\documentclass[english]{article}

\usepackage{geometry}
\usepackage[T1]{fontenc}
\usepackage[latin9]{inputenc}
\usepackage{bm}
\usepackage{amsmath,mathtools}
\usepackage{amssymb}
\usepackage[unicode=true,
 bookmarks=false,
 breaklinks=false,pdfborder={0 0 1},colorlinks=false]
 {hyperref}
\hypersetup{
 colorlinks,citecolor=blue,filecolor=blue,linkcolor=blue,urlcolor=blue}
\usepackage{xcolor,colortbl}
\definecolor{Gray}{gray}{0.85}

\makeatletter
\usepackage{amsthm}
\usepackage{cite}  
\usepackage{comment}
\usepackage{natbib}
\usepackage{booktabs}
\usepackage{graphicx}
\usepackage{enumitem}
\usepackage[linesnumbered,ruled,vlined]{algorithm2e}
\usepackage{algorithmic}

\usepackage{float}
\usepackage{multirow}
\usepackage{dsfont}
\usepackage{tcolorbox}
\usepackage{changepage}
\usepackage{color}
\definecolor{yxc}{RGB}{255,0,0}
\definecolor{yjc}{RGB}{125,0,0}
\definecolor{ytw}{RGB}{255,69,0}
\definecolor{gen}{RGB}{200,0,0}
\definecolor{laixi}{RGB}{138,43,226}

\newcommand{\yxc}[1]{\textcolor{yxc}{[YXC: #1]}}

\allowdisplaybreaks

\DeclareMathOperator{\ind}{\mathds{1}}  

\newcommand{\defn}{\coloneqq}

\newcommand{\one}{1}

\newcommand{\pib}{\pi^{\mathsf{b}}}
\newcommand{\rhob}{\rho^{\mathsf{b}}}
\newcommand{\taumax}{\tau_{\max}}

\newcommand{\myrho}{d^{\mathsf{b}}}
\newcommand{\mub}{\mu^{\mathsf{b}}}

\newcommand{\infs}{}

\newcommand{\Tpess}{\widehat{\mathcal{T}}_{\mathsf{pe}}}
\newcommand{\tildeTpess}{\widetilde{\mathcal{T}}_{\mathsf{pe}}}

\newcommand{\cS}{\mathcal{S}}
\newcommand{\cA}{\mathcal{A}}

\newcommand{\mymid}{\,|\,}
\newcommand{\cb}{c_{\mathsf{b}}}

\newcommand{\Ntrim}{N^{\mathsf{trim}}}
\newcommand{\Nmain}{N^{\mathsf{main}}}
\newcommand{\Naux}{N^{\mathsf{aux}}}
\newcommand{\Dtrim}{\mathcal{D}^{\mathsf{trim}}}
\newcommand{\Dtrimaug}{\mathcal{D}^{\mathsf{trim,aug}}}
\newcommand{\Dmain}{\mathcal{D}^{\mathsf{main}}}
\newcommand{\Daux}{\mathcal{D}^{\mathsf{aux}}}
\newcommand{\Diid}{\mathcal{D}^{\mathsf{i.i.d.}}}
\newcommand{\Cstar}{C^{\star}_{\mathsf{clipped}}}

\date{\today}

\makeatother

\begin{document}

\theoremstyle{plain} \newtheorem{lemma}{\textbf{Lemma}}\newtheorem{proposition}{\textbf{Proposition}}\newtheorem{theorem}{\textbf{Theorem}}\newtheorem{assumption}{\textbf{Assumption}}\newtheorem{definition}{\textbf{Definition}}

\theoremstyle{remark}\newtheorem{remark}{\textbf{Remark}}

\title{Settling the Sample Complexity of \\ Model-Based Offline Reinforcement Learning}

\author{Gen Li\thanks{Department of Statistics and Data Science, Wharton School, University of Pennsylvania, Philadelphia, PA 19104, USA.} \\
UPenn    \\
\and
Laixi Shi\thanks{Department of Electrical and Computer Engineering, Carnegie Mellon University, Pittsburgh, PA 15213, USA.} \\
CMU    \\
	\and
	Yuxin Chen\footnotemark[1] \thanks{Department of Electrical and Systems Engineering, University of Pennsylvania, Philadelphia, PA 19104, USA.} \\
	UPenn\\
	\and
	Yuejie Chi\footnotemark[2]\\
	CMU\\
	\and
	Yuting Wei\footnotemark[1] \\
 UPenn  \\
	}
\date{April 2022; ~Revised: February 2024}

\maketitle

\begin{abstract}

This paper is concerned with offline reinforcement learning (RL), 
which learns using pre-collected data without further exploration. 
Effective offline RL would be able to accommodate distribution shift and limited data coverage. 
However, prior algorithms or analyses either suffer from suboptimal sample complexities or incur high burn-in cost to reach sample optimality,  
thus posing an impediment to efficient offline RL in sample-starved applications. 

We demonstrate that the model-based (or ``plug-in'') approach achieves minimax-optimal sample complexity  without burn-in cost for tabular Markov decision processes (MDPs). 
Concretely, consider a $\gamma$-discounted infinite-horizon (resp.~finite-horizon) MDP with $S$ states and effective horizon $\frac{1}{1-\gamma}$ (resp.~horizon $H$), and suppose the distribution shift of data is reflected by some single-policy clipped concentrability coefficient $C^{\star}_{\mathsf{clipped}}$. We prove that model-based offline RL yields $\varepsilon$-accuracy with a sample complexity of
\[
\begin{cases}
\frac{SC_{\mathsf{clipped}}^{\star}}{(1-\gamma)^{3}\varepsilon^{2}} & (\text{infinite-horizon MDPs}) \\
\frac{H^{4}SC_{\mathsf{clipped}}^{\star}}{\varepsilon^{2}} & (\text{finite-horizon MDPs})
\end{cases}
\]
up to log factor, which is minimax optimal for the {\em entire $\varepsilon$-range}. 
The proposed algorithms are ``pessimistic'' variants of value iteration with Bernstein-style penalties, 
and do not require sophisticated variance reduction. 
Our analysis framework is established upon delicate leave-one-out decoupling arguments in conjunction with careful self-bounding techniques tailored to MDPs. 

\end{abstract}


\setcounter{tocdepth}{2}
\tableofcontents



\section{Introduction}
\label{sec:intro}


Reinforcement learning (RL) has recently achieved superhuman performance  in the gaming frontier, such as the game of Go \citep{silver2017mastering},  
under the premise that vast amounts of training data can be obtained. 
However, limited capability of online data collection in other real-world applications --- e.g., clinical trials and online advertising, where real-time data acquisition is expensive, high-stakes, and/or time-consuming, --- presents a fundamental bottleneck for carrying such RL success over to broader scenarios. 
To circumvent this bottleneck, one plausible strategy is to make more effective use of data collected previously,
given that such historical data might contain useful information that readily transfers to new tasks (for instance, the state transitions in a historical task might sometimes resemble what happen in new tasks).  
The potential of this data-driven approach has been explored and recognized in a diverse array of contexts including but not limited to robotic manipulation \citep{ebert2018visual}, autonomous driving \citep{diehl2021umbrella}, and healthcare \citep{tang2021model}; see \citet{levine2020offline,prudencio2022survey} for overviews of recent development.  
Nowadays, the subfield of reinforcement learning using historical data, without further exploration of the environment, is commonly referred to as {\em offline RL} or {\em batch RL} \citep{lange2012batch,levine2020offline}.  A desired offline RL algorithm would achieve the target statistical accuracy using as few samples as possible.

%

%

\subsection{Challenges: distribution shift and limited data coverage}

In contrast to online exploratory RL, 
offline RL has to deal with several critical issues resulting from the absence of active exploration. 
Below we single out two representative issues surrounding offline RL. 


\begin{itemize}

\item {\em Distribution shift.}
For the most part,  the historical data is generated by a certain behavior policy that departs from the optimal one. 
A key challenge in offline RL thus stems from the shift of data distributions:
how to leverage past data to the most effect, even though the distribution induced by the target policy differs from what we have available?

\item {\em Limited data coverage.} 
	Ideally, if the dataset contained sufficiently many data samples for every state-action pair, then there would be hope to simultaneously learn the performance of every policy.  Such a uniform coverage requirement, however, is oftentimes not only unrealistic (given that we can no longer change the past data) but also unnecessary (given that we might only be interested in identifying a single optimal policy).    

\end{itemize}

\noindent
Whether one can effectively cope with distribution shift and insufficient data coverage becomes a major factor that governs the feasibility and statistical efficiency of offline RL.  


In order to address the aforementioned issues, 
a recent strand of works put forward the {\em principle of pessimism or conservatism}  (e.g., \citet{buckman2020importance,kumar2020conservative,jin2021pessimism,rashidinejad2021bridging,xie2021policy,yin2021towards,shi2022pessimistic,yan2022efficacy,uehara2021pessimistic,liu2020provably,cui2022offline,chen2021pessimism,zhong2022pessimistic,zanette2021provable}).  
This is reminiscent of the optimism principle in the face of uncertainty for online exploration \citep{lai1985asymptotically,jaksch2010near,azar2017minimax,bourel2020tightening,jin2018q}, but works for drastically different reasons (as we shall elucidate momentarily).  
One plausible idea of the pessimism principle, which has been incorporated into both model-based and model-free approaches, is to penalize  value estimation of those state-action pairs that have been poorly covered. Informally speaking,  insufficient coverage of a state-action pair inevitably results in low confidence and high uncertainty in the associated value estimation, and it is hence advisable to act cautiously by tuning down the corresponding value estimate.  
Proper use of pessimism amid uncertainty brings about several provable benefits \citep{rashidinejad2021bridging,xie2021policy}:  
(i) it allows for a reduced sample size that adapts to the degree of distribution shift; 
(ii) as opposed to uniform data coverage, it only requires coverage of the part of the state-action space reachable by the target policy.  
Details to follow momentarily.


\subsection{Inadequacy of prior works}
\label{sec:inadequacy-prior}

In the present paper, we evaluate and compare the statistical performance of offline RL algorithms mainly through the lens of sample complexity --- namely, the number of samples needed for an algorithm to output, with probability approaching one, a policy whose resultant value function is at most $\varepsilon$ away from optimal (called ``$\varepsilon$-accuracy'' throughout).  
An ultimate goal is to design an algorithm to achieve the smallest possible sample complexity.

Despite extensive recent activities, however, existing statistical guarantees for the above paradigm remain inadequate, as we shall elaborate on below. For concreteness, our discussions focus on two widely studied Markov decision processes (MDPs) with $S$ states and $A$ actions \citep{bertsekas2017dynamic}: 
(a) $\gamma$-discounted infinite-horizon MDPs, with effective horizon $\frac{1}{1-\gamma}$; 
(b) finite-horizon MDPs with horizon length $H$ and nonstationary transition kernels.  
We shall bear in mind that all of these salient problem parameters (i.e., $S$, $A$, $\frac{1}{1-\gamma}$, $H$) could be enormous in modern RL applications. 
In addition, previous works have isolated an important parameter $C^{\star}\geq 1$ --- called the single-policy concentrability coefficient \citep{rashidinejad2021bridging,xie2021policy} --- that measures the mismatch of distributions induced by the target policy against the behavior policy; see Sections~\ref{sec:models-finite} and \ref{sec:models-infinite} for precise definitions.  Naturally, the statistical performance of desirable algorithms would degrade gracefully as the distribution mismatch worsens (i.e., as $C^{\star}$ increases). 
In the sequel, we shall discuss two dinstinctive RL paradigms --- model-based RL and model-free RL --- separately. 
Throughout this paper, the standard notation $\widetilde{O}(\cdot)$ indicates the order of a function with all log terms in $S,A,\frac{1}{1-\gamma},H, \frac{1}{\varepsilon},$ and $\frac{1}{\delta}$ (with $1-\delta$  the target success probability)  hidden.

\begin{figure}

\begin{tabular}{cc} 
\includegraphics[width=0.48\textwidth]{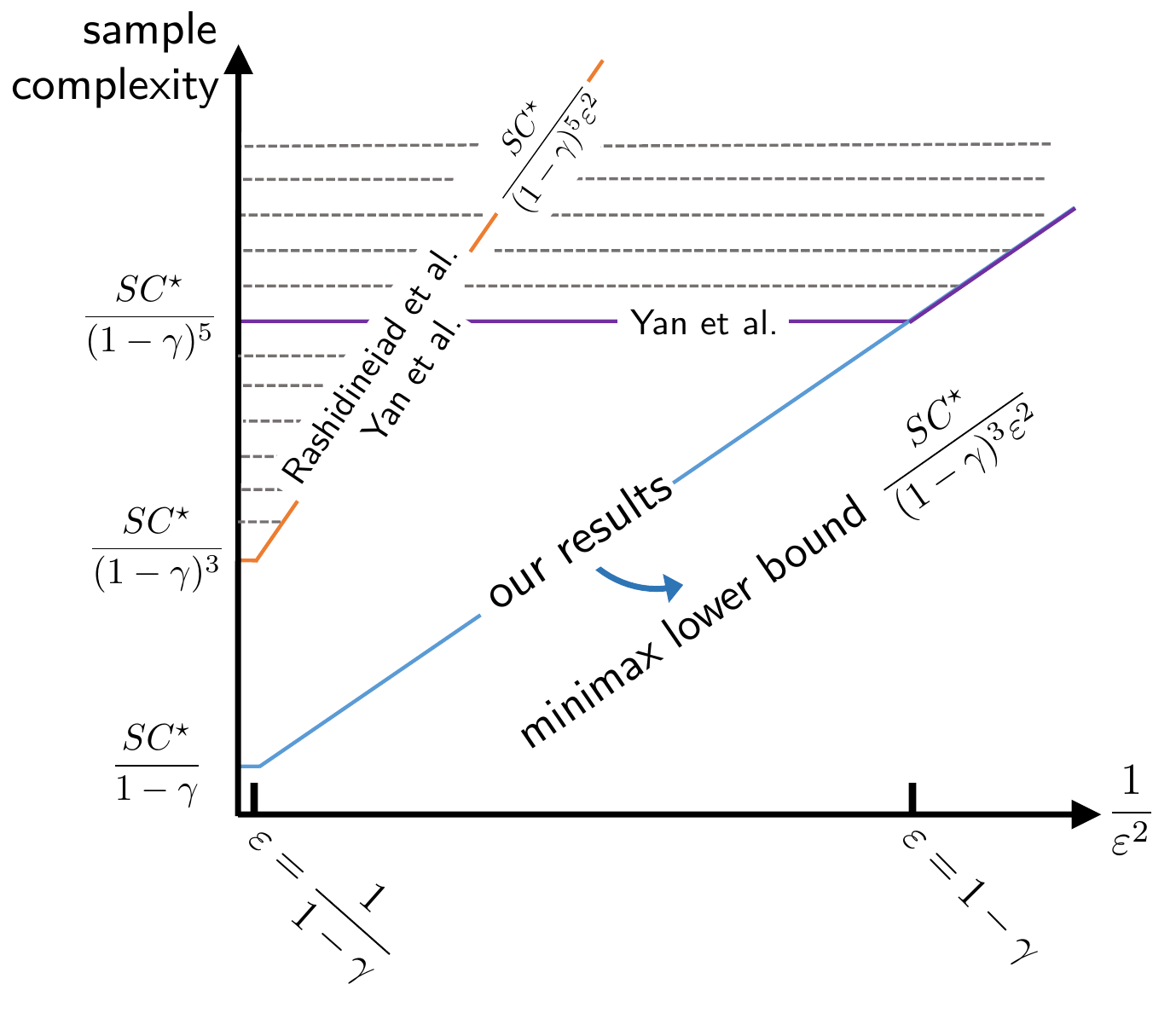} & \includegraphics[width=0.5\textwidth]{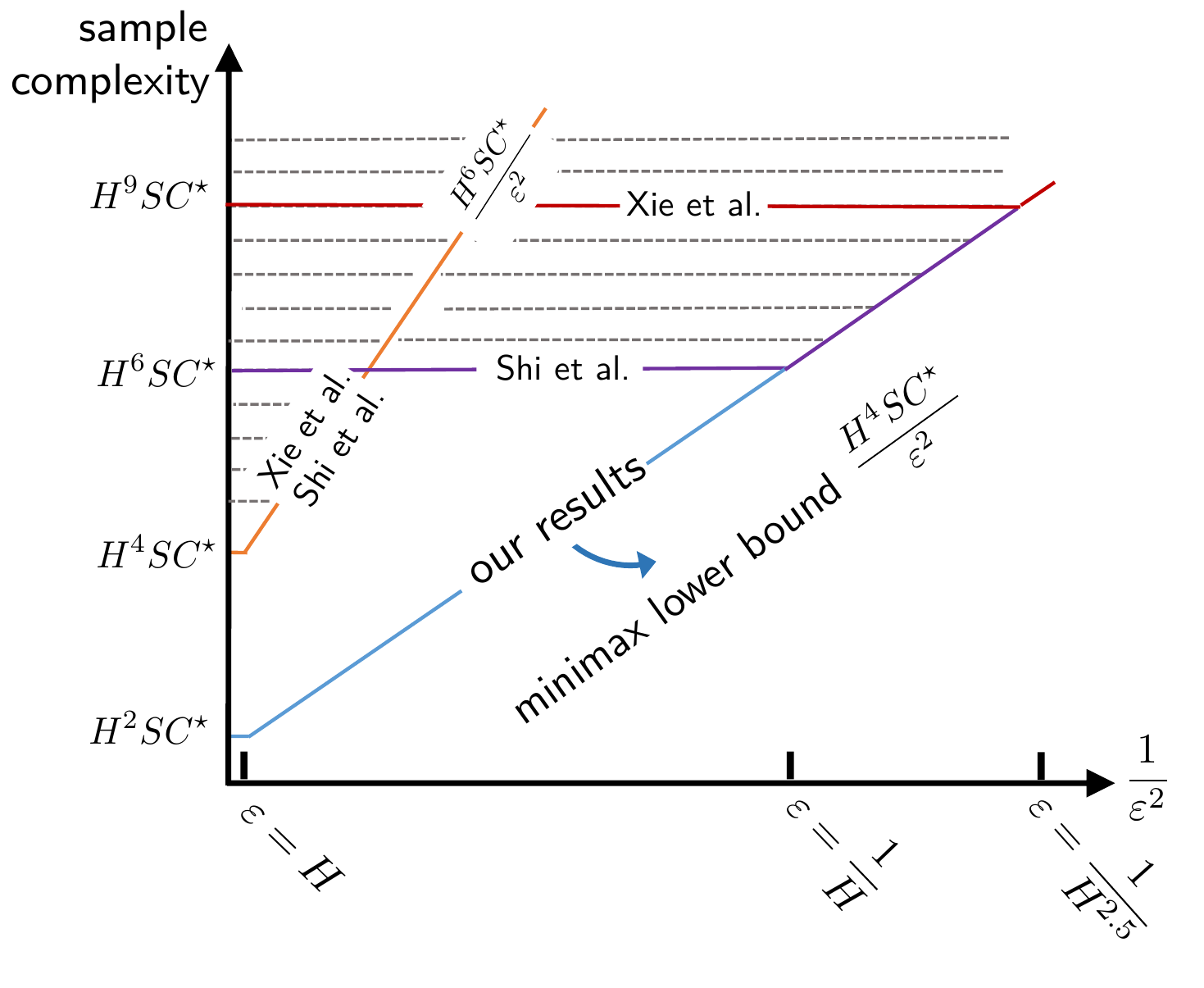}\tabularnewline
		(a) discounted infinite-horizon MDPs & (b) finite-horizon MDPs \tabularnewline
\end{tabular}

	\caption{ An illustration of prior works, where (a) is about discounted infinite-horizon MDPs and (b) is about finite-horizon MDPs. 	To facilitate comparisons, we replace $\Cstar$ with $C^{\star}$ in our results when drawing the plots given that $\Cstar\leq C^{\star}$. The shaded regions indicate the state-of-the-art achievability results. Our work manages to close the gaps between the prior achievable regions and the minimax lower bounds.   \label{fig:prior-works}}
\end{figure}

\paragraph{Model-based offline RL.} 
Model-based algorithms --- which can be interpreted as a ``plug-in'' statistical approach ---  
start by computing an empirical model for the unknown MDP, and output a policy that is (near)-optimal in accordance with the empirical MDP. 
When coupled with the pessimism principle, the model-based approach has been shown to enjoy the following sample complexity bounds.  
\begin{itemize}
	\item By incorporating Hoeffding-style lower confidence bounds into value iteration,  \citet{rashidinejad2021bridging,xie2021policy} demonstrated that 
		a sample complexity of
		\begin{align}
			\begin{cases}
				\widetilde{O}\big(\frac{SC^{\star}}{(1-\gamma)^5\varepsilon^2} \big) \quad &\text{for infinite-horizon MDPs} \\
				\widetilde{O}\big( \frac{H^6 SC^{\star}}{\varepsilon^2} \big) \quad  &\text{for finite-horizon MDPs} 
			\end{cases}
			\label{eq:Hoeffding-model-based-prior-work}
		\end{align}
		suffices to yield $\varepsilon$-accuracy. Such a sample complexity bound, however, is a large factor of $\frac{1}{(1-\gamma)^2}$ (resp.~$H^2$) above the minimax lower limit derived for infinite-horizon MDPs (resp.~finite-horizon MDPs) \citep{rashidinejad2021bridging,xie2021policy,yin2021towards}. 
	
	\item In an attempt to optimize the sample complexity, \citet{xie2021policy} leveraged the idea of variance reduction --- a powerful strategy originating from the stochastic optimization literature \citep{johnson2013accelerating} --- in model-based RL and obtained a strengthened sample complexity of
		\begin{align}
			\widetilde{O}\Big( \frac{H^4 SC^{\star}}{\varepsilon^2} +  \frac{H^{6.5} SC^{\star}}{\varepsilon} \Big)
		\end{align}
		for finite-horizon MDPs. This sample complexity bound approaches the minimax lower limit (i.e., the order of $\frac{H^4 SC^{\star}}{\varepsilon^2}$) once the sample size exceeds the order of
		\begin{align}
			\text{(burn-in cost)} \qquad H^9 SC^{\star} ;
			\label{eq:burn-in-VR-model-based}
		\end{align}
		in other words, an enormous burn-in sample size is needed in order to attain sample optimality. 
\end{itemize}

\paragraph{Model-free offline RL.} 
The model-free approach forms a contrastingly different class of RL algorithms, 
which bypasses the model estimation stage and directly learns the optimal values. Noteworthily, Q-learning and its variants \citep{watkins1992q}, which apply stochastic approximation \citep{robbins1951stochastic} based on the Bellman optimality condition,  are among the most widely used model-free paradigms. 
The principle of pessimism amid uncertainty has recently been integrated into model-free algorithms as well, with the state-of-the-art statistical guarantees listed below \citep{shi2022pessimistic,yan2022efficacy}. 
\begin{itemize}
	\item When Q-learning is implemented in conjunction with Hoeffding-style lower confidence bounds, 
		it has been shown to achieve the same sample complexity as \eqref{eq:Hoeffding-model-based-prior-work}, which is suboptimal by a factor of either $\frac{1}{(1-\gamma)^2}$ or $H^2$. 

	\item A variance-reduced variant of pessimistic Q-learning allows for further sample size benefits,  
		achieving a sample complexity of
		\begin{align}
			\begin{cases}
				\widetilde{O}\big(\frac{SC^{\star}}{(1-\gamma)^3\varepsilon^2} + \frac{SC^{\star}}{(1-\gamma)^4\varepsilon} \big) \quad &\text{for infinite-horizon MDPs} \\
				 \widetilde{O}\big( \frac{H^4 SC^{\star}}{\varepsilon^2} + \frac{H^5 SC^{\star}}{\varepsilon} \big) \quad  &\text{for finite-horizon MDPs}
			\end{cases}
			\label{eq:VR-model-free-prior-work}
		\end{align}
		for any target accuracy level $\varepsilon$. This means that the algorithm is guaranteed to be sample-optimal only after the total sample size exceeds the order of
		\begin{align}
			\text{(burn-in cost)} \qquad 
			\begin{cases}
				 \frac{SC^{\star}}{(1-\gamma)^5}  \quad &\text{for infinite-horizon MDPs}, \\
				H^6 SC^{\star}  \quad  &\text{for finite-horizon MDPs},
			\end{cases}
			\label{eq:VR-model-free-prior-work}
		\end{align}
		which again manifests itself as a significant burn-in cost for long-horizon problems.  
\end{itemize}

\paragraph{Summary.} As elucidated above, existing algorithms either suffer from suboptimal sample complexities, or require sophisticated techniques like variance reduction to approach minimax optimality. Even when variance reduction is employed, prior algorithms incur an enormous burn-in cost in order to work optimally, 
thus posing an impediment to achieving sample efficiency in data-starved applications. 
Table~\ref{tab:prior-work} summarizes quantitatively the previous results,
whereas Figure~\ref{fig:prior-works} illustrates the gaps between the state-of-the-art upper bounds and the minimax lower bounds (as derived by \citet{xie2021policy,rashidinejad2021bridging}).  All this motivates the studies of the following natural questions: 
{
\setlist{rightmargin=\leftmargin}
\begin{itemize}
\item[] {\em Can we develop an offline RL algorithm that achieves near-optimal sample complexity without burn-in cost? 
	If so, can we accomplish this goal by means of a simple algorithm without resorting to sophisticated schemes like variance reduction? }
\end{itemize}
}
\noindent 
The current paper answers these questions affirmatively by studying the model-based approach.


\subsection{Main contributions}


In this paper, we settle the sample complexity of model-based offline RL by studying a pessimistic variant of value iteration --- called VI-LCB --- applied to some empirical MDP. 
Encouragingly, for both discounted infinite-horizon and finite-horizon MDPs, the model-based algorithms provably achieve minimax-optimal sample complexities for any given target accuracy level $\varepsilon$ --- namely, any $\varepsilon \in \big(0, \frac{1}{1-\gamma}\big]$ for discounted infinite-horizon MDPs and $\varepsilon \in (0,H]$ for finite-horizon MDPs. 

To be more precise, we introduce a slightly modified version $\Cstar$ of the concentrability coefficient $C^{\star}$, which always satisfies $\Cstar\leq C^{\star}$  and shall be termed the single-policy clipped concentrability coefficient (see Sections \ref{sec:models-infinite} and~\ref{sec:models-finite} for more details as well as the advantages of this coefficient). The introduction of this new parameter leads to slightly improved sample complexity compared to the one based on $C^{\star}$.    
The main contributions are summarized as follows.



\begin{itemize}

	\item For $\gamma$-discounted infinite-horizon MDPs, we demonstrate that with high probability, the VI-LCB algorithm with Bernstein-style penalty finds an $\varepsilon$-optimal policy with a sample complexity of   
\begin{equation}
	\widetilde{O}\bigg(  \frac{S \Cstar}{(1-\gamma)^3 \varepsilon^{2}} \bigg)
	\label{eq:intro-sample-size-infinite}
\end{equation}
		for any given accuracy level $\varepsilon\in \big(0, \frac{1}{1-\gamma} \big]$ (see Theorem~\ref{thm:infinite}).  
		Our algorithm reuses all samples across all iterations in order to achieve data efficiency, 
		and our analysis builds upon a novel leave-one-out argument to decouple complicated statistical dependency across iterations. 
		Moreover, the above sample complexity \eqref{eq:intro-sample-size-infinite} remains valid if $\Cstar$ is replaced by $C^{\star}$. 
		
	\item For finite-horizon MDPs with nonstationary transition kernels, we propose a variant of VI-LCB that adopts the Bernstein-style penalty to enforce pessimism in the face of uncertainty. We prove that for any given  $\varepsilon\in (0,H]$, the proposed algorithm yields an $\varepsilon$-optimal policy  using  
\begin{equation}
	\widetilde{O}\bigg( \frac{H^4 S \Cstar}{ \varepsilon^{2}} \bigg)
		\label{eq:intro-sample-size-finite}
\end{equation}
		samples with high probability (see Theorem~\ref{thm:finite}).  
A key ingredient in the algorithm design is 
 a two-fold subsampling trick that helps decouple the statistical dependency along the sample rollouts. 
		Note that the above sample complexity result \eqref{eq:intro-sample-size-finite} continues to hold if one replaces $\Cstar$ with $C^{\star}$. 

	\item To assess the tightness and optimality of our results, 
we further develop minimax lower bounds in Theorems~\ref{thm:infinite-lwoer-bound} and \ref{thm:finite-lower-bound},
which match the above upper bounds modulo some logarithmic factors.


\end{itemize}
%
%
\noindent 
Remarkably, our algorithms do not require sophisticated variance reduction schemes, 
as long as suitable confidence bounds are adopted.  
Detailed theoretical comparisons with prior art can be found in Table~\ref{tab:prior-work} and are also illustrated in Figure~\ref{fig:prior-works}.  
Finally, we have conducted a series of numerical experiments to evaluate the performance of the proposed algorithms in Section~\ref{sec:numerics}.

\paragraph{Statistical contributions: solving the most sample-hungry regime.} 
The offline RL problem considered herein is statistical in nature, in that it seeks to learn from pre-collected data in the face of uncertainty. 
As far we know, our theory is the first to identify an offline algorithm that provably attains minimax-optimal statistical efficiency for the entire $\varepsilon$-range, 
which in turn makes clear that {\em no burn-in phase is needed} to achieve optimal statistical accuracy.  
Achieving this requires developing a new suite of statistical theory that works all the way to {\em the most data-hungry regime}. 
It is noteworthy that the existing statistical toolbox --- not merely for offline RL, but for online exploratory RL as well (as we shall detail in Section~\ref{sec:related-works}) --- 
is only guaranteed to work when the total sample size already exceeds a fairly large threshold, 
a (often unnecessary) requirement that substantially simplifies statistical analysis.  
In this sense, the regime we aim to solve 
is reminiscent of the subfield of high-dimensional statistics \citep{wainwright2019high,donoho2000high} 
that helps extend the frontier of classical statistics to the sample-starved regime, for which an enriched statistical toolbox is critically needed.


%

\newcommand{\topsepremove}{\aboverulesep = 0mm \belowrulesep = 0mm} \topsepremove

\begin{table}[t]

\begin{center}

	{\small
\begin{tabular}{c|c|c|c|cc}
\hline 
\toprule
\multirow{2}{*}{horizon} & \multirow{2}{*}{algorithm} &    sample & $\varepsilon$-range to attain & \multirow{2}{*}{type}\tabularnewline
	&    & complexity & sample optimality & \tabularnewline
\hline 
\toprule 
\multirow{8}{*}{infinite} & VI-LCB $\vphantom{\frac{1^{7}}{1^{7}}}$ &   \multirow{2}{*}{$\frac{SC^{\star}}{(1-\gamma)^{5}\varepsilon^{2}}$} & \multirow{2}{*}{------} & \multirow{2}{*}{model-based} \tabularnewline
	&\citep{rashidinejad2021bridging}   &   &  &  \tabularnewline
\cline{2-6} 
 &Q-LCB $\vphantom{\frac{1^{7}}{1^{7}}}$ &    \multirow{2}{*}{$\frac{SC^{\star}}{(1-\gamma)^{5}\varepsilon^{2}}$} & \multirow{2}{*}{------} & \multirow{2}{*}{model-free} \tabularnewline 
&\citep{yan2022efficacy}   &  &   &  \tabularnewline
\cline{2-2} \cline{3-6} \cline{4-6} \cline{5-6} \cline{6-6} 
&VR-Q-LCB  $\vphantom{\frac{1^{7}}{1^{7}}}$ &   \multirow{2}{*}{$\frac{SC^{\star}}{(1-\gamma)^{3}\varepsilon^{2}} +\frac{SC^\star}{(1-\gamma)^4\varepsilon}$} & \multirow{2}{*}{$(0,1-\gamma ]$} & \multirow{2}{*}{model-free} \tabularnewline
& \citep{yan2022efficacy}   &  &  &  \tabularnewline
\cline{2-6} 
	& {\cellcolor{Gray}\textbf{VI-LCB}} $\vphantom{\frac{1^{7}}{1^{7}}}$  &  {\cellcolor{Gray}} &  {\cellcolor{Gray}}   & {\cellcolor{Gray}} \tabularnewline
	&{\cellcolor{Gray}\textbf{(this paper: Theorem~\ref{thm:infinite})} }  & \multirow{-2}{*}{{\cellcolor{Gray}$\frac{S\Cstar}{(1-\gamma)^{3}\varepsilon^{2}}$ ~($\leq \frac{SC^{\star}}{(1-\gamma)^{3}\varepsilon^{2}}$)}} &  {\cellcolor{Gray}\multirow{-2}{*}{$\big(0,\frac{1}{1-\gamma}\big]$}} &  \multirow{-2}{*}{{\cellcolor{Gray}model-based}} \tabularnewline
	\cline{2-6} 
	& {\cellcolor{Gray}\textbf{lower bound}} $\vphantom{\frac{1^{7}}{1^{7}}}$  &  {\cellcolor{Gray}} &  {\cellcolor{Gray}}   & {\cellcolor{Gray}} \tabularnewline
	&{\cellcolor{Gray}\textbf{(this paper: Theorem~\ref{thm:infinite-lwoer-bound})} }  & \multirow{-2}{*}{{\cellcolor{Gray}$\frac{S\Cstar}{(1-\gamma)^{3}\varepsilon^{2}}$ }} &  \multirow{-2}{*}{{\cellcolor{Gray}------}} &  \multirow{-2}{*}{{\cellcolor{Gray}------}} \tabularnewline	
\hline 
\toprule
\multirow{14}{*}{finite} & VI-LCB $\vphantom{\frac{1^{7}}{1^{7}}}$ &    \multirow{2}{*}{$\frac{H^{6}SC^{\star}}{\varepsilon^{2}}$} & \multirow{2}{*}{------}  & \multirow{2}{*}{model-based} \tabularnewline
&\citep{xie2021policy} &    &     \tabularnewline
\cline{2-6} 
&\multirow{1}{*}{VPVI} $\vphantom{\frac{1^{7}}{1^{7}}}$  &   \multirow{2}{*}{$\frac{H^{5}SC^{\star}}{\varepsilon^{2}}$} & \multirow{2}{*}{------}    & \multirow{2}{*}{model-based} \tabularnewline
& \citep{yin2021towards}  &       & \tabularnewline
\cline{2-6}  
	&PEVI-Adv $\vphantom{\frac{1^{7}}{1^{7}}}$  &   \multirow{2}{*}{$\frac{H^{4}SC^{\star}}{\varepsilon^{2}}+\frac{H^{6.5}SC^{\star}}{\varepsilon}$} & \multirow{2}{*}{$\big(0, \frac{1}{H^{2.5}}\big]$}  & \multirow{2}{*}{model-based} \tabularnewline
& \citep{xie2021policy}  &  &     & \tabularnewline
\cline{2-6} 
	&LCB-Q-Advantage  $\vphantom{\frac{1^{7}}{1^{7}}}$ &    \multirow{2}{*}{$\frac{H^{4}SC^{\star}}{\varepsilon^{2}}+\frac{H^{5}SC^{\star}}{\varepsilon}$} & \multirow{2}{*}{$\big(0, \frac{1}{H}\big]$}  & \multirow{2}{*}{model-free} \tabularnewline
&\citep{shi2022pessimistic}   &  &    & \tabularnewline
\cline{2-6} 
& APVI/LCBVI $\vphantom{\frac{1^{7}}{1^{7}}}$   &  \multirow{2}{*}{$\frac{H^{4}SC^{\star}}{\varepsilon^{2}}+\frac{H^{4}}{\myrho_{\min}\varepsilon}$}   & \multirow{2}{*}{$(0,SC^\star \myrho_{\min}]$}  & \multirow{2}{*}{model-based} \tabularnewline
&\citep{yin2021towards}   &  &    &  \tabularnewline
\cline{2-2} \cline{4-6} 
	& {\cellcolor{Gray} \textbf{VI-LCB}} $\vphantom{\frac{1^{7}}{1^{7^3}}}$    & {\cellcolor{Gray}} &  {\cellcolor{Gray}}  &  {\cellcolor{Gray}}  \tabularnewline
	& {\cellcolor{Gray}\textbf{(this paper: Theorem~\ref{thm:finite})}}  &  {\cellcolor{Gray}\multirow{-2}{*}{$\frac{H^{4}S\Cstar}{\varepsilon^{2}}$ ~($\leq \frac{H^{4}SC^{\star}}{\varepsilon^{2}}$)}} & {\cellcolor{Gray}\multirow{-2}{*}{$(0,H]$}} & {\cellcolor{Gray}\multirow{-2}{*}{model-based}}    \tabularnewline
\cline{2-2} \cline{4-6} 
	& {\cellcolor{Gray} \textbf{lower bound}} $\vphantom{\frac{1^{7}}{1^{7^3}}}$    & {\cellcolor{Gray}} &  {\cellcolor{Gray}}  &  {\cellcolor{Gray}}  \tabularnewline
	& {\cellcolor{Gray}\textbf{(this paper: Theorem~\ref{thm:finite-lower-bound})}}  &  {\cellcolor{Gray}\multirow{-2}{*}{$\frac{H^{4}S\Cstar}{\varepsilon^{2}}$ }} & {\cellcolor{Gray}\multirow{-2}{*}{{\cellcolor{Gray}------}}} & {\cellcolor{Gray}\multirow{-2}{*}{\cellcolor{Gray}------}}    \tabularnewline	
\hline 
\toprule 

\end{tabular}}
\end{center}
	\caption{Comparisons with prior results (up to log terms) regarding finding an $\varepsilon$-optimal policy in offline RL. 
	The $\varepsilon$-range  stands for the range of accuracy level $\varepsilon$ for which the derived sample complexity is optimal.
	Here, one always has $\Cstar \leq C^{\star}$; and the parameter $\myrho_{\min} \coloneqq \frac{1}{\min_{s,a,h} \{\myrho_{h}(s,a):\, \myrho_{h}(s,a)>0\}  }$ employed in \citet{yin2021towards} could be exceedingly small, with $\myrho_h$ the occupancy distribution of the dataset. 
	While multiple algorithms are referred to as VI-LCB in the table, they correspond to different variants of VI-LCB. 
	Our results are the first to achieve sample optimality for the full $\varepsilon$-range. 
	\label{tab:prior-work}}  
\end{table}

\subsection{Notation}

Throughout this paper, we adopt the convention that $0/0=0$. We use $\Delta(\cS)$ to indicate the probability simplex over the set $\cS$, and denote by $[H]$ the set $\{1,\cdots, H\}$ for any positive integer $H$. We use $\ind(\cdot)$ to represent the indicator function. For any vector $x = [x(s,a)]_{(s,a)\in \cS\times \cA}\in \mathbb{R}^{SA}$, we overload the notation by letting $x^2 = [x(s,a)^2]_{(s,a)\in \cS\times \cA}$. For two vectors $a=[a_i]_{1\leq i\leq n}$ and $b=[b_i]_{1\leq i\leq n}$, $a\circ b = [a_i b_i]_{1\leq i\leq n}$ denotes their Hadamard product,
and $a\geq b$ (resp.~$a\leq b$) means $a_i \geq b_i$ (resp.~$a_i\leq b_i$) for all $i$. 
Following the convention in RL \citep{agarwal2019reinforcement}, the norm $\|\cdot \|_1$ of a matrix $P=[P_{ij}]$ is defined to be $\|P\|_1\coloneqq \max_{i} \sum_j |P_{ij}|$. 
For any probability vector $q\in \mathbb{R}^{1\times S}$ (which is a row vector) and any vector $V \in \mathbb{R}^S$,   define
\begin{align}
	\mathsf{Var}_{q} (V)\coloneqq q  \big(V \circ V \big)-\left(q V \right)^{2} \,\in \mathbb{R}
	\label{eq:defn-Var-P-V}
\end{align}
with $qV=\sum_i q_iV_i$, which corresponds to the variance of $V$ w.r.t.~the distribution $q$. The standard notation $O(\cdot)$ is adopted to represent the orderwise scaling of a function.

%


\section{Algorithm and theory: discounted infinite-horizon MDPs}
\label{sec:main-infinite}

We begin by studying offline RL in discounted infinite-horizon Markov decision processes.  
In the following, we shall first introduce the models and assumptions for discounted infinite-horizon MDPs, 
followed by algorithm design and main results.

\subsection{Models and assumptions}
\label{sec:models-infinite}

Let us begin with some preliminary concepts and notation of discounted infinite-horizon MDPs, followed by a concrete setting specific to offline RL. 
A more detailed introduction of discounted infinite-horizon MDPs can be found in classical textbooks like \citet{bertsekas2017dynamic}.

\paragraph{Basics of discounted infinite-horizon MDPs.}  Consider a discounted infinite-horizon  MDP \citep{bertsekas2017dynamic} represented by a tuple $\mathcal{M} = \{\cS, \cA, P, \gamma, r\}$. The key components of $\mathcal{M}$ are: (i) $\cS = \{1,2,\cdots, S\}$: a finite state space of size $S$; (ii) $\cA = \{1,2,\cdots, A\}$:  an action space of size $A$; (iii) $P: \cS\times \cA \rightarrow \Delta(\cS)$: the transition probability kernel of the MDP (i.e., $P (\cdot \mymid s,a)$ denotes the transition probability from state $s$ when action $a$ is executed); (iv) $\gamma \in [0,1)$: the discount factor, so that $\frac{1}{1-\gamma}$ represents the effective horizon; (v) $r: \cS\times \cA \rightarrow [0,1]$: the deterministic reward function (namely, $r(s,a)$ indicates the immediate reward received when the current state-action pair is $(s,a)$). Without loss of generality, the immediate rewards are normalized so that they are contained within the interval $[0,1]$. 
Throughout this section, we introduce the convenient notation
\begin{align}
	P_{s,a} \coloneqq P(\cdot \mymid s,a) \in \mathbb{R}^{1\times S}.
\end{align}

\paragraph{Policy, value function and Q-function.}
A stationary policy $\pi: \cS \rightarrow \Delta(\cA)$ is a possibly randomized action selection rule; that is, $\pi(a \mymid s)$ represents the probability of choosing $a$ in state $s$. When $\pi$ is a deterministic policy, we abuse the notation by letting $\pi(s)$ represent the action chosen by the policy $\pi$ in state $s$. 
A sample trajectory induced by the MDP under policy $\pi$ can be written as $\{(s_t,a_t)\}_{t\geq 0}$, with $s_t$ (resp.~$a_t$) denoting the state (resp.~action) of the trajectory at time $t$. 
To proceed, we shall also introduce the value function $V^\pi$ and Q-value function $Q^\pi$ associated with policy $\pi$. Specifically, the value function $V^\pi: \cS \rightarrow \mathbb{R}$ of policy $\pi$ is defined as the expected discounted cumulative reward as follows:
\begin{align}
	\forall s\in\cS: \qquad 
	V^\pi(s) \coloneqq \mathbb{E} \left[\sum_{t=0}^\infty \gamma^t r(s_t, a_t) \mid s_0 = s; \pi\right],
	\label{eq:defn-Vpi-inf}
\end{align}
where the expectation is taken over the sample trajectory $\{(s_t, a_t)\}_{t\geq 0}$ generated in a way that $a_t \sim \pi(\cdot \mymid s_t)$ and $s_{t+1} \sim P(\cdot \mymid s_t, a_t)$ for all $t\geq 0$. Given that all immediate rewards lie within $[0,1]$, it is easily verified that $0\leq V^\pi(s) \leq \frac{1}{1-\gamma}$ for any policy $\pi$. The Q-function (or action-state function) of policy $\pi$ can be defined analogously as follows:
\begin{align}
	\forall (s,a)\in\cS \times \cA: \qquad 
	Q^\pi(s,a) \coloneqq \mathbb{E} \left[\sum_{t=0}^\infty \gamma^t r(s_t, a_t) \mid s_0 = s, a_0 = a; \pi\right],
\end{align}
which differs from \eqref{eq:defn-Vpi-inf} in that it is also conditioned on $a_0=a$.

Let $\rho \in \Delta(\cS)$ be a given state distribution. 
If the initial state is randomly drawn from $\rho$, then we can define the following weighted value function of policy $\pi$:
\begin{align}
	V^\pi(\rho) \coloneqq \mathop{\mathbb{E}}\limits_{s\sim \rho}\big[V^\pi(s)\big].
\end{align}
In addition, we introduce the {\em discounted occupancy distributions} associated with policy $\pi$ as follows:
\begin{align}
	\forall s\in \cS: \quad d^\pi(s; \rho) &\coloneqq (1-\gamma) \sum_{t=0}^\infty \gamma^t \mathbb{P}(s_t = s \mid s_0 \sim \rho;\pi),\\
	\forall (s,a)\in \cS \times \cA: \qquad d^\pi(s,a; \rho) &\coloneqq (1-\gamma) \sum_{t=0}^\infty \gamma^t \mathbb{P}(s_t = s, a_t =a \mid s_0 \sim \rho;\pi),
\end{align}
where we consider the randomness over a sample trajectory that starts from an initial state $s_0 \sim \rho$ and that follows policy $\pi$ (i.e., $a_t \sim \pi(\cdot \mymid s_t)$ and $s_{t+1} \sim P(\cdot \mymid s_t, a_t)$ for all $t\geq 0$).

It is known that there exists at least one deterministic policy --- denoted by $\pi^\star$ --- that simultaneously maximizes $V^\pi(s)$ and $Q^\pi(s,a)$ for all state-action pairs $(s,a)\in \cS\times \cA$ \citep{bertsekas2017dynamic}. We use the following shorthand notation to represent respectively the resulting optimal value and optimal Q-function:
\begin{align}
	\forall (s,a) \in \cS\times \cA: \quad V^\star(s) \coloneqq V^{\pi^\star}(s) \quad \text{ and } \quad Q^\star(s,a) \coloneqq Q^{\pi^\star}(s,a).
\end{align}
Correspondingly, we adopt the notation of the discounted occupancy distributions associated with $\pi^\star$ as:
\begin{align}
	\forall (s,a) \in \cS\times \cA: \quad d^\star(s) \coloneqq d^{\pi^\star}(s; \rho) \quad \text{ and } \quad 
	d^\star (s,a) \coloneqq d^{\pi^\star}(s,a; \rho) = d^{\star}(s) \ind\big( a = \pi^\star(s) \big), \label{eq:infinite-d-star-def}
\end{align}
where the last equality is valid since $\pi^\star$ is assumed to be deterministic.

\paragraph{Offline/batch data.}
Let us work with an independent sampling model as studied in the prior work \citet{rashidinejad2021bridging}. To be precise, imagine that we observe a batch dataset $\mathcal{D} = \{(s_i,a_i,s_i')\}_{1\leq i\leq N}$ containing $N$ sample transitions. These samples are independently generated based on a distribution $\myrho \in \Delta(\cS\times \cA)$ and the transition kernel $P$ of the MDP, namely, 
		\begin{align}
			(s_i,a_i) \overset{\mathrm{ind.}}{\sim} \myrho 
			\qquad \text{and} \qquad
			s_i' \overset{\mathrm{ind.}}{\sim} P(\cdot \mymid s_i, a_i),
			\qquad\quad 1\leq i\leq N.
			\label{eq:sampling-offline-inf-iid}
		\end{align}
		%
In addition, it is assumed that the learner is aware of the reward function.

In order to capture the distribution shift between the desired occupancy measure and the data distribution, 
we introduce a key quantity previously introduced in \citet{rashidinejad2021bridging}. 
\begin{definition}[Single-policy concentrability for infinite-horizon MDPs] 
\label{assumption:concentrate-infinite-simple}
The single-policy concentrability coefficient of a batch dataset $\mathcal{D}$ is defined as
\begin{align}
	C^{\star} \coloneqq \max_{(s, a) \in \cS \times \cA}\, \frac{d^{\star}(s, a)}{\myrho(s, a)} . 
	\label{eq:concentrate-infinite-simple}
\end{align}
Clearly, one necessarily has $C^{\star} \geq 1$. 
\end{definition}
In words, $C^{\star}$ measures the distribution mismatch in terms of the maximum density ratio. 
The batch dataset can be viewed as expert data when $C^{\star}$ approaches 1, meaning that the batch dataset is close to the target policy in terms of the induced distributions.   
Moreover, this coefficient $C^{\star}$ is referred to as the ``single-policy'' concentrability coefficient since it is concerned with a single policy $\pi^{\star}$; this is clearly a much weaker assumption compared to the all-policy concentrability assumption (as adopted in, e.g., \citet{munos2007performance,farahmand2010error,fan2019theoretical,chen2019information,ren2021nearly,xie2021batch}), the latter of which assumes a uniform density-ratio bound over all policies and requires the dataset to be highly exploratory.


%
%

In the current paper, we also introduce a slightly improved version of $C^{\star}$ as follows.
\begin{definition}[Single-policy clipped concentrability for infinite-horizon MDPs] 
\label{assumption:concentrate-infinite}
The single-policy clipped concentrability coefficient of a batch dataset $\mathcal{D}$ is defined as
\begin{align}
	\Cstar \coloneqq \max_{(s, a) \in \cS \times \cA}\frac{\min\big\{d^{\star}(s, a), \frac{1}{S}\big\}}{\myrho(s, a)} .
	\label{eq:concentrate-infinite}
\end{align}
%
\end{definition}
\begin{remark}
A direct comparison of Conditions~\eqref{eq:concentrate-infinite-simple} and \eqref{eq:concentrate-infinite} implies that for a given batch dataset $\mathcal{D}$,
\begin{equation}
	\Cstar \leq C^{\star}.
\end{equation}
%
As we shall see later, while our sample complexity upper bounds will be mainly stated in terms of $\Cstar$, 
all of them remain valid if $\Cstar$ is replaced with $C^{\star}$. 
Additionally, in contrast to $C^{\star}$ that is always lower bounded by 1,  we have a smaller lower bound as follows (directly from the definition \eqref{eq:concentrate-infinite})
\begin{align}\label{eq:Cstar_lower_bound_infinite}
\Cstar \geq 1/S,
\end{align}
which is nearly tight.\footnote{As a concrete example, suppose that $d^{\star}(s)=\begin{cases}
1-\frac{S-1}{S^{3}} & \text{if }s=1\\
\frac{1}{S^{3}} & \text{else}
\end{cases}$ and $\myrho(s,a)=\begin{cases}
1-\frac{S-1}{S^{2}} & \text{if }a=\pi^{\star}(s)\text{ and }s=1,\\
\frac{1}{S^{2}} & \text{if }a=\pi^{\star}(s)\text{ and }s\neq1,\\
0, & \text{else}.
\end{cases}$ 
Then it can be easily verified that $\Cstar=\frac{1}{S-1+\frac{1}{S}}$. Nonetheless, caution should be exercised that an exceedingly small $\Cstar$ requires highly compressible structure of $d^{\star}$,  and the real-world data often do not fall within this benign range of $\Cstar$.}  
This attribute could lead to  sample size saving in some cases, to be detailed shortly.
\end{remark}

Let us take a moment to further interpret the coefficient in Definition~\ref{assumption:concentrate-infinite}, which says that
\begin{align}
	\myrho(s, a) \geq 
	\begin{cases} \frac{1}{\Cstar}d^{\star}(s, a) , \qquad & \text{if } d^{\star}(s, a) \leq 1/S \\ \frac{1}{\Cstar S}, & \text{if }d^{\star}(s, a) > 1/S   \end{cases}
\end{align}
holds for any pair $(s,a)$. Consider, for instance, the case where $\Cstar=O(1)$:  if a state-action pair is infrequently (or rarely) visited by the optimal policy, 
then it is fine for the associated density in the batch data to be very small (e.g., a density proportional to that of the optimal policy); by contrast, if a state-action pair is visited fairly often by the optimal policy, 
then Definition~\ref{assumption:concentrate-infinite} might only require $\myrho(s,a)$ to exceed the order of $1/S$. In other words, the required level of $\myrho(s,a)$ is clipped at the level $\frac{1}{\Cstar S}$ regardless of the value of $d^{\star}(s, a)$.

\paragraph{Goal.} Armed with the batch dataset $\mathcal{D}$, the objective of offline RL in this case is to find a policy $\widehat{\pi}$ that attains near-optimal value functions --- with respect to a given test state distribution $\rho\in \Delta(\cS)$ --- in a sample-efficient manner. 
To be precise, for a prescribed accuracy level $\varepsilon$, we seek to identify an $\varepsilon$-optimal policy $\widehat{\pi}$ satisfying
\begin{align}
	V^\star(\rho) - V^{\widehat{\pi}}(\rho) \leq \varepsilon
\end{align}
with high probability, using a batch dataset $\mathcal{D}$ (cf.~\eqref{eq:sampling-offline-inf-iid}) containing as few samples as possible.
Particular emphasis is placed on achieving minimal sample complexity for the entire range of accuracy levels (namely, for any $\varepsilon \in \big(0,\frac{1}{1-\gamma}\big]$).


\subsection{Algorithm:  VI-LCB for infinite-horizon MDPs}
\label{sec:VI-LCB-infinite-horizon}

In this subsection, we introduce a model-based offline RL algorithm that incorporates lower concentration bounds in value estimation. 
The algorithm, called VI-LCB, applies value iteration (based on some pessimistic Bellman operator) to the empirical MDP, 
with the key ingredients described below.

\paragraph{The empirical MDP.} Recall that we are given $N$ independent sample transitions $\{(s_i, a_i, s_i')\}_{i = 1}^N$ in the dataset $\mathcal{D}$.
For any given state-action pair $(s,a)$, we denote by 
\begin{align}
N(s, a) \coloneqq \sum_{i = 1}^{N} \ind\big((s_i, a_i) = (s, a)\big)
	\label{eq:defn-Nsa-infinite}
\end{align}
the number of samples transitions from $(s, a)$.
We then construct an empirical transition matrix $\widehat{P}$ such that
\begin{align}
	\widehat{P}(s'\mymid s,a) = 
	\begin{cases} \frac{1}{N(s,a)} \sum\limits_{i=1}^N \mathds{1} \big\{ (s_i, a_i, s_i') = (s,a,s') \big\}, & \text{if } N(s,a) > 0 \\
		\frac{1}{S}, & \text{else}
	\end{cases}  
	\label{eq:empirical-P-infinite}
\end{align}
for each $(s,a,s')\in \cS\times \cA \times \cS$.

\paragraph{The pessimistic Bellman operator.} 
Our offline algorithm is developed based on finding the fixed point of 
some variant of the classical Bellman operator.   
Let us first introduce this key operator and eludicate how the pessimism principle is enforced.  
Recall that the Bellman operator  $\mathcal{T}(\cdot): \mathbb{R}^{SA}\rightarrow \mathbb{R}^{SA} $ w.r.t.~the transition kernel $P$ is defined such that
for any vector $Q \in \mathbb{R}^{SA}$, 
\begin{align}
	\mathcal{T}(Q)(s, a) \coloneqq r(s, a) + \gamma P_{s, a}V \qquad\text{for all }(s, a)\in\cS\times\cA ,
	\label{eq:classical-Bellman-operator-inf}
\end{align}
where $V =[V(s)]_{s\in \cS}$ with  $V(s) \coloneqq \max_a Q(s, a)$.
We propose to penalize the original Bellman operator w.r.t.~the empirical kernel $\widehat{P}$ as follows: 
\begin{align}
	\Tpess (Q)(s, a) \coloneqq \max\Big\{r(s, a) + \gamma\widehat{P}_{s, a} V - b(s, a; V) , 0\Big\}
	\qquad\text{for all }(s, a)\in\cS\times\cA ,
	\label{eq:empirical-Bellman-infinite}
\end{align}
%
where  $b(s, a; V)$ denotes the penalty term employed to enforce pessimism amid uncertainty. 
As one can anticipate, the properties of the fixed point of $\Tpess(\cdot)$ relies heavily upon the choice of the penalty terms $\{b_h(s,a; V)\}$, often derived based on certain concentration bounds. In this paper, we focus on the following Bernstein-style penalty to exploit the importance of certain variance statistics:  
\begin{align}
	b(s, a; V) \defn \min \Bigg\{ \max \bigg\{\sqrt{\frac{\cb\log\frac{N}{(1-\gamma)\delta}}{N(s, a)}\mathsf{Var}_{\widehat{P}_{ s, a}} (V )} ,\,\frac{2\cb \log\frac{N}{(1-\gamma)\delta}}{(1-\gamma)N(s, a)} \bigg\} ,\, \frac{1}{1-\gamma}  \Bigg\} +\frac{5}{N} 
	\label{def:bonus-Bernstein-infinite}
\end{align}
for every $(s,a)\in \cS\times \cA$, 
where $\cb> 0$ is some numerical constant (e.g., $\cb=144$), and $\delta \in (0,1)$ is some given quantity (in fact, $1-\delta$ is the target success probability).
Here, for any vector $V \in \mathbb{R}^S$, 
we recall that $\mathsf{Var}_{\widehat{P}_{s,a}} (V)$ is the variance of $V$ w.r.t.~the distribution $\widehat{P}_{s,a}$ (see \eqref{eq:defn-Var-P-V}). 
%
%


We immediately isolate several useful properties as follows, whose proof is postponed to Appendix~\ref{sec:proof-lemma:contraction}. 
%
\begin{lemma} \label{lem:contraction}
For any $\gamma \in [\frac{1}{2},1)$,  
the operator $\Tpess(\cdot)$ (cf.~\eqref{eq:empirical-Bellman-infinite}) with the Bernstein-style penalty~\eqref{def:bonus-Bernstein-infinite} 
is a $\gamma$-contraction w.r.t.~$\|\cdot\|_{\infty}$,  that is, 
\begin{align}
	\big\| \Tpess (Q_1) - \Tpess(Q_2) \big\|_{\infty} \le \gamma \|Q_1 - Q_2\|_{\infty} 
\end{align}
for any $Q_1, Q_2\in \mathbb{R}^{S\times A}$ obeying $Q_1(s,a),Q_2(s,a)\in \big[0, \frac{1}{1-\gamma}\big]$ for all $(s,a)\in \cS\times \cA$. 
	In addition, there exists a unique fixed point $\widehat{Q}_{\mathsf{pe}}^{\star}$ of  the operator $\Tpess(\cdot)$, 
which also obeys $0\leq \widehat{Q}_{\mathsf{pe}}^{\star}(s,a) \leq \frac{1}{1-\gamma}$ for all $(s,a)\in \cS\times \cA$. 
%
%
\end{lemma}
%
\noindent 
In words, even though $\Tpess(\cdot)$ integrates the penalty terms, it still preserves the $\gamma$-contraction property and admits a unique fixed point,  
thereby resembling the classical Bellman operator \eqref{eq:classical-Bellman-operator-inf}.

%
%
%
%

\begin{algorithm}[t]
\DontPrintSemicolon
	\textbf{input:} dataset $\mathcal{D}$; reward function $r$; target success probability $1-\delta$; max iteration number $\tau_{\max}$. \\
	\textbf{initialization:} $\widehat{Q}_{0}=0$, $\widehat{V}_0=0$. 

	construct the empirical transition kernel $\widehat{P}$ according to \eqref{eq:empirical-P-infinite}. \\

   \For{$\tau =1,2,\cdots, \taumax $}
	{
		
		\For{$s\in \cS, a\in \cA$}{
			compute the penalty term $b\big(s,a; \widehat{V}_{\tau-1}\big)$ according to \eqref{def:bonus-Bernstein-infinite}. \\
			set $\widehat{Q}_{\tau}(s, a) = \max\big\{r(s, a) + \gamma\widehat{P}_{s, a}\widehat{V}_{\tau-1} - b\big(s, a; \widehat{V}_{\tau-1}\big), 0\big\}$. \\
		}

	\For{$s\in \cS$}{
		set $\widehat{V}_{\tau}(s) = \max_a \widehat{Q}_{\tau}(s,a)$. \label{alg:infinite-q-update}
	}
	}
	\textbf{output:} $\widehat{\pi}$ s.t.~$\widehat{\pi}(s) \in \arg\max_a \widehat{Q}_{\taumax}(s,a)$ for any $s\in \cS$. 
	\caption{Offline value iteration with LCB (VI-LCB) for discounted infinite-horizon MDPs}
 \label{alg:vi-lcb-infinite}
\end{algorithm}


\paragraph{The VI-LCB algorithm.} We are now positioned to introduce the VI-LCB algorithm, which can be regarded as classical value iteration applied in conjunction with pessimism.  Specifically, the algorithm applies the Bernstein-style pessimistic operator $\Tpess$ (cf.~\eqref{eq:empirical-Bellman-infinite}) iteratively in order to find its fixed point:    
%
\begin{align}
	\widehat{Q}_{\tau}(s, a) = \Tpess\big(\widehat{Q}_{\tau-1}\big)(s, a) = \max\Big\{r(s, a) + \gamma\widehat{P}_{s, a}\widehat{V}_{\tau-1} - b\big(s, a; \widehat{V}_{\tau-1} \big), 0\Big\},
	\qquad \tau = 1,2,\cdots
	\label{eq:VI-LCB-iterations-basic-inf}
\end{align}
We shall initialize it to $\widehat{Q}_0=0$, implement \eqref{eq:VI-LCB-iterations-basic-inf} for $\tau_{\max}$ iterations, 
and output $\widehat{Q}= \widehat{Q}_{\tau_{\max}}$ as the final Q-estimate. 
The final policy estimate $\widehat{\pi}$ is chosen on the basis of $\widehat{Q}$ as follows:  
\begin{align}
	\widehat{\pi}(s) \in \arg\max_a \widehat{Q}(s, a) \qquad \text{for all }s\in \cS,
\end{align}
with the whole algorithm summarized in Algorithm~\ref{alg:vi-lcb-infinite}.


Let us pause to explain the rationale of the pessimism principle on a high level. If a pair $(s,a)$ has been insufficiently visited in $\mathcal{D}$ (i.e., $N(s,a)$ is small), 
then the resulting Q-estimate $\widehat{Q}_{\tau}(s, a)$ could suffer from high uncertainty and become unreliable, which might in turn mislead value estimation. 
By enforcing suitable penalization $b(s,a; \widehat{V}_{\tau-1})$ based on certain lower confidence bounds, 
we can suppress the negative influence of such poorly visited state-action pairs. 
Fortunately, suppressing these state-action pairs might not result in significant bias in value estimation when $\Cstar$ is small;
for instance,  when the behavior policy $\pib$ resembles $\pi^{\star}$,  the poorly visited state-action pairs correspond primarily to suboptimal actions (as they are not selected by $\pi^{\star}$), making it acceptable to neglect these pairs.

Interestingly, Algorithm~\ref{alg:vi-lcb-infinite} is guaranteed to converge rapidly. 
In view of the $\gamma$-contraction property in Lemma~\ref{lem:contraction}, 
the iterates $\{\widehat{Q}_{\tau}\}_{\tau\geq 0}$ converge linearly to the fixed point  $\widehat{Q}_{\mathsf{pe}}^{\star}$, 
as asserted below. 

\begin{lemma} \label{lem:monotone-contraction}
	Suppose $\widehat{Q}_0=0$. Then the iterates of Algorithm~\ref{alg:vi-lcb-infinite} obey
	\begin{align}
		\widehat{Q}_{\tau} \leq \widehat{Q}_{\mathsf{pe}}^{\star} 
		\qquad \text{and}\qquad
		\big\|\widehat{Q}_{\tau} - \widehat{Q}_{\mathsf{pe}}^{\star} \big\|_{\infty} \le  \frac{\gamma^{\tau}}{1-\gamma} 
		\qquad\quad
		\text{for all }\tau \geq 0,
	\end{align}
	where $\widehat{Q}_{\mathsf{pe}}^{\star}$ is the unique fixed point of $\Tpess$. 
	As a consequence, by choosing $\taumax \geq \frac{\log\frac{N}{1-\gamma}}{\log (1/\gamma)}$ one fulfills
	\begin{align}
		\big\|\widehat{Q}_{\taumax} - \widehat{Q}_{\mathsf{pe}}^{\star} \big\|_{\infty} \le  1/N. 
		\label{eq:taumax-Q-converge}
	\end{align}
\end{lemma}
\noindent 
The proof of this lemma is deferred to Appendix~\ref{sec:proof-lemma:monotone-contraction}.

\paragraph{Algorithmic comparison with \citet{rashidinejad2021bridging}.} 
VI-LCB has been studied in the prior work \citet{rashidinejad2021bridging}. 
The difference between our algorithm and the version therein is two-fold: 
\begin{itemize}
	\item {\em Sample reuse vs.~$\widetilde{O}\big(\frac{1}{1-\gamma}\big)$-fold sample splitting.} Our algorithm reuses the same set of samples across all iterations, which is in sharp contrast to  \citet{rashidinejad2021bridging} that employs fresh samples in each of the $\widetilde{O}\big(\frac{1}{1-\gamma}\big)$ iterations. This results in considerably better usage of available information. 

	\item {\em Bernstein-style vs.~Hoeffding-style penalty.} 
		Our algorithm adopts the Bernstein-type penalty, as opposed to the Hoeffding-style penalty in \citet{rashidinejad2021bridging}. 
		This choice leads to more effective exploitation of the variance structure across time. 
	
\end{itemize}

\paragraph{Pessimism vs.~optimism in the face of uncertainty. }
The careful reader might also notice the similarity between the pessimism principle and the optimism principle utilized in online RL.   
A well-developed paradigm that balances exploration and exploitation in online RL is 
optimistic exploration based on uncertainty quantification \citep{lai1985asymptotically}. 
The earlier work \citet{jaksch2010near} put forward an algorithm called UCRL2 that computes an optimistic policy with the aid of Hoeffding-style confidence regions for the probability transition kernel. 
Later on, \citet{azar2017minimax} proposed to build upper confidence bounds (UCB) for the optimal values instead, which leads to significantly improved sample complexity;  
see, e.g., \citet{wangq,he2021nearly} for the application of this strategy to discounted infinite-horizon MDPs. 
Note, however, that the rationales behind optimism and pessimism are remarkably different. 
In offline RL (which does not allow further data collection), 
the uncertainty estimates are employed to identify, and then rule out, poorly-visited actions;  
this stands in sharp contrast to the online counterpart where poorly-visited actions might be more favored during exploration.

\subsection{Performance guarantees}
\label{sec:theory-infinite-horizon}

When the Bernstein-style concentration bound \eqref{def:bonus-Bernstein-infinite} is adopted, 
the VI-LCB algorithm in Algorithm~\ref{alg:vi-lcb-infinite} yields $\varepsilon$-accuracy with a near-minimal number of samples, 
as stated below. 
\begin{theorem} \label{thm:infinite}
Suppose  $\gamma \in [\frac{1}{2},1)$, and consider any $0<\delta <1$ and $ \varepsilon \in \big(0, \frac{1}{1-\gamma} \big]$. 
Suppose that the total number of iterations exceeds $\taumax \geq \frac{1}{1-\gamma}\log\frac{N}{1-\gamma}$. 
With probability at least $1-2\delta,$  the policy $\widehat{\pi}$ returned by Algorithm~\ref{alg:vi-lcb-infinite} obeys
\begin{align}
	V^{\star}(\rho) - V^{\widehat{\pi}}(\rho) \leq \varepsilon,
\end{align}
provided that $\cb$ (cf.~the Bernstein-style penalty term in \eqref{def:bonus-Bernstein-infinite}) is some sufficiently large numerical constant and the total sample size exceeds 
\begin{align}
	N\geq\frac{c_1 S\Cstar \log\frac{NS }{(1-\gamma)\delta}}{(1-\gamma)^{3}\varepsilon^{2}} 
	\label{eq:N-range-epsilon2-inf-thm}
\end{align}
for some large enough numerical constant $c_1>0$, 
where $\Cstar$ is introduced in Definition~\ref{assumption:concentrate-infinite}. 
In addition, the above result continues to hold if $\Cstar$ is replaced with $C^{\star}$ (introduced in Definition~\ref{assumption:concentrate-infinite-simple}).
\end{theorem}
\begin{remark}
	Regarding the numerical constants in Theorem~\ref{thm:infinite}, 
	a conservative yet concrete sufficient condition is that 
	 $\cb \ge 144$ and  $c_1= 21000\cb$, which we shall rigorize in the proof.  
\end{remark}

The proof of this theorem is postponed to Section~\ref{sec:analysis-infinite}.
In general, the total sample size characterized by Theorem~\ref{thm:infinite} could be far smaller than the ambient dimension (i.e., $S^2A$) of the transition kernel $P$, thus precluding one from estimating $P$ in a reliable fashion. 
As a crucial insight from Theorem~\ref{thm:infinite},  the model-based (or plug-in) approach enables reliable offline learning even when model estimation is completely off.



Before discussing key implications of Theorem~\ref{thm:infinite}, we develop matching minimax lower bounds that help confirm the efficacy of the proposed model-based algorithm, whose proof can be found in Appendix~\ref{proof:thm-infinite-lb}.

\begin{theorem}\label{thm:infinite-lwoer-bound}
For any $(\gamma,S,\Cstar,\varepsilon)$ obeying $\gamma\in\big[\frac{2}{3},1\big),$ $S\geq 2$, $\Cstar\geq\frac{8\gamma}{S}$,
and $\varepsilon\leq\frac{1}{42(1-\gamma)}$, one can construct two
MDPs $\mathcal{M}_{0},\mathcal{M}_{1}$, an initial state distribution
$\rho$, and a batch dataset with $N$ independent samples and single-policy
clipped concentrability coefficient $\Cstar$ such that
\[
	\inf_{\widehat{\pi}}\max\left\{ \mathbb{P}_{0}\big( V^{\star}(\rho)-V^{\widehat{\pi}}(\rho)>\varepsilon\big), \,
	\mathbb{P}_{1}\big( V^{\star}(\rho) - V^{\widehat{\pi}}(\rho) >\varepsilon\big)\right\} \geq\frac{1}{8},
\]
provided that $$N\leq\frac{c_2S\Cstar}{(1-\gamma)^{3}\varepsilon^{2}}$$ 
for some numerical constant $c_{2}>0$. 
Here, the infimum is over all estimator $\widehat{\pi}$, and  
$\mathbb{P}_{0}$ (resp.~$\mathbb{P}_{1}$) denotes the probability
when the MDP is $\mathcal{M}_{0}$ (resp.~$\mathcal{M}_{1}$). 
\end{theorem}
\begin{remark}
	As a more concrete (yet conservative) condition for $c_2$, Theorem~\ref{thm:infinite-lwoer-bound} is valid when $c_2=1/ 25088$. 
\end{remark}
%

\paragraph{Implications.} 
In the following, we take a moment to interpret the above two theorems and single out several key implications about the proposed model-based algorithm.

\begin{itemize}
	\item 

{\em Optimal sample complexities.} 
In the presence of the Bernstein-style penalty, the total number of samples needed for our algorithm to yield $\varepsilon$-accuracy is 
\begin{equation}
	\widetilde{O}\bigg( \frac{ S\Cstar  }{(1-\gamma)^3\varepsilon^2}  \bigg). 
	\label{eq:sample-size-Bernstein-infinite}
\end{equation}
This taken together with the minimax lower bound asserted in Theorem~\ref{thm:infinite-lwoer-bound} 
confirms the optimality of the proposed model-based approach  (up to some logarithmic factor). 
In comparison, the sample complexity derived in \citet{rashidinejad2021bridging} exhibits a worse dependency on the effective horizon (i.e., $\frac{1}{(1-\gamma)^{5}}$). 
		Theorem~\ref{thm:infinite-lwoer-bound} also enhances the lower bound developed in  \citet{rashidinejad2021bridging} to accommodate the scenario where $\Cstar$ can be much smaller than $C^{\star}$, i.e., $\Cstar =O(1/S)$.


	\item
		{\em No burn-in cost.} The fact that the sample size bound \eqref{eq:N-range-epsilon2-inf-thm} holds for the full $\varepsilon$-range (i.e., any given $\varepsilon \in \big( 0, \frac{1}{1-\gamma} \big]$) means that there is no burn-in cost required to achieve sample optimality. 		This not only drastically improves upon, but in fact eliminates, the burn-in cost of the best-known sample-optimal result (cf.~\eqref{eq:VR-model-free-prior-work}), the latter of which required a burn-in cost at least on the order of $\frac{SC^{\star}}{(1-\gamma)^5}$. Accomplishing this requires one to tackle the sample-hungry regime, which is statistically challenging to cope with.

	\item 
{\em No need of sample splitting.}  It is noteworthy that prior works typically required sample splitting. For instance,  \citet{rashidinejad2021bridging} analyzed the VI-LCB algorithm with fresh samples employed in each iteration, which effectively split the data into $\widetilde{O}\big(\frac{1}{1-\gamma}\big)$ disjoint subsets. In contrast, the algorithm studied herein permits the reuse of all samples across all iterations. This is an important feature in sample-starved applications to effectively maximize information utilization, and is a crucial factor that assists in improving the sample complexity compared to \citet{rashidinejad2021bridging}.

\item 
{\em Sample size saving when $\Cstar<1$.} 
In view of Theorem~\ref{thm:infinite}, the sample complexity of the proposed algorithm can be as low as 
\[
	\widetilde{O}\bigg( \frac{ 1  }{(1-\gamma)^3\varepsilon^2}  \bigg)
\]
when $\Cstar$ is on the order of $1/S$. This might seem somewhat surprising at first glance, given that the minimax sample complexity for policy evaluation is at least $\widetilde{O}\big( \frac{ S  }{(1-\gamma)^3\varepsilon^2}  \big)$ even in the presence of a simulator \citep{azar2013minimax}. To elucidate this, we note that the condition $\Cstar=O(1/S)$ implicitly imposes special --- in fact, highly compressible --- structure on the MDP that enables sample size reduction. 
As we shall see from the lower bound construction in Theorem~\ref{thm:infinite-lwoer-bound}, 
the case with $\Cstar=O(1/S)$ might require $d^{\star}(s,a)$ to concentrate on one or a small number of important states, 
with exceedingly small probability assigned to the remaining ones. If this occurs, then it often suffices to focus on what happens on these important states,  thus requiring much fewer samples.

\end{itemize}

\paragraph{Comparisons with prior statistical analysis.} 
Before concluding this section, we highlight the innovations of our statistical analysis compared to past theory when it comes to discounted infinite-horizon MDPs.  
To begin with, our sample size improvement over \citet{rashidinejad2021bridging} stems from the two algorithmic differences mentioned in Section~\ref{sec:VI-LCB-infinite-horizon}: 
the sample-reuse feature allows one to improve a factor of $\frac{1}{1-\gamma}$, while the use of Bernstein-style penalty yields an additional gain of $\frac{1}{1-\gamma}$. 
In addition, while the design of data-driven Bernstein-style bounds has been extensively studied in online RL in discounted MDPs (e.g., \citet{zhang2021model,he2021nearly}), 
all of these past results were either sample-suboptimal, or required a huge burn-in sample size (e.g., $\frac{S^3A^2}{(1-\gamma)^4}$ in \citet{he2021nearly}).  In other words,  sample optimality was not previously achieved in the most data-hungry regime.  
In comparison, our theory ensures optimality of our algorithm even for the most sample-constrained scenario, 
which relies on much more delicate statistical tools. 
In a nutshell, our statistical analysis is built upon at least two ideas: 
(i) a leave-one-out analysis framework that allows to decouple complicated statistical dependency across iterations without losing statistical tightness; (ii) a delicate self-bounding trick that allows us to simultaneously control multiple crucial statistical quantities (e.g., empirical variance) in the most sample-starved regime.

\section{Algorithm and theory: episodic finite-horizon MDPs}

In this section, we turn attention to the studies of offline RL for episodic finite-horizon MDPs.

\subsection{Models and assumptions}
\label{sec:models-finite}

As before, we briefly state some preliminaries about finite-horizon MDPs, before moving on to the sampling model and the goal. 
The readers can consult \citet{bertsekas2017dynamic} for more details about finite-horizon MDPs.

\paragraph{Basics of finite-horizon MDPs.}  Consider the setting of a finite-horizon Markov decision process, as denoted by $\mathcal{M}=\{\cS, \cA, H, P, r\}$.  It consists of the following key components: 
(i) $\cS=\{1,\cdots,S\}$: a state space of size $S$;  
(ii) $\cA=\{1,\cdots,A\}$: an action space of size $A$; 
(iii) $H$: the horizon length; 
(iv) $P=\{P_h\}_{1\leq h\leq H}$, with $P_h: \cS\times \cA\rightarrow \Delta(\cS)$ denoting the probability transition kernel at step $h$ (namely,  $P_h(\cdot \mymid s,a)$ stands for the transition probability of the MDP at step $h$ when the current state-action pair is $(s,a)$); 
(v) $r = \{r_h\}_{1\leq h\leq H}$, with $r_h: \cS\times \cA \rightarrow [0,1]$ denoting the reward function at step $h$ (namely, $r_h(s,a)$ indicates the immediate reward gained at step $h$ when the current state-action pair is $(s,a)$). 
It is assumed without loss of generality that the immediate rewards fall within the interval $[0,1]$ and are deterministic. Conveniently, we introduce the following $S$-dimensional row vector 
\begin{equation} \label{eq:transition_vector}
	P_{h,s,a} \coloneqq P_h(\cdot \mymid s,a )  
\end{equation}
for any $(s,a,h)\in \cS\times \cA \times [H]$. 

A (possibly randomized) policy $\pi=\{\pi_h\}_{1\leq h\leq H}$ with $\pi_h: \cS \rightarrow \Delta (\cA)$ is an action selection rule,   
such that $\pi_h(a \mymid s)$ specifies the probability of choosing action $a$ when in state $s$ and step $h$.  When $\pi$ is a deterministic policy, we overload the notation and let $\pi_h(s)$ represent the action selected by $\pi$ in state $s$ at step $h$. 
We can generate a sample trajectory $\{(s_h,a_h)\}_{1\leq h\leq H}$ by implementing policy $\pi$ in the MDP $\mathcal{M}$,
where $s_h$ and $a_h$ denote the state and the action in step $h$, respectively.  
We then introduce the value function $V^{\pi}=\{V_h^{\pi}\}_{1\leq h\leq H}$
and the Q-function $Q^{\pi}=\{Q_h^{\pi}\}_{1\leq h\leq H}$ associated with policy $\pi$;  
specifically, the value function $V_h: \cS\rightarrow \mathbb{R}$ of policy $\pi$ at step $h$ is defined to the be the expected cumulative reward from step $h$ on as a result of policy $\pi$, namely, 
\begin{align}
	\forall s\in \cS: \qquad 
	V_h^{\pi}(s) \coloneqq \mathbb{E} \left[ \sum_{t=h}^H r_t(s_t, a_t) \mid s_h = s; \pi \right],
	\label{eq:defn-V-pi-finite}
\end{align}
where the expectation is taken over the randomness over the sample trajectory $\{(s_t, a_t)\}_{t=h}^H$ when policy $\pi$ is implemented
(i.e., $a_t\sim \pi_t(\cdot\mymid s_t)$ and $s_{t+1}\sim P_t(\cdot\mymid s_t, a_t)$ for all $t\geq h$).   
Correspondingly, the Q-function of policy $\pi$ at step $h$ is defined to be
\begin{align}
	\forall (s,a)\in \cS \times \cA: \qquad 
	Q_h^{\pi}(s,a) \coloneqq \mathbb{E} \left[ \sum_{t=h}^H r_t(s_t, a_t) \mid s_h = s, a_h=a; \pi \right]
	\label{defn:Q-pi-finite}
\end{align}
when conditioned on the state-action pair $(s,a)$ at step $h$. 
If the initial state is drawn from a distribution $\rho \in \Delta(\cS)$, we find it convenient to define the following weighted value function of policy $\pi$: 
\begin{equation}
	V_1^{\pi}(\rho) \coloneqq \mathop{\mathbb{E}}_{s\sim \rho} \big[ V^{\pi}_1(s) \big].
	\label{eq:defn-V-rho-finite}
\end{equation}
Additionally, we introduce the following {\em occupancy distributions} associated with policy $\pi$ at step $h$: 
\begin{subequations} 	\label{eq:dh-pi-defn}
\begin{align}
	d_{h}^{\pi}(s; \rho) & \coloneqq\mathbb{P}\big(s_{h}=s\mid s_{1}\sim\rho;\pi\big), \\
	d_{h}^{\pi}(s,a; \rho) & \coloneqq \mathbb{P}\big(s_{h}=s,a_{h}=a\mid s_{1}\sim\rho;\pi\big) 
	= d_{h}^{\pi}(s; \rho) \pi( a \mymid s),
\end{align}
\end{subequations}%
which are conditioned on the initial state distribution $s_1\sim \rho$ and the event that all actions are selected according to $\pi$. In particular, it is self-evident that 
\begin{equation}
	d_1^{\pi}(s; \rho) = \rho(s) \qquad \text{for any policy }\pi\text{ and any state }s\in \cS.
	\label{eqn:defn-d1rho-finite}
\end{equation}

It is well known that there exists at least one deterministic policy that simultaneously maximizes the value function and the Q-function for all $(s,a,h)\in \cS\times \cA\times [H]$ \citep{bertsekas2017dynamic}. 
In light of this, we shall denote by $\pi^{\star}=\{\pi^{\star}_h\}_{1\leq h\leq H}$ an {\em optimal deterministic} policy throughout this paper; this allows us to employ $\pi^{\star}_h(s)$ to indicate the corresponding optimal action chosen in state $s$ at step $h$. 
The resulting optimal value function and optimal Q-function are denoted respectively by $V^{\star}=\{V^{\star}_h\}_{1\leq h\leq H}$ and $Q^{\star}=\{Q^{\star}_h\}_{1\leq h\leq H}$: 
\[
	\forall (s,a,h) \in \cS\times \cA\times [H]:
	\qquad V_{h}^{\star} \coloneqq V_h^{\pi^{\star}} \qquad \text{and} \qquad Q_{h}^{\star} \coloneqq Q_h^{\pi^{\star}} .
\]
Furthermore, 
we adopt the following notation for convenience:
\begin{equation}
	\forall (s,a,h)\in \cS\times \cA \times [H]: \quad
	d_h^{\star}(s) \coloneqq  d_h^{\pi^{\star}}(s; \rho) \quad \text{and} \quad 
	d_h^{\star}(s, a) \coloneqq d_h^{\pi^{\star}}(s,a; \rho) = d_h^{\star}(s) \mathds{1} \{a= \pi^{\star}(s)\},
	\label{eq:dhstar-finite}
\end{equation}
where the last identity holds given that $\pi^{\star}$ is assumed to be deterministic.

\paragraph{Offline/batch data.}
Suppose that we have access to a batch dataset (or historical dataset) $\mathcal{D}$, which comprises a collection of $K$ i.i.d.~sample trajectories generated by a behavior policy $\pib=\{\pib_h\}_{1\leq h\leq H}$. 
More specifically, the $k$-th sample trajectory ($1 \le k \le K$) consists of a data sequence
\begin{align}
	\big( s_1^k, a_1^k, s_2^k, a_2^k, \ldots, s_H^k, a_H^k, s_{H+1}^k \big),
	\label{eq:finite-sequence-k}
\end{align}
which is generated by the MDP $\mathcal{M}$ under the behavior policy $\pib$ in the following manner:
\begin{align}\label{eq:finite-batch-size-def}
s_1^k \sim \rhob,
\qquad a_h^k \sim \pib_h(\cdot\mymid s_h^k) 
\qquad\text{and}\qquad 
	s_{h+1}^k \sim P_h(\cdot\mymid s_h^k, a_h^k) ,
	\qquad 1 \le h \le H.
\end{align}
Here and throughout, $\rhob$ stands for some predetermined initial state distribution associated with the batch dataset. 
In addition to the above dataset (cf.~\eqref{eq:finite-sequence-k} for all $1\leq k\leq K$), 
the learner also has access to the reward function. 
For notational simplicity, we introduce the following short-hand notation for the occupancy distribution w.r.t.~the behavior policy $\pib$:
\begin{equation}
	\forall (s,a,h)\in \cS\times \cA \times [H]: \qquad
	\myrho_h(s) \coloneqq  d_h^{\pib}(s; \rhob) \quad \text{and} \quad \myrho_h(s, a) \coloneqq d_h^{\pib}(s,a; \rhob). 
	\label{eq:dhb-finite}
\end{equation}
In particular, it is easily seen that $\myrho_1(s)=\rhob(s)$ for all $s\in \cS$. Note that the initial state distribution $\rhob$ of the batch dataset might not coincide with the test state distribution $\rho$.

Akin to Definition~\ref{assumption:concentrate-infinite-simple}, 
prior works (e.g., \citet{xie2021policy}) have introduced the following concentrability coefficient to capture the distribution shift between the desired distribution and the one induced by the behavior policy. 
%
\begin{definition}[Single-policy concentrability for finite-horizon MDPs] 
\label{assumption:concentrate-finite-simple}
The single-policy concentrability coefficient of a batch dataset $\mathcal{D}$ is defined as
\begin{align}
	C^{\star} \coloneqq \max_{(s, a, h) \in \cS \times \cA \times [H]} \, \frac{d_h^{\star}(s, a) }{\myrho_h(s, a)} ,
	\label{eq:concentrate-finite-simple}
\end{align}
which clearly satisfies $C^{\star} \geq 1$. 
\end{definition}
%

%

Similar to the discounted infinite-horizon counterpart, 
$C^{\star}$ employs the largest density ratio (using the occupancy distributions defined above) to measure the distribution mismatch;
it concerns the behavior policy vs.~a single policy $\pi^{\star}$, and does not require uniform coverage of the state-action space (namely, it suffices to cover the part reachable by $\pi^{\star}$). 
As before, we further introduce a slightly modified version of $C^{\star}$ as follows.

%
\begin{definition}[Single-policy clipped concentrability for finite-horizon MDPs] 
\label{assumption:concentrate-finite}
The single-policy clipped concentrability coefficient of a batch dataset $\mathcal{D}$ is defined as
\begin{align}
	\Cstar \coloneqq \max_{(s, a, h) \in \cS \times \cA \times [H]}\frac{\min\big\{d_h^{\star}(s, a), \frac{1}{S}\big\}}{\myrho_h(s, a)} .
	\label{eq:concentrate-finite}
\end{align}
%
\end{definition}

From the definition above, it holds trivially that
\begin{align}\label{eq:Cstar_lower_bound_finite}
	\Cstar \leq C^{\star} \qquad \text{and} \qquad \Cstar \geq \frac{1}{S}.
\end{align}
As we shall see shortly, 
while all sample complexity upper bounds developed herein remain valid if we replace $\Cstar$ with $C^{\star}$, 
the use of $\Cstar$ might yield some sample size reduction when $\Cstar$ drops below 1.

\paragraph{Goal.} 
With the above batch dataset $\mathcal{D}$ in hand, our aim is to compute, in a sample-efficient fashion, a policy $\widehat{\pi}$ that results in near-optimal values w.r.t.~a given test state distribution $\rho\in \Delta(\cS)$. Formally speaking, the current paper focuses on achieving  
\[
	V_1^{\star}(\rho) - V_1^{\widehat{\pi}}(\rho) \leq \varepsilon 
\]
with high probability using as few samples as possible, where $\varepsilon$ stands for the target accuracy level. 
We seek to achieve sample optimality for the full $\varepsilon$-range, i.e.,  for any $\varepsilon \in (0, H]$.

\subsection{A model-based offline RL algorithm: VI-LCB}
\label{sec:VI-LCB-simple-finite}

Suppose for the moment that we have access to a dataset $\mathcal{D}_0$ containing $N$ sample transitions $\{(s_i, a_i, h_i, s_i')\}_{i=1}^{N}$, 
where $(s_i, a_i, h_i, s_i')$ denotes the transition from state $s_i$ at step $h_i$ to state $s_i'$ in the next step when action $a_i$ is taken.  
We now describe a pessimistic variant of the model-based approach on the basis of $\mathcal{D}_0$.

\paragraph{Empirical MDP.}
For each $(s,a,h)\in \cS\times \cA\times [H]$, we denote by 
\begin{subequations}
\label{eq:defn-Nh-sa-finite}
\begin{align}
	N_h(s,a) &\coloneqq \sum_{i=1}^N \mathds{1} \big\{ (s_i, a_i, h_i) = (s,a,h) \big\} \\ 
	N_h(s)   &\coloneqq \sum_{i=1}^N \mathds{1} \big\{ (s_i, h_i) = (s,h) \big\}
\end{align}
\end{subequations}
the total number of sample transitions at step $h$ that transition from $(s,a)$ and from $s$, respectively.  We can then compute the empirical estimate $\widehat{P}=\{\widehat{P}_h\}_{1\leq h\leq H}$ of the transition kernel $P$ as follows:
\begin{align}
	\widehat{P}_{h}(s'\mymid s,a) = 
	\begin{cases} \frac{1}{N_h(s,a)} \sum\limits_{i=1}^N \mathds{1} \big\{ (s_i, a_i, h_i, s_i') = (s,a,h,s') \big\}, & \text{if } N_h(s,a) > 0 \\
		\frac{1}{S}, & \text{else}
	\end{cases}  
	\label{eq:empirical-P-finite}
\end{align}
for each $(s,a,h,s')\in \cS\times \cA\times [H] \times \cS$. 

\paragraph{The VI-LCB algorithm.} 
With this estimated model in place, the VI-LCB algorithm (i.e., value iteration with lower confidence bounds) maintains the value function estimate $\{\widehat{V}_h\}$ and Q-function estimate $\{\widehat{Q}_h\}$, and works   
backward from $h=H$ to $h=1$ as in classical dynamic programming  with the terminal value $\widehat{V}_{H+1}=0$ \citep{xie2021policy,jin2021pessimism}. Specifically,  the algorithm adopts the following update rule:  
\begin{align}
	\widehat{Q}_h(s, a) = \max\Big\{r_h(s, a) + \widehat{P}_{h, s, a} \widehat{V}_{h+1}   - b_h(s, a), 0\Big\} ,
	\label{eq:VI-LCB-finite}
\end{align}
where $\widehat{P}_{h, s, a}$ is the empirical estimate of $P_{h, s, a}$ (cf.~\eqref{eq:transition_vector}), 
\begin{equation}
	\widehat{V}_{h+1}(s) = \max_a \widehat{Q}_{h+1}(s, a),
	\label{eq:defn-widehat-V-h+1-finite}
\end{equation}
and $b_h(s, a) \geq 0$ denotes some penalty term that is a decreasing function in $N_h(s,a)$ (as we shall specify momentarily). In addition, the policy $\widehat{\pi}$ is selected greedily in accordance to the Q-estimate:
\begin{align}
	\forall(s,h)\in \cS\times [H]: \qquad
	\widehat{\pi}_h(s) \in \arg\max_a \widehat{Q}_h(s, a).
\end{align}
In a nutshell, the VI-LCB algorithm --- as summarized in Algorithm~\ref{alg:vi-lcb-finite} --- applies the classical value iteration approach to the empirical model $\widehat{P}$, 
and in addition, implements the principle of pessimism via certain lower confidence penalty terms $\{b_h(s,a)\}$.

%

\begin{algorithm}[t]
\DontPrintSemicolon
	\textbf{input:} dataset $\mathcal{D}_0$; reward function $r$; target success probability $1-\delta$. \\
	\textbf{initialization:} $\widehat{V}_{H+1}=0$. \\

   \For{$h=H,\cdots,1$}
	{
		compute the empirical transition kernel $\widehat{P}_h$ according to \eqref{eq:empirical-P-finite}. \\
		\For{$s\in \cS, a\in \cA$}{
			compute the penalty term $b_h(s,a)$ according to \eqref{def:bonus-Bernstein-finite}. \\
			set $\widehat{Q}_h(s, a) = \max\big\{r_h(s, a) + \widehat{P}_{h, s, a} \widehat{V}_{h+1} - b_h(s, a), 0\big\}$. \\ 
		}
		\For{$s\in \cS$}{
			set $\widehat{V}_h(s) = \max_a \widehat{Q}_h(s, a)$ and $\widehat{\pi}_h(s) \in \arg\max_{a} \widehat{Q}_h(s,a)$.
		}
	}

	\textbf{output:} $\widehat{\pi}=\{\widehat{\pi}_h\}_{1\leq h\leq H}$. 
	\caption{Offline value iteration with LCB (VI-LCB) for finite-horizon MDPs.}
 \label{alg:vi-lcb-finite}
\end{algorithm}

\paragraph{The Bernstein-style penalty terms.} 
As before, we adopt Bernstein-style penalty in order to better capture the variance structure over time; that is,  
\begin{align}
	\forall (s,a,h) \in \cS\times \cA\times [H]:
	\; 
	b_h(s, a) = \min \Bigg\{ \sqrt{\frac{\cb\log\frac{NH}{\delta}}{N_h(s, a)}\mathsf{Var}_{\widehat{P}_{h, s, a}}\big(\widehat{V}_{h+1}\big)} + \cb H\frac{\log\frac{NH}{\delta}}{N_h(s, a)} ,\, H \Bigg\} 
	\label{def:bonus-Bernstein-finite}
\end{align}
for some universal constant $\cb> 0$ (e.g., $\cb=16$). Here, $\mathsf{Var}_{\widehat{P}_{h, s, a}}\big(\widehat{V}_{h+1}\big)$ corresponds to the variance of $\widehat{V}_{h+1}$ w.r.t.~the distribution $\widehat{P}_{h,s,a}$ (see the definition \eqref{eq:defn-Var-P-V}).
%
%
Note that we choose $\widehat{P}$ as opposed to $P$ (i.e., $\mathsf{Var}_{P_{h,s,a}} \big(\widehat{V}_{h+1}\big)$) in the variance term, mainly because we have no access to the true  transition kernel $P$.

Finally, it is worth noting that the Bernstein-style uncertainty estimates have been widely studied 
when performing online exploration in episodic finite-horizon MDPs (e.g., \citet{fruit2020improved,azar2017minimax,talebi2018variance,jin2018q,li2021breaking,zhang2020almost}). Once again, the main purpose therein is to encourage exploration of the insufficiently visited states/actions, 
a mechanism that is not applicable to offline RL due to the absence of further data collection.

\subsection{VI-LCB with two-fold subsampling}
\label{sec:sample-split}

Given that the batch dataset $\mathcal{D}$ is composed of several sample trajectories each of length $H$, the sample transitions in $\mathcal{D}$ cannot be viewed as being independently generated (as the sample transitions at step $h$ might influence the sample transitions in the subsequent steps). As one can imagine, the presence of such temporal statistical dependency considerably complicates analysis.

In order to circumvent this technical difficulty, we propose a two-fold subsampling trick that allows one to exploit the desired statistical independence. Informally, we propose the following steps:  
\begin{itemize}
	\item First of all, we randomly split the dataset into two halves $\Dmain$ and $\Daux$, where $\Dmain$ consists of $\Nmain_h(s)$ sample transitions from state $s$ at step $h$.   
	\item For each $(s,h)\in \cS\times [H]$, we use the dataset $\Daux$ to construct a high-probability lower bound $\Ntrim_h(s)$ on $\Nmain_h(s)$, and then subsample $\Ntrim_h(s)$ sample transitions w.r.t.~$(s,h)$ from $\Dmain$; this results in a new subsampled dataset $\Dtrim$. 
	\item Run VI-LCB on the subsampled dataset $\Dtrim$ (i.e., Algorithm~\ref{alg:vi-lcb-finite}).
\end{itemize}
The whole procedure is detailed in Algorithm~\ref{alg:vi-lcb-finite-split}. 
A few important features are worth highlighting, under the assumption that the sample trajectories in $\mathcal{D}$ are independently generated from the same distribution. 
\begin{itemize}

\item
Given that $\{\Ntrim_h(s)\}$ are computed on the basis of the dataset $\Daux$ and that $\Dtrim$ is subsampled from another dataset $\Dmain$, 
one can clearly see that $\{\Ntrim_h(s)\}$ are statistically independent from the sample transitions in $\Dtrim$. 

\item
As we shall justify in the analysis (i.e., Section~\ref{sec:independence-finite}), the samples in $\Dtrim$ can almost be treated as being statistically independent, a key attribute resulting from the subsampling trick. 

\item The proposed algorithm only splits the data into two subsets, which is in stark contrast to prior variants of VI-LCB that perform $H$-fold sample splitting (e.g., \citet{xie2021policy}). 
	Eliminating the $H$-fold splitting requirement plays a crucial role in enabling optimal sample complexity. 
\end{itemize}

\begin{algorithm}[t]
\DontPrintSemicolon
	\textbf{input:} a dataset $\mathcal{D}$; reward function $r$. \\

	\textbf{subsampling:} run the following procedure to generate the subsampled dataset $\Dtrim$. 
	{ \begin{itemize}
	\item[1)] {\em Data splitting.} Split $\mathcal{D}$ into two halves:  $\Dmain$ (which contains the first $K/2$ trajectories), and $\Daux$ (which contains the remaining $K/2$ trajectories);    we let $\Nmain_h(s)$ (resp.~$\Naux_h(s)$) denote the number of sample transitions in $\Dmain$ (resp.~$\Daux$) that transition from state $s$ at step $h$.  

	\item[2)] {\em Lower bounding $\{\Nmain_h(s)\}$ using $\Daux$.} For each $s\in \cS$ and $1\leq h\leq H$, compute
%
\begin{align}
	\label{eq:defn-Ntrim}
	\Ntrim_h(s) &\coloneqq \max\left\{\Naux_h(s) - 10\sqrt{\Naux_h(s)\log\frac{HS}{\delta}}, \, 0\right\} ;
\end{align}
%


	\item[3)] {\em Random subsampling.} Let ${\Dmain}'$ be the set of all sample transitions (i.e., the quadruples taking the form $(s,a,h,s')$) from $\Dmain$. 
		Subsample ${\Dmain}'$ to obtain $\Dtrim$, such that for each $(s,h)\in \cS\times [H]$, $\Dtrim$ contains $\min \{ \Ntrim_h(s), \Nmain_h(s)\}$ sample transitions randomly drawn from ${\Dmain}'$. 
\end{itemize}
}
	\textbf{run VI-LCB:} set $\mathcal{D}_0 = \Dtrim$; run Algorithm~\ref{alg:vi-lcb-finite} to compute a policy $\widehat{\pi}$.  \\

	\caption{Subsampled VI-LCB for episodic finite-horizon MDPs}
 	\label{alg:vi-lcb-finite-split}

\end{algorithm}

%
%

Before proceeding, we formally justify that $\Ntrim_h(s)$ --- as computed in \eqref{eq:defn-Ntrim} --- is a valid lower bound on $\Nmain_h(s)$. 
Here and below, we denote by $\Ntrim_h(s, a)$ the number of sample transitions in $\Dtrim$ that are associated with the state-action pair $(s,a)$ at step $h$.  The proof of this lemma can be found in Appendix~\ref{sec:proof-lemma:Ntrim-LB}. 
\begin{lemma}
\label{lemma:Ntrim-LB} 
	Suppose that the $K$ trajectories in $\mathcal{D}$ are generated in an i.i.d.~fashion (see Section~\ref{sec:models-finite}). 
With probability at least $1-8\delta$, the quantities constructed in \eqref{eq:defn-Ntrim} obey
\begin{subequations}
\label{eq:samples-finite}
\begin{align}
	\Ntrim_h(s)  &\le \Nmain_{h}(s), \label{eq:samples-finite-UB}\\
	\Ntrim_h(s, a) &\ge \frac{ K\myrho_h(s, a) }{8} - 5 \sqrt{ K\myrho_h(s, a)\log\frac{KH}{\delta}}
	\label{eq:samples-finite-LB}
\end{align}
\end{subequations}
simultaneously for all $1 \le h \le H$ and all $(s, a) \in \cS \times \cA$. 
\end{lemma}
%


\subsection{Performance guarantees}

In what follows, we characterize the sample complexity of Algorithm~\ref{alg:vi-lcb-finite-split}, as formalized below.  
\begin{theorem} \label{thm:finite}
Consider any $\varepsilon \in (0, H]$ and any $0<\delta <1$. 
With probability exceeding $1-12\delta,$  the policy $\widehat{\pi}$ returned by Algorithm~\ref{alg:vi-lcb-finite-split} obeys
\begin{align}
	V^{\star}_1(\rho) - V^{\widehat{\pi}}_1(\rho) \leq \varepsilon 
\end{align}
as long as the penalty terms are chosen according to the Bernstein-style quantity \eqref{def:bonus-Bernstein-finite}  for some large enough numerical constant 
$\cb >0$, and the total number of sample trajectories exceeds 
\begin{align}\label{eq:K-complexity-finite}
	K\geq \frac{c_{\mathsf{k}}H^3S\Cstar\log\frac{KH}{\delta}}{\varepsilon^2} 
\end{align}
for some sufficiently large numerical constant $c_{\mathsf{k}}>0$, where $\Cstar$ is introduced in Definition~\ref{assumption:concentrate-finite}.   Additionally, the above result continues to hold if $\Cstar$ is replaced with $C^{\star}$ (introduced in Definition~\ref{assumption:concentrate-finite-simple}).  
\end{theorem}
\begin{remark}
	One concrete yet conservative requirement on $\cb$ and $c_{\mathsf{k}}$ for Theorem~\ref{thm:finite} to hold is: $\cb \geq 16$ and $c_{\mathsf{k}}=12800\cb$, as we shall solidify in the proof of Theorem~\ref{thm:finite}. 
\end{remark}

The proof of this result is postponed to Section~\ref{sec:analysis-finite}. 
In general, the total sample size characterized by Theorem~\ref{thm:finite} could be far smaller than the ambient dimension (i.e., $S^2AH$) of the probability transition kernel $P$, thus precluding one from estimating $P$ in a reliable fashion. 
As a crucial insight from Theorem~\ref{thm:finite},  the model-based (or plug-in) approach enables reliable policy learning even when model estimation is completely off.  
Our analysis of Theorem~\ref{thm:finite} relies heavily on (i) suitable decoupling of complicated statistical dependency via subsampling, 
and (ii) careful control of the variance terms in the presence of Bernstein-style penalty.


In order to help assess the tightness and optimality of~Theoerem~\ref{thm:finite}, we further develop a minimax lower bound as follows; the proof can be found in Appendix~\ref{proof:thm-finite-lb}.

\begin{theorem}
\label{thm:finite-lower-bound}
For any $(H, S, \Cstar, \varepsilon)$ obeying $H\geq 12$, $\Cstar \geq 8/S$ and $\varepsilon \leq c_3 H$,
one can construct a collection of MDPs $\{\mathcal{M}_{\theta} \mid \theta \in \Theta\}$,  
an initial state distribution $\rho$, and a batch dataset with $K$ independent sample trajectories each of length $H$, such that
\begin{align}
    \inf_{\widehat{\pi}} \max_{\theta\in\Theta} \mathbb{P}_\theta\left\{ V_1^{\star}(\rho) - V_1^{\widehat{\pi}}(\rho) \ge \varepsilon\right\} \geq \frac{1}{4},
\end{align}
provided that the total sample size
\begin{align}
	N =KH \leq \frac{c_4\Cstar S H^4}{\varepsilon^2}.
\end{align}
Here, $c_3,c_4>0$ are some small enough numerical constants, 
the infimum is over all estimator $\widehat{\pi}$, 
and $\mathbb{P}_{\theta}$ denotes the probability when the MDP is $\mathcal{M}_{\theta}$. 
\end{theorem}
\begin{remark}
	More concretely, one (conservative) condition regarding $c_3$ and $c_4$ that is sufficient for the validity of Theorem~\ref{thm:finite-lower-bound} is: $c_3= 1/2^{14}$ and $c_4 = 1 / 2^{36}$, as we shall see in the proof. 
\end{remark}


\paragraph{Implications.} 
In what follows, let us take a moment to discuss several other key implications of Theorem~\ref{thm:finite}.

\begin{itemize}

	\item {\em Near-optimal sample complexities.} In the presence of the Bernstein-style penalty, the total number of samples (i.e., $KH$) needed for our algorithm to yield $\varepsilon$-accuracy is 
\begin{equation}
  \widetilde{O}\bigg( \frac{ H^4S\Cstar  }{\varepsilon^2} \bigg). 
	\label{eq:sample-size-Bernstein}
\end{equation}
This confirms the optimality of the proposed model-based approach (up to some logarithmic term) when Bernstein-style penalty is employed, since Theorem~\ref{thm:finite-lower-bound} reveals that at least 
$  \frac{ H^4S\Cstar  }{\varepsilon^2}$
samples are needed regardless of the algorithm in use. 
%

	\item {\em Full $\varepsilon$-range and no burn-in cost.} The sample complexity bound \eqref{eq:K-complexity-finite} stated in Theorem~\ref{thm:finite} holds for an arbitrary $\varepsilon \in (0,H]$. 
In other words, no burn-in cost is needed for the algorithm to work sample-optimally. 
This improves substantially upon the state-of-the-art results for model-based and model-free offline algorithms, 
both of which require a significant level of burn-in sample size 
($H^9SC^{\star}$ and $H^6SC^{\star}$, respectively).  






	\item {\em Sample reduction and model compressibility when $\Cstar<1$.}		
Given that $\Cstar$ might drop below 1, the sample complexity of our algorithm might be as low as $\widetilde{O}\big( \frac{H^4S}{\varepsilon^2} \big)$. In fact, recognizing that $\Cstar$ can be as small as $\frac{1+o(1)}{S}$, we see that the sample complexity can sometimes be reduced to 
\begin{equation}
	\widetilde{O} \bigg(  \frac{H^4}{\varepsilon^2} \bigg) ,
\end{equation}
resulting in significant sample size saving compared to prior works.  
Caution needs to be exercised, however, that this sample size improvement is made possible as a result of certain {\em model compressibility} implied by a small $\Cstar$.  For instance, $\Cstar=O(1/S)$ might happen when a small number of states accounts for a dominant fraction of probability mass in $d^{\star}_h(s)$, with the remaining states exhibiting vanishingly small occupancy probability 
(see also the lower bound construction in the proof of Theorem~\ref{thm:finite-lower-bound}); if this happens, then it often suffices to focus  on learning those dominant states.


%


\end{itemize}

\paragraph{(In)-feasibility of estimating $C^{\star}_{\mathsf{clipped}}$.}
With the sample complexity \eqref{eq:sample-size-Bernstein} in mind, 
one natural question arises as to whether it is possible to estimate $C^{\star}_{\mathsf{clipped}}$ from the batch dataset.  
Unfortunately, this is in general infeasible, as demonstrated by the following example. 
\begin{itemize}
	\item (A hard example) 
	Consider an MDP with horizon $H=2$. 
	In step $h=1$, we have a singleton state space $\cS_1 = \{0\}$ and an action space $\cA_1 = \{0, 1\}$, 
		whereas in step $h=2$, we have a state space $\cS_2 = \{0, 1\}$ and a singleton  action space $\cA_2 = \{0\}$. 
		The reward function and the transition kernel are given by: 
	\begin{align*}
		&r_1(0, 0) = 0, \quad r_1(0, 1) = 0 , \quad
		r_2(0, 0) = 0, \quad r_2(1, 0) = 1
	\\
		P_1(0\mymid 0, 0) &= 0.5, \quad P_1(1\mymid 0, 0) = 0.5, \quad
	P_1(0\mymid 0, 1) = p,  \quad P_1(1\mymid 0, 1) = 1-p
	\end{align*}
		for some unknown parameter $p\in (0,1)$. 
	We have $K$ independent trajectories as usual,
		and let 
		\begin{equation}
			d^{\mathsf{b}}_1(0, 0) = 1 - \frac{1}{K} \qquad \text{and} \qquad d^{\mathsf{b}}_1(0, 1) = \frac{1}{K}. 
			\label{eq:hard-example-db}
		\end{equation}
		Elementary calculation then reveals that: 
		$C^{\star}_{\mathsf{clipped}} = K$ when $p < \frac{1}{2}$, and $C^{\star}_{\mathsf{clipped}} = 1 + \frac{1}{K-1}$ when $p > \frac{1}{2}$.  
		Such a remarkable difference in $C^{\star}_{\mathsf{clipped}}$ depends on the value of $p$, which is only reflected in $(s,a)=(0,1)$ at step 1. 
		However, by construction,  there is nonvanishing probability (i.e., $\big(1-d^{\mathsf{b}}_1(0, 1)\big)^K\approx 1/e$ for large $K$) 
		such that the dataset does not visit $(s,a)=(0,1)$ in step $h=1$ at all, 
		which in turn precludes one from  distinguishing $C^{\star}_{\mathsf{clipped}}= 1 + \frac{1}{K-1}$ from $C^{\star}_{\mathsf{clipped}} = K$ given only the available dataset. 
		
\end{itemize}
Fortunately, implementing our algorithm does not require prior knowledge of $C^{\star}_{\mathsf{clipped}}$ at all, 
and the algorithm succeeds once the task becomes  feasible. 
On the other hand, we won't be able to tell how large a sample size is enough {\em a priori}, 
but this is in general information-theoretically infeasible as illustrated by the above example.

\paragraph{Towards instance optimality.} 
While the primary focus of the current paper is minimax-optimal algorithm design, 
the theoretical framework developed herein enables instance-dependent analysis as well.  
Take episodic finite-horizon MDPs for example: our analysis framework directly leads to the following instance-dependent guarantee for Algorithm~\ref{alg:vi-lcb-finite-split}:
\begin{align}
	&V_{h}^{\star}(\rho)-V_{h}^{\widehat{\pi}}(\rho) = \big\langle d_{1}^{\star},V_{1}^{\star}-V_{1}^{\widehat{\pi}}\big\rangle \notag \\
	&\quad\le 12\sum_{j=h}^{H}\sum_s d_j^{\star}(s)\sqrt{\frac{\cb \log\frac{NH}{\delta}}{K \myrho_j \big(s, \pi_j^{\star}(s) \big)}\mathsf{Var}_{P_{j, s, \pi^{\star}_j(s)}}(V_{j+1}^{\star})} +  \Big(\frac{100\cb H^3S C^{\star}\log\frac{NH}{\delta}}{K}\Big)^{3/4}, \label{eq:instance-optimal}
\end{align}
with the proviso that $K \ge 100\cb HS C^{\star} \log\frac{NH}{\delta}$. 
Encouragingly, the dominate term (i.e., the first term in the bound \eqref{eq:instance-optimal}) matches the instance-dependent lower bound established in~\cite[Theorem 4.3]{yin2021towards}, 
thus confirming the instance optimality of the proposed algorithm for a large enough sample size. 
The proof of \eqref{eq:instance-optimal} can be found in Appendix~\ref{sec:proof-instance-optimal}.

\paragraph{Comparisons with prior statistical analysis.} 
We now briefly discuss the novelty of our statistical analysis compared with past theory.  
Perhaps the most related prior work is \citet{xie2021policy}, which proposed two algorithms.  
The first algorithm therein is VI-LCB with $H$-fold sample splitting and Hoeffding-style penalty, 
and each of these two features adds an $H$ factor to the total sample complexity.  
The second algorithm therein combines VI-LCB with variance reduction, which leads to optimal sample complexity for sufficiently small $\varepsilon$ (i.e., a large burn-in cost is required). Note, however, that none of the existing statistical tools for variance reduction 
is able to work without imposing a large burn-in cost, regardless of the sampling mechanism in use (e.g., generative model, offline RL, online RL) \citep{sidford2018near,zhang2020almost,li2021breaking,xie2021policy}. 
In contrast, our theory makes apparent that variance reduction is unnecessary, 
which leads to both simpler algorithm and tighter analysis.   
Additionally, while Bernstein-style confidence bounds have been deployed in online RL for finite-horizon MDPs \citep{fruit2020improved,azar2017minimax,jin2018q,zhang2020almost}, 
none of these works was able to yield optimal sample complexity without a large burn-in cost (e.g., \citet{azar2017minimax} incurred a burn-in cost as large as $S^3AH^6$). 
This in turn underscores the power of our statistical analysis when coping with the most data-hungry regime.


%


\section{Numerical experiments}
\label{sec:numerics}

\begin{figure}[t]
	\centering
	\begin{tabular}{ccc}
		\includegraphics[width=0.31\linewidth]{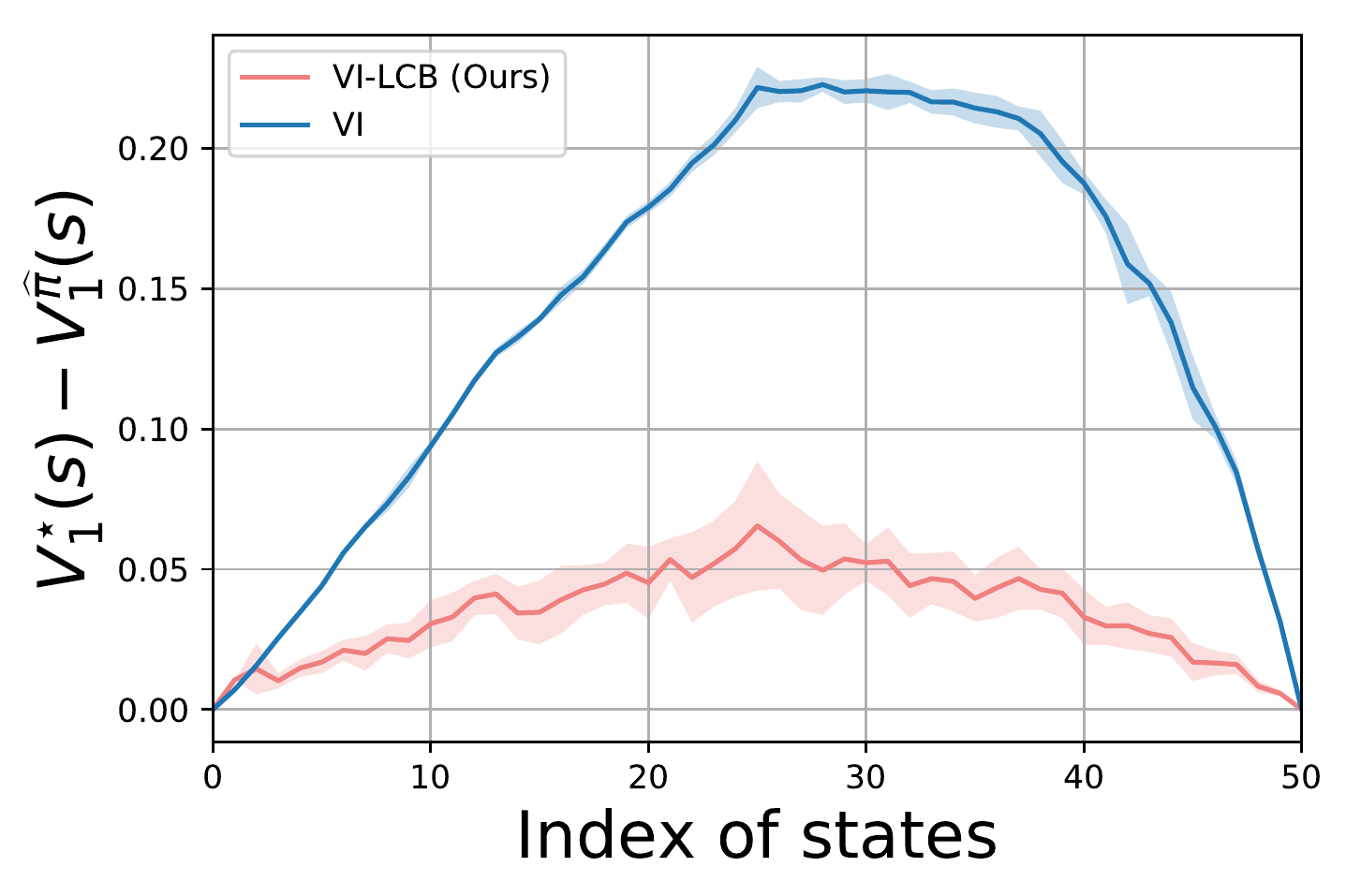} & \includegraphics[width=0.31\linewidth]{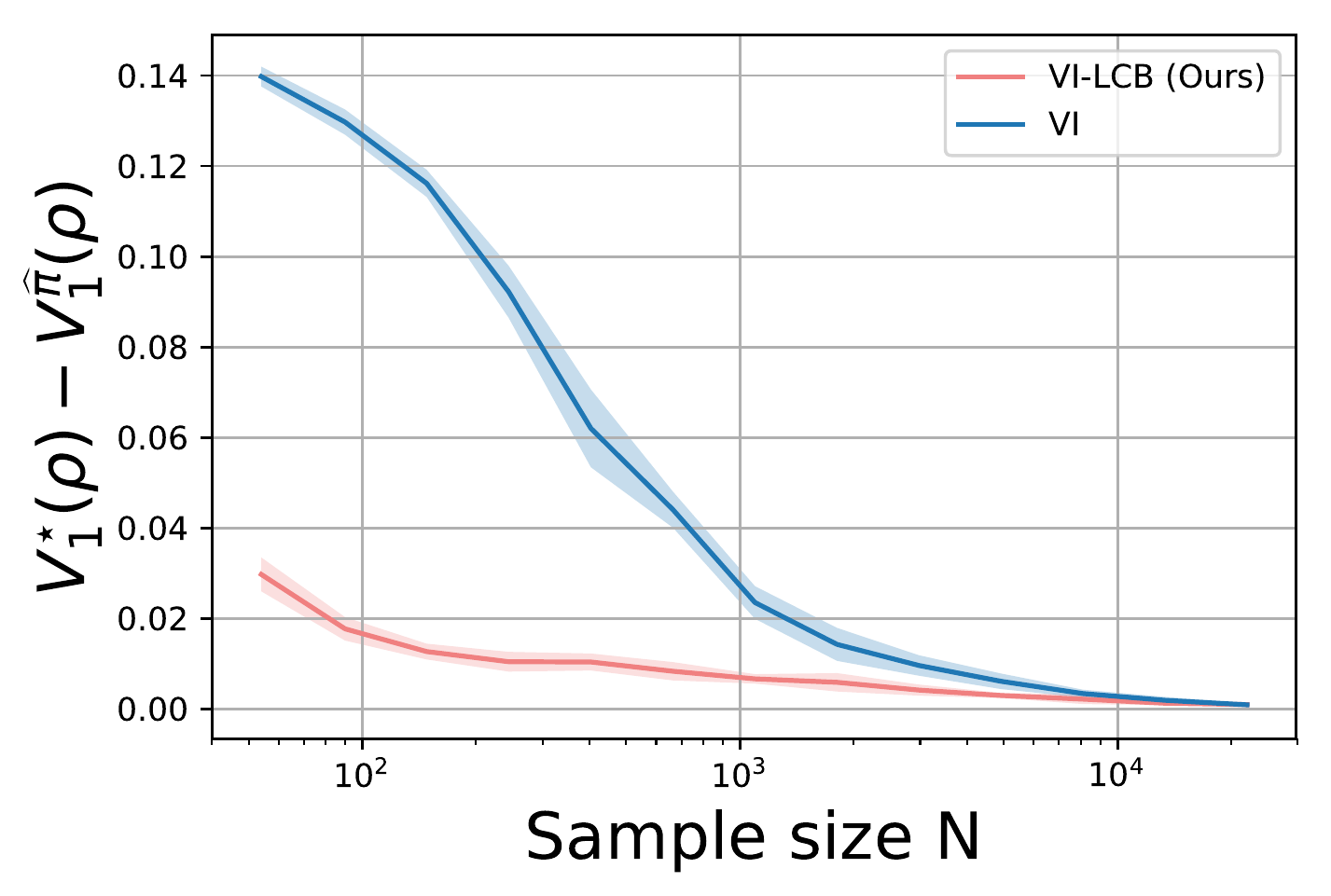} & \includegraphics[width=0.31\linewidth]{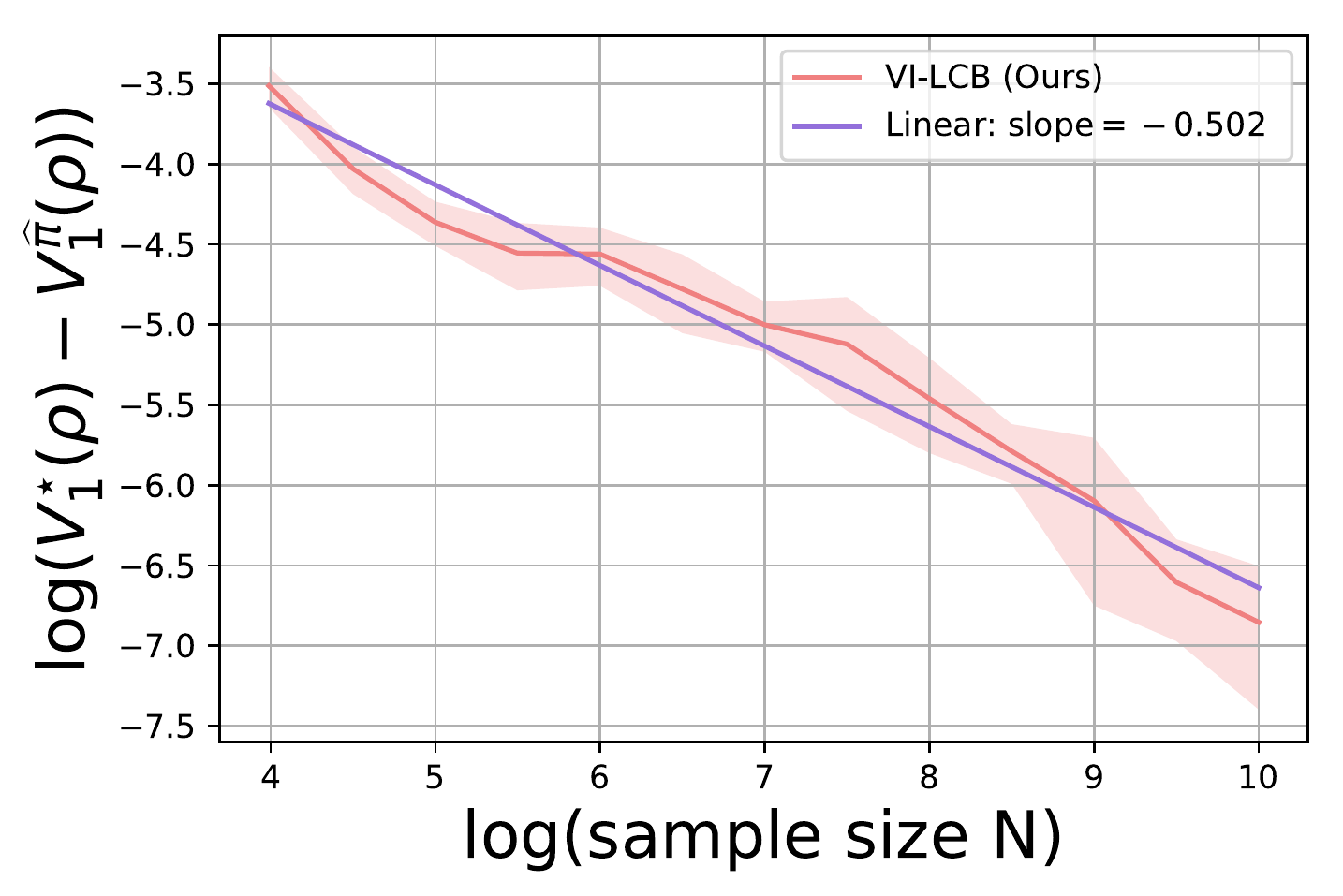} \\
		(a) Value gap versus states & (b) Value gap versus sample size & (c) Value gap dependency w.r.t. $N$
	\end{tabular}
	\caption{The performances of the proposed method VI-LCB and the baseline value iteration (VI) in the gambler's problem. It shows that VI-LCB outperforms VI by taking advantage of the pessimism principle and achieves approximately $1/\sqrt{N}$ sample complexity dependency w.r.t. the sample size $N$.}
		\label{fig:test_on_s}
\end{figure}

To confirm the practical applicability of the proposed VI-LCB algorithm, 
we evaluate its performance in the gambler's problem \citep{sutton2018reinforcement,zhou2021finite,panaganti2022sample,shi2022distributionally}. The code can be accessed at: 
\begin{center}
\url{https://github.com/Laixishi/Model-based-VI-LCB}.
\end{center}

\paragraph{Gambler's problem.} We start by introducing the formulation of the gambler's problem and its underlying MDP. An agent plays a gambling game in which she bets on a sequence of random coin flips, winning when the coins are heads and losing when they are tails. To bet on each random clip, the agent's policy chooses an integer number of dollars based on an initial balance. If the number of bets hits the maximum length $H$, or if the agent reaches $50$ dollars (win) or $0$ dollars (lose), the game ends. Without loss of generality, the problem can be formulated as an episodic finite-horizon MDP. Here, $\cS$ is the state space $\{0, 1, \cdots, 50\}$ and the associated accessible actions obey $a\in \big\{0,1,\cdots, \min\{s, 50-s\} \big\}$, $H =100$ is the horizon length, the reward is set to $0$ for all other states unless $s=50$. For the transition kernel, we fix the probability of heads as $p_{\mathrm{head}} = 0.45$ at all steps $h\in[H]$ in the episode. Moreover, the initial state/balance distribution of the agent $\rho$ is taken as a uniform distribution over $\cS$. The offline historical dataset is constructed by collecting $N$ independent samples drawn randomly over each state-action pair and time step.



\paragraph{Evaluation results.}
First, we evaluate the performance of our proposed method VI-LCB (cf.~Algorithm~\ref{alg:vi-lcb-finite}) with comparisons to the well-known value iteration (VI) method without the pessimism principle. To begin with, Fig.~\ref{fig:test_on_s}(a) shows the average and standard derivations of the performance gap $V_1^{\star}(s) - V_1^{\widehat{\pi}}(s)$ over all states $s\in\cS$, over $10$ independent experiments with a fixed sample size $N = 50$. The results indicate that the proposed VI-LCB method outperforms the baseline VI method uniformly over the entire state space, showing that pessimism brings significant advantages in this sample-scarce regime.
Secondly, we evaluate the performance gap $V_1^{\star}(\rho) - V_1^{\widehat{\pi}}(\rho)$ with varying sample size $N \in \big\{54, 90, 148,\cdots, 22026\} \approx \{e^{4}, e^{4.5}, e^{5}, \cdots, e^{10}\}$, over $10$ independent trials. Note that throughout the experiments, we fix the parameter $\cb = 0.05$, which determines the level of the pessimism penalty of VI-LCB (cf.~\eqref{def:bonus-Bernstein-finite}). Fig.~\ref{fig:test_on_s}(b) shows the average and standard derivations of the performance gap 
$V_1^{\star}(\rho) - V_1^{\widehat{\pi}}(\rho)$ with respect to the sample size $N$. Clearly,
  as the sample size increases, both our method VI-LCB and the baseline VI method perform better. Moreover, our VI-LCB method consistently outperforms the baseline VI method over the entire range of the sample size $N$, especially in the sample-starved regime. In addition, to corroborate the scaling of the sample size on the performance gap, we plot the sub-optimality performance gap of VI-LCB w.r.t. the sample size on a log-log scale in Fig.~\ref{fig:test_on_s}(c). Fitting using linear regression leads to a slope estimate of $-0.502$, with the corresponding fitted line plotted in Fig.~\ref{fig:test_on_s}(c) as well. This nicely matches the finding of Theorem~\ref{thm:finite}, which says the performance gap of VI-LCB scales as $N^{-1/2}$.
  

\section{Related works}
\label{sec:related-works}

In this section, we provide further discussions about prior art, with an emphasis on settings that are most relevant to the current paper.

\paragraph{Off-policy evaluation and offline RL.} 
Broadly speaking, at least two families of problems have been investigated in the literature that tackle offline batch data: off-policy evaluation, where the goal is to estimate the value function of a target policy that deviates from the behavior policy used in data collection; and offline policy learning, where the goal is to identify a near-optimal policy (or at least an improved one compared to the behavior policy). Our work falls under the second category. A topic of its own interest, off-policy evaluation has been extensively studied in the recent literature; we excuse ourselves from enumerating the works in that space but only provide pointers to a few examples including \citet{uehara2020minimax,li2014minimax,yang2020off,duan2020minimax,jiang2016doubly,jiang2020minimax,kallus2020double,duan2021optimal,xu2021unified,ren2021nearly,thomas2016data}.

\paragraph{Offline RL with the pessimism principle.}
The prior works that are the most relevant to this paper are \citet{rashidinejad2021bridging,xie2021policy,jin2021pessimism,shi2022pessimistic,yan2022efficacy,yin2021towards}, 
which incorporated lower confidence bounds into value estimation in order to avoid overly uncertain regions not covered by the target policy.  
In addition to the ones discussed in Section~\ref{sec:inadequacy-prior} that focus on minimax performance, 
the recent works \citet{yin2021towards,yin2022near} further developed instance-dependent statistical guarantees for the pessimistic model-based approach. 
The results in \citet{yin2021towards}, however, required a large burn-in sample size $\frac{H^4}{SC^{\star}(\myrho_{\min})^2}$ (since $\myrho_{\min}$ could be exceedingly small),
thus preventing it from attaining minimax optimality for the entire $\varepsilon$-range.   
It is noteworthy that the principle of pessimism has been incorporated into policy optimization and actor-critic methods as well by searching for some least-favorable models (e.g., \citet{uehara2021pessimistic,zanette2021provable}), which is quite different from the approach studied herein. On the empirical side, model-based algorithms \citep{yu2020mopo,kidambi2020morel} have been shown to achieve superior performance than their model-free counterpart for offline RL.
In addition, a number of recent works studied offline RL under various function approximation assumptions, e.g., \cite{jin2021pessimism,yin2022near,nguyen2021sample,uehara2021pessimistic,uehara2022representation,zhan2022offline,zanette2021provable}, which are beyond the scope of the current paper. 
Recently, the insights gleaned from the studies of offline RL have inspired improved algorithm designs for online and hybrid RL as well \citep{li2023minimax,li2024reward}.

\paragraph{Online RL and the optimism principle.}  
The optimism principle in the face of uncertainty has received widespread adoption  from bandits to online RL  \citep{lai1985asymptotically,lattimore2020bandit,agarwal2019reinforcement}. 
In the context of online RL, 
\citet{jaksch2010near} 
constructed confidence regions for the probability transition kernel to help select optimistic policies in the setting of weakly communicating MDPs, 
based on a variant (called UCRL2) of the UCRL algorithm originally proposed in 
\citet{auer2006logarithmic};  
 see also \citet{filippi2010optimism,bourel2020tightening,talebi2018variance} 
 for other variants of UCRL.  
When applied to episodic finite-horizon MDPs, the regret bound in \citet{jaksch2010near} was suboptimal by a factor of at least $\sqrt{H^2S}$;  see discussion in \citet{azar2017minimax,jin2018q}.  
\citet{fruit2020improved} developed an improved regret bound for UCRL2 by using empirical Bernstein-style bounds, which however was still suboptimal by a factor of at least $\sqrt{HS}$ when specialized to  episodic finite-horizon MDPs. 
In comparison, a more sample-efficient paradigm is to build Bernstein-style UCBs for the optimal values to help select exploration policies, 
which has been recently adopted in both model-based \citep{azar2017minimax,zhang2023settling} and model-free algorithms \citep{jin2018q}. 
Note that Bernstein-style uncertainty estimation alone is not enough to ensure regret optimality in model-free algorithms, 
thereby motivating the design of more sophisticated variance reduction strategies \citep{li2021breaking,zhang2020almost}. 
Finally, the optimism principle has been studied in undiscounted infinite-horizon MDPs too (e.g.,  \citet{qian2019exploration}), 
which is beyond the scope of this paper.



\paragraph{Model-based RL.} 
The algorithms studied herein fall under the category of model-based RL, which decouples the model estimation and the planning. 
This popular paradigm has been deployed and studied under various data collection mechanisms beyond offline RL, 
including but not limited to the generative model (or simulator) setting \citep{azar2013minimax,agarwal2019optimality,li2020breaking,li2023breaking} and the online exploratory setting \citep{azar2017minimax,jin2020provably,zhang2021reinforcement,zhang2023settling}.   
The leave-one-out analysis (and the construction of absorbing MDPs) adopted in the proof of Theorem~\ref{thm:infinite} 
has been inspired by several recent works \citet{agarwal2019optimality,li2020breaking,pananjady2020instance,cui2021minimax}, 
and has recently been shown to be effective for multi-agent offline RL as well \citep{yan2022model}.

\paragraph{Model-free RL.}  
Another widely used paradigm is model-free RL,  
which attempts to learn the optimal value function without explicit construction of the model. 
Arguably the most famous example of model-free RL is Q-learning, 
which applies the stochastic approximation paradigm to find the fixed point of the Bellman operator \citep{watkins1992q,even2003learning,beck2012error,szepesvari1998asymptotic,murphy2005generalization,qu2020finite,li2021tightening,li2021sample,xiong2020finite,shi2022statistically,chen2020finiteNIPS}.  
It is worth noting that the  asynchronous Q-learning, which aims to learn the optimal Q-function from a data trajectory collected by following a certain behavior policy, 
shares some similarity with offline RL; note that prior results on vanilla asynchronous Q-learning require a strong uniform coverage requirement \citep{qu2020finite,li2021tightening,chen2021lyapunov}, which is stronger than the single-policy concentrability considered herein. 
Moreover, Q-learning alone is known to be sub-optimal in terms of the sample complexity in various settings \citep{li2021tightening,jin2018q,shi2022pessimistic,wainwright2019stochastic,bai2019provably}.  This motivates the incorporation of the variance reduction  in order to further improve the sample complexity  \citep{wainwright2019variance,shi2022pessimistic,yan2022efficacy,li2021breaking,du2017stochastic,zhang2021model,li2021sample,zhang2020almost}. Note, however, variance-reduced model-free RL typically requires a large burn-in cost in order to operate in a sample-optimal fashion, and is hence outperformed by the model-based approach under multiple sampling mechanisms.

\section{Analysis: discounted infinite-horizon MDPs}
\label{sec:analysis-infinite}


This section is devoted to establishing Theorem~\ref{thm:infinite}. 
Towards this end, we claim that it is sufficient to prove the following theorem.

\begin{theorem} \label{thm:Bernstein-auxiliary-infinite}
Consider any $0<\delta <1$ and any  $\gamma \in [\frac{1}{2},1)$. Suppose that the penalty terms are set to be \eqref{def:bonus-Bernstein-infinite} for any numerical constant $\cb\ge144$. 
Then with probability exceeding $1-2\delta,$  for any estimate $\widehat{Q}$ obeying $\big\|\widehat{Q} - \widehat{Q}^{\star}_{\mathsf{pe}} \big\|_{\infty} \le 1/N$ one has
\begin{align} 
	V^{\star}(\rho) - V^{\widehat{\pi}}(\rho) \le 120\sqrt{\frac{\cb S\Cstar \log\frac{NS  }{(1-\gamma)\delta}}{(1-\gamma)^3N}} + \frac{3464\cb \infs S \Cstar \log\frac{NS }{(1-\gamma)\delta}}{(1-\gamma)^2N}, 
	\label{eq:V-error-regret-inf-auxiliary}
\end{align}
where $\widehat{\pi}(s) \in \arg\max_a \widehat{Q}(s, a)$ for any $s \in \cS$. 
\end{theorem}

As we have demonstrated in Lemma~\ref{lem:monotone-contraction}, the output of Algorithm~\ref{alg:vi-lcb-infinite}
satisfies $\big\|\widehat{Q} - \widehat{Q}^{\star}_{\mathsf{pe}} \big\|_{\infty} \le 1/N$ once the iteration number exceeds 
$\taumax \geq \frac{\log\frac{N}{1-\gamma}}{\log (1/\gamma)}$, thus making Theorem~\ref{thm:Bernstein-auxiliary-infinite} applicable. 
Taking the right-hand side of \eqref{eq:V-error-regret-inf-auxiliary} to be no larger than $\varepsilon $ reveals that:  
$ V^{\star}(\rho) - V^{\widehat{\pi}}(\rho) \le \varepsilon $ 
holds as long as $N$ exceeds 
\begin{align}
	N \geq \frac{21000 \cb S\Cstar \log\frac{NS}{(1-\gamma)\delta} }{(1-\gamma)^3 \varepsilon^2}, 
	\label{eq:N-range-epsilon2-inf}
\end{align}
%
given that $\varepsilon \in \big(0, \frac{1}{1-\gamma} \big]$. 

The remainder of this section is thus dedicated to establishing Theorem~\ref{thm:Bernstein-auxiliary-infinite}. 
Throughout the proof, it suffices to focus on the case where
\begin{align}
	N \ge \frac{ c_3 S \Cstar \infs \log \frac{NS }{(1-\gamma)\delta}}{1-\gamma}
	\label{eq:N-condition-basic-burnin-proof-inf}
\end{align}
for some  large constant $c_3 \geq 2880000$; 
otherwise the claim \eqref{eq:V-error-regret-inf-auxiliary} follows directly since $V^{\star}(\rho) - V^{\widehat{\pi}}(\rho) \le \frac{1}{1-\gamma}$.


\subsection{Preliminary facts}

Before embarking on the proof, we collect a couple of preliminary facts that will be used multiple times.

 \paragraph{Properties of $N(s,a)$.}
 To begin with,  the quantity $N(s,a)$ --- the total number of sample transitions from $(s,a)$ ---  can be bounded through the following lemma; the proof is provided in Appendix~\ref{proof:lem:sample-split-infinite}. 
%
\begin{lemma} 
	\label{lem:sample-split-infinite}
Consider any $\delta \in (0,1)$. With probability at least $1-\delta$, the quantities $\{N(s,a)\}$ in \eqref{eq:defn-Nsa-infinite} obey
\begin{align}
	\max\Big\{N(s, a), \frac{2}{3}\infs \log\frac{N}{\delta}\Big\} &\ge \frac{ N \myrho(s, a) }{12} 
	\label{eq:samples-infinite-LB}
\end{align}
simultaneously for all $(s, a) \in \cS \times \cA$. 
\end{lemma}

\paragraph{Properties about $\widehat{V}$ and $\widehat{V}^{\star}_{\mathsf{pe}}$.}
First of all, note that the assumption
\begin{align} \label{eq:Q-condition-Hoeffding}
	\big\| \widehat{Q} - \widehat{Q}^{\star}_{\mathsf{pe}} \big\|_{\infty} \le \frac{1}{N}
\end{align}
has the following direct consequence:
\begin{align}
	\big\| \widehat{V} - \widehat{V}^{\star}_{\mathsf{pe}} \big\|_{\infty} 
	 = \max_{s}\left| \max_a \widehat{Q}(s,a)- \max_a \widehat{Q}^{\star}_{\mathsf{pe}}(s,a)\right| 
	\le 
	\big\| \widehat{Q} - \widehat{Q}^{\star}_{\mathsf{pe}} \big\|_{\infty} \le \frac{1}{N}.
	\label{eq:Vhar-Vstarhat-diff-1/T-inf}
\end{align}
Given the proximity of $\widehat{V}$ and $\widehat{V}^{\star}_{\mathsf{pe}}$, 
we can bound the difference of the corresponding variance terms as follows: 
\begin{align}
	 \left|\mathsf{Var}_{\widehat{P}_{s,a}}\big(\widehat{V}_{\mathsf{pe}}^{\star}\big) - \mathsf{Var}_{\widehat{P}_{s,a}}\big(\widehat{V}\big)\right| &\overset{\mathrm{(i)}}{=}  \left|\widehat{P}_{s,a}\big( \widehat{V}_{\mathsf{pe}}^{\star}\circ \widehat{V}_{\mathsf{pe}}^{\star} \big)-\widehat{P}_{s,a}\big( \widehat{V} \circ \widehat{V} \big)  + \big(\widehat{P}_{s,a}\widehat{V} \big)^2 - \big(\widehat{P}_{s,a}\widehat{V}_{\mathsf{pe}}^{\star}\big)^2 \right|\notag\\
	& = \left| \widehat{P}_{s,a} \Big( \big( \widehat{V} +  \widehat{V}_{\mathsf{pe}}^{\star}\big) \circ \big( \widehat{V}_{\mathsf{pe}}^{\star} - \widehat{V} \big) \Big) + \left(\widehat{P}_{s,a} \big( \widehat{V} +  \widehat{V}_{\mathsf{pe}}^{\star}\big)\right) \left(\widehat{P}_{s,a} \big( \widehat{V}_{\mathsf{pe}}^{\star} - \widehat{V} \big)\right)  \right|\notag\\
	&\leq \big\|\widehat{P}_{s,a} \big\|_1 \big\| \widehat{V} +  \widehat{V}_{\mathsf{pe}}^{\star}\big\|_\infty \big\| \widehat{V}_{\mathsf{pe}}^{\star} - \widehat{V}\big\|_\infty + \big\|\widehat{P}_{s,a} \big\|_1^2 \big\| \widehat{V} +  \widehat{V}_{\mathsf{pe}}^{\star}\big\|_\infty \big\| \widehat{V}_{\mathsf{pe}}^{\star} - \widehat{V}\big\|_\infty\notag\\
	&\leq\left( \big\|\widehat{P}_{s,a} \big\|_1 + \big\|\widehat{P}_{s,a} \big\|_1^2 \right) \left(2\big\| \widehat{V} \big\|_{\infty} + \big\| \widehat{V}_{\mathsf{pe}}^{\star} - \widehat{V} \big\|_{\infty}  \right) \big\| \widehat{V}_{\mathsf{pe}}^{\star} - \widehat{V}\big\|_\infty \notag\\
	& \leq  \frac{2}{N} \left( \frac{2}{1-\gamma} + \frac{1}{N} \right)
	\leq \frac{6}{(1-\gamma)N}. 
	\label{infinite-Vpe-to-V}
\end{align}
Here, (i) follows from the definition \eqref{eq:defn-Var-P-V},  the penultimate inequality follows from \eqref{eq:Vhar-Vstarhat-diff-1/T-inf} and the basic facts $\|\widehat{P}_{s,a} \|_1=1$ and $\| \widehat{V} \|_{\infty}  \leq \frac{1}{1-\gamma}$,
while the last line relies on \eqref{eq:N-condition-basic-burnin-proof-inf}.

Armed with \eqref{infinite-Vpe-to-V}, one can further control the difference of the associated penalty terms. Note that the definition of $b(s,a; V)$ in \eqref{def:bonus-Bernstein-infinite} tells us that
\begin{align}
	\left|b\big(s,a; \widehat{V}_{\mathsf{pe}}^{\star}\big)-b\big(s,a; \widehat{V}\big) \right| 
	& =  \Bigg|\min\Bigg\{\max\bigg\{\sqrt{\frac{\cb\log\frac{N }{(1-\gamma)\delta}}{N(s, a)}\mathsf{Var}_{\widehat{P}_{s, a}}\big(\widehat{V}_{\mathsf{pe}}^{\star}\big)}, \frac{2\cb \log\frac{N }{(1-\gamma)\delta}}{(1-\gamma)N(s, a)}\bigg\}, \frac{1}{1-\gamma} \Bigg\} \notag\\
	&\qquad - \min\Bigg\{\max\bigg\{\sqrt{\frac{\cb\log\frac{N }{(1-\gamma)\delta}}{N(s, a)}\mathsf{Var}_{\widehat{P}_{s, a}}\big(\widehat{V}\big)}, \frac{2\cb \log\frac{N }{(1-\gamma)\delta}}{(1-\gamma)N(s, a)}\bigg\}, \frac{1}{1-\gamma} \Bigg\} \Bigg| .
	\label{eq:b-difference-definition-inf}
\end{align}
If at least one of the variance terms is not too small in the sense that
\begin{equation}
	\max\left\{ \mathsf{Var}_{\widehat{P}_{s, a}}\big(\widehat{V}_{\mathsf{pe}}^{\star}\big) , \mathsf{Var}_{\widehat{P}_{s, a}}\big(\widehat{V}\big) \right\} 
	\geq \frac{4\cb \log\frac{N }{(1-\gamma)\delta}}{(1-\gamma)^2 N(s, a)}, 
	\label{eq:max-variance-term-LB-inf}
\end{equation}
then \eqref{eq:b-difference-definition-inf} implies that
\begin{align}
	\eqref{eq:b-difference-definition-inf} &\leq   \sqrt{\frac{\cb\log\frac{N  }{(1-\gamma)\delta}}{N(s, a)}}\left|\sqrt{\mathsf{Var}_{\widehat{P}_{s, a}}\big(\widehat{V}_{\mathsf{pe}}^{\star}\big)} - \sqrt{\mathsf{Var}_{\widehat{P}_{s, a}}\big(\widehat{V}\big)}\right|  
	 = 
\sqrt{\frac{\cb\log\frac{N  }{(1-\gamma)\delta}}{N(s,a)}}\frac{\big| \mathsf{Var}_{\widehat{P}_{s,a}}\big(\widehat{V}_{\mathsf{pe}}^{\star}\big)-\mathsf{Var}_{\widehat{P}_{s,a}}\big(\widehat{V}\big) \big| }{\sqrt{\mathsf{Var}_{\widehat{P}_{s,a} }\big(\widehat{V}_{\mathsf{pe}}^{\star}\big)}+\sqrt{\mathsf{Var}_{\widehat{P}_{s,a}}\big(\widehat{V}\big)}}  \notag\\	
	&  \overset{\mathrm{(i)}}{\leq} \frac{1-\gamma}{2} \Big| \mathsf{Var}_{\widehat{P}_{s,a}}\big(\widehat{V}_{\mathsf{pe}}^{\star}\big)-\mathsf{Var}_{\widehat{P}_{s,a}}\big(\widehat{V}\big) \Big| 
	\overset{\mathrm{(ii)}}{\leq}  
	\frac{3}{N} ,
\end{align}
where (i) results from \eqref{eq:max-variance-term-LB-inf}, and (ii) holds due to \eqref{infinite-Vpe-to-V}. 
On the other hand, if \eqref{eq:max-variance-term-LB-inf} is not satisfied, then one clearly has
$b\big(s,a; \widehat{V}_{\mathsf{pe}}^{\star}\big) = b\big(s,a; \widehat{V}\big)$. In conclusion, in all cases we have
\begin{align}
	\left|b\big(s,a; \widehat{V}_{\mathsf{pe}}^{\star}\big)-b\big(s,a; \widehat{V}\big) \right|  \leq  
	\frac{3}{N} .\label{infinite-bpe-to-bV}
\end{align}

\subsection{Proof of Theorem~\ref{thm:Bernstein-auxiliary-infinite}}

Armed with the preceding preliminary facts, we can readily turn to the proof of Theorem~\ref{thm:Bernstein-auxiliary-infinite}. 
By virtue of Lemma~\ref{lem:sample-split-infinite}, our proof shall --- unless otherwise noted --- operate on the high-probability event that 
\begin{align}
	\forall (s,a) \in \cS\times \cA:
	\qquad \max\bigg\{N(s, a), \frac{2}{3}\infs \log\frac{SN }{\delta}\bigg\} &\ge \frac{ N\myrho(s, a) }{12} .
	\label{eq:assumption-Nsa-LB-infinite}
\end{align}
In addition, from the sampling model \eqref{eq:sampling-offline-inf-iid}, 
the sample transitions employed to form $\widehat{P}$ are statistically independent conditional on $\{N(s,a) \}$. 
Our proof consists of four steps as detailed below.

\paragraph{Step 1: Bernstein-style inequalities and leave-one-out decoupling argument.}
We are in need of tight control of the size of 
$\big(\widehat{P}_{ s, a} - P_{s, a}\big)\widehat{V}$. 
However, this  becomes challenging due to the statistical dependency between $\widehat{P}$ and the value estimate  $\widehat{V}$ (given that we reuse samples in all iterations of Algorithm~\ref{alg:vi-lcb-infinite}). 
In order to circumvent this difficulty, we resort to a leave-one-out argument to decouple the statistical dependency, 
as motivated by \citet{agarwal2019optimality,li2020breaking}. The result stated below establishes Bernstein-style inequalities despite the complicated dependency.

\begin{lemma} \label{lem:Bernstein-infinite}
	Suppose that $\gamma \in [\frac{1}{2}, 1)$, and consider any $\delta\in (0,1)$. 
With probability at least $1-\delta$, we have
\begin{subequations}
\begin{align}
\Big|\big(\widehat{P}_{ s, a} - P_{s, a}\big)\widetilde{V} \Big| 
	&\le 12\sqrt{\frac{\log\frac{2N}{(1-\gamma)\delta}}{N(s,a)}\mathsf{Var}_{\widehat{P}_{s,a}}\big(\widetilde{V}\big)}+\frac{74\log\frac{2N}{(1-\gamma)\delta}}{(1-\gamma)N(s,a)} ,
	\label{eq:bonus-Bernstein-infinite} \\
	\mathsf{Var}_{\widehat{P}_{s, a}}\big( \widetilde{V} \big) &\le 2\mathsf{Var}_{P_{s, a}}\big(\widetilde{V}\big) 
	+ \frac{41\log\frac{2N}{(1-\gamma)\delta}}{(1-\gamma)^2N(s, a)} \label{eq:empirical-var-infinite}
\end{align}
\end{subequations}
simultaneously for all $(s, a) \in \cS \times \cA$
	and all $\widetilde{V}$ with $\big\|\widetilde{V} - \widehat{V}^{\star}_{\mathsf{pe}} \big\|_{\infty} \le \frac{1}{N}$ and $\|\widetilde{V}\|_{\infty}\leq \frac{1}{1-\gamma}$.
\end{lemma}
\begin{proof}[High-level proof ideas]
	In short, the proof consists of contructing a finite collection of auxiliary MDPs $\{\widehat{\mathcal{M}}^{s,u}\}$ for each state $s$ obeying the following properties: 
	(i) each $\widehat{\mathcal{M}}^{s,u}$ is constructed without using any sample transition that comes from state $s$, and is hence statistically independent from $\widehat{P}_{s,a}$ for all $a\in \cA$ 
	(instead, the useful information is embedded into the corresponding immediate reward, which is a low-dimensional object and easier to control); 
	(ii) at least one of the MDPs in $\{\widehat{\mathcal{M}}^{s,u}\}$ is extremely close to the true MDP in terms of the resulting value function. 
	With the aid of these leave-one-out auxiliary MDPs, one can control  $\big(\widehat{P}_{ s, a} - P_{s, a}\big)\widetilde{V}$ 
	by first exploiting the statistical independence between $\widehat{P}_{s,a}$ and $\{\widehat{\mathcal{M}}^{s,u}\}$ and then transferring the concentration bound back to the original MDP using the proximity property (ii). 
	The construction of these auxiliary MDPs and the proof details can be found in Appendix~\ref{proof:lem:Bernstein-infinite}. 
\end{proof}

Note that \eqref{eq:bonus-Bernstein-infinite} has been derived only for those pairs $(s,a)$ with $N(s,a)>0$. 
For every $(s,a)$ with $N(s,a)= 0$, one can directly obtain
\[
	\Big|\big(\widehat{P}_{ s, a} - P_{s, a}\big)\widetilde{V} \Big| 
	= \big| P_{s, a}\widetilde{V} \big| \leq \big\| P_{s,a} \big\|_{1} \big\|\widetilde{V} \big\|_{\infty} \leq \frac{1}{1-\gamma}. 
\]
Putting these bounds together with the definition \eqref{def:bonus-Bernstein-infinite} of $b(s,a;V)$ reveals that 
\begin{align}
	\Big|\big(\widehat{P}_{ s, a} - P_{s, a}\big)\widetilde{V} \Big|  + \frac{5}{N}
	\leq b\big(s,a; \widetilde{V} \big) 
	\qquad \text{for all }(s,a)\in \cS\times \cA 
	\label{eq:infitnie-bernstein-b-bound-diff}
\end{align}
for all $\widetilde{V}$ obeying $\big\|\widetilde{V} - \widehat{V}^{\star}_{\mathsf{pe}} \big\|_{\infty} \le \frac{1}{N}$ and $\|\widetilde{V}\|_{\infty}\leq \frac{1}{1-\gamma}$, 
provided that the constant $\cb$ is sufficiently large. 
The remainder of the proof should then also operate on the high-probability events \eqref{eq:infitnie-bernstein-b-bound-diff} and \eqref{eq:empirical-var-infinite}, in addition to assuming that the event \eqref{eq:assumption-Nsa-LB-infinite} occurs.

\paragraph{Step 2: showing that $\widehat{Q}(s, a)$ is a lower bound on $Q^{\widehat{\pi}}(s, a)$.}
We now justify that $\widehat{Q}(s, a)$ (resp.~$\widehat{V}(s)$) is a ``pessimistic'' estimate of $Q^{\widehat{\pi}}(s, a)$ (resp.~$V^{\widehat{\pi}}(s)$);
 this is enabled by the pessimism principle (so that the algorithm effectively seeks lower estimates of the value iteration) and the Bernstein-style bounds in Lemma~\ref{lem:Bernstein-infinite} (so that the penalty term always dominates the uncertainty incurred by using the empirical MDP).

To begin with, recall that $\widehat{Q}^{\star}_{\mathsf{pe}} (s,a)$ is the unique fixed point of the pessimistic Bellman operator that obeys
\begin{align}
	\widehat{Q}^{\star}_{\mathsf{pe}} (s,a) = \max\Big\{r(s, a) + \gamma\widehat{P}_{s, a}\widehat{V}^{\star}_{\mathsf{pe}} - b\big(s, a; \widehat{V}^{\star}_{\mathsf{pe}}\big) , 0\Big\}.
	\label{eq:hat-Qstar-pe-Bellman-123}
\end{align}
In the sequel, we divide the set of state-action pairs $(s,a)$ into two types. 
\begin{itemize}
	\item {\em Case 1: $\widehat{Q}^{\star}_{\mathsf{pe}} (s,a) = 0$.}  Given that $\widehat{Q}_0=0$, Lemma~\ref{lem:monotone-contraction} tells us that
		\[
			\widehat{Q}(s,a) = \widehat{Q}_{\tau_{\max}}(s,a) \leq \widehat{Q}^{\star}_{\mathsf{pe}} (s,a) = 0.
		\]
		This combined with the basic fact $Q^{\widehat{\pi}}\geq 0$ immediately yields $0 = \widehat{Q}(s,a)\leq Q^{\widehat{\pi}}(s,a)$.

\item {\em Case 2: $\widehat{Q}^{\star}_{\mathsf{pe}} (s,a) =r(s, a) + \gamma\widehat{P}_{s, a} \widehat{V}^{\star}_{\mathsf{pe}} - b\big(s, a; \widehat{V}^{\star}_{\mathsf{pe}}\big)>0$.} It is first observed that  
\begin{align}
	\widehat{Q} (s,a) &\overset{\text{(i)}}{\le} \widehat{Q}^{\star}_{\mathsf{pe}} (s,a) + \frac{1}{N} 
	\overset{\text{(ii)}}{=} r(s,a) - b\big(s, a; \widehat{V}^{\star}_{\mathsf{pe}}\big) + \gamma\widehat{P}_{s,a}\widehat{V}^{\star}_{\mathsf{pe}} + \frac{1}{N}  \nonumber \\
	& \leq r(s,a) - b\big(s, a; \widehat{V}^{\star}_{\mathsf{pe}}\big) + \gamma \widehat{P}_{s,a} \widehat{V} + \frac{1}{N} + \gamma \big\|\widehat{P}_{s, a}\big\|_1 \big\|\widehat{V} - \widehat{V}_{\mathsf{pe}}^\star\big\|_\infty \notag\\
	& \overset{\text{(iii)}}{\le} r(s,a) - b\big(s, a; \widehat{V}^{\star}_{\mathsf{pe}}\big) + \gamma \widehat{P}_{s,a} \widehat{V} + \frac{2}{N} \notag\\
	& \leq r(s,a) - b\big(s,a; \widehat{V} \big) + \gamma P_{s,a} \widehat{V} + \frac{2}{N}  + \gamma \Big|\big(\widehat{P}_{s, a} - P_{s, a}\big)\widehat{V}\Big| + \left|b\big(s,a; \widehat{V}_{\mathsf{pe}}^{\star}\big)-b\big(s,a; \widehat{V}\big) \right|
	 \notag \\
	& \overset{\text{(iv)}}{\le} r(s,a) + \gamma P_{s,a}\widehat{V}. \label{eq:infinite-q-upper2}
\end{align}
Here, (i) and (iii) arise from the assumption \eqref{eq:Q-condition-Hoeffding},  (ii) relies on the fact that $\widehat{Q}_{\mathsf{pe}}^{\star}$ is the fixed point of the operator $\Tpess$,
whereas (iv) takes advantage of \eqref{infinite-bpe-to-bV} and \eqref{eq:infitnie-bernstein-b-bound-diff}. 
Combining \eqref{eq:infinite-q-upper2} with the Bellman equation $Q^{\widehat{\pi}} = r + \gamma P V^{\widehat{\pi}}$ results in
\begin{align}
	Q^{\widehat{\pi}}(s, a) - \widehat{Q}(s, a) 
	\geq  r(s,a) + \gamma P_{s,a}V^{\widehat{\pi}}
	-  \big(r(s,a) + \gamma P_{s,a}\widehat{V} \big) 
	= \gamma P_{s,a} \big( V^{\widehat{\pi}}
	-  \widehat{V} \big).
	\label{eq:Case2-Q-Qhat-V-Vhat-inf}
\end{align}
Suppose for the moment that there exists some $(s,a)$ obeying $Q^{\widehat{\pi}}(s, a) - \widehat{Q}(s, a) < 0$ (which clearly cannot happen in Case 1),   
then $\arg\min_{s, a}\big[Q^{\widehat{\pi}}(s, a) - \widehat{Q}(s, a)\big]$ must belong to Case 2. 
Thus, taking the minimum over $(s,a)$ and using the above inequality \eqref{eq:Case2-Q-Qhat-V-Vhat-inf} give
\begin{align}
	\min_{s, a}\big[Q^{\widehat{\pi}}(s, a) - \widehat{Q}(s, a)\big] &\ge \min_{s, a}\big[\gamma P_{s, a} \big(V^{\widehat{\pi}} - \widehat{V} \big)\big] \overset{\text{(i)}}{\ge} \gamma\min_{s}\big[V^{\widehat{\pi}}(s) - \widehat{V}(s)\big] \notag\\
	&  = \gamma\min_{s} \big[Q^{\widehat{\pi}}\big( s, \widehat{\pi}(s) \big) - \widehat{Q}\big( s, \widehat{\pi}(s) \big) \big]  \geq \gamma\min_{s, a}\big[Q^{\widehat{\pi}}(s, a) - \widehat{Q}(s, a)\big],
	\label{eq:min-Q-lower-bound-gamma-infinite}
\end{align}
where (i) holds since $P_{s, a} \in \Delta(\cS)$.
Given that $1 > \gamma > 0$, inequality~\eqref{eq:min-Q-lower-bound-gamma-infinite} holds only when 
$\min_{s, a}\big[Q^{\widehat{\pi}}(s, a) - \widehat{Q}(s, a)\big] \geq 0$. 
We therefore conclude that in this case, one also has $Q^{\widehat{\pi}}(s, a) \geq \widehat{Q}(s, a)$.
\end{itemize}
With the arguments for the above two cases in place, we arrive at
\begin{equation}\label{eq:infinite-q-upper3}
	Q^{\widehat{\pi}}(s, a) \geq  \widehat{Q}(s, a)\qquad \text{for all }(s,a)\in \cS\times \cA,  
\end{equation}
and evidently, 
\begin{equation}
	V^{\star}(s) \geq V^{\widehat{\pi}}(s) =  Q^{\widehat{\pi}}\big(s, \widehat{\pi}(s)\big) \geq  
	 \widehat{Q}\big(s, \widehat{\pi}(s)\big)  =\max_a\widehat{Q}(s, a) = \widehat{V}(s)
	\qquad  \text{for all } s\in \cS.
	\label{eq:Vhat-pi-V-pihat-order-infinite}
\end{equation}

\paragraph{Step 3: bounding $V^{\star}(s) - V^{\widehat{\pi}}(s)$.} 
Recall that the Bellman optimality equation gives
\begin{align}
	V^{\star}(s) = r\big(s, \pi^{\star}(s)\big) + \gamma P_{s, \pi^\star(s)}V^{\star}. \label{eq:infinite-Bellman-opt}
\end{align}
Before continuing, we make note of the following lower bound on $\widehat{V}$: 
\begin{align}
	\widehat{V}(s) &= \max_a \widehat{Q}(s, a) \ge \widehat{Q}\big(s, \pi^{\star}(s)\big) 
	\overset{\mathrm{(i)}}{\ge} \widehat{Q}^{\star}_{\mathsf{pe}}\big(s, \pi^{\star}(s)\big) - \frac{1}{N}  \notag\\
	&\overset{\mathrm{(ii)}}{\ge} r\big(s, \pi^{\star}(s)\big) - b\big(s, \pi^{\star}(s); \widehat{V}^{\star}_{\mathsf{pe}}\big) + \gamma \widehat{P}_{s, \pi^\star(s)}\widehat{V}_{\mathsf{pe}}^\star - \frac{1}{N} \notag\\
	& \notag = r\big(s, \pi^{\star}(s)\big) - b\big(s, \pi^{\star}(s); \widehat{V}^{\star}_{\mathsf{pe}}\big) + \gamma \widehat{P}_{s, \pi^\star(s)} \widehat{V} - \frac{1}{N} - 
	\gamma \widehat{P}_{s, \pi^\star(s)}  \big( \widehat{V} - \widehat{V}_{\mathsf{pe}}^\star\big )\\
	& \overset{\mathrm{(iii)}}{\ge} r\big(s, \pi^{\star}(s)\big) - b\big(s, \pi^{\star}(s); \widehat{V}^{\star}_{\mathsf{pe}}\big) + \gamma \widehat{P}_{s, \pi^\star(s)} \widehat{V} - \frac{2}{N} \notag \\
	& \ge r\big(s, \pi^{\star}(s)\big) - b\big(s, \pi^{\star}(s); \widehat{V} \big) + \gamma P_{s, \pi^\star(s)} \widehat{V} - \frac{2}{N} - \gamma \Big|\big(\widehat{P}_{s, \pi^\star(s)} - P_{s, \pi^\star(s)}\big)\widehat{V}\Big| \notag\\
	&\qquad - \left|b\big(s,\pi^{\star}(s); \widehat{V}_{\mathsf{pe}}^{\star}\big)-b\big(s,\pi^{\star}(s); \widehat{V}\big) \right|
	  \notag \\
	&\overset{\mathrm{(iv)}}{\ge} r\big(s, \pi^{\star}(s)\big) - 2b\big(s, \pi^{\star}(s); \widehat{V}\big) + \gamma P_{s, \pi^\star(s)}\widehat{V}. \label{eq:infinite-v-lower}
\end{align}
Here, (i) results from the assumption~\eqref{eq:Q-condition-Hoeffding}, 
(ii) relies on \eqref{eq:hat-Qstar-pe-Bellman-123}, 
(iii) is valid since $\widehat{P}_{s, \pi^\star(s)}  \big( \widehat{V} - \widehat{V}_{\mathsf{pe}}^\star\big ) \leq \big\|\widehat{P}_{s, \pi^\star(s)}\big\|_1 \big\|\widehat{V} - \widehat{V}_{\mathsf{pe}}^\star\big\|_\infty \leq 1/N$, 
whereas (iv) holds by virtue of \eqref{infinite-bpe-to-bV} and \eqref{eq:infitnie-bernstein-b-bound-diff}. 
Armed with the results in \eqref{eq:infinite-Bellman-opt} and \eqref{eq:infinite-v-lower}, we can readily show that 
\begin{align}
	\big\langle \rho, V^{\star}-\widehat{V} \big\rangle & = \sum_{s\in\cS} \rho(s) \big( V^{\star}(s) - \widehat{V} (s) \big) \notag\\
	& \le \sum_{s\in\cS} \rho(s) \left\{ r\big( s, \pi^{\star}(s)\big) + \gamma P_{s, \pi^\star(s)}V^{\star} - \left( r\big(s, \pi^{\star}(s)\big) - 2b\big(s, \pi^{\star}(s); \widehat{V}\big) + \gamma P_{s, \pi^\star(s)}\widehat{V} \right) \right\} \nonumber \\
 & \leq\gamma\sum_{s\in\cS}\rho(s)P_{s,\pi^{\star}(s)}\big(V^{\star}-\widehat{V}\big)+2\sum_{s\in\cS}\rho(s)b\big(s, \pi^{\star}(s); \widehat{V}\big). \label{eq:rho-Vstar-hat-UB-recursive-inf}
\end{align}

For notational convenience, let us introduce a matrix $P^{\star} \in \mathbb{R}^{S\times S}$ and a vector $b^{\star} \in \mathbb{R}^{S\times 1}$ 
whose $s$-th row are given respectively by 
\begin{align}
	\big[ P^{\star} \big]_{s,\cdot} \defn P_{s, \pi^{\star}(s)} 
	\qquad \text{and} \qquad b^{\star}(s) \defn b\big(s, \pi^{\star}(s); \widehat{V}\big) 
	\qquad\quad \text{for all }s\in \cS. 
	\label{eq:P-b-star-infinite}
\end{align}
This allows us to rewrite \eqref{eq:rho-Vstar-hat-UB-recursive-inf} in the following matrix/vector form:
\begin{align}
\rho^{\top}\big(V^{\star}-\widehat{V}\big) \leq  \gamma \rho^\top P^{\star}\big(V^{\star}-\widehat{V}\big)+ 2\rho^{\top}b^{\star} .
\end{align}
Note that this relation holds for any arbitrary $\rho$. Apply it recursively to arrive at
\begin{align}
\label{eqn:brahms-reduction}
	\notag \rho^{\top}\big(V^{\star}-\widehat{V}\big) & \leq  \big( \gamma \rho^\top P^{\star} \big) \big(V^{\star}-\widehat{V}\big)+ 2\rho^{\top}b^{\star}\\
	\notag & \leq  \gamma \big( \gamma \rho^\top P^{\star} \big) P^{\star} \big(V^{\star}-\widehat{V}\big)+2 \big(\gamma  \rho^{\top} P^{\star} \big)  b^{\star}
	+2\rho^{\top}b^{\star}\\
	\notag & = \gamma^2 \rho^\top \big( P^{\star}\big)^{2}\big(V^{\star}-\widehat{V}\big)+2\gamma  \rho^{\top} P^{\star} b^{\star}
	+2\rho^{\top}b^{\star}\\
	\notag & \leq\cdots \leq \left\{ \lim_{i\rightarrow \infty}\gamma^{i} \rho^\top \big( P^{\star}\big)^{i}\big(V^{\star}-\widehat{V}\big) \right\}
	+ 2\rho^{\top} \bigg\{\sum_{i=0}^{\infty}\gamma^{i}\big(P^{\star}\big)^{i}\bigg\} b^{\star} 
	 \\
	\notag & \overset{\mathrm{(i)}}{=} 2\rho^{\top} \bigg\{\sum_{i=0}^{\infty}\gamma^{i}\big(P^{\star}\big)^{i}\bigg\} b^{\star} 
	= 2\rho^{\top}\big(I-\gamma P^{\star}\big)^{-1}b^{\star} \\
	& 
	  = \frac{2}{1-\gamma} \langle d^{\star}, b^{\star} \rangle, 
\end{align}
where (i) holds since $ \lim_{i \to \infty} \gamma^i \rho^\top \big(P^\star\big)^i\big(V^\star -\widehat{V}\big) = 0$ (given that $\lim_{i\rightarrow \infty}\gamma^{i}=0$ and $\|\rho^\top \big(P^\star\big)^i \|_1 =1$ for any $i\geq 0$), and the last equality results from the definition of $d^{\star}$ (see \eqref{eq:infinite-d-star-def}) expressed in the following matrix/vector form: 
\begin{align}
	\big(d^\star\big)^\top = (1-\gamma)  \sum_{t=0}^\infty\gamma^t  \rho^\top \big(P^\star\big)^t =  (1-\gamma) \rho^\top \big(I-\gamma P^{\star}\big)^{-1}.
	\label{eq:infinite-d-matrix-form}
\end{align}
Combine the above inequality with \eqref{eq:Vhat-pi-V-pihat-order-infinite} to reach
\begin{align}
\big\langle \rho, V^{\star} - V^{\widehat{\pi}}\big\rangle \le \big\langle \rho, V^{\star} - \widehat{V}\big\rangle 
	\leq \frac{2\big\langle d^{\star}, b^{\star} \big\rangle}{1-\gamma}. \label{eq:Vall-Bernstein-infinite}
\end{align}

\paragraph{Step 4: using concentrability to control $\big \langle d^\star, b^\star \big \rangle$.}
We shall control $\big \langle d^\star, b^\star \big \rangle$ by dividing the state set $\cS$ into the following two disjoint subsets:
%
\begin{subequations}
\label{eq:defn-Ssmall-Slarge-inf}
\begin{align}
	\cS^{\mathsf{small}} \coloneqq \bigg\{ s\in \cS \mid  N \myrho \big(s, \pi^{\star}(s) \big) \leq 8\infs \log \frac{NS}{(1-\gamma)\delta} \bigg\}; \label{eq:defn-Ssmall-inf}\\
	\cS^{\mathsf{large}} \coloneqq \bigg\{ s\in \cS \mid  N \myrho \big(s, \pi^{\star}(s) \big) > 8\infs\log \frac{NS}{(1-\gamma)\delta} \bigg\}. \label{eq:defn-Slarge-inf}
\end{align}
\end{subequations}
\begin{itemize}
\item
To begin with,  consider any state $s \in \cS^{\mathsf{small}}$. Applying Definition~\ref{assumption:concentrate-infinite} and the definition of $\cS^{\mathsf{small}}$ yields 
\begin{align}
	\min \Big\{ d^{\star}(s) , \frac{1}{S} \Big\} 
	\le \Cstar  \myrho \big(s, \pi^{\star}(s) \big) \le \frac{8 \Cstar \infs \log \frac{NS}{(1-\gamma)\delta}}{N} < \frac{1}{S}, 
\end{align}
provided that $N > 8S \Cstar \infs \log \frac{NS}{(1-\gamma)\delta}$ (see \eqref{eq:N-condition-basic-burnin-proof-inf}). This inequality necessarily implies that
\begin{align}
	 d^{\star}(s)  
	\le  \frac{8 \Cstar \infs \log \frac{NS}{(1-\gamma)\delta}}{N} < \frac{1}{S}.  
\end{align}
Combining the preceding inequality with the following fact (see the definition \eqref{def:bonus-Bernstein-infinite}) 
\begin{align}
	b^{\star}(s) \coloneqq b\big(s, \pi^{\star}(s); \widehat{V} \big) \leq \frac{1}{1-\gamma} + \frac{5}{N},
\end{align}
we arrive at 
\begin{align}
\sum_{s\in\cS^{\mathsf{small}}}	d^{\star}(s) b^{\star}(s) &\leq \sum_{s\in\cS^{\mathsf{small}}} \left( \frac{8 \Cstar \infs \log \frac{NS}{(1-\gamma)\delta}}{(1-\gamma)N} +   d^{\star}(s) \frac{5}{N} \right) \leq  \frac{8 S \Cstar \infs \log \frac{NS}{(1-\gamma)\delta}}{(1-\gamma)N} +  \frac{5}{N}. \label{eq:infinite-bernstein-small-result}
\end{align}
%

%

\item Next, we turn to any state $s \in \cS^{\mathsf{large}}$. Using the definition \eqref{def:bonus-Bernstein-infinite} of $b(s,a;V) $, we obtain 
\begin{align}
	b^{\star}(s) 
	& = b\big( s, \pi^{\star}(s); \widehat{V} \big)
	\le \sqrt{\frac{\cb\log\frac{N }{(1-\gamma)\delta}}{N\big(s, \pi^{\star}(s)\big)}\mathsf{Var}_{\widehat{P}_{s, \pi^{\star}(s)}}\big(\widehat{V}\big)} +  \frac{2\cb\log\frac{N }{(1-\gamma)\delta}}{(1-\gamma)N\big(s, \pi^{\star}(s)\big)} + \frac{5}{N}\notag\\
	& \overset{\mathrm{(i)}}{\le} \sqrt{\frac{\cb\log\frac{N }{(1-\gamma)\delta}}{N\big(s, \pi^{\star}(s)\big)}\left(2\mathsf{Var}_{P_{s, \pi^{\star}(s)}}\big(\widehat{V}\big) 
	+ \frac{41\log\frac{2N}{(1-\gamma)\delta}}{(1-\gamma)^2N\big(s, \pi^{\star}(s)\big)}\right)} + \frac{2\cb \log\frac{N}{(1-\gamma)\delta}}{(1-\gamma)N\big(s, \pi^{\star}(s)\big)}   +\frac{5}{N} \notag\\
	& \overset{\mathrm{(ii)}}{\le} \sqrt{\frac{2\cb\log\frac{N }{(1-\gamma)\delta}}{N\big(s, \pi^{\star}(s)\big)}\mathsf{Var}_{P_{s, \pi^{\star}(s)}}\big(\widehat{V}\big)} +  \frac{4\cb\log\frac{N }{(1-\gamma)\delta}}{(1-\gamma)N\big(s, \pi^{\star}(s)\big)}, \label{eq:infinite-bernstein-large-result-middle}
\end{align}
where (i) arises from Lemma~\ref{lem:Bernstein-infinite} and \eqref{eq:Vhar-Vstarhat-diff-1/T-inf}, (ii) applies the elementary inequality $\sqrt{x+y}\leq \sqrt{x}+\sqrt{y}$ for any $x,y\geq 0$ and the fact $N \geq N(s,a)$, in addition to assuming that $\cb$ is large enough. To continue, we observe that
\begin{align}
	\frac{1}{N\big(s, \pi^{\star}(s)\big)} \overset{\mathrm{(i)}}{\le} \frac{12}{N\myrho \big(s, \pi^{\star}(s)\big)} \overset{\mathrm{(ii)}}{\le} \frac{12\Cstar}{N \min\big\{d^{\star}(s), \frac{1}{S}\big\}} \leq \frac{12\Cstar}{N}\left(\frac{1}{d^{\star}(s)} + S\right), \label{eq:infinite-N-to-d-bound}
\end{align}
where (i) follows from the assumption \eqref{eq:assumption-Nsa-LB-infinite} and the definition of $\cS^{\mathsf{large}}$,
and (ii) results from Assumption~\ref{assumption:concentrate-infinite}.
Substitution into \eqref{eq:infinite-bernstein-large-result-middle} yields 
\begin{align}
	b^{\star}(s) 
	&\leq \underset{\eqqcolon\, \alpha_1(s) }{\underbrace{ \sqrt{\frac{24\cb  \Cstar \log\frac{N }{(1-\gamma)\delta}}{N} \mathsf{Var}_{P_{s, \pi^{\star}(s)}}\big(\widehat{V}\big)}\left(\frac{1}{\sqrt{d^{\star}(s)}} + \sqrt{S}\right) }} + \underset{\eqqcolon\, \alpha_2(s) }{\underbrace{ \frac{48\cb \Cstar \log\frac{N  }{(1-\gamma)\delta}}{(1-\gamma)N}\left(\frac{1}{d^{\star}(s)} + S\right) }}, \label{eq:infinite-bernstein-large-result}
\end{align}
where the last line comes from the elementary inequality $\sqrt{x+y}\leq \sqrt{x}+\sqrt{y}$ for any $x,y\geq 0$.

To proceed,  observe that the sum of the first terms in \eqref{eq:infinite-bernstein-large-result} satisfies 
\begin{align}
	&\sum_{s\in\cS^{\mathsf{large}}}d^{\star}(s) \alpha_1(s) \notag\\
	&\quad = \sqrt{\frac{24\cb \Cstar \log\frac{N  }{(1-\gamma)\delta}}{N} } \left( \sum_{s\in\cS^{\mathsf{large}}} \sqrt{d^{\star}(s)  \mathsf{Var}_{P_{s, \pi^{\star}(s)}}\big(\widehat{V}\big)} + \sum_{s\in\cS^{\mathsf{large}}}\sqrt{d^{\star}(s)} \sqrt{Sd^{\star}(s)\mathsf{Var}_{P_{s, \pi^{\star}(s)}}\big(\widehat{V}\big)} \right)\notag\\
	&\quad \overset{\mathrm{(i)}}{\leq} \sqrt{\frac{24\cb \Cstar \log\frac{N  }{(1-\gamma)\delta}}{N} } 
	\Bigg( \sqrt{S} \cdot \sqrt{\sum_{s\in\cS^{\mathsf{large}}}  d^{\star}(s)  \mathsf{Var}_{P_{s, \pi^{\star}(s)}}\big(\widehat{V}\big)}  
	 + \sqrt{\sum_{s\in\cS^{\mathsf{large}}} S d^{\star}(s) \mathsf{Var}_{P_{s, \pi^{\star}(s)}}\big(\widehat{V}\big)} \Bigg) \notag\\
	&\quad = \sqrt{\frac{96\cb S\Cstar \log\frac{N }{(1-\gamma)\delta}}{N} } \sqrt{\sum_{s\in\cS^{\mathsf{large}}} d^{\star}(s) \mathsf{Var}_{P_{s, \pi^{\star}(s)}}\big(\widehat{V}\big)}, \label{eq:infinite-bernstein-large-result1}
\end{align}
where (i) arises from the Cauchy-Schwarz inequality and the fact $\sum_s d^{\star}(s)=1$. 
In addition, it is easily verified that the sum of the second terms in \eqref{eq:infinite-bernstein-large-result} obeys
\begin{align}
	\sum_{s\in\cS^{\mathsf{large}}}d^{\star}(s) \alpha_2(s)
	\leq \frac{96\cb S \Cstar \log\frac{N }{(1-\gamma)\delta}}{(1-\gamma)N}, \label{eq:infinite-bernstein-large-result2}
\end{align}
which also makes use of the identity $\sum_s d^{\star}(s)=1$. 
Combining \eqref{eq:infinite-bernstein-large-result1} and \eqref{eq:infinite-bernstein-large-result2} with \eqref{eq:infinite-bernstein-large-result} gives
\begin{align}
	& \sum_{s\in\cS^{\mathsf{large}}}d^{\star}(s) b^{\star}\big(s, \pi^{\star}(s) \big)
		\leq \sum_{s\in\cS^{\mathsf{large}}}d^{\star}(s) \alpha_1(s) + \sum_{s\in\cS^{\mathsf{large}}}d^{\star}(s) \alpha_2(s)
\notag\\
	& \qquad \le \sqrt{\frac{96\cb S\Cstar \log\frac{N  }{(1-\gamma)\delta}}{N} } \sqrt{\sum_{s\in\cS^{\mathsf{large}}} d^{\star}(s) \mathsf{Var}_{P_{s, \pi^{\star}(s)}}\big(\widehat{V}\big)} + \frac{96\cb S \Cstar \log\frac{N }{(1-\gamma)\delta}}{(1-\gamma)N}. \label{eq:infinite-bernstein-large-result-all}
\end{align}
\end{itemize}

\noindent The above results \eqref{eq:infinite-bernstein-small-result} and \eqref{eq:infinite-bernstein-large-result-all} taken collectively give
\begin{align*}
	& \big\langle d^{\star}, b^\star \big\rangle  
	=\sum_{s\in\cS^{\mathsf{large}}}d^{\star}(s)b^{\star}(s) + \sum_{s\in\cS^{\mathsf{small}}}d^{\star}(s)b^{\star}(s) \\
&\quad \leq \sqrt{\frac{96\cb S\Cstar \log\frac{N }{(1-\gamma)\delta}}{N} } \sqrt{\sum_{s\in\cS^{\mathsf{large}}} d^{\star}(s) \mathsf{Var}_{P_{s, \pi^{\star}(s)}}\big(\widehat{V}\big)} + \frac{96\cb  S \Cstar \log\frac{N  }{(1-\gamma)\delta}}{(1-\gamma)N}\\
&\quad\qquad + \frac{8 S \Cstar \infs \log \frac{NS}{(1-\gamma)\delta}}{(1-\gamma)N} +  \frac{5}{N} \\
&\quad \overset{\mathrm{(i)}}{\leq} \sqrt{\frac{96\cb S\Cstar \log\frac{NS  }{(1-\gamma)\delta}}{N} } \sqrt{\sum_{s\in\cS} d^{\star}(s) \mathsf{Var}_{P_{s, \pi^{\star}(s)}}\big(\widehat{V}\big)} + \frac{98\cb \infs S \Cstar \log\frac{NS  }{(1-\gamma)\delta}}{(1-\gamma)N}\\
&\quad \overset{\mathrm{(ii)}}{\leq} \frac{2}{\gamma}\sqrt{\frac{96\cb S\Cstar \log\frac{NS }{(1-\gamma)\delta}}{(1-\gamma)N}  \big\langle d^{\star}, b^{\star}\big\rangle}  + \frac{1}{\gamma}\sqrt{\frac{192\cb S\Cstar \log\frac{NS  }{(1-\gamma)\delta}}{(1-\gamma)N}  }+ \frac{98\cb\infs S \Cstar \log\frac{N S }{(1-\gamma)\delta}}{(1-\gamma)N}, \notag \\
&\quad \overset{\mathrm{(iii)}}{\leq} 4\sqrt{\frac{96\cb S\Cstar \log\frac{NS  }{(1-\gamma)\delta}}{(1-\gamma)N}  \big\langle d^{\star}, b^{\star}\big\rangle}  + 2\sqrt{\frac{192\cb S\Cstar \log\frac{N S }{(1-\gamma)\delta}}{(1-\gamma)N}  }+ \frac{98\cb\infs S \Cstar \log\frac{N S }{(1-\gamma)\delta}}{(1-\gamma)N}, \notag \\
&\quad \overset{\mathrm{(iv)}}{\leq} \frac{1}{2} \big\langle d^{\star}, b^{\star}\big\rangle + \frac{768\cb S\Cstar \log\frac{NS  }{(1-\gamma)\delta}}{(1-\gamma)N}  + \sqrt{\frac{768\cb S\Cstar \log\frac{NS  }{(1-\gamma)\delta}}{(1-\gamma)N}  }+ \frac{98\cb \infs S \Cstar \log\frac{N  S}{(1-\gamma)\delta}}{(1-\gamma)N}. 
\end{align*}
Here, (i) follows when $\cb$ is sufficiently large and $\Cstar \geq 1/S$ (see \eqref{eq:Cstar_lower_bound_infinite}), 
(ii) would hold as long as the following inequality could be established:  
\begin{align}
	\sum_{s\in\cS} d^{\star}(s) \mathsf{Var}_{P_{s, \pi^{\star}(s)}}\big(\widehat{V}\big) 
	\le \frac{2}{\gamma^2(1-\gamma)} + \frac{4}{\gamma^2(1-\gamma)}\big\langle d^{\star}, b^{\star}\big\rangle; 
	\label{eq:sum-d-Var-UB-123-infinite}
\end{align}
(iii) is valid since $\gamma \in [\frac{1}{2}, 1)$, and (iv) follows from the elementary inequality $2xy\leq x^2 + y^2$. 
Rearranging terms, we are left with
\begin{align}
	\big\langle d^{\star}, b^{\star} \big\rangle & \leq \sqrt{\frac{3072\cb S\Cstar \log\frac{NS  }{(1-\gamma)\delta}}{(1-\gamma)N}} + \frac{1732\cb \infs S \Cstar \log\frac{NS  }{(1-\gamma)\delta}}{(1-\gamma)N},
\end{align}
which combined with \eqref{eq:Vall-Bernstein-infinite} yields
\begin{align}
	\big\langle \rho, V^{\star} - V^{\widehat{\pi}}\big\rangle 
	\leq \frac{2\big\langle d^{\star}, b^{\star} \big\rangle}{1-\gamma} \leq 120\sqrt{\frac{\cb S\Cstar \log\frac{NS  }{(1-\gamma)\delta}}{(1-\gamma)^3N}} + \frac{3464\cb \infs S \Cstar \log\frac{N S }{(1-\gamma)\delta}}{(1-\gamma)^2N}. 
\end{align}
This concludes the proof, as long as the inequality~\eqref{eq:sum-d-Var-UB-123-infinite} can be established.

\paragraph{Proof of inequality~\eqref{eq:sum-d-Var-UB-123-infinite}.}
To begin with, we make the observation that 
\begin{align}
	\big(\widehat{V} \circ \widehat{V}\big) - \big(\gamma P^{\star}\widehat{V}) \circ \big(\gamma P^{\star}\widehat{V}) &= \big(\widehat{V} - \gamma P^{\star}\widehat{V}\big) \circ \big(\widehat{V} + \gamma P^{\star}\widehat{V}\big) \notag\\
	& \overset{\mathrm{(i)}}{\le} \big(\widehat{V} - \gamma P^{\star}\widehat{V} + 2 b^\star \big) \circ \big(\widehat{V} + \gamma P^{\star}\widehat{V}\big) \notag \\
	&\overset{\mathrm{(ii)}}{\le} \frac{2}{1-\gamma}\big(\widehat{V} - \gamma P^{\star}\widehat{V} + 2 b^\star \big), \label{eq:infinite-extra-lemma-au1}
\end{align}
where (i) holds since $b^\star \geq 0 $ and $\widehat{V} + \gamma P^{\star}\widehat{V} \geq 0$, (ii) follows from the basic property $\|\widehat{V} + \gamma P^{\star}\widehat{V}\|_\infty \leq 2\|\widehat{V} \|_\infty \leq \frac{2}{1-\gamma}$ and the fact $\widehat{V} - \gamma P^{\star}\widehat{V} + 2 b^\star \geq 0$, the latter of which has been verified in \eqref{eq:infinite-v-lower}. 
Armed with this fact, one can deduce that
\begin{align*}
\sum_{s} d^{\star}(s) \mathsf{Var}_{P_{s, \pi^{\star}(s)}}\big(\widehat{V}\big) 
&\overset{\mathrm{(i)}}{=} \left\langle d^{\star}, P^{\star}\big(\widehat{V} \circ \widehat{V}\big) - \big(P^{\star}\widehat{V}\big) \circ \big(P^{\star}\widehat{V}\big)\right\rangle \\
& \overset{\mathrm{(ii)}}{\le} \left\langle d^{\star}, P^{\star}\big(\widehat{V} \circ \widehat{V}\big) -\frac{1}{\gamma^2} \widehat{V} \circ \widehat{V} + \frac{2}{\gamma^2(1-\gamma)}\big(\widehat{V} - \gamma P^{\star}\widehat{V} + 2 b^\star \big)\right\rangle \\
& \overset{\mathrm{(iii)}}{\le} \left\langle d^{\star}, P^{\star}\big(\widehat{V} \circ \widehat{V}\big) -\frac{1}{\gamma} \widehat{V} \circ \widehat{V} + \frac{2}{\gamma^2(1-\gamma)} (I - \gamma P^{\star})\widehat{V} + \frac{4}{\gamma^2(1-\gamma)} b^\star \right\rangle \notag\\
&=\left\langle d^{\star}, \frac{1}{\gamma}\big(\gamma P^{\star} - I\big)\big(\widehat{V} \circ \widehat{V}\big) + \frac{2}{\gamma^2(1-\gamma)} (I - \gamma P^{\star})\widehat{V} + \frac{4}{\gamma^2(1-\gamma)} b^\star \right\rangle \\
&=d^{\star\top}\big(I-\gamma P^{\star}\big)\left\{ -\frac{1}{\gamma}\widehat{V}\circ\widehat{V}+\frac{2}{\gamma^{2}(1-\gamma)}\widehat{V}\right\} +\frac{4}{\gamma^{2}(1-\gamma)}\langle d^{\star},b^{\star}\rangle \\
	&\overset{\mathrm{(iv)}}{\le} (1-\gamma)\rho^{\top} \left\{ -\frac{1}{\gamma} \widehat{V} \circ \widehat{V} + \frac{2}{\gamma^2(1-\gamma)}\widehat{V} \right\} + \frac{4}{\gamma^2(1-\gamma)}\left\langle d^{\star}, b^{\star}\right\rangle \\
	&\leq \frac{2}{\gamma^2} \rho^{\top}  \widehat{V} + \frac{4}{\gamma^2(1-\gamma)}\left\langle d^{\star}, b^{\star}\right\rangle \\	
	& \overset{\mathrm{(v)}}{\le} \frac{2}{\gamma^2(1-\gamma)} + \frac{4}{\gamma^2(1-\gamma)}\left\langle d^{\star}, b^{\star}\right\rangle. 
\end{align*}
Here, (i) follows by invoking the definition \eqref{eq:defn-Var-P-V}, (ii) holds due to \eqref{eq:infinite-extra-lemma-au1}, (iii) is valid since $\gamma < 1$,  (iv) is a direct consequence of \eqref{eq:infinite-d-matrix-form}, while (v) comes from the basic facts  $\|\rho^{\top}\|_1=1$ and $\|\widehat{V} \|_\infty \leq \frac{1}{1-\gamma}$.




\section{Analysis: episodic finite-horizon MDPs}
\label{sec:analysis-finite}

\subsection{Preliminary facts and notation}

We first collect a few preliminary facts that are useful for the analysis. The first fact determines the range of our estimates $\widehat{Q}_{h}$ and $\widehat{V}_{h}$. 
\begin{lemma}
	\label{lem:range-V-finite}
	The iterates of Algorithm~\ref{alg:vi-lcb-finite} obey
	\begin{align}
		0\leq \widehat{Q}_{h}(s,a) \leq H-h+1 \quad \text{and} \quad 0\leq \widehat{V}_{h}(s) \leq H-h+1
		\qquad \text{for all }(s,a,h)\in \cS\times \cA\times [H]. 
		\label{eq:range-Q-V-finite-horizon}
	\end{align}
\end{lemma}
\begin{proof}
	The non-negativity of  $\widehat{Q}_{h}$ (and hence $\widehat{V}_{h}$) follows directly from the update rule \eqref{eq:VI-LCB-finite}. Regarding the upper bound, we suppose for the moment that  $\widehat{V}_{h+1}(s)\leq H-h$ for step $h+1$. Then \eqref{eq:VI-LCB-finite} tells us that
	\[
		\widehat{Q}_{h}(s,a) \leq 1 + \big\| \widehat{V}_{h+1} \big\|_{\infty} \leq 1+ H-h,
	\]
	which together with $\widehat{V}_{h}(s)=\max_a\widehat{Q}_{h}(s,a)$ justifies the claim \eqref{eq:range-Q-V-finite-horizon} for step $h$ as well. Taking this together with the base case $\widehat{V}_{H+1}=0$ and the standard induction argument concludes the proof. 
\end{proof}

The second fact is concerned with the vector $d_{h}^{\star} \coloneqq [d_{h}^{\star} (s)]_{s\in \cS} \in \mathbb{R}^S$. 
For any $h\in [H]$, denote by $P_h^{\star} \in \mathbb{R}^{S\times S}$ a matrix whose $s$-th row is given by $P_h\big( \cdot\mymid s,\pi_h^{\star}(s) \big)$. 
Then the Markovian property of the MDP indicates that: for any $j > h$, one has 
\begin{align}
	\big( d_{j}^{\star} \big)^{\top} = \big( d_{h}^{\star} \big)^{\top} P_{h}^{\star} \cdots P_{j-1}^{\star}.  
	\label{eq:djstar-connection-dh}
\end{align}

\paragraph{Notation.} 
We remind the reader that $P_{h,s,a} \in \mathbb{R}^{1\times S}$ represents the probability transition vector $P_h(\cdot\mymid s,a)$, and the associated variance parameter  $\mathsf{Var}_{P_{h,s,a}}(V)$ is defined to be the $(h,s,a)$-th row of  $\mathsf{Var}_{P}(V)$ (cf.~\eqref{eq:defn-Var-P-V}), namely, 
\begin{equation} 
	\mathsf{Var}_{P_{h,s,a}} (V) \coloneqq
	\sum_{s' \in S} P_h(s' \mymid s,a) \big( V(s') \big)^2 - \bigg( \sum_{s' \in S} P_h(s' \mymid s,a) V(s') \bigg)^2 
	\label{eq:defn-Var-P-hsa-notation}
\end{equation}
for any given vector $V\in \mathbb{R}^S$. 
The vector $\widehat{P}_{h,s,a} \in \mathbb{R}^{1\times S}$ and the variance parameter  $\mathsf{Var}_{\widehat{P}_{h,s,a}}(V)$
are defined analogously. 

\subsection{A crucial statistical independence property} 
\label{sec:independence-finite}

This subsection demonstrates that the subsampling trick introduced in Section~\ref{sec:sample-split} leads to some crucial statistical independence property. 
To be precise, let us consider the following two data-generating mechanisms; here and below, a sample transition refers to a quadruple $(s,a,h,s')$ that indicates a transition from state $s$ to state $s'$ when action $a$ is taken at step $h$.   
\begin{itemize}

	\item {\bf Model 1 (augmented dataset).} Augment $\Dtrim$ to yield a dataset $\Dtrimaug$ via the following steps. For every $(s,h)\in \cS \times [H]$:
		\begin{itemize}
			\item[1)] Add to $\Dtrimaug$ all $\Ntrim_h(s)$ sample transitions in $\Dtrim$ that transition from state $s$ at step $h$;
			\item[2)] If $\Ntrim_h(s) > \Nmain_h(s)$, then we further add to $\Dtrimaug$
				another set of $\Ntrim_h(s) - \Nmain_h(s)$ independent sample transitions $\big\{ \big(s,a^{(i)}_{h,s}, h, s^{\prime\,(i)}_{h,s}\big) \big\}$ obeying  
		\begin{align}
			a^{(i)}_{h,s} \overset{\mathrm{i.i.d.}}{\sim} \pib_h(\cdot \mymid s), \qquad
			s^{\prime\,(i)}_{h,s} \overset{\mathrm{i.i.d.}}{\sim} P_h\big(\cdot\mymid s,a^{(i)}_{h,s} \big), \qquad \Nmain_h(s) < i\leq \Ntrim_h(s).
		\end{align}

		\end{itemize}

	\item {\bf Model 2 (independent dataset).} For every $(s,h)\in \cS \times [H]$, generate $\Ntrim_h(s)$ independent sample transitions $\big\{ \big(s, a^{(i)}_{h,s}, h, s^{\prime\,(i)}_{h,s}\big) \big\}$ as follows: 
		\begin{align}
			a^{(i)}_{h,s} \overset{\mathrm{i.i.d.}}{\sim} \pib_h(\cdot \mymid s), \qquad
			s^{\prime\,(i)}_{h,s} \overset{\mathrm{i.i.d.}}{\sim} P_h(\cdot\mymid s,a), \qquad 1\leq i\leq \Ntrim_h(s).
		\end{align}
		This forms the following dataset: 
		\begin{align}
			\Diid \coloneqq 
			\Big\{ \big( s, a^{(i)}_{h,s}, h,s^{\prime\,(i)}_{h,s} \big) \mid s\in \cS, 1\leq h\leq H, 1\leq i\leq \Ntrim_h(s) \Big\} .
		\end{align}
\end{itemize}

In words, the dataset $\Dtrimaug$ generated in Model 1 differs from $\Dtrim$ only if $\Ntrim_h(s) > \Nmain_h(s)$ occurs; this data generating mechanism ensures that $\Dtrimaug$ comprises exactly $\Ntrim_h(s)$ sample transitions from state $s$ at step $h$.  
Two key features are: (a) the samples in $\Dtrimaug$ are statistically independent, 
and (b) $\Dtrimaug$ is essentially equivalent to $\Dtrim$ with high probability, as asserted below.  
\begin{lemma}
	\label{lem:distribution-equivalent-finite}
	The above two datasets $\Dtrimaug$ and $\Diid$ have the same distributions. 
	In addition, with probability exceeding $1-8\delta$, $\Dtrimaug=\Dtrim$.  
\end{lemma}
\begin{proof}
	Both $\Dtrimaug$ and $\Diid$ contain exactly $\Ntrim_h(s)$ sample transitions from state $s$ at step $h$. 
	where $\{\Ntrim_h(s)\}$ are statistically independent from the randomness of the samples. 
	It is easily seen that: given $\{\Ntrim_h(s)\}$, the sample transitions in $\Dtrimaug$ across different steps are statistically independent. As a result, 
	$\Dtrim$ and $\Diid$ both consist of independent samples and are of the same distribution. 

	Furthermore, Lemma~\ref{lemma:Ntrim-LB} tells us that with probability at least $1-8\delta$, $\Ntrim_h(s)\leq \Nmain_h(s)$ holds for all $(s,h)\in \cS \times [H]$, implying that that all data in $\Dtrimaug$ come from $\Dmain$ and hence $\Dtrimaug=\Dtrim$. 
\end{proof}

\subsection{Proof of Theorem~\ref{thm:finite}}

We first demonstrate that Theorem~\ref{thm:finite} is valid as long as the following theorem can be established. 
\begin{theorem} \label{thm:Bernstein-finite-aux}
Consider the dataset $\mathcal{D}_0$ described in Section~\ref{sec:VI-LCB-simple-finite}, and any $0<\delta <1$. Suppose that $\mathcal{D}_0$ contains $N$ sample transitions, and that the non-negative integers $\{N_h(s,a) \}$ defined in \eqref{eq:defn-Nh-sa-finite} obey
\begin{align}
	\forall (s,a,h) \in \cS\times \cA\times [H]:
	\qquad N_h(s,a) \geq \frac{K\myrho_h(s,a)}{8} - 5 \sqrt{K \myrho_h(s,a) \log \frac{NH}{\delta}} ,
	\label{eq:assumption-Nhsa-LB-finite-Bern}
\end{align}
with $K$ some quantity obeying $K\geq 3872 HS\Cstar \log \frac{NH}{\delta}$. 
Assume that conditional on $\{N_h(s,a) \}$, the sample transitions $\{ (s, a, h, s'_{(i)}) \mymid 1\leq i\leq N_h(s,a), (s,a,h) \in \cS \times \cA \times [H] \}$	
are statistically independent. 
The penalty terms are taken to be \eqref{def:bonus-Bernstein-finite}, where 
$\cb \ge 16$ is chosen to be some constant. 
Then with probability at least $1- 4\delta$, one has
\begin{align}
	\sum_s d_h^{\star}(s) \big(V^{\star}_h(s) - V^{\widehat{\pi}}_h(s)\big) 
	\le 80 \sqrt{\frac{2\cb H^3S\Cstar\log\frac{NH}{\delta}}{K}} ,
	\qquad 1\leq h\leq H. 
	\label{eq:weighted-sum-V-error-aux-Bern}
\end{align}
\end{theorem}
%

By construction, $\{\Ntrim_h(s,a)\}$ are computed using $\Daux$, and hence are independent from the empirical model $\widehat{P}_h$ generated based on $\Dtrim$. 
Additionally, Lemma~\ref{lem:distribution-equivalent-finite} permits us to treat the samples in $\Dtrim$ as being statistically independent. 
Recalling that the lower bound \eqref{eq:samples-finite-LB} holds with probability at least $1-8\delta$, we can readily invoke Theorem~\ref{thm:Bernstein-finite-aux} by taking $N_h(s,a)=\Ntrim_h(s,a)$ and the property \eqref{eqn:defn-d1rho-finite} to show that
\begin{align}
	\sum_{s\in \cS} 
	\rho(s) \big(V^{\star}_1(s) - V^{\widehat{\pi}}_1(s)\big) =
	\sum_{s\in \cS} 
	d_1^{\star}(s) \big(V^{\star}_1(s) - V^{\widehat{\pi}}_1(s)\big) \le 80 \sqrt{\frac{2\cb H^3S\Cstar\log\frac{NH}{\delta}}{K}}
	\label{eq:sum-rho-V-finite-123}
\end{align}
with probability at least $1-12\delta$, provided that $K \geq 3872 HS\Cstar\log\frac{KH}{\delta}$. 
Setting the right-hand side of \eqref{eq:sum-rho-V-finite-123} to be smaller than $\varepsilon$ immediately concludes the proof of Theorem~\ref{thm:finite}, where we have used the fact that $N\leq KH$ in $\mathcal{D}_0$. 
As a consequence, it suffices to establish Theorem~\ref{thm:Bernstein-finite-aux}. 
In the sequel, we shall assume without loss of generality that we are working on the high-probability event \eqref{eq:samples-finite}.

\subsubsection{Proof of Theorem~\ref{thm:Bernstein-finite-aux}}

\paragraph{Step 1: showing that $\widehat{Q}_h(s, a) \le Q_h^{\widehat{\pi}}(s, a)$.}
This part relies crucially on the following lemma.
\begin{lemma} \label{lem:Bernstein}
	Consider any $1\leq h\leq H$, and any vector $V\in \mathbb{R}^S$ independent of $\widehat{P}_h$ obeying $\|V\|_{\infty} \le H$. 
	With probability at least $1-4\delta/H $, one has
	\begin{align}
		\big|\big(\widehat{P}_{h, s, a} - P_{h, s, a}\big)V\big| 
		&\leq \sqrt{\frac{48\mathsf{Var}_{\widehat{P}_{h,s,a}}(V)\log\frac{NH}{\delta}}{N_{h}(s,a)}}+\frac{48 H\log\frac{NH}{\delta}}{N_{h}(s,a)}
		\label{eq:bonus-Bernstein} \\
		\mathsf{Var}_{\widehat{P}_{h, s, a}} ( V ) &\le 2\mathsf{Var}_{P_{h, s, a}}\big(V\big) + \frac{5H^2\log\frac{NH}{\delta}}{3N_h(s, a)}
		\label{eq:empirical-var}
	\end{align}
	simultaneously for all $(s, a) \in \cS \times \cA $.
\end{lemma}

 \begin{proof} 
The proof follows from exactly the same argument as that of Lemma~\ref{lem:Bernstein-infinite-proof}, except that the assumed upper bound on $\|V\|_{\infty}$ is now $H$ (as opposed to $\frac{1}{1-\gamma}$) and $\delta$ is replaced with $\delta/H$. We thus omit the proof details for brevity. 
 \end{proof}

Additionally, we make note of the crude bound $\big| (\widehat{P}_{h, s, a} - P_{h, s, a} ) \widehat{V}_{h+1}\big| \leq \|\widehat{V}_{h+1}\|_{\infty}\leq H$.  
Also, given that Algorithm~\ref{alg:vi-lcb-finite} works backwards, 
the iterate $\widehat{V}_{h+1}$ does not use $\widehat{P}_h$, and is hence statistically independent from $\widehat{P}_h$.  
Thus, we can readily apply Lemma~\ref{lem:Bernstein} to obtain
\begin{equation}
	\forall(s,a,h)\in \cS\times \cA\times[H]:
	\quad \big|\big(\widehat{P}_{h, s, a} - P_{h, s, a}\big) \widehat{V}_{h+1}\big| \leq b_h(s,a) 
	\label{eq:Phat-P-diff-Vhat-Bern-finite}
\end{equation}
in the presence of the Bernstein-style penalty \eqref{def:bonus-Bernstein-finite}, 
provided that the constant $\cb>0$ is sufficiently large.  

In the sequel, 
we shall work with the high-probability events \eqref{eq:Phat-P-diff-Vhat-Bern-finite} and 
\eqref{eq:empirical-var}, in addition to  \eqref{eq:samples-finite}. We intend to prove the following relation
\begin{align}
	\forall(s,a,h)\in \cS\times \cA\times[H]:
	\quad
	\widehat{Q}_h(s,a)\leq Q_h^{\widehat{\pi}}(s,a) \quad \text{and} \quad \widehat{V}_h(s) \leq V_h^{\widehat{\pi}}(s)
	\label{eq:Q-LCB-Bernstein-finite}
\end{align}
hold with probability exceeding $1-4\delta$. Note that the latter assertion concerning $\widehat{V}_h$ is implied by the former, according to the following relation: 
\begin{equation}
	\widehat{V}_{h}(s) = \max_a \widehat{Q}_{h}(s,a) 
	= \widehat{Q}_{h}\big( s,\widehat{\pi}_h(s) \big)
	\le  Q_{h}^{\widehat{\pi}}\big( s,\widehat{\pi}_h(s) \big) = V_{h}^{\widehat{\pi}}(s) . 
	\label{eq:V-LCB-claim-finite}
\end{equation}  
Therefore, we focus on the first assertion and will show it by induction.
 First of all, the claim~\eqref{eq:Q-LCB-Bernstein-finite} holds trivially for the base case with $h=H+1$, given that $\widehat{Q}_{H+1}(s, a) = Q_{H+1}^{\widehat{\pi}}(s, a) = 0$.
Next, suppose that $\widehat{Q}_{h+1}(s, a) \le Q_{h+1}^{\widehat{\pi}}(s, a)$ holds for all $(s, a) \in \cS \times \cA$ and some step $h+1$. 
%
We would like to show that the claimed inequality holds for step $h$ as well. 
If $\widehat{Q}_h(s, a) = 0$, then the claim holds trivially; otherwise, our update rule \eqref{eq:VI-LCB-finite} reveals that
\begin{align*}
	\widehat{Q}_{h}(s,a) & =r_{h}(s,a)+\widehat{P}_{h,s,a}\widehat{V}_{h+1}-b_{h}(s,a)\\
 	& =r_{h}(s,a)+P_{h,s,a}\widehat{V}_{h+1}+\big(\widehat{P}_{h,s,a}-P_{h,s,a}\big)\widehat{V}_{h+1}-b_{h}(s,a)\\
 	& \overset{(\mathrm{i})}{\leq}r_{h}(s,a)+P_{h,s,a}V_{h+1}^{\widehat{\pi}}\overset{(\mathrm{ii})}{=}Q_{h}^{\widehat{\pi}}(s,a),
\end{align*}
with probability at least $1- \delta/2$, 
where (i) results from~\eqref{eq:Phat-P-diff-Vhat-Bern-finite} and \eqref{eq:V-LCB-claim-finite} (i.e., $\widehat{V}_{h+1}(s) \le   V_{h+1}^{\widehat{\pi}}(s)$), and (ii) arises from the Bellman equation.  
We have thus established \eqref{eq:Q-LCB-Bernstein-finite} via a standard induction argument. 
%


\paragraph{Step 2: bounding $V_h^{\star}(s) - V^{\widehat{\pi}}_h(s)$.}
%
%
In view of \eqref{eq:V-LCB-claim-finite}, we make the observation that
\begin{align}
	0 \leq V_h^{\star}(s) - V_h^{\widehat{\pi}}(s) 
	&\le V_h^{\star}(s) - \widehat{V}_h(s) \le Q_h^{\star} \big(s, \pi^{\star}_h(s) \big) - \widehat{Q}_h \big(s, \pi^{\star}_h(s) \big),
	\label{eq:Vhstar-Vhpihat-UB-246}
\end{align}
where the last inequality holds true since $V_h^{\star}(s) = Q_h^{\star}(s, \pi^{\star}_h(s))$ and $\widehat{V}_h(s) = \max_a\widehat{Q}_h(s,a) \geq  \widehat{Q}_h(s, \pi^{\star}_h(s))$. 
Recognizing that 
\begin{align*}
Q_{h}^{\star}\big(s,\pi_{h}^{\star}(s)\big) & =r\big(s,\pi_{h}^{\star}(s)\big)+P_{h,s,\pi_{h}^{\star}(s)}V_{h+1}^{\star},\\
	\widehat{Q}_{h}\big(s,\pi_{h}^{\star}(s)\big) & = \max\Big\{ r\big(s,\pi_{h}^{\star}(s)\big)+\widehat{P}_{h,s,\pi_{h}^{\star}(s)}\widehat{V}_{h+1}-b_{h}\big(s,\pi_{h}^{\star}(s)\big),\, 0 \Big\},
\end{align*}
we can continue the derivation of \eqref{eq:Vhstar-Vhpihat-UB-246} to obtain
\begin{align}
V_{h}^{\star}(s)-\widehat{V}_{h}(s) & \leq r\big(s,\pi_{h}^{\star}(s)\big)+P_{h,s,\pi_{h}^{\star}(s)}V_{h+1}^{\star}-\left\{ r\big(s,\pi_{h}^{\star}(s)\big)+\widehat{P}_{h,s,\pi_{h}^{\star}(s)}\widehat{V}_{h+1}-b_{h}\big(s,\pi_{h}^{\star}(s)\big)\right\} \nonumber\\
 & =P_{h,s,\pi_{h}^{\star}(s)}V_{h+1}^{\star}-\widehat{P}_{h,s,\pi_{h}^{\star}(s)}\widehat{V}_{h+1}+b_{h}\big(s,\pi_{h}^{\star}(s)\big)\nonumber\\
 & =P_{h,s,\pi_{h}^{\star}(s)}\big(V_{h+1}^{\star}-\widehat{V}_{h+1}\big)-\Big(\widehat{P}_{h,s,\pi_{h}^{\star}(s)}-P_{h,s,\pi_{h}^{\star}(s)}\Big)\widehat{V}_{h+1}+b_{h}\big(s,\pi_{h}^{\star}(s)\big)\nonumber\\
 & \leq P_{h,s,\pi_{h}^{\star}(s)}\big(V_{h+1}^{\star}-\widehat{V}_{h+1}\big)+2b_{h}\big(s,\pi_{h}^{\star}(s)\big)
\label{eq:Vhstar-Vhpihat-UB-357}
\end{align}
with probability at least $1-\delta$, where the last inequality is valid due to \eqref{eq:Phat-P-diff-Vhat-Bern-finite}. 
%
For notational convenience, let us introduce a sequence of matrices $P_{h}^{\star} \in \mathbb{R}^{S\times S}$ ($1\leq h\leq H$) and vectors $b_h^{\star} \in \mathbb{R}^S$ ($1\leq h\leq H$), with their $s$-th rows given by
\begin{align}
	\big[ P_{h}^{\star} \big]_{s,\cdot} \defn P_{h, s, \pi^{\star}_h(s)}
	\qquad\text{and}\qquad 
	b_h^{\star}(s) \defn b_h \big( s, \pi^{\star}_h(s) \big). 
	\label{eq:P-b-star}
\end{align}
This allows us to rewrite \eqref{eq:Vhstar-Vhpihat-UB-357} in matrix/vector form as follows:
\begin{align}
	0 \leq V_{h}^{\star}-\widehat{V}_{h} & \leq P_{h}^{\star}\big(V_{h+1}^{\star}-\widehat{V}_{h+1}\big)+2b_{h}^{\star}.
	\label{eq:Vhstar-Vhpihat-UB-579}
\end{align}
The inequality \eqref{eq:Vhstar-Vhpihat-UB-579} plays a key role in the analysis since it establishes a connection between the value estimation errors in step $h$ and step $h+1$.


Given that $b_h^{\star}$, $P_{h}^{\star}$ and $V_{h}^{\star}-\widehat{V}_{h}$ are all non-negative,
applying~\eqref{eq:Vhstar-Vhpihat-UB-579} recursively with the boundary condition $V_{H+1}^{\star}=\widehat{V}_{H+1}=0$  leads to 
\begin{align*}
0\leq V_{h}^{\star}-\widehat{V}_{h} & \leq P_{h}^{\star}\big(V_{h+1}^{\star}-\widehat{V}_{h+1}\big)+2b_{h}^{\star}\\
 & \leq P_{h}^{\star}P_{h+1}^{\star}\big(V_{h+2}^{\star}-\widehat{V}_{h+2}\big)+2P_{h}^{\star}b_{h+1}^{\star}+2b_{h}^{\star}\leq\cdots\\
	& \leq  2  \sum_{j=h}^{H} \bigg( \prod_{k=h}^{j-1} P_{k}^{\star} \bigg) b_{j}^{\star},
\end{align*}
where we adopt the following notation for convenience (note the order of the product)
%
\[
	\prod_{k=h}^{h-1}P_{k}^{\star}=I 
	\qquad \text{and} \qquad
	\prod_{k=h}^{j-1}P_{k}^{\star}= P_h^{\star}\cdots P_{j-1}^{\star} ~~~~\text{if }j > h.
\]
With this inequality in mind, we can let $d_h^{\star} \coloneqq [d_h^{\star}(s)]_{s\in \cS}$ be a $S$-dimensional vector and derive
\begin{align}
	\big\langle d_{h}^{\star},V_{h}^{\star}-V_{h}^{\widehat{\pi}}\big\rangle \le \big\langle d_{h}^{\star},V_{h}^{\star}-\widehat{V}_{h}\big\rangle & \le\Bigg <d_{h}^{\star},2\sum_{j=h}^{H}\bigg(\prod_{k=h}^{j-1}P_{k}^{\star} \bigg) b_{j}^{\star}\Bigg> \nonumber \\
	& =2\sum_{j=h}^{H} \big(d_{h}^{\star}\big)^{\top} \bigg(\prod_{k=h}^{j-1}P_{k}^{\star}\bigg) b_{j}^{\star}
	=2\sum_{j=h}^{H}\big\langle d_{j}^{\star},b_{j}^{\star}\big\rangle, 
	\label{eq:Vall-Bernstein-finite}
\end{align}
where we have made use of~\eqref{eq:Vhstar-Vhpihat-UB-246} and the elementary identity \eqref{eq:djstar-connection-dh}. 
%


\paragraph{Step 3: using concentrability to bound $\langle d_{j}^{\star},b_{j}^{\star}\rangle$.}
To finish up, we need to make use of the concentrability coefficient. In what follows, we look at two cases separately.

\begin{itemize}
	\item {\em Case 1: $K\myrho_j \big(s, \pi_j^{\star}(s)\big) \le 4\cb \log\frac{NH}{\delta}$.} Given that $b_h(s,a)\leq H$ (cf.~\eqref{def:bonus-Bernstein-finite}), we necessarily have
		\begin{align}
			b_j^{\star}(s) \le H \le 
			H \cdot \frac{4\cb \log\frac{NH}{\delta}}{K \myrho_j\big(s, \pi_j^{\star}(s)\big)}
			\le \frac{4\cb \Cstar H \log\frac{NH}{\delta}}{K\min\big\{d_j^{\star}(s), \frac{1}{S}\big\}} 
		\end{align}
		in this case, where the last inequality arises from Definition~\ref{assumption:concentrate-finite}. 

	\item {\em Case 2: $K\myrho_j\big(s, \pi_j^{\star}(s)\big) > 4\cb \log\frac{NH}{\delta}$.} It follows from the assumption \eqref{eq:assumption-Nhsa-LB-finite-Bern} that 
\begin{align}
	N_j\big(s, \pi_j^{\star}(s) \big) 
	&\geq \frac{K \myrho_j \big(s, \pi_j^{\star}(s) \big)}{ 8} - 5 \sqrt{K \myrho_j \big(s, \pi_j^{\star}(s) \big) \log \frac{N}{\delta}}
	\geq \frac{K \myrho_j \big(s, \pi_j^{\star}(s) \big)}{ 16} \notag\\
	&\geq \frac{K  \min \big\{ d_j^{\star} \big(s,\pi_{j}^{\star}(s)\big), \frac{1}{S} \big\}}{ 16 \Cstar} 
	= \frac{K  \min \big\{ d_j^{\star} (s), \frac{1}{S} \big\}}{ 16 \Cstar},
	\label{eq:Nj-lower-bound-finite-Bern}
\end{align}
		as long as $\cb > 0$ is sufficiently large. Here, the last line results from Definition~\ref{assumption:concentrate-finite}
and the assumption that $\pi^{\star}$ is a deterministic policy (so that $d_j^{\star} (s)=d_j^{\star} \big(s,\pi_{j}^{\star}(s)\big)$).

This further leads to
		\begin{align*}
b_{j}^{\star}(s) & \leq\sqrt{\frac{\cb\log\frac{NH}{\delta}}{N_j\big( s, \pi_j^{\star}(s) \big)}\mathsf{Var}_{\widehat{P}_{j,s,\pi_{j}^{\star}(s)}}(\widehat{V}_{j+1})}+\cb H\frac{\log\frac{NH}{\delta}}{N_j\big( s, \pi_j^{\star}(s) \big)}\\
 & \overset{(\mathrm{i})}{\leq}\sqrt{\frac{2\cb\log\frac{NH}{\delta}}{N_j\big( s, \pi_j^{\star}(s) \big)}\mathsf{Var}_{P_{j,s,\pi_{j}^{\star}(s)}}(\widehat{V}_{j+1})}+3\cb H\frac{\log\frac{NH}{\delta}}{N_j\big( s, \pi_j^{\star}(s) \big)}\\
 & \overset{(\mathrm{ii})}{\leq}\sqrt{\frac{32\cb \Cstar\log\frac{NH}{\delta}}{K\min\big\{ d_{j}^{\star}(s),\frac{1}{S}\big\}}\mathsf{Var}_{P_{j,s,\pi_{j}^{\star}(s)}}(\widehat{V}_{j+1})}+48\cb \Cstar H\frac{\log\frac{NH}{\delta}}{K\min\big\{ d_{j}^{\star}(s),\frac{1}{S}\big\}} .
		\end{align*}
		Here, (i) comes from \eqref{eq:empirical-var} and the elementary inequality $\sqrt{x+y}\leq \sqrt{x}+\sqrt{y}$ for any $x,y\geq 0$, provided that $\cb$ is large enough; and (ii) relies on \eqref{eq:Nj-lower-bound-finite-Bern}. 

\end{itemize}
Putting the above two cases together, we arrive at
\begin{align}
	\sum_s d_j^{\star}(s)b_j^{\star}(s) &\le \sum_s d_j^{\star}(s)\sqrt{\frac{32\cb \Cstar\log\frac{NH}{\delta}}{K \min\big\{d_j^{\star}(s), \frac{1}{S}\big\}}\mathsf{Var}_{P_{j, s, \pi^{\star}_j(s)}}(\widehat{V}_{j+1})} + 48\cb H\sum_s d_j^{\star}(s)\frac{\Cstar\log\frac{NH}{\delta}}{K \min\big\{d_j^{\star}(s), \frac{1}{S}\big\}} \notag\\
	& \le \sum_s d_j^{\star}(s)\sqrt{\frac{32\cb \Cstar\log\frac{NH}{\delta}}{K \min\big\{d_j^{\star}(s), \frac{1}{S}\big\}}\mathsf{Var}_{P_{j, s, \pi^{\star}_j(s)}}(\widehat{V}_{j+1})} +  \frac{96\cb H S \Cstar\log\frac{NH}{\delta}}{K },
	\label{eq:sum-djstar-bjstar-UB-167}
\end{align}
where the last inequality holds since
\begin{align*}
\sum_{s}\frac{d_{j}^{\star}(s)}{\min\big\{ d_{j}^{\star}(s),\frac{1}{S}\big\}} & \le\sum_{s}d_{j}^{\star}(s)\bigg\{\frac{1}{d_{j}^{\star}(s)}+\frac{1}{1/S}\bigg\}\leq\sum_{s}1+S\sum_{s}d_{j}^{\star}(s)\le 2S .
\end{align*}

In addition, we make the observation that
\begin{align*}
 & \sum_{j=h}^{H}\sum_{s}d_{j}^{\star}(s)\sqrt{\frac{\mathsf{Var}_{P_{j,s,\pi_{j}^{\star}(s)}}\big(\widehat{V}_{j+1}\big)}{\min\big\{ d_{j}^{\star}(s),\frac{1}{S}\big\}}}\leq\sum_{j=h}^{H}\sum_{s}d_{j}^{\star}(s)\sqrt{\frac{\mathsf{Var}_{P_{j,s,\pi_{j}^{\star}(s)}}\big(\widehat{V}_{j+1}\big)}{d_{j}^{\star}(s)}}+\sum_{j=h}^{H}\sum_{s}d_{j}^{\star}(s)\sqrt{\frac{\mathsf{Var}_{P_{j,s,\pi_{j}^{\star}(s)}}\big(\widehat{V}_{j+1}\big)}{1/S}}\\
 & \quad=\sum_{j=h}^{H}\sum_{s}\sqrt{d_{j}^{\star}(s)\mathsf{Var}_{P_{j,s,\pi_{j}^{\star}(s)}}\big(\widehat{V}_{j+1}\big)}+\sqrt{S}\sum_{j=h}^{H}\sum_{s}d_{j}^{\star}(s)\sqrt{\mathsf{Var}_{P_{j,s,\pi_{j}^{\star}(s)}}\big(\widehat{V}_{j+1}\big)}\\
 & \quad\leq\sqrt{HS}\cdot\sqrt{\sum_{j=h}^{H}\sum_{s}d_{j}^{\star}(s)\mathsf{Var}_{P_{j,s,\pi_{j}^{\star}(s)}}\big(\widehat{V}_{j+1}\big)}+\sqrt{S}\sqrt{\sum_{j=h}^{H}\sum_{s}d_{j}^{\star}(s)}\sqrt{\sum_{j=h}^{H}\sum_{s}d_{j}^{\star}(s)\mathsf{Var}_{P_{j,s,\pi_{j}^{\star}(s)}}\big(\widehat{V}_{j+1}\big)}\\
 & \quad=2\sqrt{HS}\cdot\sqrt{\sum_{j=h}^{H}\sum_{s}d_{j}^{\star}(s)\mathsf{Var}_{P_{j,s,\pi_{j}^{\star}(s)}}\big(\widehat{V}_{j+1}\big)}  \notag\\
	& \quad \leq 4\sqrt{HS \bigg( H^{2}+ H\sum_{j=h}^{H}\big\langle d_{j}^{\star},b_{j}^{\star}\big\rangle \bigg) }
	\leq 4\sqrt{H^{3}S}+4\sqrt{H^{2}S\sum_{j=h}^{H}\big\langle d_{j}^{\star},b_{j}^{\star}\big\rangle} ,
\end{align*}
where the third line makes use of the Cauchy-Schwarz inequality, 
and the last line would hold as long as we could establish the following inequality
\begin{align}
	\sum_{j=h}^{H}\sum_{s}d_{j}^{\star}(s)\mathsf{Var}_{P_{j,s,\pi_{j}^{\star}(s)}}\big(\widehat{V}_{j+1}\big)
	\le 4H^{2}+4H\sum_{j=h}^{H}\big\langle d_{j}^{\star},b_{j}^{\star}\big\rangle 
	\label{eq:sum-d-Var-UB-123}
\end{align}
for all $h\in [H]$ with probability exceeding $1-4\delta$. 
Substitution into \eqref{eq:sum-djstar-bjstar-UB-167} yields
\begin{align*}
	& \sum_{j=h}^{H}\sum_{s}d_{j}^{\star}(s)b_{j}^{\star}(s)  \le\sqrt{\frac{32\cb \Cstar\log\frac{NH}{\delta}}{K}}\left\{ 4\sqrt{H^{3}S}+4\sqrt{H^{2}S\sum_{j=h}^{H}\big\langle d_{j}^{\star},b_{j}^{\star}\big\rangle}\right\} +\sum_{j=h}^{H}\frac{96\cb HS\Cstar\log\frac{NH}{\delta}}{K}\\
 & \quad \leq 16\sqrt{\frac{2\cb H^{2}S\Cstar\log\frac{NH}{\delta}}{K}}\sqrt{\sum_{j=h}^{H}\big\langle d_{j}^{\star},b_{j}^{\star}\big\rangle}+ 16\sqrt{\frac{2\cb H^{3}S\Cstar\log\frac{NH}{\delta}}{K}}+\frac{96\cb H^{2}S\Cstar\log\frac{NH}{\delta}}{K}\\
 & \quad \leq\frac{1}{2}\sum_{j=h}^{H}\big\langle d_{j}^{\star},b_{j}^{\star}\big\rangle+\frac{256\cb H^{2}S\Cstar\log\frac{NH}{\delta}}{K}+ 16\sqrt{\frac{2\cb H^{3}S\Cstar\log\frac{NH}{\delta}}{K}}+\frac{96\cb H^2S\Cstar\log\frac{NH}{\delta}}{K},
\end{align*}
where the last inequality follows from the elementary inequality $2xy\leq x^2+y^2$. 
Rearranging terms, we are left with
\begin{align*}
\sum_{j=h}^{H}\sum_{s}d_{j}^{\star}(s)b_{j}^{\star}(s) & \leq 32\sqrt{\frac{2\cb H^{3}S\Cstar\log\frac{NH}{\delta}}{K}}+\frac{704\cb H^{2}S\Cstar\log\frac{NH}{\delta}}{K}\\
 & \leq 40\sqrt{\frac{2\cb H^{3}S\Cstar\log\frac{NH}{\delta}}{K}},
\end{align*}
provided that $K\geq 3872 H S\Cstar \log \frac{NH}{\delta}$. 
This taken collectively with \eqref{eq:Vall-Bernstein-finite} completes the proof of Theorem~\ref{thm:Bernstein-finite-aux}, 
as long as the inequality \eqref{eq:sum-d-Var-UB-123} can be validated.

\paragraph{Proof of inequality \eqref{eq:sum-d-Var-UB-123}.}
First of all, we observe that
\begin{align*}
\widehat{V}_{j}(s)+2b_{j}^{\star}\big(s,\pi_{j}^{\star}(s)\big) - P_{j,s,\pi_{j}^{\star}(s)}\widehat{V}_{j+1} 
& =\widehat{V}_{j}(s)-\widehat{P}_{j,s,\pi_{j}^{\star}(s)}\widehat{V}_{j+1}+2b_{j}^{\star}\big(s,\pi_{j}^{\star}(s)\big)+\big(\widehat{P}_{j,s,\pi_{j}^{\star}(s)}-P_{j,s,\pi_{j}^{\star}(s)}\big)\widehat{V}_{j+1}\\
 & \overset{(\mathrm{i})}{\geq}\widehat{V}_{j}(s) -\widehat{P}_{j,s,\pi_{j}^{\star}(s)}\widehat{V}_{j+1} + b_{j}^{\star}\big(s,\pi_{j}^{\star}(s)\big)\\
	& \geq\widehat{V}_{j}(s)- \Big\{ r\big(s,\pi_{j}^{\star}(s)\big) + \widehat{P}_{j,s,\pi_{j}^{\star}(s)}\widehat{V}_{j+1} - b_{j}^{\star}\big(s,\pi_{j}^{\star}(s)\big) \Big\}  \\
 & \geq \max_{a}\widehat{Q}_{j}(s,a)-\widehat{Q}_{j}\big(s,\pi_{j}^{\star}(s)\big)\geq0
\end{align*}
for any $s\in\cS$, where (i) is a consequence of \eqref{eq:Phat-P-diff-Vhat-Bern-finite}, and the last line arises from \eqref{eq:VI-LCB-finite} and \eqref{eq:defn-widehat-V-h+1-finite}. 
This implies the non-negativity
of the vector $\widehat{V}_{j}+2b_{j}^{\star}-P_{j}^{\star}\widehat{V}_{j+1}$, 
which in turn allows one to deduce that
\begin{align}
\widehat{V}_{j}\circ\widehat{V}_{j}-\big(P_{j}^{\star}\widehat{V}_{j+1}\big)\circ\big(P_{j}^{\star}\widehat{V}_{j+1}\big) & =\big(\widehat{V}_{j}+P_{j}^{\star}\widehat{V}_{j+1}\big)\circ\big(\widehat{V}_{j}-P_{j}^{\star}\widehat{V}_{j+1}\big) \notag\\
 & \leq\big(\widehat{V}_{j}+P_{j}^{\star}\widehat{V}_{j+1}\big)\circ\big(\widehat{V}_{j}+2b_{j}^{\star}-P_{j}^{\star}\widehat{V}_{j+1}\big) \notag\\
 & \leq2H\big(\widehat{V}_{j}+2b_{j}^{\star}-P_{j}^{\star}\widehat{V}_{j+1}\big),
	\label{eq:hat-Vj-replace-bound}
\end{align}
where the last line relies on Lemma~\ref{lem:range-V-finite}. 
Consequently, we can demonstrate that
\begin{align*}
 & \sum_{j=h}^{H}\sum_{s}d_{j}^{\star}(s)\mathsf{Var}_{P_{j,s,\pi_{j}^{\star}(s)}}\big(\widehat{V}_{j+1}\big)
	=\sum_{j=h}^{H}\big\langle d_{j}^{\star},P_{j}^{\star}\big(\widehat{V}_{j+1}\circ\widehat{V}_{j+1}\big)-\big(P_{j}^{\star}\widehat{V}_{j+1}\big)\circ\big(P_{j}^{\star}\widehat{V}_{j+1}\big)\big\rangle\\
	& 	=\sum_{j=h}^{H}  \big(d_{j}^{\star} \big)^{\top} P_{j}^{\star} \widehat{V}_{j+1}\circ\widehat{V}_{j+1} 
	-  \big\langle d_{j}^{\star}, \big(P_{j}^{\star}\widehat{V}_{j+1}\big)\circ\big(P_{j}^{\star}\widehat{V}_{j+1}\big)\big\rangle\\
	& \overset{\mathrm{(i)}}{\leq} \sum_{j=h}^{H}\Big(\big\langle d_{j+1}^{\star},\widehat{V}_{j+1}\circ\widehat{V}_{j+1}\big\rangle-\big\langle d_{j}^{\star},\widehat{V}_{j}\circ\widehat{V}_{j}\big\rangle+2H\big\langle d_{j}^{\star},\widehat{V}_{j}+2b_{j}^{\star}-P_{j}^{\star}\widehat{V}_{j+1}\big\rangle\Big)\\
	& \overset{\mathrm{(ii)}}{=} \sum_{j=h}^{H}\Big(\big\langle d_{j+1}^{\star},\widehat{V}_{j+1}\circ\widehat{V}_{j+1}\big\rangle-\big\langle d_{j}^{\star},\widehat{V}_{j}\circ\widehat{V}_{j}\big\rangle\Big)+2H\sum_{j=h}^{H}\Big(\big\langle d_{j}^{\star},\widehat{V}_{j}\big\rangle-\big\langle d_{j+1}^{\star},\widehat{V}_{j+1}\big\rangle\Big)+4H\sum_{j=h}^{H}\big\langle d_{j}^{\star},b_{j}^{\star}\big\rangle\\
 & =\big\langle d_{H+1}^{\star},\widehat{V}_{H+1}\circ\widehat{V}_{H+1}\big\rangle-\big\langle d_{h}^{\star},\widehat{V}_{h}\circ\widehat{V}_{h}\big\rangle+2H\Big(\big\langle d_{h}^{\star},\widehat{V}_{h}\big\rangle-\big\langle d_{H+1}^{\star},\widehat{V}_{H+1}\big\rangle\Big)+4H\sum_{j=h}^{H}\big\langle d_{j}^{\star},b_{j}^{\star}\big\rangle\\
	& \leq \big\| d_{H+1}^{\star} \big\|_1 \big\| \widehat{V}_{H+1}\circ\widehat{V}_{H+1}\big\|_{\infty}
	+2H \big\| d_{h}^{\star} \big\|_1 \big\|\widehat{V}_{h}\big\|_{\infty}
	+4H\sum_{j=h}^{H}\big\langle d_{j}^{\star},b_{j}^{\star}\big\rangle\\
	& \overset{\mathrm{(iii)}}{\le} 3H^{2}+4H\sum_{j=h}^{H}\big\langle d_{j}^{\star},b_{j}^{\star}\big\rangle,
\end{align*}
where (i) arises from \eqref{eq:hat-Vj-replace-bound} as well as the basic property $\big(d_j^{\star}\big)^{\top} P_j^{\star}= \big(d_{j+1}^{\star}\big)^{\top}$,
 (ii) follows by rearranging terms and using the property $\big(d_j^{\star}\big)^{\top} P_j^{\star}= \big(d_{j+1}^{\star}\big)^{\top}$ once again,
and (iii) holds due to the fact that $\|\widehat{V}_h\|_{\infty}\leq H$ and $\|d_{h}^{\star}\|_1=1$. 
This concludes the proof of \eqref{eq:sum-d-Var-UB-123}. 

\section{Discussion}

Our primary contribution has been to pin down the sample complexity of model-based offline RL for the tabular settings, 
by establishing its (near) minimax optimality for both infinite- and finite-horizon MDPs. 
While reliable estimation of the transition kernel is often infeasible in the sample-starved regime, 
it does not preclude the success of this ``plug-in'' approach in learning the optimal policy. 
Encouragingly, the sample complexity characterization we have derived holds for the entire range of target accuracy level $\varepsilon$, 
thus revealing that sample optimality comes into effect  without incurring any burn-in cost.   
This is in stark contrast to all prior results, which either suffered from sample sub-optimality or required a large burn-in sample size in order to yield optimal efficiency.  
We have demonstrated that sophisticated techniques like variance reduction are not necessary, as long as Bernstein-style lower confidence bounds are carefully employed to capture the variance of the estimates in each iteration.

Turning to future directions, we first note that the two-fold subsampling adopted in Algorithm~\ref{alg:vi-lcb-finite-split} is likely unnecessary; it would be of interest to develop sharp analysis for the VI-LCB algorithm without sample splitting, which would call for more refined analysis in order to handle the complicated statistical dependency between different time steps. 
Notably, while avoiding sample splitting cannot improve the sample complexity in an order-wise sense, 
the potential gain in terms of the pre-constants as well as the algoritmic simplicity might be of practical interest. 
 Moreover, given the appealing memory efficiency of model-free algorithms, 
understanding whether one can design sample-optimal model-free offline algorithms with minimal burn-in periods 
is another open direction. Moving beyond tabular settings, it would be of great interest to extend our analysis to 
accommodate model-based offline RL in more general scenarios; examples include MDPs with low-complexity linear representations, and offline RL involving multiple agents.

\section*{Acknowledgements}

Y.~Wei is supported in part by the NSF grants CCF-2106778, DMS-2147546/2015447, the NSF CAREER award DMS-2143215, and the Google Research Scholar Award.  
Y.~Chen is supported in part by the Alfred P.~Sloan Research Fellowship, the Google Research Scholar Award, the AFOSR grants FA9550-19-1-0030 and FA9550-22-1-0198, the ONR grant N00014-22-1-2354, 
and the NSF grants CCF-2221009, CCF-1907661, DMS-2014279, IIS-2218713 and IIS-2218773. 
L.~Shi and Y.~Chi are supported in part by the grants ONR N00014-19-1-2404, NSF CCF-2106778 and DMS-2134080, and CAREER award ECCS-1818571. L. Shi is also gratefully supported by the Leo Finzi Memorial Fellowship, Wei Shen and Xuehong Zhang Presidential Fellowship, and
Liang Ji-Dian Graduate Fellowship at Carnegie Mellon University.
Part of this work was done while G.~Li, Y.~Chen and Y.~Wei were visiting the Simons Institute for the Theory of Computing.

\appendix

\section{Proof of auxiliary lemmas: infinite-horizon MDPs}
\label{sec:proof-auxiliary-infinite}

\subsection{Proof of Lemma~\ref{lem:contraction}} \label{sec:proof-lemma:contraction}


Before embarking on the proof, we introduce several notation. 
To make explicit the dependency on $V$, we shall express the penalty term using the following notation throughout this subsection:  
\begin{align}
	b(s, a; V) = \min \Bigg\{ \max \bigg\{\sqrt{\frac{\cb\log\frac{N}{(1-\gamma)\delta}}{N(s, a)}\mathsf{Var}_{\widehat{P}_{ s, a}} (V )} ,\,\frac{2\cb \log\frac{N}{(1-\gamma)\delta}}{(1-\gamma)N(s, a)} \bigg\} ,\, \frac{1}{1-\gamma}  \Bigg\} +\frac{5}{N} 
	\label{eq:defn-bV-sa-inf}
\end{align}

For any $Q, Q_1, Q_2 \in \mathbb{R}^{SA}$, we write
\begin{align}
	V(s) \coloneqq \max_a Q(s,a), \qquad 
	V_1(s) \coloneqq \max_a Q_1(s,a) \qquad \text{and} \qquad V_2(s) \coloneqq \max_a Q_2(s,a) 
\end{align}
for all $s\in \cS$. 
Unless otherwise noted, we assume that
\[
	Q(s,a), Q_1(s,a), Q_2(s,a) \in \Big[0, \frac{1}{1-\gamma}\Big] \qquad \text{for all }(s,a)\in \cS\times \cA
\]
throughout this subsection. 
In addition, let us define another operator $\tildeTpess$ obeying
\begin{equation}
	\tildeTpess(Q)(s,a) = r(s,a) - b(s,a; V) + \gamma \widehat{P}_{s,a} V
	\qquad \text{for all }(s,a)\in \cS\times \cA
	\label{eq:defn-tilde-Tpess}
\end{equation}
for any $Q\in \mathbb{R}^{SA}$. 
It is self-evident that 
\begin{align}
	\Tpess(Q)(s,a) = \max\big\{ \tildeTpess(Q)(s,a), 0 \big\}
	\qquad \text{for all }(s,a)\in \cS\times \cA.
	\label{eq:relation-Tpress-Ttilde-press}
\end{align}

\paragraph{$\gamma$-contraction.}
The main step of the proof lies in showing the monotonicity of the operator $\tildeTpess$ in the sense that
\begin{align}
	\tildeTpess(Q) \le \tildeTpess(\widetilde{Q}) \qquad\text{for any } Q \le \widetilde{Q}.
	\label{eq:monotonicity-Tpess-inf}
\end{align}
Suppose that this claim is valid for the moment, then one can demonstrate that: for any $Q_1, Q_2 \in \mathbb{R}^{SA}$, 
\begin{subequations}
\label{eq:Tpess-Q1-Q2-LB-UB-inf}
\begin{align}
\tildeTpess(Q_{1})-\tildeTpess(Q_{2})
	&\le\tildeTpess\big(Q_{2}+\|Q_{1}-Q_{2}\|_{\infty}1\big)-\tildeTpess(Q_{2}), \\ 
 \tildeTpess(Q_{1})-\tildeTpess(Q_{2})
	&\ge \tildeTpess\big(Q_{2}-\|Q_{1}-Q_{2}\|_{\infty}1\big)
	-\tildeTpess(Q_{2}),
\end{align}
\end{subequations}
with $1$ denoting the all-one vector. Additionally, observe that
\begin{align*}
	\mathsf{Var}_{\widehat{P}_{s, a}}(V) = \mathsf{Var}_{\widehat{P}_{s, a}}(V + c\cdot 1)
	\qquad\text{and hence}\qquad
	b(s,a; V) = b(s,a; V+c\cdot 1)
\end{align*}
for any constant $c$, which together with the identity $\widehat{P}1=1$ immediately  leads to
\begin{align*}
\Big\|\tildeTpess\big(Q_{2}-\|Q_{1}-Q_{2}\|_{\infty}1\big)-\tildeTpess(Q_{2})\Big\|_{\infty} & \leq\gamma\Big\|\widehat{P}\big(\|Q_{1}-Q_{2}\|_{\infty}1\big)\Big\|_{\infty}=\gamma\|Q_{1}-Q_{2}\|_{\infty},\\
\Big\|\tildeTpess\big(Q_{2}+\|Q_{1}-Q_{2}\|_{\infty}1\big)-\tildeTpess(Q_{2})\Big\|_{\infty} & \leq\gamma\Big\|\widehat{P}\big(\|Q_{1}-Q_{2}\|_{\infty}1\big)\Big\|_{\infty}=\gamma\|Q_{1}-Q_{2}\|_{\infty}.
\end{align*}
Taking this together with \eqref{eq:Tpess-Q1-Q2-LB-UB-inf} yields
\begin{align*}
 & \big\|\tildeTpess(Q_{1})-\tildeTpess(Q_{2})\big\|_{\infty}
  \leq\gamma\|Q_{1}-Q_{2}\|_{\infty}, 
\end{align*}
which combined with the basic property 
$\big\|\Tpess(Q_{1})-\Tpess(Q_{2})\big\|_{\infty}\leq \big\|\tildeTpess(Q_{1})-\tildeTpess(Q_{2})\big\|_{\infty}$ (as a result of \eqref{eq:relation-Tpress-Ttilde-press}) justifies that
\begin{align}
 & \big\|\Tpess(Q_{1})-\Tpess(Q_{2})\big\|_{\infty}
  \leq\gamma\|Q_{1}-Q_{2}\|_{\infty}.  
	\label{eq:contraction-Tpess-proof-inf}
\end{align}
The remainder of the proof is thus devoted to establishing the monotonicity property \eqref{eq:monotonicity-Tpess-inf}.

\paragraph{Proof of the monotonicity property \eqref{eq:monotonicity-Tpess-inf}.}
Consider any point $Q\in \mathbb{R}^{SA}$, and we would like to examine the derivative of $\tildeTpess$ at point $Q$. 
Towards this end, we consider any $(s,a)\in \cS\times \cA$ and divide into several cases. 
\begin{itemize}
	\item {\em Case 1: $\max \bigg\{\sqrt{\frac{\cb\log\frac{N}{(1-\gamma)\delta}}{N(s, a)}\mathsf{Var}_{\widehat{P}_{ s, a}}(V)} ,\,\frac{2\cb \log\frac{N}{(1-\gamma)\delta}}{(1-\gamma)N(s, a)} \bigg\} > \frac{1}{1-\gamma}$.} In this case, the penalty term \eqref{eq:defn-bV-sa-inf} simplifies to  
		\[
			b(s,a; V) = \frac{1}{1-\gamma} + \frac{5}{N}. 
		\]
		Taking the derivative of $\tildeTpess(Q)(s,a)$ w.r.t.~the $s'$-th component of $V$ leads to
		\begin{align}
\frac{\partial\big(\tildeTpess(Q)(s,a)\big)}{\partial V(s')}=\frac{\partial\big(r(s,a)-\frac{1}{1-\gamma}+\gamma\widehat{P}_{s,a}V\big)}{\partial V(s')}=\gamma\widehat{P}(s'\mymid s,a)\geq 0
		\label{eq:derivative-Ttilde-press-case1}
		\end{align}
		for any $s'\in  \cS$.

	\item {\em Case 2: $\sqrt{\frac{\cb\log\frac{N}{(1-\gamma)\delta}}{N(s,a)}\mathsf{Var}_{\widehat{P}_{s,a}}(V)}<\frac{2\cb\log\frac{N}{(1-\gamma)\delta}}{(1-\gamma)N(s,a)}<\frac{1}{1-\gamma}
$. 
		} The penalty \eqref{eq:defn-bV-sa-inf} in this case reduces to  
		\[
			b(s,a; V) = \frac{2\cb\log\frac{N}{(1-\gamma)\delta}}{(1-\gamma)N(s,a)} +\frac{5}{N},
		\]
		an expression that is independent of $V$. As a result, repeating the argument for Case 1 indicates that \eqref{eq:derivative-Ttilde-press-case1} continues to hold for this case. 

	\item {\em Case 3: 
$\frac{2\cb\log\frac{N}{(1-\gamma)\delta}}{(1-\gamma)N(s,a)}<\sqrt{\frac{\cb\log\frac{N}{(1-\gamma)\delta}}{N(s,a)}\mathsf{Var}_{\widehat{P}_{s,a}}(V)}<\frac{1}{1-\gamma}$. } In this case, the penalty term is given by
		\[
			b(s,a; V) = \sqrt{\frac{\cb\log\frac{N}{(1-\gamma)\delta}}{N(s,a)}\mathsf{Var}_{\widehat{P}_{s,a}}(V)} + \frac{5}{N}. 
		\]
		Note that in this case, we necessarily have
\begin{align*}
	\mathsf{Var}_{\widehat{P}_{ s, a}}(V) \geq \frac{4\cb \log\frac{N }{(1-\gamma)\delta}}{(1-\gamma)^2 N(s, a)},
\end{align*}
which together with the definition in \eqref{eq:defn-Var-P-V} indicates that
\begin{align}
	 \widehat{P}_{s, a}(V \circ V) - \big(\widehat{P}_{s, a}V\big)^2 \geq  \frac{4\cb\log\frac{N }{(1-\gamma)\delta}}{(1-\gamma)^2N(s, a)} > 0.
	 \label{eq:P-VV-PV-condition-Lemma5}
\end{align}
As a result, for any $s^{\prime}\in \cS$, taking the derivative of $b(s,a; V)$ w.r.t.~the $s^{\prime}$-th component of $V$ gives
\begin{align*}
\frac{\partial b(s,a; V)}{\partial V(s^{\prime})} & =\sqrt{\frac{\cb\log\frac{N }{(1-\gamma)\delta}}{N(s,a)}}\frac{\partial\sqrt{\widehat{P}_{s,a}(V\circ V)-\big(\widehat{P}_{s,a}V\big)^{2}}}{\partial V(s^{\prime})}\\
	& =\sqrt{\frac{\cb\log\frac{N }{(1-\gamma)\delta}}{N(s,a)}}\frac{\widehat{P}(s^{\prime}\mymid s,a)V(s^{\prime})-\big(\widehat{P}_{s,a}V\big)\widehat{P}(s^{\prime}\mymid s,a)}{\sqrt{\widehat{P}_{s,a}(V\circ V)-\big(\widehat{P}_{s,a}V\big)^{2}}}\\
 & \le\sqrt{\frac{\cb\log\frac{N }{(1-\gamma)\delta}}{N(s,a)}}\frac{\widehat{P}(s^{\prime}\mymid s,a)V(s^{\prime})}{\sqrt{\widehat{P}_{s,a}(V\circ V)-\big(\widehat{P}_{s,a}V\big)^{2}}}\\
 & \le\frac{1}{2}(1-\gamma)\widehat{P}(s^{\prime}\mymid s,a)V(s^{\prime})\leq\gamma\widehat{P}(s^{\prime}\mymid s,a),
\end{align*}
where the penultimate inequality relies on \eqref{eq:P-VV-PV-condition-Lemma5}, and the last inequality is valid since 
		$V(s^{\prime}) = \max_a Q(s^{\prime},a) \leq \frac{1}{1-\gamma}$ and $\gamma \geq 1/2$. 
In turn, the preceding relation allows one to derive 
\begin{align*}
	\frac{\partial \big( \tildeTpess(Q)(s, a) \big)}{\partial V(s^{\prime})} = \gamma \widehat{P}(s^{\prime} \mymid s, a) - \frac{\partial b(s,a; V)}{\partial V(s^{\prime})} \ge 0 
\end{align*}
for any $s^{\prime}\in \cS$.

\end{itemize}

Putting the above cases together reveals that $$\frac{\partial \big( \tildeTpess(Q)(s, a) \big)}{\partial V(s^{\prime})}\geq 0
\qquad \text{for all }(s,a,s^{\prime}) \in \cS\times \cA\times \cS $$ 
holds  almost everywhere (except for the boundary points of these cases). 
Recognizing that $\tildeTpess(Q)$ is continuous in $Q$ and that $V$ is non-decreasing in $Q$,  
one can immediately conclude that
\begin{align}
	\tildeTpess({Q}) \le \tildeTpess(\widetilde{Q}) \qquad\text{for any } Q \le \widetilde{Q}. 
\end{align}

\paragraph{Existence and uniqueness of fixed points.} 
To begin with, note that 
for any $0\leq Q\leq \frac{1}{1-\gamma}\cdot 1$, one has $0\leq \Tpess(Q) \leq \frac{1}{1-\gamma}\cdot 1$. 
If we produce the following sequence recursively:
\[
	Q^{(0)}=0 \qquad \text{and} \qquad Q^{(t+1)}= \Tpess(Q^{(t)}) \quad \text{for all }t\geq 0,
\]
then the standard proof for the Banach fixed-point theorem (e.g., \citet[Theorem 1]{agarwal2001fixed}) tells us that $Q^{(t)}$ converges to some point $Q^{(\infty)}$ as $t\rightarrow \infty$. Clearly, $Q^{(\infty)}$ is a fixed point of $\Tpess$ obeying $0\leq Q^{(\infty)} \leq \frac{1}{1-\gamma}\cdot 1$.

We then turn to justifying the uniqueness of fixed points of $\Tpess$. 
Suppose that there exists another point $\widetilde{Q}$ obeying $\widetilde{Q}= \Tpess \big( \widetilde{Q} \big)$, 
which clearly satisfies $\widetilde{Q}\geq 0$. 
If $\|\widetilde{Q}\|_{\infty} > \frac{1}{1-\gamma}$, then 
\[
	\big\| \widetilde{Q} \big\|_{\infty} = \big\| \Tpess \big( \widetilde{Q} \big) \big\|_{\infty}
	\leq \| r \|_{\infty} + \gamma \| \widehat{P} \|_1 	\big\| \widetilde{Q} \big\|_{\infty} 
	\leq 1 + \gamma \big\| \widetilde{Q} \big\|_{\infty} 
	< (1-\gamma) \big\| \widetilde{Q} \big\|_{\infty} + \gamma \big\| \widetilde{Q} \big\|_{\infty} = \big\| \widetilde{Q} \big\|_{\infty} ,
\]
resulting in contradiction. Consequently, one necessarily has $0 \leq \widetilde{Q} \leq \frac{1}{1-\gamma}\cdot 1$. 
Further,  the $\gamma$-contraction property \eqref{eq:contraction-Tpess-proof-inf} implies that
\[
	\big\| \widetilde{Q} - Q^{(\infty)} \big\|_{\infty}
	= \big\| \Tpess \big( \widetilde{Q} \big) - \Tpess \big( Q^{(\infty)} \big) \big\|_{\infty}
	\leq \gamma \big\| \widetilde{Q} - Q^{(\infty)} \big\|_{\infty}. 
\]
Given that $\gamma < 1$, this inequality cannot happen unless $\widetilde{Q} = Q^{{\infty}}$, 
thus confirming the uniqueness of $Q^{{\infty}}$.

\subsection{Proof of Lemma~\ref{lem:monotone-contraction}}\label{sec:proof-lemma:monotone-contraction}

Let us first recall the monotone non-decreasing property \eqref{eq:monotonicity-Tpess-inf} of the operator $\tildeTpess$ defined in \eqref{eq:defn-tilde-Tpess}, which taken together with the property \eqref{eq:relation-Tpress-Ttilde-press} readily yields
\begin{align}
	\Tpess(Q) \le \Tpess(\widetilde{Q})  
	\label{eq:monotonicity-Tpess-inf-recall}
\end{align}
for any $Q$ and $\widetilde{Q}$ obeying $Q \le \widetilde{Q}$, $0\leq Q \leq \frac{1}{1-\gamma} \cdot 1$ and $0\leq \widetilde{Q} \leq \frac{1}{1-\gamma} \cdot 1$ (with $1$ the all-one vector). 
Given that $\widehat{Q}_0 = 0 \leq \widehat{Q}_{\mathsf{pe}}^{\star}$, we can apply \eqref{eq:monotonicity-Tpess-inf-recall} to obtain
\[
	\widehat{Q}_1 = \Tpess\big( Q_0 \big) \leq \Tpess\big( \widehat{Q}_{\mathsf{pe}}^{\star} \big) = \widehat{Q}_{\mathsf{pe}}^{\star}. 
\]
Repeat this argument recursively to arrive at
\[
	\widehat{Q}_{\tau} \leq \widehat{Q}_{\mathsf{pe}}^{\star} \qquad \text{for all }\tau \geq 0. 
\]

In addition, it comes directly from Lemma~\ref{lem:contraction} that
\begin{align}
	\big\|\widehat{Q}_{\tau} - \widehat{Q}_{\mathsf{pe}}^{\star} \big\|_{\infty} 
	&= \big\| \Tpess\big(\widehat{Q}_{\tau-1} \big) - \Tpess\big( \widehat{Q}_{\mathsf{pe}}^{\star} \big) \big\|_{\infty} 
	\leq  \gamma\big\|\widehat{Q}_{\tau-1} -  \widehat{Q}_{\mathsf{pe}}^{\star}  \big\|_{\infty} \notag\\
	&\le \cdots \le \gamma^{\tau} \big\|\widehat{Q}_0 - \widehat{Q}_{\mathsf{pe}}^{\star} \big\|_{\infty} \notag\\
	&\le \frac{\gamma^{\tau} }{1-\gamma} 
	\label{eq:Q-convergence-proof-135}
\end{align}
for any $\tau \geq 0$, 
where the last inequality is valid since $\widehat{Q}_0=0$ and $\|\widehat{Q}_{\mathsf{pe}}^{\star}\|_{\infty}\leq \frac{1}{1-\gamma}$ (see Lemma~\ref{lem:contraction}). 
The other claim \eqref{eq:taumax-Q-converge} also follows immediately by taking the right-hand side of \eqref{eq:Q-convergence-proof-135} to be no larger than $1/N$.

\subsection{Proof of Lemma~\ref{lem:sample-split-infinite}}\label{proof:lem:sample-split-infinite}

For any $(s,a)\in \cS\times \cA$, if $\frac{N\myrho(s,a)}{12} < \frac{2}{3} \infs \log \frac{SN}{\delta}$, 
then it is self-evident that this pair satisfies \eqref{eq:samples-infinite-LB}. 
As a consequence, it suffices to focus attention on the following set of state-action pairs: 
\begin{align}
	\mathcal{N}_{\mathsf{large}} \coloneqq \bigg\{ (s,a) \,\Big|\, \myrho(s,a) \geq \frac{ 8 \infs \log \frac{SN}{\delta} }{N} \bigg\}.  
\end{align}
To bound the cardinality of $\mathcal{N}_{\mathsf{large}}$, we make the observation that
\[
	\big|\mathcal{N}_{\mathsf{large}}\big|\cdot\frac{8\infs\log\frac{SN}{\delta}}{N}
	\leq\sum_{(s,a)\in\mathcal{N}_{\mathsf{large}}}\myrho(s,a)\leq\sum_{(s,a)\in\cS\times\cA}\myrho(s,a) \leq 1 ,
\]
thus leading to the crude bound 
\begin{align}
	\big| \mathcal{N}_{\mathsf{large}} \big| \leq \frac{N}{8\log\frac{SN}{\delta}} \leq  \frac{N}{8}.
	\label{eq:Nlarge-cardinality-UB}
\end{align}

Let us now look at any $(s,a)\in \mathcal{N}_{\mathsf{large}}$. 
Given that $N(s,a)$ can be viewed as the sum of $N$ independent Bernoulli random variables each with mean $\myrho(s,a)$,
we can apply the Bernstein inequality to yield
\begin{align*}
	\mathbb{P}\Big\{ \left|N(s,a)- N\myrho(s,a)\right|\geq\tau\Big\}  
	& \leq 2 \exp\left(-\frac{\tau^{2}/2}{ v_{s,a} + \tau/3}\right) 
\end{align*}
for any $\tau\geq 0$, 
where we define
$$
	v_{s,a} \coloneqq N \mathsf{Var}\Big( \mathds{1} \big\{ (s_i, a_i) = (s,a) \big\} \Big) 
	 \leq  N\myrho(s,a) .
$$
A little algebra then yields that with probability at least $1-\delta$, 
\begin{align}
	\left|N(s,a)- N\myrho(s,a)\right| 
	& \leq\sqrt{4v_{s,a}\log\frac{2}{\delta}}+\frac{2}{3}\log\frac{2}{\delta} 
  	\leq \sqrt{4N\myrho(s,a)\log\frac{2}{\delta}}+\log\frac{2}{\delta} .
\end{align}
Combining this result with the union bound over $(s,a)\in \mathcal{N}_{\mathsf{large}}$ and making use of \eqref{eq:Nlarge-cardinality-UB} give: with probability at least $1-\delta$,
\begin{align}
	\left|N(s,a)- N\myrho(s,a)\right|  
  	\leq \sqrt{4N\myrho(s,a)\log\frac{N}{\delta}}+\log\frac{N}{\delta} 
	\label{eq:deviation-Naux-h-s}
\end{align}
holds simultaneously for all $(s,a)\in \mathcal{N}_{\mathsf{large}}$. 
Recalling that $N \myrho(s, a) \geq 8 \infs \log\frac{NS}{\delta}$ holds for any $(s,a)\in \mathcal{N}_{\mathsf{large}}$,
we can easily verify that
\begin{align}
	N(s,a) \geq N\myrho(s,a) - \left(\sqrt{4N\myrho(s,a)\log\frac{N}{\delta}}+\log\frac{N}{\delta}\right) \geq \frac{N\myrho(s,a)}{12},
\end{align}
thereby establishing \eqref{eq:samples-infinite-LB} for any $(s,a)\in \mathcal{N}_{\mathsf{large}}$. This concludes the proof.

\subsection{Proof of Lemma~\ref{lem:Bernstein-infinite}}\label{proof:lem:Bernstein-infinite}

If $N(s, a) = 0$, then the inequalities hold trivially. Hence, it is sufficient to focus on the case where $N(s, a) > 0$.
Before proceeding, we make note of a key Bernstein-style result; the proof is deferred to Appendix~\ref{sec:lem:Bernstein-infinite-proof}. 
%
\begin{lemma} \label{lem:Bernstein-infinite-proof}
	Consider any given pair $(s,a)\in \cS\times \cA$ with $N(s,a)>0$. Let $V\in \mathbb{R}^S$ be any vector  independent of $\widehat{P}_{s,a}$ obeying $\|V\|_{\infty} \le \frac{1}{1-\gamma}$. 
	With probability at least $1-4\delta  $, one has
	\begin{subequations}
	\begin{align}
		\big|\big(\widehat{P}_{s, a} - P_{s, a}\big)V\big| 
		&\leq \sqrt{\frac{48\mathsf{Var}_{\widehat{P}_{s,a}}(V)\log\frac{N}{\delta}}{N(s,a)}}+\frac{48 \log\frac{N}{\delta}}{(1-\gamma) N(s,a)}
		\label{eq:bonus-Bernstein-infinite-proof} \\
		\mathsf{Var}_{\widehat{P}_{s, a}} ( V ) &\le 2\mathsf{Var}_{P_{s, a}}\big(V\big) + \frac{5\log\frac{N}{\delta }}{3(1-\gamma)^2 N(s, a)}
		\label{eq:empirical-var-infinite-proof}
	\end{align}
	\end{subequations}
\end{lemma}
\begin{remark}
	In words, Lemma~\ref{lem:Bernstein-infinite-proof} develops a Bernstein bound \eqref{eq:bonus-Bernstein-infinite-proof} on $\big|\big(\widehat{P}_{s, a} - P_{ s, a}\big)V\big|$ 
	that makes clear the importance of the variance parameter.  Lemma~\ref{lem:Bernstein-infinite-proof} (cf.~\eqref{eq:empirical-var-infinite-proof}) also ascertains that the variance w.r.t.~the empirical distribution $\widehat{P}_{s,a}$ does not deviate much from the variance w.r.t.~the true distribution $P_{s,a}$. 
\end{remark}

Equipped with this result, we are now ready to present the proof of Lemma~\ref{lem:Bernstein-infinite}, which is built upon a leave-one-out decoupling argument and consists of the following steps.

\paragraph{Step 1: construction of auxiliary state-absorbing MDPs.}  
Recall that $\widehat{\mathcal{M}}$ is the empirical MDP. 
For each state $s \in \cS$ and each scalar $u\geq 0$, we construct an auxiliary state-absorbing MDP $\widehat{\mathcal{M}}^{s, u}$ in a way that makes it identical to the empirical MDP $\widehat{\mathcal{M}}$ except for state $s$. 
More specifically,  the transition kernel of the auxiliary MDP $\widehat{\mathcal{M}}^{s, u}$ --- denoted by $P^{s,u}$ --- is chosen such that
\begin{align*}
	P^{s,u}(\,\widetilde{s} \mymid s, a) &= \mathds{1}(\widetilde{s} = s) \qquad\quad &&\text{for all }(\,\widetilde{s},a) \in \cS\times \cA, \\
	P^{s,u}(\cdot\mymid s^{\prime}, a) &= \widehat{P}(\cdot\mymid s^{\prime}, a)
	\qquad &&\text{for all }(s^{\prime}, a) \in \cS \times \cA \text{ and }s^{\prime}\neq s; 
\end{align*}
and the reward function of $\widehat{\mathcal{M}}^{s,u}$ --- denoted by $r^{s,u}$ --- is set to be
\begin{align*}
	r^{s,u}(s, a) &= u \qquad\qquad\quad~&&\text{for all }a \in \cA, \\
	r^{s,u}(s^{\prime}, a) &= r(s^{\prime}, a) \qquad&&\text{for all }(s^{\prime}, a) \in \cS \times \cA \text{ and }s^{\prime}\neq s.
\end{align*}
In words, the probability transition kernel of $\widehat{\mathcal{M}}^{s, u}$ is obtained by dropping all randomness of $\widehat{P}_{s,a}$ ($a\in \cA$) that concerns state $s$ and making $s$ an absorbing state. 
In addition, let us define the pessimistic Bellman operator $\Tpess^{s,u}$ based on the auxiliary MDP $\widehat{\mathcal{M}}^{s,u}$  such that
\begin{align}
	\Tpess^{s,u} (Q)(s, a) \coloneqq \max\Big\{r^{s,u}(s, a) + \gamma P^{s,u}_{s,a} V - b^{s,u}(s, a; V) 
	 ,\, 0\Big\}
	\label{eq:empirical-Bellman-infinite-Lemma8}
\end{align}
for any $(s,a)\in \cS\times \cA$, 
where the penalty term is taken to be  
\begin{align}
	b^{s,u}(s, a; V)
	= \min \Bigg\{ \max \bigg\{\sqrt{\frac{\cb\log\frac{N}{(1-\gamma)\delta}}{N(s, a)}\mathsf{Var}_{P^{s,u} (\cdot \mymid s, a)} (V )} ,\,\frac{2\cb \log\frac{N}{(1-\gamma)\delta}}{(1-\gamma)N(s, a)} \bigg\} ,\, \frac{1}{1-\gamma}  \Bigg\} +\frac{5}{N}.
\end{align}

\paragraph{Step 2: the correspondence between the empirical MDP and auxiliary MDP.} Taking 
\begin{align}
	u^{\star} = (1-\gamma) \widehat{V}_{\mathsf{pe}}^{\star}(s) + \min \left\{ \frac{2\cb \log\frac{N}{(1-\gamma)\delta}}{(1-\gamma)\max_a N(s,a)}, \frac{1}{1-\gamma}\right\} 
	+\frac{5}{N},
	\label{eq:u-star-choice-inf-123}
\end{align}
we claim that there exists a fixed point $\widehat{Q}^{\star}_{s,u^{\star}}$  of $\Tpess^{s,u^{\star}}$ whose corresponding value function $\widehat{V}_{s,u^{\star}}^{\star}$ coincides with $\widehat{V}_{\mathsf{pe}}^{\star}$. 
To justify this, it suffices to verify the following properties: 
\begin{itemize}
	\item Consider any $a\in \cA$. Given that $P^{s,u}(\cdot \mymid s,a)$ only has a single non-zero entry (equal to 1), it is easily seen that 
		$ \mathsf{Var}_{P^{s,u} (\cdot \mymid s, a)} (V ) = 0$ holds for any $V$ and any $u$, thus indicating that 
		\begin{align}
			b^{s,u}(s, a; V) = \min \Bigg\{ \frac{2\cb \log\frac{N }{(1-\gamma)\delta}}{(1-\gamma)N(s, a)} ,\, \frac{1}{1-\gamma}  \Bigg\} +\frac{5}{N} .
			\label{eq:defn-bsu-sa-inf}
		\end{align}
		Consequently, for state $s$, one has
		\begin{align}
			\max_a\Big\{ r^{s,u^{\star}}(s,a)-b^{s,u^{\star}}\big(s,a; \widehat{V}_{\mathsf{pe}}^{\star}\big)+\gamma\big\langle P^{s,u^{\star}}(\cdot\mymid s,a),\,\,\widehat{V}_{\mathsf{pe}}^{\star}\big\rangle\Big\}
			&=\max_a\Big\{u^{\star}-b^{s,u^{\star}}\big(s,a;\widehat{V}_{\mathsf{pe}}^{\star}\big)+\gamma\widehat{V}_{\mathsf{pe}}^{\star}(s)\Big\} \notag\\
			&=u^{\star}-\min_a b^{s,u^{\star}}\big(s,a;\widehat{V}_{\mathsf{pe}}^{\star}\big)+\gamma\widehat{V}_{\mathsf{pe}}^{\star}(s)\notag\\
			&=(1-\gamma)\widehat{V}_{\mathsf{pe}}^{\star}(s)+\gamma\widehat{V}_{\mathsf{pe}}^{\star}(s) \notag\\
			&=\widehat{V}_{\mathsf{pe}}^{\star}(s), 
			\label{eq:Bellman-su-first-s-inf}
		\end{align}
		where the third identity makes use of our choice \eqref{eq:u-star-choice-inf-123} of $u^{\star}$ and \eqref{eq:defn-bsu-sa-inf}.

	\item Next, consider any $s'\neq s$ and any $a\in \cA$. We make the observation that
		\begin{align}
			&\max\Big\{ r^{s,u^{\star}}(s',a)-b^{s,u^{\star}}\big(s',a; \widehat{V}_{\mathsf{pe}}^{\star}\big)+\gamma\big\langle P^{s,u^{\star}}(\cdot\mymid s',a),\,\,\widehat{V}_{\mathsf{pe}}^{\star}\big\rangle, \,0\Big\} \notag\\
			& \qquad =\max\Big\{ r(s',a)-b\big(s',a; \widehat{V}_{\mathsf{pe}}^{\star}\big)+\gamma\big\langle\widehat{P}(\cdot\mymid s',a),\,\widehat{V}_{\mathsf{pe}}^{\star}\big\rangle, \,0\Big\}
			= \widehat{Q}_{\mathsf{pe}}^{\star}(s',a),
			\label{eq:Bellman-su-non-s-a-inf}
		\end{align}
		where the last relation holds since $\widehat{Q}_{\mathsf{pe}}^{\star}$ is a fixed point of $\Tpess$. 
\end{itemize}
%
Armed with \eqref{eq:Bellman-su-first-s-inf} and \eqref{eq:Bellman-su-non-s-a-inf}, 
we see that $\widehat{Q}^{\star}_{s,u^{\star}} = \Tpess^{s,u^{\star}}\big( \widehat{Q}^{\star}_{s,u^{\star}} \big)$  by taking 
\begin{align*}
	\max_{a\in \cA} \widehat{Q}_{s,u^{\star}}^{\star}(s,a) & =\widehat{V}_{\mathsf{pe}}^{\star}(s),\qquad &&\\
	\widehat{Q}_{s,u^{\star}}^{\star}(s',a) & =\widehat{Q}_{\mathsf{pe}}^{\star}(s',a)\qquad &&\text{for all }s'\neq s\text{ and }a\in\mathcal{A}.
\end{align*}
This readily confirms the existence of a fixed point of $\Tpess^{s,u^{\star}}$ whose corresponding value coincides with  $\widehat{V}_{\mathsf{pe}}^{\star}$.

\paragraph{Step 3: building an $\epsilon$-net.}
Consider any $(s,a)\in \cS\times \cA$ with $N(s,a)>0$. Construct a set $\mathcal{U}_{\text{cover}}$ as follows  
\begin{align}
	\mathcal{U}_{\text{cover}} \coloneqq \left\{ \frac{i}{N} \mid 1\leq i\leq N u_{\max} \right\},  
	\label{eq:defn-Ucover-infinite-Bernstein}
\end{align}
with $u_{\max} =  \min \Big\{ \frac{2\cb \log\frac{N }{(1-\gamma)\delta}}{(1-\gamma)N(s, a)} ,\, \frac{1}{1-\gamma}  \Big\} +\frac{5}{N} + 1 $. 
This can be viewed as the $\epsilon$-net \citep{vershynin2018high} of the range $[0, u_{\max}]\subseteq \big[0, \frac{2}{1-\gamma}\big]$ with $\epsilon = 1 / N$. 
Let us construct an auxiliary MDP $\widehat{\mathcal{M}}^{s, u}$ as in Step 1 for each $u\in \mathcal{U}_{\text{cover}}$. 
Repeating the argument in the proof of Lemma~\ref{lem:contraction} (see Section~\ref{sec:proof-lemma:contraction}), 
we can easily show that there exists a unique fixed point $\widehat{Q}^{\star}_{s,u}$ of $\widehat{\mathcal{M}}^{s, u}$, 
which also obeys
$0\leq \widehat{Q}^{\star}_{s,u}  \leq \frac{1}{1-\gamma} \cdot 1$. 
In what follows, we denote by $\widehat{V}^{\star}_{s,u}$ the corresponding value function of $\widehat{Q}^{\star}_{s,u}$.

Recognizing that $\widehat{\mathcal{M}}^{s, u}$ is statistically independent from $\widehat{P}_{s,a}$ for any $u\in \mathcal{U}_{\text{cover}}$ (by construction),  
we can apply Lemma~\ref{lem:Bernstein-infinite-proof} in conjunction with the union bound (over all $u\in \mathcal{U}_{\mathsf{cover}}$) to show that, 
with probability exceeding $1-\delta$, 
\begin{subequations}
\label{eq:Hoeffding-infinite-net-Bernstein-appendix}
\begin{align} 
\Big|\big(\widehat{P}_{s,a}-P_{s,a}\big)\widehat{V}_{s,u}^{\star}\Big| & \le\sqrt{\frac{48\log\frac{8N^{2}}{(1-\gamma)\delta}}{N(s,a)}\mathsf{Var}_{\widehat{P}_{s,a}}\big(\widehat{V}_{s,u}^{\star}\big)}+\frac{48\log\frac{8N^{2}}{(1-\gamma)\delta}}{(1-\gamma)N(s,a)},\\
\mathsf{Var}_{\widehat{P}_{s,a}}\big(\widehat{V}_{s,u}^{\star}\big) & \le2\mathsf{Var}_{P_{s,a}}\big(\widehat{V}_{s,u}^{\star}\big)+\frac{5\log\frac{8N^{2}}{(1-\gamma)\delta}}{3(1-\gamma)^{2}N(s,a)} 
\end{align}
\end{subequations}
hold simultaneously for all $u\in \mathcal{U}_{\text{cover}}$.  
Clearly, the total number of $(s,a)$ pairs with $N(s,a)>0$ cannot exceed $N$. Thus, taking the union bound over all these pairs yield that, 
with probability at least $1-\delta$, 
\begin{subequations}
\label{eq:Hoeffding-infinite-net-Bernstein-all}
\begin{align}
\Big|\big(\widehat{P}_{s,a}-P_{s,a}\big)\widehat{V}_{s,u}^{\star}\Big| & \le\sqrt{\frac{48\log\frac{8N^{3}}{(1-\gamma)\delta}}{N(s,a)}\mathsf{Var}_{\widehat{P}_{s,a}}\big(\widehat{V}_{s,u}^{\star}\big)}+\frac{48\log\frac{8N^{3}}{(1-\gamma)\delta}}{(1-\gamma)N(s,a)}, \label{eq:Hoeffding-infinite-net-Bernstein-all-1}\\
\mathsf{Var}_{\widehat{P}_{s,a}}\big(\widehat{V}_{s,u}^{\star}\big) & \le 2\mathsf{Var}_{P_{s,a}}\big(\widehat{V}_{s,u}^{\star}\big)+\frac{5\log\frac{8N^{3}}{(1-\gamma)\delta}}{3(1-\gamma)^{2}N(s,a)} 
\label{eq:Hoeffding-infinite-net-Bernstein-all-2}
\end{align}
\end{subequations}
hold simultaneously for all $(s,a,u)\in \cS\times \cA\times \mathcal{U}_{\text{cover}}$ obeying $N(s,a)>0$.

\paragraph{Step 4: a covering argument.} 
In this step, we shall work on the high-probability event \eqref{eq:Hoeffding-infinite-net-Bernstein-appendix} that holds simultaneously for all $u\in \mathcal{U}_{\text{cover}}$. 
Given that  $\widehat{V}^{\star}_{\mathsf{pe}}$ satisfies the trivial bound $0 \le \widehat{V}^{\star}_{\mathsf{pe}}(s) \le \frac{1}{1-\gamma}$ for all $s\in \cS$, 
one can find some $u_0 \in \mathcal{U}_{\text{cover}}$ such that $| u_0 - u^{\star} | \leq 1/N$,  
where we recall the choice of $u^{\star}$ in \eqref{eq:u-star-choice-inf-123}.  
From the definition of the MDP $\widehat{\mathcal{M}}^{s, u}$ and the operator  \eqref{eq:empirical-Bellman-infinite-Lemma8}, 
it is readily seen that
\[
	\big\| \Tpess^{s,u_0}(Q) - \Tpess^{s,u^{\star}}(Q) \big\|_{\infty} \leq \big| u_0 - u^{\star} \big| \leq \frac{1}{N}
\]
holds for any $Q\in \mathbb{R}^{SA}$.  Consequently, we can use $\gamma$-contraction of the operator to obtain  
\begin{align*}
\big\|\widehat{Q}_{s,u_{0}}^{\star}-\widehat{Q}_{s,u^{\star}}^{\star}\big\|_{\infty} & =\Big\|\Tpess^{s,u_{0}}\big(\widehat{Q}_{s,u_{0}}^{\star}\big)-\Tpess^{s,u^{\star}}\big(\widehat{Q}_{s,u^{\star}}^{\star}\big)\Big\|_{\infty}\\
 & \leq\Big\|\Tpess^{s,u^{\star}}\big(\widehat{Q}_{s,u_{0}}^{\star}\big)-\Tpess^{s,u^{\star}}\big(\widehat{Q}_{s,u^{\star}}^{\star}\big)\Big\|_{\infty}+\Big\|\Tpess^{s,u_{0}}\big(\widehat{Q}_{s,u_{0}}^{\star}\big)-\Tpess^{s,u^{\star}}\big(\widehat{Q}_{s,u_{0}}^{\star}\big)\Big\|_{\infty}\\
 & \leq\gamma\big\|\widehat{Q}_{s,u_{0}}^{\star}-\widehat{Q}_{s,u^{\star}}^{\star}\big\|_{\infty}+\frac{1}{N},
\end{align*}
which implies that
\[
	\big\|\widehat{Q}_{s,u_{0}}^{\star}-\widehat{Q}_{s,u^{\star}}^{\star}\big\|_{\infty} \leq \frac{1}{(1-\gamma)N} 
\]
%
%
%
and therefore
\[
	\big\|\widehat{V}_{s,u_{0}}^{\star}-\widehat{V}_{s,u^{\star}}^{\star}\big\|_{\infty} 
	\leq \big\|\widehat{Q}_{s,u_{0}}^{\star}-\widehat{Q}_{s,u^{\star}}^{\star}\big\|_{\infty} \leq \frac{1}{(1-\gamma)N}. 
\]
This in turn allows us to demonstrate that
\begin{align*}
	& \mathsf{Var}_{P_{s,a}}\big(\widehat{V}_{s,u_{0}}^{\star}\big)-\mathsf{Var}_{P_{s,a}}\big(\widehat{V}_{s,u^{\star}}^{\star}\big) \nonumber \\
 & \qquad = P_{s,a}\Big(\big(\widehat{V}_{s,u_{0}}^{\star} - P_{s,a}\widehat{V}_{s,u_{0}}^{\star}\big) \circ\big(\widehat{V}_{s,u_{0}}^{\star} - P_{s,a}\widehat{V}_{s,u_{0}}^{\star}\big) - \big(\widehat{V}_{s,u^{\star}}^{\star} - P_{s,a}\widehat{V}_{s,u^{\star}}^{\star}\big) \circ\big(\widehat{V}_{s,u^{\star}}^{\star} - P_{s,a}\widehat{V}_{s,u^{\star}}^{\star}\big)\Big)\nonumber \\
 & \qquad \leq P_{s,a}\Big(\big(\widehat{V}_{s,u_{0}}^{\star} - P_{s,a}\widehat{V}_{s,u^{\star}}^{\star}\big) \circ\big(\widehat{V}_{s,u_{0}}^{\star} - P_{s,a}\widehat{V}_{s,u^{\star}}^{\star}\big) - \big(\widehat{V}_{s,u^{\star}}^{\star} - P_{s,a}\widehat{V}_{s,u^{\star}}^{\star}\big) \circ\big(\widehat{V}_{s,u^{\star}}^{\star} - P_{s,a}\widehat{V}_{s,u^{\star}}^{\star}\big)\Big) \nonumber \\
 & \qquad \leq P_{s,a}\Big(\big(\widehat{V}_{s,u_{0}}^{\star} - P_{s,a}\widehat{V}_{s,u^{\star}}^{\star} + \widehat{V}_{s,u^{\star}}^{\star} - P_{s,a}\widehat{V}_{s,u^{\star}}^{\star}\big) \circ\big(\widehat{V}_{s,u_{0}}^{\star} - \widehat{V}_{s,u^{\star}}^{\star}\big)\Big) \nonumber \\
 & \qquad \leq \frac{2}{1-\gamma}\Big|P_{s,a}\big(\widehat{V}_{s,u_{0}}^{\star}-\widehat{V}_{s,u^{\star}}^{\star}\big)\Big|\leq\frac{2}{1-\gamma}\big\|\widehat{V}_{s,u_{0}}^{\star}-\widehat{V}_{s,u^{\star}}^{\star}\big\|_{\infty}\leq\frac{2}{(1-\gamma)^{2}N} ,
\end{align*}
where the third line comes from the fact that $\mathbb{E}[X] = \arg\min_c \mathbb{E}[(X-c)^2]$, and the last line relies on the property $0 \le \widehat{V}_{s,u_{0}}^{\star}, \widehat{V}_{s,u^{\star}}^{\star} \leq \frac{1}{1-\gamma}$. 
In addition, by swapping $\widehat{V}_{s,u_{0}}^{\star}$ and $\widehat{V}_{s,u^{\star}}^{\star}$, we can derive
\begin{align*}
\mathsf{Var}_{P_{s,a}}\big(\widehat{V}_{s,u^{\star}}^{\star}\big) - \mathsf{Var}_{P_{s,a}}\big(\widehat{V}_{s,u_{0}}^{\star}\big) \leq \frac{2}{(1-\gamma)^{2}N},
\end{align*}
and then
\begin{align}
\Big|\mathsf{Var}_{P_{s,a}}\big(\widehat{V}_{s,u_{0}}^{\star}\big)-\mathsf{Var}_{P_{s,a}}\big(\widehat{V}_{s,u^{\star}}^{\star}\big)\Big| \leq\frac{2}{(1-\gamma)^{2}N} .
	\label{eq:Var-perturbation-V-su-inf}
\end{align}
Clearly, this bound \eqref{eq:Var-perturbation-V-su-inf} continues to be valid if we replace $P_{s,a}$ with $\widehat{P}_{s,a}$. 
%

With the above perturbation bounds in mind, we can invoke the triangle inequality and \eqref{eq:Hoeffding-infinite-net-Bernstein-all-1} to reach
\begin{align}
\Big|\big(\widehat{P}_{s,a}-P_{s,a}\big)\widehat{V}_{\mathsf{pe}}^{\star}\Big| & =\Big|\big(\widehat{P}_{s,a}-P_{s,a}\big)\widehat{V}_{s,u^{\star}}^{\star}\Big|\leq\Big|\big(\widehat{P}_{s,a}-P_{s,a}\big)\widehat{V}_{s,u_{0}}^{\star}\Big|+\Big|\big(\widehat{P}_{s,a}-P_{s,a}\big)\big(\widehat{V}_{s,u^{\star}}^{\star}-\widehat{V}_{s,u_{0}}^{\star}\big)\Big| \notag\\
 & \leq\Big|\big(\widehat{P}_{s,a}-P_{s,a}\big)\widehat{V}_{s,u_{0}}^{\star}\Big|+\frac{2}{N(1-\gamma)} \notag\\
 & \le\sqrt{\frac{48\log\frac{8N^{3}}{(1-\gamma)\delta}}{N(s,a)}\mathsf{Var}_{\widehat{P}_{s,a}}\big(\widehat{V}_{s,u_{0}}^{\star}\big)}+\frac{48\log\frac{8N^{3}}{(1-\gamma)\delta}}{(1-\gamma)N(s,a)}+\frac{2}{N(1-\gamma)} \notag\\
 & \leq\sqrt{\frac{48\log\frac{8N^{3}}{(1-\gamma)\delta}}{N(s,a)}\mathsf{Var}_{\widehat{P}_{s,a}}\big(\widehat{V}_{s,u^{\star}}^{\star}\big)}+\sqrt{\frac{96\log\frac{8N^{3}}{(1-\gamma)\delta}}{(1-\gamma)^{2}N(s,a)}}+\frac{48\log\frac{8N^{3}}{(1-\gamma)\delta}}{(1-\gamma)N(s,a)}+\frac{2}{N(1-\gamma)} \notag\\
 & \leq\sqrt{\frac{48\log\frac{8N^{3}}{(1-\gamma)\delta}}{N(s,a)}\mathsf{Var}_{\widehat{P}_{s,a}}\big(\widehat{V}_{s,u^{\star}}^{\star}\big)}+\frac{60\log\frac{8N^{3}}{(1-\gamma)\delta}}{(1-\gamma)N(s,a)},
	\label{eq:Phat-sa-Vsa-Vpe-bound-proof-123}
\end{align}
where the second line holds since 
\[
\Big|\big(\widehat{P}_{s,a}-P_{s,a}\big)\big(\widehat{V}_{s,u^{\star}}^{\star}-\widehat{V}_{s,u_{0}}^{\star}\big)\Big|\leq\big(\big\|\widehat{P}_{s,a}\big\|_{1}+\big\| P_{s,a}\big\|_{1}\big)\big\|\widehat{V}_{s,u^{\star}}^{\star}-\widehat{V}_{s,u_{0}}^{\star}\big\|_{\infty}\leq\frac{2}{N(1-\gamma)},
\]
the penultimate line is valid due to \eqref{eq:Var-perturbation-V-su-inf}, and the last line holds true under the conditions that $T\geq N(s,a)$ and that $T$ is sufficiently large. Moreover, apply \eqref{eq:Hoeffding-infinite-net-Bernstein-all-2} and the triangle inequality to arrive at
\begin{align}
\mathsf{Var}_{\widehat{P}_{s,a}}\big(\widehat{V}_{\mathsf{pe}}^{\star}\big) & =\mathsf{Var}_{\widehat{P}_{s,a}}\big(\widehat{V}_{s,u^{\star}}^{\star}\big)\leq\mathsf{Var}_{\widehat{P}_{s,a}}\big(\widehat{V}_{s,u_{0}}^{\star}\big)+\Big|\mathsf{Var}_{\widehat{P}_{s,a}}\big(\widehat{V}_{s,u^{\star}}^{\star}\big)-\mathsf{Var}_{\widehat{P}_{s,a}}\big(\widehat{V}_{s,u_{0}}^{\star}\big)\Big| \notag\\
 & \overset{(\mathrm{i})}{\leq}2\mathsf{Var}_{P_{s,a}}\big(\widehat{V}_{s,u_{0}}^{\star}\big)+\frac{5\log\frac{8N^{3}}{(1-\gamma)\delta}}{3(1-\gamma)^{2}N(s,a)}+\frac{2}{(1-\gamma)^{2}N} \notag\\
 & \leq2\mathsf{Var}_{P_{s,a}}\big(\widehat{V}_{s,u^{\star}}^{\star}\big)+2\Big|\mathsf{Var}_{P_{s,a}}\big(\widehat{V}_{s,u^{\star}}^{\star}\big)-\mathsf{Var}_{P_{s,a}}\big(\widehat{V}_{s,u_{0}}^{\star}\big)\Big|+\frac{5\log\frac{8N^{3}}{(1-\gamma)\delta}}{3(1-\gamma)^{2}N(s,a)}+\frac{2}{(1-\gamma)^{2}N} \notag\\
 & \overset{(\mathrm{ii})}{\leq}2\mathsf{Var}_{P_{s,a}}\big(\widehat{V}_{\mathsf{pe}}^{\star}\big)+\frac{6}{(1-\gamma)^{2}N}+\frac{5\log\frac{8N^{3}}{(1-\gamma)\delta}}{3(1-\gamma)^{2}N(s,a)} \notag\\
 & \leq2\mathsf{Var}_{P_{s,a}}\big(\widehat{V}_{\mathsf{pe}}^{\star}\big)+\frac{23\log\frac{8N^{3}}{(1-\gamma)\delta}}{3(1-\gamma)^{2}N(s,a)},
	\label{eq:Var-Vpe-hat-diff-proof-135}
\end{align}
where (i) arise from \eqref{eq:Hoeffding-infinite-net-Bernstein-all-2} and \eqref{eq:Var-perturbation-V-su-inf}, (ii) follows from \eqref{eq:Var-perturbation-V-su-inf}, and the last line holds true since $N\geq N(s,a)$.

\paragraph{Step 5: extending the bounds to $\widetilde{V}$.} 
Consider any $\widetilde{V}$ obeying $\|\widetilde{V} - \widehat{V}_{\mathsf{pe}}^{\star} \|_{\infty}\leq \frac{1}{N}$ and $\|\widetilde{V}\|_{\infty} \leq \frac{1}{1-\gamma}$. 
Invoke \eqref{eq:Phat-sa-Vsa-Vpe-bound-proof-123} and the triangle inequality to arrive at 
\begin{align}
\Big|\big(\widehat{P}_{s,a}-P_{s,a}\big)\widetilde{V}\Big| & \leq\Big|\big(\widehat{P}_{s,a}-P_{s,a}\big)\widehat{V}_{\mathsf{pe}}^{\star}\Big|+\Big|\big(\widehat{P}_{s,a}-P_{s,a}\big)\big(\widehat{V}_{\mathsf{pe}}^{\star}-\widetilde{V}\big)\Big| \notag\\
 & \leq\sqrt{\frac{48\log\frac{8N^{3}}{(1-\gamma)\delta}}{N(s,a)}\mathsf{Var}_{\widehat{P}_{s,a}}\big(\widehat{V}_{s,u^{\star}}^{\star}\big)}+\frac{60\log\frac{2N}{(1-\gamma)\delta}}{(1-\gamma)N(s,a)}+\frac{2}{N}, \notag\\
 & \leq 12\sqrt{\frac{\log\frac{2N}{(1-\gamma)\delta}}{N(s,a)}\mathsf{Var}_{\widehat{P}_{s,a}}\big(\widehat{V}_{s,u^{\star}}^{\star}\big)}+\frac{62\log\frac{2N}{(1-\gamma)\delta}}{(1-\gamma)N(s,a)} \notag\\
 & = 12\sqrt{\frac{\log\frac{2N}{(1-\gamma)\delta}}{N(s,a)}\mathsf{Var}_{\widehat{P}_{s,a}}\big(\widehat{V}_{\mathsf{pe}}^{\star}\big)}+\frac{62\log\frac{2N}{(1-\gamma)\delta}}{(1-\gamma)N(s,a)}, 
	\label{eq:Psa-hat-perturb-13579}
\end{align}
where the penultimate inequality relies on $N\geq N(s,a)$, and the second line holds since 
\[
\Big|\big(\widehat{P}_{s,a}-P_{s,a}\big)\big(\widehat{V}_{\mathsf{pe}}^{\star}-\widetilde{V}\big)\Big|\leq\big(\big\|\widehat{P}_{s,a}\big\|_{1}+\big\| P_{s,a}\big\|_{1}\big)\big\|\widehat{V}_{\mathsf{pe}}^{\star}-\widetilde{V}\big\|_{\infty}\leq\frac{2}{N}.
\]
Given that $\big\|\widetilde{V}-\widehat{V}_{\mathsf{pe}}^{\star} \big\|_{\infty}\leq 1/N$, 
we can repeat the argument for \eqref{eq:Var-perturbation-V-su-inf} allows one to demonstrate that
\[
\Big|\mathsf{Var}_{\widehat{P}_{s,a}}\big(\widehat{V}_{\mathsf{pe}}^{\star}\big)-\mathsf{Var}_{\widehat{P}_{s,a}}\big(\widetilde{V}\big)\Big|\leq\frac{2}{(1-\gamma)^{2}N}
\]
which taken together with \eqref{eq:Psa-hat-perturb-13579} and the basic inequality $\sqrt{x+y}\leq \sqrt{x}+\sqrt{y}$ gives
\begin{align*}
\Big|\big(\widehat{P}_{s,a}-P_{s,a}\big)\widetilde{V}\Big| & \leq12\sqrt{\frac{\log\frac{2N}{(1-\gamma)\delta}}{N(s,a)}\mathsf{Var}_{\widehat{P}_{s,a}}\big(\widetilde{V}\big)}+12\sqrt{\frac{\log\frac{2N}{(1-\gamma)\delta}}{N(s,a)}\cdot\frac{2}{(1-\gamma)^{2}N}}+\frac{62\log\frac{2N}{(1-\gamma)\delta}}{(1-\gamma)N(s,a)}\\
 & \leq12\sqrt{\frac{\log\frac{2N}{(1-\gamma)\delta}}{N(s,a)}\mathsf{Var}_{\widehat{P}_{s,a}}\big(\widetilde{V}\big)}+\frac{74\log\frac{2N}{(1-\gamma)\delta}}{(1-\gamma)N(s,a)}. 
\end{align*}
Additionally, repeating the argument for \eqref{eq:Var-Vpe-hat-diff-proof-135} leads to another desired inequality: 
\begin{align*}
\mathsf{Var}_{\widehat{P}_{s,a}}\big(\widetilde{V}\big) & \leq2\mathsf{Var}_{P_{s,a}}\big(\widetilde{V}\big)+\frac{6}{(1-\gamma)N}+\frac{23\log\frac{8N^{3}}{(1-\gamma)\delta}}{3(1-\gamma)^{2}N(s,a)}\\
 & \leq2\mathsf{Var}_{P_{s,a}}\big(\widetilde{V}\big)+\frac{41\log\frac{2N}{(1-\gamma)\delta}}{(1-\gamma)^{2}N(s,a)} .
\end{align*}

\subsubsection{Proof of Lemma~\ref{lem:Bernstein-infinite-proof}}
\label{sec:lem:Bernstein-infinite-proof}

In this proof, we shalle often use $\mathsf{Var}_{s, a}$ to abbreviate $\mathsf{Var}_{P_{s, a}}$ for notational simplicity.
Before proceeding, let us define the following vector
\begin{align}
	\overline{V} = V - (P_{s, a}V) \one , 
\end{align}
with $\one$ denoting the all-one vector. It is clearly seen that
\begin{align}
P_{s,a}\big(\overline{V}\circ\overline{V}\big) & =P_{s,a}\big(V\circ V\big)-\big(P_{s,a}V\big)^{2}=\mathsf{Var}_{s,a}(V).
\end{align}
In addition, we make note of the following basic facts that will prove useful:
\begin{subequations}
\label{eq:V-Vbar-UB-infinite}
\begin{align}
	\| V \|_{\infty} &\leq \frac{1}{1-\gamma}, \qquad \| \overline{V} \|_{\infty} \leq \frac{1}{1-\gamma} ,
	\qquad  \| \overline{V} \circ \overline{V} \|_{\infty} \leq \| \overline{V} \|_{\infty}^2 \leq H^2 , \\
	\mathsf{Var}_{s, a}\big(\overline{V} \circ \overline{V}\big) &\le P_{s, a}\big(\overline{V} \circ \overline{V} \circ \overline{V} \circ \overline{V}\big) 
	\le \frac{1}{(1-\gamma)^2} P_{s, a}\big(\overline{V} \circ \overline{V}\big) 
	= \frac{1}{(1-\gamma)^2} \mathsf{Var}_{s, a}\big(V\big). 
\end{align}
\end{subequations}

\paragraph{Proof of inequality~\eqref{eq:bonus-Bernstein-infinite-proof}.}
If $0< N(s,a) < 48 \log \frac{N}{\delta}$, then we can immediately see that 
\begin{align}
	\left|\big(\widehat{P}_{s, a} - P_{s, a}\big)V\right| 
	\leq \|V\|_{\infty} \leq \frac{1}{1-\gamma} \leq \frac{48 \log\frac{N}{\delta}}{(1-\gamma) N(s,a)},
\end{align}
and hence the claim~\eqref{eq:bonus-Bernstein-infinite-proof} is valid. As a result, it suffices to focus on the case where
\begin{align}
	N(s,a) \geq  48 \log \frac{N}{\delta}.
	\label{eq:Nhsa-lower-bound-infinite}
\end{align}
%

%
Note that the total number of pairs $(s,a)$ with nonzero $N(s,a)$ cannot exceed $N$. 
Akin to \eqref{eq:deviation-Naux-h-s}, taking the Bernstein inequality together with \eqref{eq:V-Vbar-UB-infinite} and invoking the union bound,  we can demonstrate that with probability at least $1-4\delta$, 
\begin{subequations}
	\label{eq:Bernstein-infinite}
\begin{align}
	\left|\big(\widehat{P}_{s, a} - P_{s, a}\big)V\right| 
	&\le \sqrt{\frac{4\mathsf{Var}_{s, a}(V)\log\frac{N}{\delta}}{N(s, a)}} + \frac{2\|V\|_{\infty}\log\frac{N}{\delta}}{3N(s, a)} \notag\\
	&\le \sqrt{\frac{4\mathsf{Var}_{s, a}(V)\log\frac{N}{\delta}}{N(s, a)}} + \frac{2\log\frac{N}{\delta}}{3(1-\gamma)N(s, a)}  \label{eq:Bernstein-infinite-V}\\
	\left|\big(P_{s, a} - \widehat{P}_{s, a}\big)\big(\overline{V} \circ \overline{V}\big)\right| 
	&\le \sqrt{\frac{4\mathsf{Var}_{s, a}\big(\overline{V} \circ \overline{V}\big)\log\frac{N}{\delta}}{N(s, a)}} + \frac{2\| \overline{V} \circ \overline{V} \|_{\infty} \log\frac{N}{\delta}}{3N(s, a)} \nonumber\\
	&\le \sqrt{\frac{4\mathsf{Var}_{s, a}(V)\log\frac{N}{\delta}}{(1-\gamma)^2N(s, a)}} + \frac{2\log\frac{N}{\delta}}{3(1-\gamma)^2N(s, a)}	
	\label{eq:Bernstein-infinite-V2}
\end{align}
\end{subequations}
hold simultaneously over all $(s,a)$ with  $N(s,a)>0$. 
Note, however, that the Bernstein bounds in \eqref{eq:Bernstein-infinite} involve the variance $\mathsf{Var}_{s,a}(V)$; 
we still need to connect $\mathsf{Var}_{s,a}(V)$ with its empirical estimate $\mathsf{Var}_{\widehat{P}_{s,a}}(V)$.

In the sequel, let us look at two cases separately. 
\begin{itemize}
\item \emph{Case 1: $\mathsf{Var}_{s, a}(V) \le \frac{9\log\frac{N}{\delta}}{(1-\gamma)^2N(s, a)}$.} 
	In this case, our bound \eqref{eq:Bernstein-infinite-V} immediately leads to
\begin{align}
	\left|\big(\widehat{P}_{s, a} - P_{s, a}\big)V\right| \le \frac{7\log\frac{N}{\delta}}{(1-\gamma)N(s, a)}.
\end{align}

\item \emph{Case 2: $\mathsf{Var}_{s, a}(V) > \frac{9\log\frac{N}{\delta}}{(1-\gamma)^2N(s, a)}$.} 
We first single out the following useful identity:
\begin{align}
	& \widehat{P}_{s,a}\big(\overline{V}\circ\overline{V}\big)- \mathsf{Var}_{\widehat{P}_{s,a}}(V) = 
	\widehat{P}_{s,a}\big(\overline{V}\circ\overline{V}\big)-\left[\widehat{P}_{s,a}\big(V\circ V\big)-\big(\widehat{P}_{s,a}V\big)^{2}\right] \notag\\
	& \qquad =\widehat{P}_{s,a}\big(V\circ V\big)-2\big(\widehat{P}_{s,a}V\big)\big(P_{s,a}V\big)+\big(P_{s,a}V\big)^{2}-\left[\widehat{P}_{s,a}\big(V\circ V\big)-\big(\widehat{P}_{s,a}V\big)^{2}\right] \notag\\
 	&\qquad =\big|\big(\widehat{P}_{s,a}-P_{s,a}\big)V\big|^{2}. 
	\label{eq:var-empirical-infinite}
\end{align}
Combining~\eqref{eq:var-empirical-infinite} with \eqref{eq:Bernstein-infinite-V2} then implies that, with probability exceeding $1-4\delta$, 
\begin{align}
 & \mathsf{Var}_{s,a}(V)=P_{s,a}\big(\overline{V}\circ\overline{V}\big)=\big(P_{s,a}-\widehat{P}_{s,a}\big)\big(\overline{V}\circ\overline{V}\big)+\widehat{P}_{s,a}\big(\overline{V}\circ\overline{V}\big)\nonumber\\
	& =\big(P_{s,a}-\widehat{P}_{s,a}\big)\big(\overline{V}\circ\overline{V}\big)+\left\{ \big|\big(\widehat{P}_{s,a}-P_{s,a}\big)V\big|^{2}+\mathsf{Var}_{\widehat{P}_{s,a}}(V)\right\} \label{eq:Var-hsa-Var-hat-connection}\\
 & \le\sqrt{\frac{4\log\frac{N}{\delta}}{(1-\gamma)^2N(s,a)}}\sqrt{\mathsf{Var}_{s,a}(V)}+\big|\big(\widehat{P}_{s,a}-P_{s,a}\big)V\big|^{2}+\mathsf{Var}_{\widehat{P}_{s,a}}(V) +\frac{2\log\frac{N}{\delta}}{3(1-\gamma)^2N(s,a)} \notag\\
 & \leq\frac{2}{3}\mathsf{Var}_{s,a}(V)+\big|\big(\widehat{P}_{s,a}-P_{s,a}\big)V\big|^{2}+\mathsf{Var}_{\widehat{P}_{s,a}}(V) +\frac{2\log\frac{N}{\delta}}{3(1-\gamma)^2N(s,a)} , 
\end{align}
where the second line arises from the identity \eqref{eq:var-empirical-infinite}, 
the penultimate inequality results from \eqref{eq:Bernstein-infinite-V2},
and the last inequality holds true due to the assumption $\mathsf{Var}_{s, a} (V) > \frac{9\log\frac{N}{\delta}}{(1-\gamma)^2N(s, a)}$ in this case. 
Rearranging terms of the above inequality, we are left with
\begin{align*}
\mathsf{Var}_{s,a}(V) & \leq 3\big|\big(\widehat{P}_{s,a}-P_{s,a}\big)V\big|^{2}+3\mathsf{Var}_{\widehat{P}_{s,a}}(V) +\frac{2\log\frac{N}{\delta}}{(1-\gamma)^2N(s,a)} 
\end{align*}
Taking this upper bound on $\mathsf{Var}_{s,a}(V)$ collectively with \eqref{eq:Bernstein-infinite-V} and using a little algebra lead to
\begin{align}
\left|\big(\widehat{P}_{s,a}-P_{s,a}\big)V\right| & \leq\sqrt{\frac{12\log\frac{N}{\delta}}{N(s,a)}}\big|\big(\widehat{P}_{s,a}-P_{s,a}\big)V\big|
	+ \sqrt{\frac{12\mathsf{Var}_{\widehat{P}_{s,a}}(V) \log\frac{N}{\delta}}{N(s,a)}} 
 +\frac{5\log\frac{N}{\delta}}{(1-\gamma)N(s,a)}
	\label{eq:Bernstein-bound-intermediate-135}
\end{align}
with probability at least $1-4\delta$.  
When $N(s,a) \geq 48\log\frac{N}{\delta}$ (cf.~\eqref{eq:Nhsa-lower-bound-infinite}), 
one has $\sqrt{\frac{12\log\frac{N}{\delta}}{N(s,a)}}\leq 1/2$. 
Substituting this into \eqref{eq:Bernstein-bound-intermediate-135} and rearranging terms, we arrive at
\begin{align*}
\left|\big(\widehat{P}_{s,a}-P_{s,a}\big)V\right| & 
 %
\leq \sqrt{\frac{48\mathsf{Var}_{\widehat{P}_{s,a}}(V)\log\frac{N}{\delta}}{N(s,a)}}+\frac{10\log\frac{N}{\delta}}{(1-\gamma)N(s,a)}
\end{align*}

with probability at least $1-4\delta$.

\end{itemize}

Putting the above two cases together establishes the advertised bound \eqref{eq:bonus-Bernstein-infinite-proof}.

\paragraph{Proof of inequality~\eqref{eq:empirical-var-infinite-proof}.}
It follows from \eqref{eq:Var-hsa-Var-hat-connection} and \eqref{eq:Bernstein-infinite-V} that with probability at least $1-4\delta$, 
\begin{align*}
\mathsf{Var}_{s,a} (V) & \ge -\left|\big(P_{s,a}-\widehat{P}_{s,a}\big)\big(\overline{V}\circ\overline{V}\big)\right|+\mathsf{Var}_{\widehat{P}_{s,a}}(V) \nonumber\\
 & \ge-\sqrt{\frac{4\mathsf{Var}_{s,a}\big(V\big)\log\frac{N}{\delta}}{(1-\gamma)^2N(s,a)}}-\frac{2\log\frac{N}{\delta}}{3(1-\gamma)^2N(s,a)}+\mathsf{Var}_{\widehat{P}_{s,a}}(V) ,
\end{align*}
or equivalently,
\begin{align*}
\mathsf{Var}_{\widehat{P}_{s,a}}(V) & \leq\mathsf{Var}_{s,a}(V)+ 2\sqrt{\frac{\mathsf{Var}_{s,a}\big(V\big)\log\frac{N}{\delta}}{(1-\gamma)^2N(s,a)}}+\frac{2\log\frac{N}{\delta}}{3(1-\gamma)^2N(s,a)}.
\end{align*}
Invoke the elementary inequality $2xy\leq x^2+y^2$ to establish the claimed bound: 
\begin{align*}
\mathsf{Var}_{\widehat{P}_{s,a}}(V) & \leq\mathsf{Var}_{s,a}(V)+\left(\mathsf{Var}_{s,a}\big(V\big)+\frac{\log\frac{N}{\delta}}{(1-\gamma)^2N(s,a)}\right)+\frac{2\log\frac{N}{\delta}}{3(1-\gamma)^2N(s,a)}\\
 & =2\mathsf{Var}_{s,a}(V)+\frac{5\log\frac{N}{\delta}}{3(1-\gamma)^2N(s,a)} .
\end{align*}

\section{Proof of auxiliary lemmas: episodic finite-horizon MDPs}

\subsection{Proof of Lemma~\ref{lemma:Ntrim-LB}}
\label{sec:proof-lemma:Ntrim-LB}
(a) Let us begin by proving the claim \eqref{eq:samples-finite-UB}.  Recall from our construction that $\Daux$ is composed of the second half of the sample trajectories, and hence for each $s\in \cS$ and $1\leq h\leq H$, 
$$\Naux_h(s) = \sum_{k=K/2+1}^{K} \mathds{1}\big\{ s_h^k = s \big\} $$ 
can be viewed as the sum of $K/2$ independent Bernoulli random variables, each with mean $\myrho_h(s)$.  
According to the union bound and the Bernstein inequality, we obtain
\begin{align*}
	\mathbb{P}\left\{ \exists (s,h)\in\mathcal{S}\times [H]:\left|\Naux_{h}(s)-\frac{K}{2}\myrho_{h}(s)\right|\geq\tau\right\}  & \leq\sum_{s\in\mathcal{S}, h\in [H]}\mathbb{P}\left\{ \left|\Naux_{h}(s)-\frac{K}{2}\myrho_{h}(s)\right|\geq\tau\right\} \\
& \leq2SH\exp\left(-\frac{\tau^{2}/2}{ v_{s,h} + \tau/3}\right) 
\end{align*}
for any $\tau\geq 0$, where 
$$
	v_{s,h} \coloneqq \frac{K}{2} \mathsf{Var}\big( \mathds{1} \{ s_h^t = s \} \big) 
	= \frac{ K\myrho_h(s) \big(1-\myrho_h(s)\big) }{2} \leq \frac{ K\myrho_h(s) }{2}.
$$
A little algebra then yields that with probability at least $1-2\delta$, one has
\begin{align}
	\left|\Naux_{h}(s)-\frac{K}{2}\myrho_{h}(s)\right| 
	& \leq\sqrt{4v_{s,h}\log\frac{HS}{\delta}}+\frac{2}{3}\log\frac{HS}{\delta} 
  	\leq \sqrt{2K\myrho_{h}(s)\log\frac{HS}{\delta}}+\log\frac{HS}{\delta}
	\label{eq:deviation-Naux-h-s-finite}
\end{align}
simultaneously for all $s\in \cS$ and all $1\leq h\leq H$. 
The same argument also reveals that with probability exceeding $1-2\delta$, 
\begin{align}
	\left|\Nmain_{h}(s)-\frac{K}{2}\myrho_{h}(s)\right|  
  	\leq \sqrt{2K\myrho_{h}(s)\log\frac{HS}{\delta}}+\log\frac{HS}{\delta}
	\label{eq:deviation-Nmain-h-s}
\end{align}
holds simultaneously for all $s\in \cS$ and all $1\leq h\leq H$. 
Combine \eqref{eq:deviation-Naux-h-s-finite} and \eqref{eq:deviation-Nmain-h-s} to show that
\begin{align}
	\left|\Nmain_{h}(s)-\Naux_{h}(s)\right|  
  	\leq 2 \sqrt{2K\myrho_{h}(s)\log\frac{HS}{\delta}}+ 2\log\frac{HS}{\delta}
	\label{eq:deviation-Nmain-Naux-h-s}
\end{align}
for all $s\in \cS$ and all $1\leq h\leq H$.

To establish the claimed result \eqref{eq:samples-finite-UB}, we divide into two cases. 
\begin{itemize}
	\item {\em Case 1: $\Naux_h(s) \leq 100 \log \frac{HS}{\delta}$.} By construction, it is easily seen that
\begin{align}
	\Ntrim_h(s) = \max\left\{\Naux_h(s) - 10 \sqrt{\Naux_h(s)\log\frac{HS}{\delta}},\, 0\right\} = 0 \leq \Nmain_h(s). 
	\label{eq:Ntrim-Nmain-0-trivial}
\end{align}

	\item {\em Case 2: $\Naux_h(s) > 100 \log \frac{HS}{\delta}$.} In this case, invoking \eqref{eq:deviation-Naux-h-s-finite} reveals that
\[
\frac{K}{2}\myrho_{h}(s)+\sqrt{2K\myrho_{h}(s)\log\frac{HS}{\delta}}+\log\frac{HS}{\delta}\geq\Naux_{h}(s)> 100\log\frac{HS}{\delta},
\]
and hence one necessarily has 
\begin{align}
	K\myrho_{h}(s) \geq (9\sqrt{2})^2 \log \frac{HS}{\delta}  \geq 100 \log \frac{HS}{\delta}. \label{eq:K-rhoh-UB-finite}
\end{align}
In turn, this property \eqref{eq:K-rhoh-UB-finite} taken collectively with \eqref{eq:deviation-Naux-h-s} ensures that
\begin{align}
	\Naux_h(s) \geq \frac{K}{2}\myrho_{h}(s)-\sqrt{2K\myrho_{h}(s)\log\frac{HS}{\delta}}-\log\frac{HS}{\delta}\geq\frac{K}{4}\myrho_{h}(s) .
	\label{eq:K-rhoh-UB-finite-2}
\end{align}
Therefore, in the case with $\Naux_h(s) > 100 \log \frac{HS}{\delta}$, 
we can demonstrate that
\begin{align}
	\Ntrim_h(s) &= \max\left\{\Naux_h(s) - 10\sqrt{\Naux_h(s)\log\frac{HS}{\delta}},\, 0\right\}  
	= \Naux_h(s) - 10\sqrt{\Naux_h(s)\log\frac{HS}{\delta}} \notag \\
	& \overset{\mathrm{(i)}}{\leq} \Naux_h(s) - 5\sqrt{K\myrho_{h}(s) \log\frac{HS}{\delta}}
	 \overset{\mathrm{(ii)}}{\leq} \Naux_h(s) - \left\{ 2\sqrt{2K\myrho_{h}(s) \log\frac{HS}{\delta}} + 2\log\frac{HS}{\delta} \right\} \notag\\
	& \overset{\mathrm{(iii)}}{\leq} \Nmain_h(s),
	\label{eq:Ntrim-Nmain-1}
\end{align}
where (i) comes from Condition~\eqref{eq:K-rhoh-UB-finite-2}, (ii) is valid under the condition \eqref{eq:K-rhoh-UB-finite},  
and (iii) holds true with probability at least $1-2\delta$ due to the inequality \eqref{eq:deviation-Nmain-Naux-h-s}. 
\end{itemize}
Putting the above two cases together establishes the claim \eqref{eq:samples-finite-UB}.

\bigskip
\noindent 
(b) We now turn to the second claim \eqref{eq:samples-finite-LB}.  
Towards this, we first claim that the following bound holds simultaneously for all $(s,a,h)\in \cS\times \cA\times [H]$ with probability exceeding $1-2\delta$: 
\begin{align}
	\Ntrim_{h}(s,a)\geq
	\Ntrim_{h}(s)\pi_{h}^{\mathsf{b}}(a\mymid s)-\sqrt{4\Ntrim_{h}(s)\pi_{h}^{\mathsf{b}}(a\mymid s)\log\frac{KH}{\delta}}-\log\frac{KH}{\delta} .
	\label{eq:Ntrim-hsa-LB-finite-123}
\end{align}
Let us take this claim as given for the moment, and return to establish it towards the end of this section.  
We shall discuss the following two cases separately.

%
%

%
\begin{itemize}
	\item If $K \myrho_h(s,a) = K \myrho_h(s) \pi_{h}^{\mathsf{b}}(a\mymid s) > 1600 \log \frac{KH}{\delta}$, 
	 then it follows from \eqref{eq:K-rhoh-UB-finite-2} (with slight modification) that
\begin{align}
	\Naux_h(s) \geq \frac{K}{4}\myrho_{h}(s) \geq 400 \log \frac{KH}{\delta}. 
	\label{eq:Krho-LB-Naux-LB-finite}
\end{align}
This property together with the definition of $\Ntrim_h(s)$ in turn allows us to derive
\begin{align*}
\Ntrim_{h}(s) &  \geq \Naux_{h}(s)-10\sqrt{\Naux_{h}(s)\log\frac{KH}{\delta}}\geq\frac{K}{4}\myrho_{h}(s)-10\sqrt{\frac{K}{4}\myrho_{h}(s)\log\frac{KH}{\delta}}\\
 & \geq\frac{K}{8}\myrho_{h}(s),
\end{align*}
and as a result,
\begin{align*}
	\Ntrim_{h}(s) \pi_h^{\mathsf{b}}(a\mymid s) &\geq \frac{K}{8}\myrho_{h}(s) \pi_h^{\mathsf{b}}(a\mymid s) = \frac{K}{8}\myrho_{h}(s,a) 
	\geq 200 \log \frac{KH}{\delta},
\end{align*}

where the last inequality arises from the assumption of this case.  
Taking this lower bound with \eqref{eq:Ntrim-hsa-LB-finite-123} implies that
\begin{align*}
\Ntrim_{h}(s,a) & \geq\frac{K}{8}\myrho_{h}(s,a)-\sqrt{\frac{K}{2}\myrho_{h}(s,a)\log\frac{KH}{\delta}}-\log\frac{KH}{\delta} \\
 & \geq \frac{K}{8}\myrho_{h}(s,a)- 2\sqrt{K\myrho_{h}(s,a)\log\frac{KH}{\delta}}. 
\end{align*}

	\item If $K \myrho_h(s,a) \leq 1600 \log \frac{KH}{\delta}$, then one can easily verify that
		\[
			\frac{K}{8} \myrho_h(s,a) - 5\sqrt{K \myrho_h(s,a)\log\frac{KH}{\delta}}
			\leq 0 \leq \Ntrim_{h}(s,a). 
		\]

\end{itemize}
%
%
%
Putting these two cases together concludes the proof, provided that the claim \eqref{eq:Ntrim-hsa-LB-finite-123} is valid. 

\paragraph{Proof of inequality \eqref{eq:Ntrim-hsa-LB-finite-123}.} Let us look at two cases separately.
\begin{itemize}
	\item If $\Ntrim_{h}(s)\pi_{h}^{\mathsf{b}}(a\mymid s) \leq 4\log\frac{KH}{\delta}$, then 
		the right-hand side of \eqref{eq:Ntrim-hsa-LB-finite-123} is negative, and hence the claim \eqref{eq:Ntrim-hsa-LB-finite-123} holds trivially. 

	\item We then turn attention to the following set:
		\begin{equation}
			\mathcal{A}_{\mathsf{large}} \coloneqq \bigg\{ (s,a,h)\in \cS\times \cA\times [H] \,\,\Big|\,\, 
			\Ntrim_{h}(s)\pi_{h}^{\mathsf{b}}(a\mymid s) > 4\log\frac{KH}{\delta} \bigg\}.
		\end{equation}
		Recognizing that
		\begin{align*}
			\sum_{(s,a,h)\in \cS\times \cA\times[H]} \Ntrim_{h}(s)\pi_{h}^{\mathsf{b}}(a\mymid s)
			&= \sum_{(s,h)\in \cS \times [H]} \Ntrim_{h}(s) \sum_{a\in \cA} \pi_{h}^{\mathsf{b}}(a\mymid s) 
			= \sum_{(s,h)\in \cS \times [H]} \Ntrim_{h}(s) \notag\\
			&\leq \sum_{(s,h)\in \cS \times [H]} \Naux_{h}(s) = \frac{KH}{2},
		\end{align*}
		we can immediately bound  the cardinality of $\mathcal{A}_{\mathsf{large}}$ as follows: 
		\begin{equation}
			\big| \mathcal{A}_{\mathsf{large}} \big| < \frac{\sum_{(s,a,h)} \Ntrim_{h}(s)\pi_{h}^{\mathsf{b}}(a\mymid s)} { 4\log\frac{KH}{\delta}}
			\leq KH/2.  
			\label{eq:size-Alarge}
		\end{equation}
		Additionally,  it follows from our construction that: conditional on $\Ntrim_h(s)$, $\Nmain_h(s)$ and the high-probability event \eqref{eq:samples-finite-UB}, 
		$\Ntrim_h(s,a)$ can be viewed as the sum of $\min\big\{ \Ntrim_h(s), \Nmain_h(s) \big\}=\Ntrim_h(s)$ 
		independent Bernoulli random variables each with mean $\pi_h(a\mymid s)$. 
		As a result, repeating the Bernstein-type argument in \eqref{eq:deviation-Naux-h-s} on the event \eqref{eq:samples-finite-UB} 
reveals that, with probability at least $1-2\delta/(KH)$,
\begin{align}
	\Ntrim_{h}(s,a)\geq
	\Ntrim_{h}(s)\pi_{h}^{\mathsf{b}}(a\mymid s)-\sqrt{4\Ntrim_{h}(s)\pi_{h}^{\mathsf{b}}(a\mymid s)\log\frac{KH}{\delta}}-\log\frac{KH}{\delta}
	\label{eq:Ntrim-hsa-LB-finite-12345}
\end{align}
		for any fixed triple $(s,a,h)$. Taking the union bound over all $(s,a,h)\in \mathcal{A}_{\mathsf{large}}$ and using the bound \eqref{eq:size-Alarge} imply that
		with probability exceeding $1-\delta$, 
		\eqref{eq:Ntrim-hsa-LB-finite-12345} holds simultaneously for all $(s,a,h)\in \mathcal{A}_{\mathsf{large}}$. 
\end{itemize}
Combining the above two cases allows one to conclude that with probability at least $1- \delta$, 
the advertised property \eqref{eq:Ntrim-hsa-LB-finite-123}
holds simultaneously for all $(s,a,h)\in \cS\times \cA\times [H]$.


\subsection{Proof of the instance-dependent statistical bound~\eqref{eq:instance-optimal}}
\label{sec:proof-instance-optimal}

To establish relation~\eqref{eq:instance-optimal}, we make use of relation (171) as follows: for any $1 \le h \le H$,
\begin{align}
	& \big\langle d_{h}^{\star},V_{h}^{\star}-V_{h}^{\widehat{\pi}}\big\rangle \le 2\sum_{j=h}^{H}\big\langle d_{j}^{\star},b_{j}^{\star}\big\rangle \notag\\
	& \le 2\sum_{j=h}^{H}\sum_s d_j^{\star}(s)\sqrt{\frac{32\cb \log\frac{NH}{\delta}}{K \myrho_j \big(s, \pi_j^{\star}(s) \big)}\mathsf{Var}_{P_{j, s, \pi^{\star}_j(s)}}(\widehat{V}_{j+1})} +  \frac{192\cb H^2 S \Cstar\log\frac{NH}{\delta}}{K } \notag\\
	& \le 12\sum_{j=h}^{H}\sum_s d_j^{\star}(s)\sqrt{\frac{\cb \log\frac{NH}{\delta}}{K \myrho_j \big(s, \pi_j^{\star}(s) \big)}\mathsf{Var}_{P_{j, s, \pi^{\star}_j(s)}}(V_{j+1}^{\star})} +  20\Big(\frac{\cb H^3S C^{\star} \log\frac{NH}{\delta}}{K}\Big)^{3/4}, 
	\label{eq:instance-opt-135}
\end{align}
provided that $K \ge 100\cb HS C^{\star} \log\frac{NH}{\delta}$. 
To see why the last inequality in \eqref{eq:instance-opt-135} holds, it suffices to observe that 
\begin{align}
 & \sum_{s}d_{j}^{\star}(s)\sqrt{\frac{\mathsf{Var}_{P_{j,s,\pi_{j}^{\star}(s)}}(\widehat{V}_{j+1})}{\myrho_{j}\big(s,\pi_{j}^{\star}(s)\big)}}\notag\\
 & \overset{\mathrm{(i)}}{\le}\sum_{s}d_{j}^{\star}(s)\sqrt{\frac{\mathsf{Var}_{P_{j,s,\pi_{j}^{\star}(s)}}(V_{j+1}^{\star})}{\myrho_{j}\big(s,\pi_{j}^{\star}(s)\big)}}+\sum_{s}d_{j}^{\star}(s)\sqrt{\frac{\mathsf{Var}_{P_{j,s,\pi_{j}^{\star}(s)}}(\widehat{V}_{j+1}-V_{j+1}^{\star})}{\myrho_{j}\big(s,\pi_{j}^{\star}(s)\big)}}\notag\\
 & \overset{\mathrm{(ii)}}{\le}\sum_{s}d_{j}^{\star}(s)\sqrt{\frac{\mathsf{Var}_{P_{j,s,\pi_{j}^{\star}(s)}}(V_{j+1}^{\star})}{\myrho_{j}\big(s,\pi_{j}^{\star}(s)\big)}}+\sum_{s}d_{j}^{\star}(s)\sqrt{\frac{H\cdot\big\langle P_{j,s,\pi_{j}^{\star}(s)},\widehat{V}_{j+1}-V_{j+1}^{\star}\big\rangle}{\myrho_{j}\big(s,\pi_{j}^{\star}(s)\big)}}\notag\\
 & \overset{\mathrm{(iii)}}{\le}\sum_{s}d_{j}^{\star}(s)\sqrt{\frac{\mathsf{Var}_{P_{j,s,\pi_{j}^{\star}(s)}}(V_{j+1}^{\star})}{\myrho_{j}\big(s,\pi_{j}^{\star}(s)\big)}}+\sqrt{S}\sqrt{H\sum_{s}\frac{\big(d_{j}^{\star}(s)\big)^{2}}{\myrho_{j}\big(s,\pi_{j}^{\star}(s)\big)}\big\langle P_{j,s,\pi_{j}^{\star}(s)},\widehat{V}_{j+1}-V_{j+1}^{\star}\big\rangle}\notag\\
 & \overset{\mathrm{(iv)}}{\le}\sum_{s}d_{j}^{\star}(s)\sqrt{\frac{\mathsf{Var}_{P_{j,s,\pi_{j}^{\star}(s)}}(V_{j+1}^{\star})}{\myrho_{j}\big(s,\pi_{j}^{\star}(s)\big)}}+\sqrt{HSC^{\star}\sum_{s}d_{j}^{\star}(s)\big\langle P_{j,s,\pi_{j}^{\star}(s)},\widehat{V}_{j+1}-V_{j+1}^{\star}\big\rangle}\notag\\
 & =\sum_{s}d_{j}^{\star}(s)\sqrt{\frac{\mathsf{Var}_{P_{j,s,\pi_{j}^{\star}(s)}}(V_{j+1}^{\star})}{\myrho_{j}\big(s,\pi_{j}^{\star}(s)\big)}}+\sqrt{HSC^{\star}\sum_{s}d_{j+1}^{\star}(s)\big[V_{j+1}^{\star}(s)-\widehat{V}_{j+1}(s)\big]}\notag\\
 & \overset{\mathrm{(v)}}{\le}\sum_{s}d_{j}^{\star}(s)\sqrt{\frac{\mathsf{Var}_{P_{j,s,\pi_{j}^{\star}(s)}}(V_{j+1}^{\star})}{\myrho_{j}\big(s,\pi_{j}^{\star}(s)\big)}}+\sqrt{HSC^{\star}\cdot80\sqrt{\frac{2\cb H^{3}SC^{\star}\log\frac{NH}{\delta}}{K}}}, 
\end{align}
where (i) holds due to the elementary inequality $\sqrt{\mathsf{Var}(X + Y)} \le \sqrt{\mathsf{Var}(X)} + \sqrt{\mathsf{Var}(Y)}$; 
(ii) follows since
\[
\mathsf{Var}_{P_{j,s,\pi_{j}^{\star}(s)}}(\widehat{V}_{j+1}-V_{j+1}^{\star})\leq\big\langle P_{j,s,\pi_{j}^{\star}(s)},(\widehat{V}_{j+1}-V_{j+1}^{\star})^{2}\big\rangle\leq H\big\langle P_{j,s,\pi_{j}^{\star}(s)},\widehat{V}_{j+1}-V_{j+1}^{\star}\big\rangle,
\]
which comes from the fact that $V_{j+1}^\star \geq \widehat{V}_{j+1}\geq0$ and $\|V_{j+1}^\star\|_\infty \leq H$; (iii) invokes the Cauchy-Schwarz inequality; 
(iv) makes use of the definition of $\Cstar$; and (v) is obtained by applying \eqref{eq:weighted-sum-V-error-aux-Bern} of Theorem~\ref{thm:Bernstein-finite-aux}.

\section{Proof of minimax lower bounds}

\subsection{Preliminary facts}
For any two distributions $P$ and $Q$, 
we denote by $\mathsf{KL}(P \parallel Q)$ the Kullback-Leibler (KL) divergence of $P$ and $Q$. 
Letting $\mathsf{Ber}(p)$ be the Bernoulli distribution with mean $p$,  
we also introduce 
\begin{align}
	\mathsf{KL}(p \parallel q) \coloneqq
	p\log\frac{p}{q}+(1-p)\log\frac{1-p}{1-q}  
	\qquad \text{and} \qquad 
	\chi^2 (p \parallel q) \coloneqq \frac{(p-q)^2}{q} +  \frac{(p-q)^2}{1-q},
	\label{eq:defn-KL-bernoulli}
\end{align}
which represent respectively the KL divergence and the chi-square divergence of $\mathsf{Ber}(p)$ from $\mathsf{Ber}(q)$ \citep{tsybakov2009introduction}. 
We make note of the following useful properties about the KL divergence. 
\begin{lemma}
    \label{lem:KL-key-result}
    For any $p, q \in \left[\frac{1}{2},1\right)$ and $p >q$, it holds that 
    \begin{align}
	    \mathsf{KL}(p \parallel q)  \leq \mathsf{KL}(q \parallel p) 
	    \leq \chi^2 (q \parallel p) =  \frac{(p-q)^2}{p(1-p)}. \label{eq:KL-dis-key-result} 
    \end{align}
\end{lemma}
\begin{proof}
The second inequality in \eqref{eq:KL-dis-key-result} is a well-known relation between KL divergence and chi-square divergence; see \citet[Lemma~2.7]{tsybakov2009introduction}.
As a result, it suffices to justify the first inequality. 
Towards this end, let us introduce $a=\frac{p+q}{2}\in\big[\frac{1}{2},1\big]$ and
$b=\frac{p-q}{2}\in\big[0,\frac{1}{4}\big]$, which allow us to re-parameterize $(p,q)$ as 
$p=a+b$ and $q=a-b$. The definition \eqref{eq:defn-KL-bernoulli} together with a little algebra gives
\begin{align*}
\mathsf{KL}(p\parallel q)-\mathsf{KL}(q\parallel p) & =(p+q)\log\frac{p}{q}+(2-p-q)\log\frac{1-p}{1-q}\\
 & =2a\log\left(\frac{a+b}{a-b}\right)+2(1-a)\log\frac{1-a-b}{1-a+b}\eqqcolon g(a,b).
\end{align*}
Taking the derivative w.r.t.~$b$ yields
\[
\frac{\partial g(a,b)}{\partial b}=2a\left\{ \frac{1}{a+b}+\frac{1}{a-b}\right\} -2\left(1-a\right)\left\{ \frac{1}{1-a+b}+\frac{1}{1-a-b}\right\} = f(a)- f(1-a) \leq 0 ,
\]
with $f(x)\coloneqq\frac{2x}{x+b}+\frac{2x}{x-b}$ (for $x>b$). 
Here, the last inequality follows since $f(\cdot)$ is a decreasing function and that $a\geq1-a$. 
	This implies that $g(a,b)$ is  non-increasing  in $b\geq 0$ for any given $a$, which in turn leads to 
\[
\mathsf{KL}(p\parallel q)-\mathsf{KL}(q\parallel p)=g(a,b)\leq g(a,0)=0 
\] 
as claimed. 
\end{proof}

\subsection{Proof of Theorem~\ref{thm:infinite-lwoer-bound}}\label{proof:thm-infinite-lb}

We now construct some hard problem instances and use them to establish the minimax lower bounds claimed in Theorem~\ref{thm:infinite-lwoer-bound}.  
It is assumed throughout this subsection that
\begin{equation}
	\frac{2}{3}\leq \gamma < 1 \qquad \text{and} \qquad \frac{14(1-\gamma)\varepsilon}{\gamma} \leq \frac{1}{2}.
	\label{eq:infinite-epsilon-assumption}
\end{equation}

\subsubsection{Construction of hard problem instances}

\paragraph{Construction of the hard MDPs.}
Let us introduce two MDPs $\left\{ \mathcal{M}_{\theta} = (\mathcal{S}, \mathcal{A}, P_{\theta}, r, \gamma) \mid \theta \in \{0,1\} \right\}$ 
parameterized by $\theta$, which involve $S$ states and 2 actions as follows: 
\begin{align*}
	\cS = \{0, 1, \ldots, S-1\} \qquad \text{and} \qquad  \mathcal{A} = \{0, 1\}.
\end{align*}
We single out a crucial state distribution (supported on the state subset $\{0, 1\}$) as follows:
\begin{align}  \label{infinite-mu-assumption}
	\mu(s) = \frac{1}{CS}\mathds{1}\{s = 0\} + \Big(1 - \frac{1}{CS}\Big)\mathds{1}\{s = 1\}
\end{align}
for some quantity $C>0$ obeying 
\begin{equation}
	\label{infinite-CS-assumption}
	\frac{1}{CS} \leq \frac{1}{4 \gamma}.
\end{equation}
We shall make clear the relation between $C$ and the concentrability coefficient $\Cstar$ shortly (see \eqref{eq:expression-Cstar-LB-inf}).  
Armed with this distribution, we are ready to define the transition kernel $P_{\theta}$ of the MDP $\mathcal{M}_\theta$ as follows:
\begin{align}
P_\theta(s^{\prime} \mymid s, a) = \left\{ \begin{array}{lll}
	p\mathds{1}\{s^{\prime} = 0\} + (1-p)\mu(s^{\prime}) & \text{for} & (s, a) = (0, \theta), \\
	 q\mathds{1}\{s^{\prime} = 0\} + (1-q)\mu(s^{\prime}) & \text{for} & (s, a) = (0, 1-\theta), \\
	 \mathds{1}\{s^{\prime} = 1\} & \text{for}   & (s, a) = (1, 0), \\ 
	 (2\gamma-1)\mathds{1}\{s^{\prime} = 1\} + 2(1-\gamma)\mu(s^{\prime}) & \text{for}   & (s, a) = (1, 1), \\ 
	 \gamma\mathds{1}(s^{\prime} = s\} + (1-\gamma)\mu(s^{\prime}) & \text{for}   & s > 1, \
                \end{array}\right.
	\label{eq:defn-Ptheta-defn-inf}
\end{align}
where the parameters $p$ and $q$ are chosen to be 
\begin{align}
	\label{eq:infinite-p-q}
	p = \gamma + \frac{14(1-\gamma)^2\varepsilon}{\gamma},\qquad q = \gamma - \frac{14(1-\gamma)^2\varepsilon}{\gamma}.
\end{align}
In view of the assumptions~\eqref{eq:infinite-epsilon-assumption}, one has
\begin{equation}
	p > q \geq \gamma - \frac{1-\gamma}{2} \geq \frac{1}{2}.
	\label{eq:p-q-relation-LB-inf}
\end{equation}
As can be clearly seen from the construction, if the MDP is initialized to either state 0 or state 1, then it will never leave the state subset $\{0,1\}$.  
In addition, the reward function for any MDP $\mathcal{M}_\theta$ is chosen to be
\begin{align}
r(s, a) = \left\{ \begin{array}{lll}
         1 & \text{for} & s = 0, \\
         \frac{1}{2} & \text{for}   & (s, a) = (1, 0), \\ 
         0 & \text{for}   & (s, a) = (1, 1), \\ 
         0 & \text{for}   & s > 1 ,\
                \end{array}\right.
		\label{eq:defn-r-theta-defn-inf}
\end{align}
where the reward gained in state 0 is clearly higher than that in other states.

\paragraph{Value functions and optimal policies.}
Next, let us take a moment to compute the value functions of the constructed MDPs and identify the optimal policies. 
For notational clarity, for the MDP $\mathcal{M}_{\theta}$ with $\theta \in \{0,1\}$, 
we denote by $\pi_{\theta}^{\star}$ the optimal policy,  
and let $V_{\theta}^{\pi}$ (resp.~$V_{\theta}^{\star}$) represent the value function of policy $\pi$ (resp.~$\pi_{\theta}^{\star}$). 
The lemma below collects several useful properties about the value functions and the optimal policies; the proof is deferred to Appendix~\ref{sec:proof-lemma:value-policy-LB-inf}. 
\begin{lemma}
\label{lemma:value-policy-LB-inf}
Consider any $\theta\in \{0,1\}$ and any policy $\pi$. One has
\begin{align}\label{eq:infinite-value-0-expression}
    V_\theta^\pi(0)= \frac{1 + \gamma(1-x_{\pi,\theta})\mu(1)V_\theta^{\pi}(1)}{1 - \gamma\big(\mu(1)x_{\pi,\theta} +\mu(0)\big)} 
	= V_{\theta}^{\pi}(1)+\frac{1-(1-\gamma)V_{\theta}^{\pi}(1)}{1-\gamma\mu(0)-\gamma\mu(1)x_{\pi,\theta}},
\end{align}
where we define
\begin{align}
	x_{\pi,\theta} \coloneqq p\pi(\theta\mymid 0) + q\pi(1-\theta\mymid 0) .
	\label{eq:infinite-x-h}
\end{align}
In addition, the optimal policy $\pi_{\theta}^{\star}$ and the optimal value function obey
\begin{align}
	\pi_\theta^\star(\theta \mymid 0) = 1, \qquad  
	\pi^{\star}_{\theta}(0\mymid 1) = 1, \qquad \text{and} \qquad V^{\star}_{\theta}(1)= \frac{1}{2(1-\gamma)}.
	\label{eq:optimal-pi-0-inf-135}
\end{align}
\end{lemma}

\paragraph{Construction of the batch dataset.}
Given any constructed MDP $\mathcal{M}_\theta$, we generate a dataset containing $N$ i.i.d.~samples $\{(s_i,a_i,s_i')\}_{1\leq i\leq N}$ according to \eqref{eq:sampling-offline-inf-iid},  
where the initial state distribution $\rhob$ and behavior policy $\pib$ are chosen to be: 
\begin{align*}
    \rhob (s) = \mu(s)\qquad \text{and } \qquad \pi^{\mathsf{b}}(a \mymid s) = 1/2, \qquad \forall (s,a)\in \cS\times \cA,
\end{align*}
with $\mu$ denoting the distribution defined in \eqref{infinite-mu-assumption}. 
Interestingly, the occupancy state distribution of this dataset coincides with $\mu$, in the sense that
%
%
\begin{align}
	\myrho(s) =  \mu(s) \qquad \text{and} \qquad
	\myrho(s,a) =  \mu(s)/2, \qquad \forall (s,a)\in\cS\times \cA. 
	\label{eq:db-sa-inf-LB}
\end{align}

Moreover, letting us choose the test distribution $\rho$ in a way that 
\begin{align}
	\rho(s) = 
	\begin{cases} 1, \quad &\text{if }s=0 \\
		0, &\text{if }s>0. 
	\end{cases}
	\label{eq:rho-defn-inf-LB}
\end{align}
we can also characterize the single-policy clipped concentrability coefficient $\Cstar$ of the dataset w.r.t. the constructed MDP $\mathcal{M}_{\theta}$ as follows
\begin{align}
	\Cstar = 2C. 
	\label{eq:expression-Cstar-LB-inf}
\end{align}
%
%
The proof of the claims~\eqref{eq:db-sa-inf-LB} and \eqref{eq:expression-Cstar-LB-inf} can be found in Appendix~\ref{sec:proof:eq:db-Cstar-inf-LB}.

\subsubsection{Establishing the minimax lower bound}

Equipped with the above construction, we are ready to develop our lower bounds. 
We remind the reader of the test distribution $\rho$ chosen in \eqref{eq:rho-defn-inf-LB}, 
and hence we need to control $\langle \rho, V_\theta^{\star} - V_\theta^{\widehat{\pi}}\rangle = V_\theta^{\star}(0) - V_\theta^{\widehat{\pi}}(0)$ with $\widehat{\pi}$ representing a policy estimate (computed based on the batch dataset).

\paragraph{Step 1: converting $\widehat{\pi}$ into an estimate $\widehat{\theta}$ of $\theta$. }
Consider first an arbitrary policy $\pi$. 
By combining the definition \eqref{eq:infinite-x-h} with the properties~\eqref{eq:optimal-pi-0-inf-135}, we see that $x_{\pi_{\theta}^{\star},\theta}=p$,
which together with \eqref{eq:infinite-value-0-expression} gives
\begin{align}
    \langle \rho, V_\theta^{\star} - V_\theta^{\pi}\rangle 
    &= V_\theta^{\star}(0) - V_\theta^{\pi}(0) 
	= \frac{1 + \gamma(1-p)\mu(1)V_\theta^{\star}(1)}{1 - \gamma\big(\mu(1)p +\mu(0)\big)} - \frac{1 + \gamma(1-x_{\pi,\theta})\mu(1)V_\theta^{\pi}(1)}{1 - \gamma\big(\mu(1)x_{\pi,\theta} +\mu(0)\big)} \nonumber  \\
	& \geq \frac{1 + \gamma(1-p)\mu(1)V_\theta^{\star}(1)}{1 - \gamma\big(\mu(1)p +\mu(0)\big)} - \frac{1 + \gamma(1-x_{\pi,\theta})\mu(1)V_\theta^{\star}(1)}{1 - \gamma\big(\mu(1)x_{\pi,\theta} +\mu(0)\big)} \nonumber \\
	&	 \geq \frac{21 \varepsilon  }{8} \big( 1 - \pi(\theta \mymid 0) \big) .
	\label{eq:infinite-delta-key-auxiliary}
\end{align}
Here, the second line holds since $V_{\theta}^{\pi}\leq V_{\theta}^{\star}$, 
and the last inequality will be established in Appendix~\ref{sec:proof:eq:db-Cstar-inf-LB}. 

Denoting by $\mathbb{P}_{\theta}$ the probability distribution when the MDP is $\mathcal{M}_{\theta}$, suppose for the moment that the policy estimate $\widehat{\pi}$ achieves
\[
	\mathbb{P}_{\theta} \big\{  \big\langle \rho, V_\theta^{\star} - V_\theta^{\widehat{\pi}} \big\rangle \leq \varepsilon  \big\} \geq \frac{7}{8},
\]
then in view of \eqref{eq:infinite-delta-key-auxiliary}, one necessarily has $\widehat{\pi}(\theta \mymid 0) \geq \frac{13}{21}$ with probability at least $7/8$. 
If this were true, then we could then construct the following estimate $\widehat{\theta}$ for $\theta$:
\begin{align}
	\widehat{\theta}=\arg\max_{a} \, \widehat{\pi}(a\mymid 0),
	\label{eq:defn-theta-hat-inf-LB}
\end{align}
which would necessarily satisfy
\begin{align}
	\mathbb{P}_{\theta}\big( \widehat{\theta} = \theta \big) 
	\geq \mathbb{P}_{\theta}\big\{ \widehat{\pi}(\theta \mymid 0) > 1/2 \big\} 
	\geq \mathbb{P}_{\theta}\bigg\{ \widehat{\pi}(\theta \mymid 0) \geq \frac{13}{21} \bigg\} \geq \frac{7}{8}. 
	\label{eq:P-theta-accuracy-inf}
\end{align}
In what follows, 
we would like to show that \eqref{eq:P-theta-accuracy-inf} cannot happen --- i.e., one cannot possibly find such a good estimator for $\theta$ --- without a sufficient number of samples.

\paragraph{Step 2: probability of error in testing two hypotheses.}  
The next step lies in studying the feasibility of differentiating two hypotheses $\theta = 0$ and $\theta=1$. 
Define the minimax probability of error as follows
\begin{equation}
	p_{\mathrm{e}}\coloneqq\inf_{\psi}\max\big\{ \mathbb{P}_{0}(\psi\neq0),\,\mathbb{P}_{1}(\psi\neq1)\big\} , \label{eq:error-prob-two-hypotheses-inf-LB}
\end{equation}
where the infimum is taken over all possible tests $\psi$ (based on the batch dataset in hand).   
Letting $\mub_{\theta}$ denote the distribution of a sample $(s_i,a_i,s_i')$ under the MDP $\mathcal{M}_{\theta}$
and recalling that the samples are independently generated,
one can demonstrate that
\begin{align}
p_{\mathrm{e}} & \geq\frac{1}{4}\exp\Big(-N\mathsf{KL} \big(\mub_{0}\parallel\mub_{1} \big)\Big)\nonumber\\
	& =\frac{1}{4}\exp\bigg\{-\frac{1}{2}N\mu(0)\Big(\mathsf{KL}\big(P_{0}(\cdot\mymid0,0)\parallel P_{1}(\cdot\mymid0,0)\big)+\mathsf{KL}\big(P_{0}(\cdot\mymid0,1)\parallel P_{1}(\cdot\mymid0,1)\big)\Big)\bigg\}. 
	\label{eq:infinite-remainder-KL}
\end{align}
Here, the first inequality results from \citet[Theorem~2.2]{tsybakov2009introduction} and  the additivity property of the KL divergence (cf.~\citet[Page~85]{tsybakov2009introduction}), 
and the second line holds true since 
\begin{align*}
\mathsf{KL}(\mub_{0}\parallel\mub_{1}) & =\sum_{s,a,s'}\mu(s)\pib(a\mymid s)P_{0}(s'\mymid s,a)\log\frac{\mu(s)\pib(a\mymid s)P_{0}(s'\mymid s,a)}{\mu(s)\pib(a\mymid s)P_{1}(s'\mymid s,a)}\\
 & =\frac{1}{2}\mu(0)\sum_{a}\sum_{s'}P_{0}(s'\mymid0,a)\log\frac{P_{0}(s'\mymid0,a)}{P_{1}(s'\mymid0,a)} \\
 & =\frac{1}{2}\mu(0)\sum_{a}\mathsf{KL}\big(P_{0}(\cdot\mymid0,a)\parallel P_{1}(\cdot\mymid0,a)\big) ,
\end{align*}
where the second line is valid since $P_0(\cdot\mymid s,a)$ and $P_1(\cdot\mymid s,a)$ differ only when $s=0$.

Next, we turn attention to the KL divergence of interest. Recall that
\begin{align*}
	P_0(0 \mymid 0, 0) = \left(1 -\frac{1}{CS}\right)p + \frac{1}{CS},\qquad
	P_1(0 \mymid 0, 0) = \left(1 -\frac{1}{CS}\right)q + \frac{1}{CS} .
\end{align*}
Given that $p\geq q \geq 1/2$ (see \eqref{eq:p-q-relation-LB-inf}), 
we can apply Lemma~\ref{lem:KL-key-result} to arrive at
\begin{align*}
\mathsf{KL}\big(P_{0}(\cdot\mymid0,0)\parallel P_{1}(\cdot\mymid0,0)\big) & =\mathsf{KL}\left(\left(1-\frac{1}{CS}\right)p+\frac{1}{CS}\parallel\left(1-\frac{1}{CS}\right)q+\frac{1}{CS}\right)\nonumber\\
 & \overset{\mathrm{(i)}}{\leq}\frac{\left(1-\frac{1}{CS}\right)^{2}(p-q)^{2}}{\left(\left(1-\frac{1}{CS}\right)p+\frac{1}{CS}\right)\left(1-p-(1-p)\frac{1}{CS}\right)}\nonumber\\
 & \leq\frac{\left(1-\frac{1}{CS}\right)^{2}(p-q)^{2}}{p\left((1-p)\big(1-\frac{1}{CS}\big)\right)}\nonumber\\
 & \overset{\mathrm{(ii)}}{=}\frac{784(1-\gamma)^{4}\varepsilon^{2}}{\gamma^{2}\left(\gamma+\frac{14(1-\gamma)^{2}\varepsilon}{\gamma}\right)\left(1-\gamma-\frac{14(1-\gamma)^{2}\varepsilon}{\gamma}\right)}\nonumber\\
 & \overset{(\mathrm{iii})}{\leq}\frac{1568(1-\gamma)^{4}\varepsilon^{2}}{\gamma^{3}(1-\gamma)}\overset{(\mathrm{iv})}{\leq}12544(1-\gamma)^{3}\varepsilon^{2},
\end{align*}
where (i) arises from Lemma~\ref{lem:KL-key-result}, 
(ii) follows from the definitions of $p$ and $q$ \eqref{eq:infinite-p-q}, 
(iii) holds true as long as $\frac{14(1-\gamma)^2\varepsilon}{\gamma} \leq \frac{1-\gamma}{2}$, and (iv) results from the assumption $\gamma\in [\frac{1}{2},1)$. 
Evidently, the same upper bound holds for $\mathsf{KL}\big(P_{0}(\cdot\mymid 0,1)\parallel P_{1}(\cdot\mymid 0,1)\big)$ as well. 
Substitution back into \eqref{eq:infinite-remainder-KL} reveals that: 
if the sample size does not exceed
\begin{align}
    N \leq \frac{ CS\log 2}{12544(1-\gamma)^3\varepsilon^2} = \frac{ \Cstar S\log 2}{25088(1-\gamma)^3\varepsilon^2},
	\label{eq:sample-size-condition-proof-LB-inf}
\end{align}
then one necessarily has
\begin{align}
	p_{\mathrm{e}}
     & \geq \frac{1}{4}\exp\Big( -  12544 N \mu(0) (1-\gamma)^3\varepsilon^2 \Big)
      = \frac{1}{4}\exp\Big(-    \frac{12544N (1-\gamma)^3\varepsilon^2}{CS} \Big)
     \geq \frac{1}{8}. 
	\label{eq:pe-LB-13579-inf}
\end{align}

\paragraph{Step 3: putting all this together.} 
To finish up, suppose that there exists an estimator $\widehat{\pi}$ such that
\[
	\mathbb{P}_0 \big\{  \big\langle \rho, V_0^{\star} - V_0^{\widehat{\pi}} \big\rangle > \varepsilon  \big\} < \frac{1}{8}
	\qquad \text{and} \qquad
	\mathbb{P}_1 \big\{  \big\langle \rho, V_0^{\star} - V_0^{\widehat{\pi}} \big\rangle > \varepsilon  \big\} < \frac{1}{8}.
\]
Then in view of our arguments in Step 1, the estimator $\widehat{\theta}$ defined in \eqref{eq:defn-theta-hat-inf-LB} must satisfy
\[
	\mathbb{P}_0\big(\widehat{\theta} \neq \theta\big) < \frac{1}{8} 
	\qquad \text{and} \qquad
	\mathbb{P}_1\big(\widehat{\theta} \neq \theta\big) < \frac{1}{8}. 
\]
This, however, cannot possibly happen under the sample size condition \eqref{eq:sample-size-condition-proof-LB-inf}; otherwise it contradicts the lower bound \eqref{eq:pe-LB-13579-inf}.

%
%
%
%

\subsubsection{Proof of Lemma~\ref{lemma:value-policy-LB-inf}} 
\label{sec:proof-lemma:value-policy-LB-inf}

To begin with, for any policy $\pi$, the value function of state $0$ obeys
\begin{align}
	V_\theta^\pi(0) &= \mathop{\mathbb{E}}\limits_{a\sim\pi(\cdot\mymid0)}\bigg[r(0,a)+\gamma \sum_{s'} P_{\theta}(s' \mymid 0,a)V_{\theta}^{\pi}(s')\bigg] \nonumber \\
    & = 1 + \gamma \pi(\theta \mymid 0) \Big[ \big(p + (1-p)\mu(0)\big) V_\theta^\pi(0) + (1-p) \mu(1) V_\theta^\pi(1)\Big] \nonumber \\
    & \qquad + \gamma \pi(1-\theta \mymid 0) \Big[\big(q + (1-q) \mu(0)\big) V_\theta^\pi(0) + (1-q)\mu(1) V_\theta^\pi(1)\Big] \nonumber \\
   & = 1 + \gamma\Big[p \pi(\theta \mymid 0) + q \pi(1-\theta \mymid0) + \mu(0) - p \pi(\theta \mymid 0) \mu(0) - q \pi(1-\theta \mymid0)\mu(0) \Big]V_\theta^{\pi}(0) \nonumber\\
&\qquad + \gamma \mu(1) \Big[ 1 - p\pi(\theta \mymid 0)  -q \pi(1-\theta\mymid 0)\Big]V_\theta^{\pi}(1) \nonumber\\
& \overset{\mathrm{(i)}}{=} 1 + \gamma \Big[ x_{\pi,\theta} + (1- x_{\pi,\theta})\mu(0) V_\theta^{\pi}(0) +  (1 - x_{\pi,\theta}) \mu(1) V_\theta^{\pi}(1)\Big]  \nonumber\\
& \overset{\mathrm{(ii)}}{=}  1 + \gamma \Big[ \big(\mu(1)x_{\pi,\theta} +\mu(0)\big) V_\theta^{\pi}(0) + (1-x_{\pi,\theta})\mu(1)V_\theta^{\pi}(1) \Big], \label{eq:infinite-Value-0}
\end{align}
where in (i) we have defined the following quantity 
\begin{align}
	x_{\pi,\theta} = p\pi(\theta\mymid 0) + q\pi(1-\theta\mymid 0) = q+ (p-q) \pi(\theta\mymid 0),\label{eq:infinite-x-h-proof}
\end{align}
and (ii) relies on the fact that $\mu(0)+\mu(1)=1$. 
Rearranging terms in \eqref{eq:infinite-Value-0}, we are left with
\begin{align}\label{eq:infinite-value-0-expression-proof}
    V_\theta^\pi(0)= \frac{1 + \gamma(1-x_{\pi,\theta})\mu(1)V_\theta^{\pi}(1)}{1 - \gamma\big(\mu(1)x_{\pi,\theta} +\mu(0)\big)} 
	= V_{\theta}^{\pi}(1)+\frac{1-(1-\gamma)V_{\theta}^{\pi}(1)}{1-\gamma\mu(0)-\gamma\mu(1)x_{\pi,\theta}}.
\end{align}

Additionally,  the value function of state $1$ can be calculated as
\begin{align}
        V_\theta^{\pi}(1) &= \pi(0 \mymid 1) \left( \frac{1}{2} + \gamma V_\theta^{\pi}(1)\right) 
	+ \pi(1 \mymid 1) \gamma \Big[ \big((2\gamma - 1) + 2(1-\gamma) \mu(1)\big) V_\theta^{\pi}(1) + 2(1-\gamma) \mu(0) V_\theta^{\pi}(0) \Big] \notag\\
	&= \pi(0 \mymid 1) \left( \frac{1}{2} + \gamma V_\theta^{\pi}(1)\right) + \pi(1 \mymid 1) \gamma \left[ \left(1- \frac{2(1-\gamma)}{CS}\right)V_\theta^{\pi}(1) + \frac{2(1-\gamma)}{CS}V_\theta^{\pi}(0)\right] \label{eq:Vpi-expansion-1-infinite-13}\\
        & \overset{\mathrm{(i)}}{\leq} \pi(0 \mymid 1) \left( \frac{1}{2} + \gamma V_\theta^{\pi}(1)\right) + \pi(1 \mymid 1) \gamma \left[ \left(1- \frac{2(1-\gamma)}{CS}\right)V_\theta^{\pi}(1) + \frac{2(1-\gamma)}{CS}\frac{1}{1-\gamma}\right] \notag\\
        & \overset{\mathrm{(ii)}}{\leq} \pi(0 \mymid 1) \left( \frac{1}{2} + \gamma V_\theta^{\pi}(1)\right) + \pi(1 \mymid 1) \left[ \frac{1}{2} + \gamma \left(1- \frac{2(1-\gamma)}{CS}\right)V_\theta^{\pi}(1) \right] \notag\\
        & = \frac{1}{2} + \gamma V_\theta^{\pi}(1) - \frac{2\gamma(1-\gamma)}{CS} V_\theta^{\pi}(1) \pi(1 \mymid 1) ,
	\label{eq:Vpi-expansion-1-infinite-246}
\end{align}
where (i) arises from the elementary property $ 0\leq V_\theta^{\pi}(s) \leq \frac{1}{1-\gamma}$ for any $\pi$ and $s\in\cS $, 
and (ii) comes from the assumption~\eqref{infinite-CS-assumption}.
The above observation reveals several facts:
\begin{itemize}
	\item If we take $\pi(0\mymid 1)=1$, then \eqref{eq:Vpi-expansion-1-infinite-13} tells us that
		\begin{equation}
			V_{\theta}^{\pi}(1)=\frac{1}{2}+\gamma V_{\theta}^{\pi}(1)\qquad\Longrightarrow\qquad V_{\theta}^{\pi}(1)=\frac{1}{2(1-\gamma)}.
			\label{eq:Vtheta-pi-calculation-simple}
		\end{equation}
	\item It also follows from \eqref{eq:Vpi-expansion-1-infinite-246} that for any policy $\pi$, one has
		\begin{equation}
			V_{\theta}^{\pi}(1)\leq \frac{1}{2}+\gamma V_{\theta}^{\pi}(1)\qquad\Longrightarrow\qquad V_{\theta}^{\pi}(1)\leq \frac{1}{2(1-\gamma)}.
			\label{eq:Vtheta-pi-calculation-simple}
		\end{equation}
\end{itemize}
These two facts taken collectively imply that the optimal policy and the optimal value function obey
		\begin{equation}
			\pi^{\star}_{\theta}(0\mymid 1) = 1 \qquad \text{and} \qquad V^{\star}_{\theta}(1)= \frac{1}{2(1-\gamma)}.
		\end{equation}

%

Next, we have learned from \eqref{eq:infinite-value-0-expression-proof} that
\[
	V_\theta^{\star}(0)= V_{\theta}^{\star}(1)+\frac{1-(1-\gamma)V_{\theta}^{\star}(1)}{1-\gamma\mu(0)-\gamma\mu(1)x_{\pi^{\star}_{\theta},\theta}}.
\]
Note that $1-(1-\gamma)V_{\theta}^{\star}(1)\geq 1- (1-\gamma)\frac{1}{1-\gamma}=0$. 
Since the function
\[
	g(x)= V_{\theta}^{\star}(1)+\frac{1-(1-\gamma)V_{\theta}^{\star}(1)}{1-\gamma\mu(0)-\gamma\mu(1)x}
\]
is increasing in $x$ and that $x_{\pi,\theta}$ (cf.~\eqref{eq:infinite-x-h-proof}) is increasing in $\pi(\theta\mymid 0)$ (given that $p\geq q$), 
one can easily see that the optimal policy obeys
\begin{align}
	\pi_\theta^\star(\theta \mymid 0) = 1 .
\end{align}

\subsubsection{Proof of auxiliary properties}
\label{sec:proof:eq:db-Cstar-inf-LB}

\paragraph{Proof of claim~\eqref{eq:db-sa-inf-LB}.}
We begin by proving the property~\eqref{eq:db-sa-inf-LB}. 
Towards this, let us abuse the notation by considering a MDP trajectory denoted by $\{(s_t,a_t)\}_{t\geq 0}$,
and suppose that it starts from $s_{0}\sim \rhob = \mu$. 
It can be straightforwardly calculated that
\begin{align*}
	\mathbb{P}\left\{ s_{1}=0\right\}  & = \sum\nolimits_{s}\mu(s)\Big\{ \pib(0\mymid s)\mathbb{P}\left\{ s_{1}=0\mid s_{0}=s,a_{0}=0\right\} +\pib(1\mymid s)\mathbb{P}\left\{ s_{1}=0\mid s_{0}=s,a_{0}=1\right\} \Big\} \\
 & =\mu(0)\left\{ \frac{1}{2} P_{\theta}(0\mymid0,0)+\frac{1}{2} P_{\theta}(0\mymid0,1) \right\} +\mu(1)\left\{ \frac{1}{2} P_{\theta}(0\mymid1,0)+\frac{1}{2}P_{\theta}(0\mymid1,1)\right\} \\
 & =\mu(0)\left\{ \gamma+(1-\gamma)\mu(0)\right\} +\mu(1)\left\{ (1-\gamma)\mu(0)\right\} =\mu(0),
\end{align*}
where the last identity holds since $\mu(0)+\mu(1)=1$. Similarly, one can derive $\mathbb{P}\left\{ s_{1}=1\right\} = \mu(1)$, thus indicating that $s_1\sim \mu$.  
Repeating this analysis reveals that $s_t \sim \mu$ for any $t\geq 0$. 
Consequently, one has 
\[
	\myrho(s) = (1-\gamma) \mathbb{E}\left[\sum_{t=0}^\infty \gamma^t \mathbb{P}\big(s_t=s\mid s_{0}\sim \rhob;\pib\big)\right]
 = \mu(s), \qquad \forall s\in\cS.
\]
Additionally, it it observed that 
\begin{align}
	\myrho(s,a)=\myrho(s) \pib(a\mymid s)= \mu(s)/2.
	\label{eq:myrho-sa-mus-2-inf-LB}
\end{align}

\paragraph{Proof of claim \eqref{eq:expression-Cstar-LB-inf}.}
Consider the MDP $\mathcal{M}_\theta$, whose optimal policy $\pi_\theta^{\star}$ satisfies  $\pi_\theta^{\star}(\theta \mymid 0) = 1$ (see Lemma~\ref{lemma:value-policy-LB-inf}). 
Let us generate a MDP trajectory denoted by $\{(s_t,a_t)\}_{t\geq 0}$ with $a_t\sim \pi_{\theta}^{\star}(\cdot \mymid s_t)$, 
where we have again abused notation as long as it is clear from the context.  
In this case, we can deduce that
\begin{align*}
d^{\star}(0, \theta) 
	& =  (1-\gamma) \mathbb{E}\left[\sum_{t=0}^\infty \gamma^t \mathbb{P}\big(s_t=0\mid s_{0}\sim \rho;\pi^{\star}_{\theta}\big) \pi^{\star}_{\theta}(\theta \mymid 0)\right]
	= (1-\gamma) \mathbb{E}\left[\sum_{t=0}^\infty \gamma^t \mathbb{P}\big(s_t=0\mid s_{0}\sim \rho;\pi^{\star}_{\theta} \big) \right]\\
	& \overset{\mathrm{(i)}}{\geq} (1-\gamma)\sum_{t=0}^\infty \gamma^t \rho(0) \big[ \mathbb{P}_{\theta}(0 \mymid 0, \theta) \big]^t 
	\overset{\mathrm{(ii)}}{\geq} (1-\gamma)\sum_{t=0}^\infty \rho(0)\gamma^{2t} = \frac{1-\gamma}{1-\gamma^2} = \frac{1}{1+\gamma} \geq \frac{1}{2}, 
\end{align*}
where in (i) we compute, for each $t$, the probability of a special trajectory with $s_1=\cdots=s_t=0$ and $a_0=\cdots=a_{t-1}=\theta$, 
and (ii) holds true since $P_{\theta}(0\mymid 0,\theta)\geq p\geq \gamma$. 
Taking this together with \eqref{eq:myrho-sa-mus-2-inf-LB} yields
\begin{align}
	\frac{\min\big\{ d^{\star}(0, \theta), \frac{1}{S}\big\} }{\myrho(0,\theta)} &= \frac{2}{S\mu(0)} = 2C, \notag\\
	\frac{\min\big\{ d^{\star}(0, 1-\theta), \frac{1}{S}\big\} }{\myrho(0, 1-\theta)} &= 
	\frac{\min\big\{ d^{\star}(0, 1-\theta), \frac{1}{S}\big\} }{\myrho(0, \theta)} \leq
		\frac{\min\big\{ d^{\star}(0, \theta), \frac{1}{S}\big\} }{\myrho(0,\theta)} =  2C. 
\end{align}
In addition, it is easily seen that $d^{\star}(s,a)=0$ for any $s>1$, and that
\[
\frac{\min\big\{ d^{\star}(1,a),\frac{1}{S}\big\}}{\myrho(1,a)}\leq\frac{1/S}{\mu(1)/2}=\frac{2}{S(1-1/CS)}
	\leq\frac{4}{S}\leq2C,
\]
where the first inequality comes from \eqref{eq:myrho-sa-mus-2-inf-LB}, 
the first identity uses the definition \eqref{infinite-mu-assumption}, and the last two inequalities result from an immediate consequence of \eqref{infinite-CS-assumption} and $\gamma \geq 1/2$, i.e.,
\begin{align}
	\frac{1}{CS}\leq \frac{1}{4\gamma } \leq \frac {1}{2}. 
\end{align}
As a result, putting the above relations together leads to
\begin{align}
    \Cstar  = \max_{(s, a) \in \cS \times \cA}\frac{\min\big\{d^{\star}(s, a), \frac{1}{S}\big\}}{\myrho(s, a)} = \frac{\min\big\{d^{\star}(0, \theta), \frac{1}{S}\big\}}{\myrho(0, \theta)} =  2C .
\end{align}
%

\paragraph{Proof of inequality \eqref{eq:infinite-delta-key-auxiliary}.}
%
Observing the basic identity (using $\mu(0)+\mu(1)=1$)
\[
	\frac{1+\gamma(1-x)\mu(1)V_{\theta}^{\star}(1)}{1-\gamma\big(\mu(1)x+\mu(0)\big)}=V_{\theta}^{\star}(1)+\frac{1-(1-\gamma)V_{\theta}^{\star}(1)}{1-\gamma\mu(0)-\gamma\mu(1)x},
\]
we can obtain
\begin{align}
	&\frac{1 + \gamma(1-p)\mu(1)V_\theta^{\star}(1)}{1 - \gamma\big(\mu(1)p +\mu(0)\big)} - \frac{1 + \gamma(1-x_{\pi,\theta})\mu(1)V_\theta^{\star}(1)}{1 - \gamma\big(\mu(1)x_{\pi,\theta} +\mu(0)\big)} 
	=	\frac{1-(1-\gamma)V_{\theta}^{\star}(1)}{1-\gamma\mu(0)-\gamma\mu(1)p}-\frac{1-(1-\gamma)V_{\theta}^{\star}(1)}{1-\gamma\mu(0)-\gamma\mu(1)x_{\pi,\theta}} \nonumber\\
	&=  \big( 1-(1-\gamma)V_{\theta}^{\star}(1) \big) \frac{\gamma \mu(1) (p-x_{\pi,\theta})}{ ~\underbrace{\left[1 - \gamma\big(\mu(1)p +\mu(0)\big)\right]\left[1 - \gamma\big(\mu(1)x_{\pi,\theta} +\mu(0)\big)\right] }_{\eqqcolon\, \alpha}~} \nonumber \\
    & =  \frac{\gamma \mu(1) (p-x_{\pi,\theta})}{2\alpha} , \label{eq:infinite-A-result}
\end{align}
where the last relation arises from the fact \eqref{eq:optimal-pi-0-inf-135}.

The remainder of the proof boils down to controlling $\alpha$. 
Making use of the definition of $p$ (cf.~\eqref{eq:infinite-p-q}), 
$\mu(s)$ (cf.~\eqref{infinite-mu-assumption})  and $x_\pi$ (cf.~\eqref{eq:infinite-x-h}), 
we can demonstrate that
\begin{align}
    \alpha &= \left[1- \gamma \left(\left(1-\frac{1}{CS}\right)p + \frac{1}{CS}\right)\right]\left[1- \gamma \left(\left(1-\frac{1}{CS}\right)x_{\pi,\theta} + \frac{1}{CS}\right)\right]  \leq (1-\gamma p)(1-\gamma x_{\pi,\theta})\notag \\
    & \overset{\mathrm{(i)}}{\leq} (1-\gamma p)(1-\gamma q)
     \overset{\mathrm{(ii)}}{\leq}  \Big( 1-\gamma \frac{p+q}{2} \Big)^2 \notag\\
	&  = (1-\gamma^2)^2 = (1-\gamma)^2(1+\gamma)^2 \leq 4(1-\gamma)^2, 
\end{align}
where (i) holds true owing to the trivial fact that $ x_{\pi,\theta} \geq q$ for any policy $\pi$ (as long as $p\geq q$), 
and (ii) is a consequence of the AM-GM inequality. 
Substituting it into \eqref{eq:infinite-A-result} and using the definition \eqref{eq:infinite-x-h} give
\begin{align*}
    &\frac{1 + \gamma(1-p)\mu(1)V_\theta^{\star}(1)}{1 - \gamma\big(\mu(1)p +\mu(0)\big)} - \frac{1 + \gamma(1-x_{\pi,\theta})\mu(1)V_\theta^{\star}(1)}{1 - \gamma\big(\mu(1)x_{\pi,\theta}+\mu(0)\big)} = \frac{\gamma \mu(1) (p-x_{\pi,\theta})}{2\alpha} \geq \frac{\gamma \mu(1) (p-x_{\pi,\theta})}{8(1-\gamma)^2}\\
& \qquad = \frac{\gamma \mu(1)}{8(1-\gamma)^2} (p-q) \pi(1-\theta \mymid 0)\\
& \qquad \geq \frac{3 \gamma }{32(1-\gamma)^2} \frac{28 (1-\gamma)^2\varepsilon}{\gamma} \pi(1-\theta \mymid 0) 
	= \frac{21 \varepsilon  }{8}  \big( 1 - \pi(\theta \mymid 0) \big).
\end{align*}
%

 %

\subsection{Proof of Theorem~\ref{thm:finite-lower-bound}}
\label{proof:thm-finite-lb}

To establish Theorem~\ref{thm:finite-lower-bound}, 
we shall first generate a collection of hard problem instances (including MDPs and the associated batch datasets), 
and then conduct sample complexity analyses over these hard instances.

\subsubsection{Construction of hard problem instances}

\paragraph{Construction of the hard MDPs.}
To begin with, for any integer $H\geq 32$, let us consider a set $\Theta\subseteq \{0,1\}^{H}$ of $H$-dimensional vectors, which we shall construct shortly.  
We then generate a collection of MDPs 
\begin{equation}
	\mathsf{MDP}(\Theta) = \left\{\mathcal{M}_{\theta} = 
	\big(\mathcal{S}, \mathcal{A}, P^{\theta} = \{P^{\theta_h}_h\}_{h=1}^H, \{r_h\}_{h=1}^H, H \big) 
	\mid \theta = [\theta_h]_{1\leq h\leq H} \in \Theta
	\right\},
\end{equation}
where
\begin{align*}
	\cS = \{0, 1, \ldots, S-1\}, 
	\qquad \text{and} \qquad  \mathcal{A} = \{0, 1\}.
\end{align*}
To define the transition kernel of these MDPs, we find it convenient to introduce the following state distribution supported on the state subset $\{0,1\}$:
\begin{align}\label{finite-mu-assumption}
	\mu(s) = \frac{1}{CS}\mathds{1}\{s = 0\} + \Big(1 - \frac{1}{CS}\Big)\mathds{1}\{s = 1\},
\end{align}
where $\mathds{1}(\cdot)$ is the indicator function, and $C>0$ is some constant that will determine the concentrability coefficient $\Cstar$ (as we shall detail momentarily). It is assumed that 
\begin{align}\label{finite-C-assumption}
    \frac{1}{CS} \leq \frac{1}{4} .
\end{align}
With this distribution in mind, we can specify the transition kernel $P^{\theta} = \{P^{\theta_h}_h\}_{h=1}^H$ of the MDP $\mathcal{M}_\theta$ as follows:
\begin{align}
P^{\theta_h}_h(s^{\prime} \mymid s, a) = \left\{ \begin{array}{lll}
	p\mathds{1}\{s^{\prime} = 0\} + (1-p)\mu(s^{\prime}) & \text{if} & (s, a) = (0, \theta_h) \\
	 q\mathds{1}\{s^{\prime} = 0\} + (1-q)\mu(s^{\prime}) & \text{if} & (s, a) = (0, 1-\theta_h) \\
	 \mathds{1}\{s^{\prime} = 1\} & \text{if}   & (s, a) = (1, 0) \\ 
	 \big(1 - \frac{2 c_1}{H}\big)\mathds{1}\{s^{\prime} = 1\} + \frac{2 c_1}{H}\mu(s^{\prime}) & \text{if}   & (s, a) = (1, 1) \\ 
	 \big(1 - \frac{1}{H}\big)\mathds{1}\{s^{\prime} = s\} + \frac{1}{H}\mu(s^{\prime}) & \text{if}   & s > 1 \
                \end{array}\right.
		\label{eq:Ph-construction-lower-bound-finite}
\end{align}
for any $(s,a,s',h)\in \cS\times \cA\times \cS\times [H]$, 
where $p$ and $q$ are set to be 
\begin{equation}
	p = 1 - \frac{c_1}{H} + \frac{c_2\varepsilon}{H^2}
	\qquad \text{and} \qquad
	q = 1 - \frac{c_1}{H} - \frac{c_2 \varepsilon}{H^2} \label{eq:finite-p-q-def}
\end{equation}
for $c_1 = 1/4$ and $c_2 = 4096$ such that
\begin{equation}
     \frac{c_2\varepsilon}{H^2} \leq \frac{c_1}{2H} \leq \frac{1}{8}. 
\end{equation}
It is readily seen from the above assumption that
\begin{align}
    p> q \geq \frac{1}{2} .
	\label{eq:p-q-order-LB-finite}
\end{align}
In view of the transition kernel \eqref{eq:Ph-construction-lower-bound-finite}, 
the MDP will never leave the state subset $\{0,1\}$ if its initial state belongs to $\{0,1\}$.
The reward function of all these MDPs is chosen to be
\begin{align}
r_h(s, a) = \left\{ \begin{array}{lll}
         1 & \text{if} & s = 0 \\
         \frac{1}{2} & \text{if}   & (s, a) = (1, 0) \\ 
         0 & \text{if}   & (s, a) = (1, 1) \\ 
         0 & \text{if}   & s > 1 \
                \end{array}\right.
		\label{eq:rh-construction-lower-bound-finite}
\end{align}
for any $(s,a,h)\in \cS\times \cA\times [H]$.

Finally,  let us choose the set $\Theta \subseteq  \{0, 1\}^{H}$. 
By virtue of  the Gilbert-Varshamov lemma \citep{gilbert1952comparison}, 
one can construct $\Theta \subseteq  \{0, 1\}^{H}$ in a way  that 
\begin{equation}
	|\Theta| \ge e^{H/8}
	\qquad \text{and} \qquad 
	\|\theta - \widetilde{\theta}\|_1 \ge \frac{H}{8}
	\quad 
	\text{for any }\theta,\widetilde{\theta}\in \Theta \text{ obeying }\theta \ne \widetilde{\theta}. 
	\label{eq:property-Theta} 
\end{equation}
In other words, the set $\Theta$ we construct contains an exponentially large number of vectors that are sufficiently separated. 
This property plays an important role in the ensuing analysis.


\paragraph{Value functions and optimal policies.} Next, we look at the value functions of the constructed MDPs and identify the optimal policies. 
For the sake of notational clarity, for the MDP $\mathcal{M}_\theta$, 
we denote by $\pi^{\star, \theta} = \{\pi^{\star, \theta}_h\}_{h=1}^H$ the optimal policy,  
and let $V_h^{\pi, \theta}$ (resp.~$V_h^{\star, \theta}$) indicate the value function of  policy $\pi$ (resp.~$\pi^{\star, \theta}$) at time step $h$. The following lemma collects a couple of useful properties concerning the value functions and optimal policies; the proof can be found in Appendix~\ref{proof:lem:finite-lb-value}.

\begin{lemma}\label{lem:finite-lb-value}
Consider any $\theta\in\Theta$ and any policy $\pi$. Then it holds that
\begin{align}
   	V_h^{\pi, \theta}(0) =  1 + \big(\mu(1)x_h^{\pi,\theta}+\mu(0)\big) V_{h+1}^{\pi,\theta}(0) + (1-x_h^{\pi,\theta})\mu(1)V_{h+1}^{\pi,\theta}(1)
	\label{eq:finite-lemma-value-0-pi}
\end{align}
for any $h\in [H]$, where 
\begin{align}
x_h^{\pi,\theta} = p\pi_h(\theta_h\mymid 0) + q\pi_h(1-\theta_h\mymid 0).\label{eq:finite-x-h}
\end{align}
In addition,  for any $h\in[H]$,  the optimal policies and the optimal value functions obey
\begin{subequations}
	\label{eq:finite-lb-value-lemma}
\begin{align}
	\pi_h^{\star,\theta}(\theta_h \mymid 0) &= 1,  &V_h^{\star, \theta}(0) \ge \tfrac{2}{3}(H+1-h), \\
	\pi_h^{\star,\theta}(0 \mymid 1) &= 1,      &V_h^{\star,\theta}(1)  = \tfrac{1}{2}(H+1-h),
\end{align}
\end{subequations}
provided that $0<c_1\leq 1/2$. 
\end{lemma}

\paragraph{Construction of the batch dataset.}
A batch dataset is then generated, which consists of $K$ {\em independent} sample trajectories each of length $H$.
The initial state distribution $\rhob$ and the behavior policy $\pib = \{\pib_h\}_{h=1}^H$ (according to \eqref{eq:finite-batch-size-def}) are chosen as follows: 
\begin{align*}
	\rhob (s) = \mu(s)\quad \text{and } \quad\pi_h^{\mathsf{b}}(a \mymid s) = \frac{1}{2}, \qquad \forall (s,a,h)\in \cS\times \cA\times [H], 
\end{align*}
where $\mu$ has been defined in \eqref{finite-mu-assumption}. 
As it turns out, for any MDP $\mathcal{M}_\theta$, the occupancy distributions of the above batch dataset admit the following simple characterization: 
\begin{align}
	\myrho_{h}(s) &= \mu(s), \qquad \myrho_h(s, a) = \frac{1}{2} \mu(s), \qquad \forall (s,a,h)\in\cS\times \cA\times [H]. 
	\label{eq:finite-lb-behavior-distribution}
\end{align}
Additionally, we shall choose the initial state distribution $\rho$ as follows
\begin{align}
    \rho(s) = 
    \begin{cases} 1, \quad &\text{if }s=0, \\
        0, &\text{if }s>0. 
    \end{cases}
    \label{eq:rho-defn-finite-LB}
\end{align}
With this choice of $\rho$, the single-policy clipped concentrability coefficient $\Cstar$ and the quantity $C$ are intimately connected as follows:
\begin{align}
    \Cstar = 2C. \label{eq:expression-Cstar-LB-finite}
\end{align}
The proof of the claims \eqref{eq:finite-lb-behavior-distribution} and \eqref{eq:expression-Cstar-LB-finite} can be found in Appendix~\ref{proof:eq:finite-lb-behavior-distribution}.


\subsubsection{Establishing the minimax lower bound}

We are now positioned to establish our sample complexity lower bounds. 
Recalling our choice of $\rho$ in \eqref{eq:rho-defn-finite-LB}, 
our proof seeks to control the quantity $$\big\langle \rho, V_1^{\star, \theta} - V_1^{\widehat{\pi}, \theta} \big\rangle = V_1^{\star, \theta}(0) - V_1^{\widehat{\pi}, \theta}(0),$$
where $\widehat{\pi}$ is any policy estimator computed based on the batch dataset.

\paragraph{Step 1: converting  $\widehat{\pi}$ into an estimate $\widehat{\theta}$ of $\theta$.} 
Towards this, we first make the following claim: for an arbitrary policy $\pi$ obeying
\begin{align}
	\sum_{h=1}^H \big\|\pi_h(\cdot\mymid 0) - \pi_h^{\star, \theta}(\cdot\mymid 0) \big\|_1 \ge \frac{H}{8},
\end{align}
one has
\begin{align}
    \big\langle \rho, V_1^{\star, \theta} - V_1^{\pi, \theta} \big\rangle > \varepsilon. \label{eq:finite-Value-0-recursive}
\end{align}
We shall postpone the proof of this claim to Appendix~\ref{proof:eq:finite-lb-behavior-distribution}.
Suppose for the moment that there exists a policy estimate $\widehat{\pi}$ that achieves
\begin{align}
    \mathbb{P} \left\{\big\langle \rho, V_1^{\star, \theta} - V_1^{\widehat{\pi}, \theta} \big\rangle \leq \varepsilon\right\} \geq \frac{3}{4},
	\label{eq:assumption-theta-small-LB-finite}
\end{align}
then in view of \eqref{eq:finite-Value-0-recursive}, 
we necessarily have 
\begin{equation}
	\mathbb{P}\left\{ \sum_{h=1}^H \big\| \widehat{\pi}_h(\cdot\mymid 0) - \pi_h^{\star, \theta}(\cdot\mymid 0) \big\|_1 < H/8 \right\} \geq \frac{3}{4}.
\end{equation}

With the above observation in mind, we are motivated to construct the following estimate $\widehat{\theta}$ for $\theta\in \Theta$: 
\begin{align}
	\widehat{\theta} = \arg \min_{\overline{\theta} \in\Theta} ~\sum_{h=1}^H \big\|\widehat{\pi}_h(\cdot\mymid 0) - \pi^{\star, \overline{\theta}}_h(\cdot\mymid 0)\big\|_1. \label{eq:finite-theta-estimator}
\end{align}
If $\sum_h \big\|\widehat{\pi}(\cdot\mymid 0) - \pi^{\star, \theta}(\cdot\mymid 0) \big\|_1 < H/8$ holds for some $\theta \in\Theta$, 
then for any $\widetilde{\theta} \in \Theta$ with $\widetilde{\theta}  \neq \theta$ one has
\begin{align}
	\sum_{h=1}^H \big \|\widehat{\pi}_h(\cdot\mymid 0) - \pi_h^{\star, \widetilde{\theta}}(\cdot\mymid 0) \big\|_1 
	&\geq   \sum_{h=1}^H  \big\|\pi_h^{\star, \theta}(\cdot\mymid 0) - \pi_h^{\star, \widetilde{\theta}}(\cdot\mymid 0)\big\|_1 - \sum_{h=1}^H \big\| \widehat{\pi}_h(\cdot\mymid 0) 
	-   \pi^{\star, \theta}_h(\cdot\mymid 0)\big\|_1 \notag\\
    & = 2\big\|\theta - \widetilde{\theta}\big\|_1 - \sum_{h=1}^H  \big\|\widehat{\pi}_h(\cdot\mymid 0) - \pi^{\star, \theta}_h(\cdot\mymid 0)\big\|_1  \notag\\
    & > \frac{H}{4} - \frac{H}{8} = \frac{H}{8} , \label{eq:finite-tilde-theta}
\end{align}
where the first inequality holds by the triangle inequality, 
the second line arises from the fact $\pi_h^{\star,\theta}(\theta_h \mymid 0)=1$ for all $1\leq h \leq H$ (see \eqref{eq:finite-lb-value-lemma}), and the last line comes from the properties \eqref{eq:property-Theta} about $\Theta$.
Putting \eqref{eq:finite-theta-estimator} and \eqref{eq:finite-tilde-theta} together implies that 
$\widehat{\theta} = \theta$ if 
$$
	\sum_{h=1}^H \big\| \widehat{\pi} _h (\cdot\mymid 0) - \pi^{\star, \theta}_h (\cdot\mymid 0) \big\|_1 
	<  \frac{H}{8} < \sum_{h=1}^H \big \|\widehat{\pi}_h(\cdot\mymid 0) - \pi_h^{\star, \widetilde{\theta}}(\cdot\mymid 0) \big\|_1
$$ 
is valid for all $\widetilde{\theta} \in \Theta$ with $\widetilde{\theta} \neq \theta$.
As a consequence,
\begin{align}
    \mathbb{P}\big(\widehat{\theta} = \theta\big) 
	\geq \mathbb{P}\left( \sum_{h=1}^H \big\|\widehat{\pi}_h(\cdot\mymid 0) - \pi_h^{\star, \theta}(\cdot\mymid 0)\big\|_1 < \frac{H}{8}\right) \geq  \frac{3}{4}. 
	\label{eq:P-theta-accuracy-finite}
\end{align}

In the sequel, we aim to demonstrate that \eqref{eq:P-theta-accuracy-finite} cannot possibly happen without enough samples, 
which would in turn contradict \eqref{eq:assumption-theta-small-LB-finite}.

\paragraph{Step 2: probability of error in testing multiple hypotheses.} 
Next, we turn attention to a $|\Theta|$-ary hypothesis testing problem. 
For any $\theta\in \Theta$, denote by $\mathbb{P}_\theta$ the probability distribution when the MDP is $\mathcal{M}_\theta$.
We will then study the minimax probability of error defined as follows:
\begin{equation}
    p_{\mathrm{e}}\coloneqq\inf_{\psi}\max_{\theta\in\Theta} \mathbb{P}_{\theta}(\psi\neq\theta) , \label{eq:error-prob-two-hypotheses-finite-LB}
\end{equation}
where the infimum is taken over all possible tests $\psi$ (constructed based on the batch dataset available).


Let $\mu^{\mathsf{b},\theta}$ (resp.~$\mu^{\mathsf{b},\theta_h}_h(s_h)$) represent the distribution of a sample trajectory $\{s_1, a_1,s_2, a_2,\cdots, s_H, a_H\}$ (resp.~a sample $(a_h,s_{h+1})$ conditional on $s_h$) for the MDP $\mathcal{M}_\theta$. Recalling that the $K$ trajectories in the batch dataset are independently generated, one obtains
\begin{align}
p_{\mathrm{e}} & \overset{\mathrm{(i)}}{\geq}1-\frac{K\max_{\theta,\widetilde{\theta}\in\Theta,\,\theta\neq\widetilde{\theta}}\mathsf{KL}\big(\mu^{\mathsf{b},\theta}\parallel\mu^{\mathsf{b},\widetilde{\theta}}\big)+\log2}{\log|\Theta|}\nonumber \\
 & \overset{(\mathrm{ii})}{\geq}1-\frac{8K}{H}\max_{\theta,\widetilde{\theta}\in\Theta,\,\theta\neq\widetilde{\theta}}\mathsf{KL}\big(\mu^{\mathsf{b},\theta}\parallel\mu^{\mathsf{b},\widetilde{\theta}}\big)-\frac{8\log2}{H}\nonumber \\
 & \overset{(\mathrm{iii})}{\geq}\frac{1}{2}-\frac{8K}{H}\max_{\theta,\widetilde{\theta}\in\Theta,\,\theta\neq\widetilde{\theta}}\mathsf{KL}\big(\mu^{\mathsf{b},\theta}\parallel\mu^{\mathsf{b},\widetilde{\theta}}\big),\label{eq:Fano-LB-finite}
\end{align}
where (i) arises from Fano's inequality (cf.~\citep[Corollary 2.6]{tsybakov2009introduction}) and  the additivity property of the KL divergence (cf.~\citet[Page~85]{tsybakov2009introduction}), 
(ii) holds since $|\Theta| \ge e^{H/8}$ (according to our construction \eqref{eq:property-Theta}), and (iii) is valid when $H\geq 16 \log 2$. 
Recalling that the occupancy state distribution $\myrho_h$ is the same for any MDP $\mathcal{M}_{\theta}$ with $\theta\in \Theta$ (see \eqref{eq:finite-lb-behavior-distribution}), one can invoke the chain rule of the KL divergence \citep[Lemma 5.2.8]{duchi2018introductory} and the Markovian nature of the sample trajectories to obtain
\begin{align*}
\mathsf{KL}\big(\mu^{\mathsf{b},\theta}\parallel\mu^{\mathsf{b},\widetilde{\theta}}\big) & =\sum_{h=1}^{H}\mathop{\mathbb{E}}\limits _{s_{h}\sim\myrho_{h}}\left[\mathsf{KL}\big(\mu_{h}^{\mathsf{b},\theta_h}(s_{h})\parallel\mu_{h}^{\mathsf{b},\widetilde{\theta}_h}(s_{h})\big)\right]
	=\frac{1}{2}\mu(0)\sum_{h=1}^H\sum_{a\in\{0,1\}}\mathsf{KL}\left(P_{h}^{\theta_h}(\cdot\mymid0,a)\parallel P_{h}^{\widetilde{\theta}_h}(\cdot\mymid0,a)\right), 
\end{align*}
where the last identity holds true since (by construction and \eqref{eq:finite-lb-behavior-distribution})
\begin{align*}
\mathop{\mathbb{E}}\limits _{s_{h}\sim\myrho_{h}}\left[\mathsf{KL}\big(\mu_{h}^{\mathsf{b},\theta_h}(s_{h})\parallel\mu_{h}^{\mathsf{b},\widetilde{\theta}_h}(s_{h})\big)\right] & =\sum_{s}\myrho_{h}(s)\left\{ \sum_{a,s'}\pib_{h}(a\mymid s)P_{h}^{\theta_h}(s'\mymid s,a)\log\frac{\pib_{h}(a\mymid s)P_{h}^{\theta_h}(s'\mymid s,a)}{\pib_{h}(a\mymid s)P_{h}^{\widetilde{\theta}_h}(s'\mymid s,a)}\right\} \\
 & =\frac{1}{2}\mu(0)\sum_{a}\sum_{s'}P_{h}^{\theta_h}(s'\mymid0,a)\log\frac{P_{h}^{\theta_h}(s'\mymid0,a)}{P_{h}^{\widetilde{\theta}_h}(s'\mymid0,a)}\\
 & =\frac{1}{2}\mu(0)\sum_{a}\mathsf{KL}\big(P_{h}^{\theta_h}(\cdot\mymid0,a)\parallel P_{h}^{\widetilde{\theta}_h}(\cdot\mymid0,a)\big).
\end{align*}
Substitution into \eqref{eq:Fano-LB-finite} yields
\begin{align}
	p_{\mathrm{e}}  
	\geq \frac{1}{2}- \frac{4K \mu(0)}{H} \max_{\theta,\widetilde{\theta}\in \Theta, \, \theta\neq \widetilde{\theta}} \, \sum_{h=1}^H     \left[ \mathsf{KL}\big( P_{h}^{\theta_h}(\cdot \mymid 0, 0) \parallel  P_{h}^{\widetilde{\theta}_h}(\cdot \mymid 0, 0)  \big) + \mathsf{KL}\big( P_{h}^{\theta_h}(\cdot \mymid 0, 1) \parallel  P_{h}^{\widetilde{\theta}_h}(\cdot \mymid 0, 1)  \big)  \right] .
	\label{eq:finite-remainder-KL}
\end{align}
%

%

It then boils down to bounding the KL divergence terms in \eqref{eq:finite-remainder-KL}. 
If $\theta_h = \widetilde{\theta}_h$, then it is self-evident that
\begin{align}
    \mathsf{KL}\big( P_{h}^{\theta_h}(\cdot \mymid 0, 0) \parallel  P_{h}^{\widetilde{\theta}_h}(\cdot \mymid 0, 0)  \big) + \mathsf{KL}\big( P_{h}^{\theta_h}(\cdot \mymid 0, 1) \parallel  P_{h}^{\widetilde{\theta}_h}(\cdot \mymid 0, 1)  \big)  = 0. \label{eq:finite-KL-bounded-0}
\end{align}
Consider now the case that $\theta_h \neq \widetilde{\theta}_h$,
and suppose without loss of generality that $\theta_h = 0$ and $\widetilde{\theta}_h = 1$.
It is seen that
\begin{align*}
	P_h^{\theta_h}(0 \mymid 0, 0)  &= P_h^{\theta_h}(\theta_h \mymid 0, 0) = \left(1 -\frac{1}{CS}\right)p + \frac{1}{CS},\\
	P_h^{\widetilde{\theta}_h}(0 \mymid 0, 0) &= P_h^{\widetilde{\theta}_h}\big(1-\widetilde{\theta}_h \mymid 0, 0\big) = \left(1 -\frac{1}{CS}\right)q + \frac{1}{CS}.
\end{align*}
Given that $p\geq q \geq 1/2$ (see \eqref{eq:p-q-order-LB-finite}), 
we can apply Lemma~\ref{lem:KL-key-result} to arrive at
\begin{align}
\mathsf{KL}\left(P_h^{\theta_h}(0 \mymid 0, 0)\parallel P_h^{\widetilde{\theta}_h}(0 \mymid 0, 0)\right) & =\mathsf{KL}\left(\left(1-\frac{1}{CS}\right)p+\frac{1}{CS}\parallel\left(1-\frac{1}{CS}\right)q+\frac{1}{CS}\right)\nonumber\\
 & \leq \frac{\left(1-\frac{1}{CS}\right)^2(p-q)^2}{\left(\left(1 -\frac{1}{CS}\right)p + \frac{1}{CS}\right)\left(1-p - (1-p)\frac{1}{CS} \right)} \nonumber\\
	& \overset{\mathrm{(i)}}{\leq}  \frac{\left(1-\frac{1}{CS}\right)^2(p-q)^2}{\left(\left(1 -\frac{1}{CS}\right)p \right)\big((1-p) (1- \frac{1}{CS}) \big)} 
	= \frac{4(c_2)^2 \varepsilon^2}{H^4  p(1-p)} \nonumber\\
    & \overset{\mathrm{(ii)}}{=}\frac{4 (c_2)^2 \varepsilon^2}{H^4 \left(1-\frac{c_1}{H} + \frac{c_2\varepsilon}{H^2}\right)\left(\frac{c_1}{H} - \frac{c_2\varepsilon}{H^2}\right)} \nonumber\\
    & \leq \frac{4 (c_2)^2 \varepsilon^2}{H^4 \frac{1}{2}\frac{c_1}{2H}} = \frac{16 (c_2)^2 \varepsilon^2}{c_1 H^3 },
	\label{eq:finite-KL-bounded}
\end{align}
where (i) and (ii) make use of the definition \eqref{eq:finite-p-q-def} of $(p, q)$, 
and the last line follows as long as $\frac{c_2\varepsilon}{H^2} \leq \frac{c_1}{2H} \leq \frac{1}{4}$. 
Similarly, it can be easily verified that $\mathsf{KL}\big(P_h^{\theta_h}(0 \mymid 0, 1)\parallel P_h^{\widetilde{\theta}_h}(0 \mymid 0, 1)\big)$ 
can be upper bounded in the same way. 
Substituting \eqref{eq:finite-KL-bounded} and \eqref{eq:finite-KL-bounded-0} back into \eqref{eq:finite-remainder-KL} 
indicates that: if the sample size obeys
\begin{align}
    N = KH \leq \frac{c_1 CS H^4}{512 (c_2)^2 \varepsilon^2} = \frac{c_1 \Cstar S H^4}{1024 (c_2)^2 \varepsilon^2}, \label{eq:finite-LB-sample-condition}
\end{align}
then one necessarily has
\begin{align}
     p_{\mathrm{e}}  &\geq \frac{1}{2}- \frac{4K \mu(0)}{H} \max_{\theta,\widetilde{\theta}\in \Theta, \, \theta\neq \widetilde{\theta}} \sum_{h=1}^H     \left[ \mathsf{KL}\big( P_{h}^{\theta_h}(\cdot \mymid 0, 0) \parallel  P_{h}^{\widetilde{\theta}_h}(\cdot \mymid 0, 0)  \big) + \mathsf{KL}\big( P_{h}^{\theta_h}(\cdot \mymid 0, 1) \parallel  P_{h}^{\widetilde{\theta}_h}(\cdot \mymid 0, 1)  \big)  \right]  \notag \\
     & \geq \frac{1}{2} - \frac{4K \mu(0)}{H}\sum_{h=1}^H  \frac{32 (c_2)^2 \varepsilon^2}{c_1 H^3 }  \geq \frac{1}{4}. \label{eq:finite-LB-final-pe}
\end{align}

\paragraph{Step 3: combining the above results.}
Suppose that there exists an estimator $\widehat{\pi}$ satisfying
\begin{align}
    \max_{\theta\in\Theta} \mathbb{P}_\theta \left\{\big\langle \rho, V_1^{\star, \theta} - V_1^{\widehat{\pi}, \theta} \big\rangle \geq \varepsilon \right\} < \frac{1}{4},
\end{align}
where $\mathbb{P}_{\theta}$ denotes the probability when the MDP is $\mathcal{M}_{\theta}$. 
Then in view of the analysis in Step 1, we must have 
\[
	\mathbb{P}_{\theta} \left( \sum_{h=1}^H\big\| \widehat{\pi}(\cdot\mymid 0) - \pi^{\star, \theta}(\cdot\mymid 0) \big\|_1  < \frac{H}{8} \right)
	\geq \frac{3}{4}, \qquad \text{for all }\theta \in \Theta, 
\]
and as a consequence of \eqref{eq:P-theta-accuracy-finite}, the estimator $\widehat{\theta}$ defined in \eqref{eq:finite-theta-estimator} must satisfy
\begin{align}
    \mathbb{P}_\theta \big(\widehat{\theta} \neq \theta \big) <\frac{1}{4}, \qquad  \text{for all }\theta \in \Theta.
\end{align}
Nevertheless, this cannot possibly happen under the sample size condition \eqref{eq:finite-LB-sample-condition}; otherwise it is contradictory to the result in \eqref{eq:finite-LB-final-pe}. This concludes the proof by inserting $c_1 = 1/4$ and $c_2 = 4096$.

\subsubsection{Proof of Lemma~\ref{lem:finite-lb-value}}\label{proof:lem:finite-lb-value} 
To start with, for any policy $\pi$, it is observed that the value function of state $s=0$ at step $h$ is
\begin{align}
	V_h^{\pi, \theta}(0) &= \mathop{\mathbb{E}}\limits_{a\sim\pi_h(\cdot \mymid 0)} \left[1 +  \sum_{s'} P^{\theta_h}_h(s' \mymid 0, a)V_{h+1}^{\pi,\theta}(s')\right] \notag\\
& = 1 + \pi_h(\theta_h \mymid 0) \left[ \big(p + (1-p)\mu(0)\big) V_{h+1}^{\pi,\theta}(0) + (1-p) \mu(1) V_{h+1}^{\pi,\theta}(1)\right] \nonumber \\
&\qquad +  \pi(1-\theta_h \mymid0) \left[ \big(q + (1-q)\mu(0)\big)V_{h+1}^{\pi,\theta}(0) + (1-q)\mu(1) V_{h+1}^{\pi,\theta}(1)\right] \nonumber\\
& = 1 + \Big[p \pi_h(\theta_h \mymid 0) + q \pi(1-\theta_h \mymid0) + \mu(0) - p \pi_h(\theta_h \mymid 0) \mu(0) - q \pi(1-\theta_h \mymid0)\mu(0) \Big]
	V_{h+1}^{\pi,\theta}(0) \nonumber\\
&\qquad +  \mu(1) \Big[ 1 - p\pi_h(\theta_h \mymid 0)  -q \pi(1-\theta_h \mymid 0)\Big] V_{h+1}^{\pi,\theta}(1) \nonumber\\
& \overset{\mathrm{(i)}}{=} 1 +  \left[ x_h^{\pi,\theta} + (1- x_h^{\pi,\theta})\mu(0) V^{\pi,\theta}_{h+1}(0) +  (1 - x_h^{\pi,\theta}) \mu(1) V_{h+1}^{\pi,\theta}(1)\right]  \nonumber\\
& \overset{\mathrm{(ii)}}{=}  1 + \big(\mu(1)x_h^{\pi}+\mu(0)\big) V_{h+1}^{\pi,\theta}(0) + (1-x_h^{\pi})\mu(1)V_{h+1}^{\pi,\theta}(1), \label{eq:finite-Value-0}
\end{align}
where (i) is valid due to the choice
\begin{align}
x_h^{\pi,\theta} = p\pi_h(\theta_h\mymid 0) + q\pi_h(1-\theta_h\mymid 0),
\end{align}
and (ii) holds since $\mu(0) + \mu(1) = 1$.

Additionally,  the value function of state $1$ at any step $h$ obeys 
\begin{align}
        V_h^{\pi,\theta}(1) &= \pi_h(0 \mymid 1) \left( \frac{1}{2} +  V_{h+1}^{\pi,\theta}(1)\right) + \pi_h(1 \mymid 1) \left[ \left(1- \frac{2c_1}{HCS}\right)V_{h+1}^{\pi,\theta}(1) + \frac{2c_1}{HCS}V_{h+1}^{\pi,\theta}(0)\right] \label{eq:finite-value-1-equal}\\
        & \overset{\mathrm{(i)}}{\leq} \pi_h(0 \mymid 1) \left( \frac{1}{2} +  V_{h+1}^{\pi,\theta}(1)\right) + \pi_h(1 \mymid 1) \left[ \left(1- \frac{2c_1}{HCS}\right)V_{h+1}^{\pi,\theta}(1) + \frac{2c_1}{HCS}\left(H -h\right)\right] \notag\\
        & \overset{\mathrm{(ii)}}{\leq} \pi_h(0 \mymid 1) \left( \frac{1}{2} +  V_{h+1}^{\pi,\theta}(1)\right) + \pi_h(1 \mymid 1) \left[\frac{1}{2} + \left(1- \frac{2c_1}{HCS}\right)V_{h+1}^{\pi,\theta}(1) \right] \notag\\
        & = \frac{1}{2} + V_{h+1}^{\pi,\theta}(1) - \frac{2c_1}{HCS}\pi_h(1 \mymid 1) V_{h+1}^{\pi,\theta}(1), \label{eq:finite-value-1-leq}
\end{align}
where (i) arises from the basic fact $ 0\leq V_{h}^{\pi,\theta}(s) \leq H-h+1$ for any policy $\pi$ and all $(s,h)\in\cS \times [H]$, and (ii) holds since $\frac{2c_1}{HCS}\left(H -h\right) \leq \frac{1}{2}$ for $c_1$ small enough.
The above results lead to several immediate facts.  
\begin{itemize}
    \item If we choose $\pi$ such that $\pi_h(0 \mymid 1) = 1$ for all $h\in[H]$, then \eqref{eq:finite-value-1-equal} tells us that
    \begin{align}
        V_h^{\pi,\theta}(1) = \frac{1}{2} + V_{h+1}^{\pi,\theta}(1) .
    \end{align}
	A recursive application of this relation reveals that
    \begin{align}
        V_h^{\pi,\theta}(1) = \frac{1}{2} + V_{h+1}^{\pi,\theta}(1) = \cdots = \sum_{j=h}^H \frac{1}{2} = \frac{1}{2}(H+1-h). \label{eq:finite-value-1-optimal}
    \end{align}

    \item For any policy $\pi$, applying \eqref{eq:finite-value-1-leq} recursively tells us that
    \begin{align}
        V_h^{\pi,\theta}(1) \leq \frac{1}{2} + V_{h+1}^{\pi,\theta}(1) \leq \cdots \leq \sum_{j=h}^H \frac{1}{2} = \frac{1}{2}(H+1-h). \label{eq:finite-value-1-suboptimal}
    \end{align}

\end{itemize}

 The above two facts taken collectively imply that the optimal policy and optimal value function obey
    \begin{align}
	    \pi_h^{\star,\theta}(0 \mymid 1) = 1,\qquad V_h^{\star,\theta}(1)  =  \frac{1}{2}(H+1-h), \qquad \forall h\in [H]. \label{eq:finite-value-1-lemma-summary}
    \end{align}
We then return to state $0$.  By taking $\pi$ such that $\pi_h(\theta_h\mymid 0)=1$ (and hence $x_h^{\pi,\theta} = p$) for all $h\in[H]$, one can invoke \eqref{eq:finite-Value-0} to derive
\begin{align}
	V_h^{\pi, \theta}(0)&= 1 + \big(\mu(1)p+\mu(0)\big) V_{h+1}^{\pi, \theta}(0) + (1-p)\mu(1)V_{h+1}^{\pi, \theta}(1) \nonumber\\
&\ge 1 + pV_{h+1}^{\pi, \theta}(0)  \ge \sum_{j=0}^{H-h} p^j \ge \sum_{j = 0}^{H-h} \Big(1 - \frac{c_1}{H} \Big)^j = \frac{1-\big(1-\frac{c_{1}}{H}\big)^{H-h+1}}{c_{1}/H}\notag\\
	&\ge \frac{2}{3}(H+1-h) . \label{eq:finite-value-0-lower-bounded-123}
\end{align}
To see that why the last inequality holds, it suffices to observe that
\[
	\Big(1-\frac{c_{1}}{H}\Big)^{H-h+1}\leq\exp\left(-\frac{c_{1}}{H}(H-h+1)\right)\leq1-\frac{2c_{1}(H-h+1)}{3H},
\]
as long as $c_1\leq 0.5$, which follows due to the elementary inequalities $1-x\leq \exp(-x)$ for any $x\geq 0$ and $\exp(-x)\leq 1-2x/3$ for any $0\leq x\leq 1/2$. 
Combine \eqref{eq:finite-value-0-lower-bounded-123} with \eqref{eq:finite-value-1-lemma-summary} to reach
\begin{equation}
	V_h^{\star, \theta}(0) \geq V_h^{\pi, \theta}(0) \geq \frac{2}{3}(H+1-h) > V_h^{\star, \theta}(1) . 
\label{eq:finite-value-0-lower-bounded}
\end{equation}
Moreover, it follows from \eqref{eq:finite-Value-0} that
\begin{align}
    V_h^{\star, \theta}(0) &= 1 + \big(\mu(1)x_h^{\pi^{\star,\theta},\theta}+\mu(0)\big) V_{h+1}^{\star, \theta}(0) + (1-x_h^{\pi^{\star,\theta},\theta})\mu(1)V_{h+1}^{\star, \theta}(1) \nonumber\\
    & = 1 + \mu(0)V_{h+1}^{\star, \theta}(0) + \mu(1)V_{h+1}^{\star, \theta}(1) + \mu(1)\big( V_{h+1}^{\star, \theta}(0) - V_{h+1}^{\star, \theta}(1) \big) x_h^{\pi^{\star,\theta},\theta}.
\end{align}
Observing that the function
\begin{align}
    \mu(1)\big( V_{h+1}^{\star, \theta}(0) - V_{h+1}^{\star, \theta}(1) \big) x
\end{align}
is increasing in $x$ (as a result of \eqref{eq:finite-value-0-lower-bounded}) and that $x_h^{\pi,\theta}$ is increasing in $\pi_h(\theta_h \mymid 0)$ (since $p\geq q$), 
we can readily conclude that the optimal policy in state $0$ obeys
\begin{equation}
	\pi_h^{\star,\theta}(\theta_h \mymid 0) = 1,\qquad \text{for all } h\in[H]. \label{eq:finite-lb-optimal-policy}
\end{equation}


\subsubsection{Proof of auxiliary properties}\label{proof:eq:finite-lb-behavior-distribution}

Throughout this section, we shall suppress the dependency on $\theta$ in the notation $d_h^{\star}$ whenever it is clear from the context.  

\paragraph{Proof of claim \eqref{eq:finite-lb-behavior-distribution}.}
For any MDP $\mathcal{M}_\theta$, from the definition of $\myrho_h(s, a)$ in \eqref{eq:dhb-finite} and the Markov property, it is clearly seen that
\begin{align}
\myrho_{h+1}(s) &= d_{h+1}^{\pib}(s; \rhob) = \mathbb{P}(s_{h+1} = s \mymid s_h \sim \myrho_{h}; \pib), \qquad \forall (s,h)\in \cS\times [H].
\end{align}
Recalling that $\myrho_{1}(s)= \rhob(s)= \mu(s)$ for all $s\in\cS$, one can then show that
\begin{align}
    \myrho_{2}(0) &= \mathbb{P}\{s_{2} = 0 \mymid s_1 \sim \myrho_1; \pib\} \notag \\
    & = \mu(0)\left[\pib_1(\theta_1 \mymid 0) P_1^{\theta_1}(0\mymid 0,\theta_1) + \pib_1(1-\theta_1 \mymid 0) P_1^{\theta_1}(0\mymid 0,1-\theta_1) \right]  \notag \\
	&\qquad + \mu(1) \left[ \pib_1(0 \mymid 1) P_1^{\theta_1}(0\mymid 1,0) + \pib_1(1 \mymid 1) P_1^{\theta_1}(0\mymid 1, 1) \right] \notag\\
    & = \frac{\mu(0)}{2}\left[P_1^{\theta_1}(0\mymid 0,\theta_1) + P_1^{\theta_1}(0\mymid 0,1-\theta_1) \right] + \frac{\mu(1)}{2}\left[P_1^{\theta_1}(0\mymid 1,0) + P_1^{\theta_1}(0\mymid 1, 1) \right] \notag\\
    & =  \frac{\mu(0)}{2}\big[(p+q) + (2-p-q) \mu(0)\big] + \frac{\mu(1) }{2}\mu(0)\frac{2c_1}{H} \notag\\
    & = \frac{\mu(0)}{2}\left[2-\frac{2c_1}{H} + \frac{2c_1}{H} \mu(0)\right] + \frac{\mu(1) }{2}\mu(0)\frac{2c_1}{H} 
     = \mu(0) , \notag
\end{align}
where the last inequality holds since $\mu(1) + \mu(0) = 1$. Similarly, it can be verified that $\myrho_{1}(1) = \mu(1)$, thereby implying that $\myrho_{2} = \mu$. 
Repeating this argument recursively for steps $h=2,\ldots,H$ confirms that
\begin{align}
	\myrho_{h}(s) &= \mu(s), \qquad \forall (s,h)\in\cS\times [H]. 
\end{align}
This further allows one to demonstrate that
\begin{align}
\myrho_h(s, a) &= \myrho_h(s)\pi_h^{\mathsf{b}}(a \mymid s) =  \mu(s)\pi_h^{\mathsf{b}}(a \mymid s) = \mu(s)/2, \qquad \forall (s,a,h)\in\cS\times \cA\times [H].
\end{align}

\paragraph{Proof of claim \eqref{eq:expression-Cstar-LB-finite}.}
Consider any MDP $\mathcal{M}_\theta$, for which we have shown in Lemma~\ref{lem:finite-lb-value} that $\pi_h^{\star,\theta}(\theta_h \mymid 0) = 1 $ for all $h\in[H]$.
It is observed that
\begin{align}
d_h^{\star}(0,\theta_h)  &= d_h^{\star}(0)\pi_h^{\star,\theta}(\theta_h \mymid 0) 
= d_h^{\star}(0)= \mathbb{P} \big\{ s_{h} = 0 \mymid s_{h-1} \sim d_{h-1}^{\star}; \pi^{\star,\theta} \big\} \notag\\
	& \geq d_{h-1}^{\star}(0) \pi_{h-1}^{\star,\theta}(\theta_{h-1} \mymid 0) P_{h-1}^{\theta_{h-1}}(0 \mymid 0, \theta_{h-1}) 
	=  d_{h-1}^{\star}(0)  P_{h-1}^{\theta_{h-1}}(0 \mymid 0, \theta_{h-1})  \notag\\
	& \geq \cdots \geq d_{1}^{\star}(0) \prod_{j = 0}^{h-1}P_{j}^{\theta_j}(0 \mymid 0, \theta_j) = \rho(0)\prod_{j = 0}^{h-1}P_{j}^{\theta_j}(0 \mymid 0, \theta_j) \notag \\
	&  \geq \rho(0)\prod_{j = 0}^{h-1}p  \ge \Big(1-\frac{c_1}{H}\Big)^{H} > \frac{1}{2}, \label{eq:finite-optimal-d-large}
\end{align}
where the last line makes use of the properties $p\geq 1- c_1/H$, $\rho(0)=1$, and 
\[
\left(1-\frac{c_{1}}{H}\right)^{H}\geq\Big(1-\frac{1}{2H}\Big)^{H} > \frac{1}{2}
\]
provided that $0<c_1 < 1/2$. 
Combining this with \eqref{eq:finite-lb-behavior-distribution}, we arrive at
\begin{align}
    \max_{h\in[H]} \frac{\min\big\{d_h^{\star}(0, \theta_h), \frac{1}{S}\big\}}{\myrho_h(0, \theta_h)} &= \frac{2}{S\mu(0)} = 2C, \notag\\
	\max_{h\in[H]} \frac{\min\big\{d_h^{\star}(0, 1-\theta_h), \frac{1}{S}\big\}}{\myrho_h(0, 1- \theta_h)} 
	&=  \max_{h\in[H]} \frac{\min\big\{d_h^{\star}(0) \pi_h^{\star}(1-\theta_h \mymid 0), \frac{1}{S}\big\}}{\myrho_h(0, 1- \theta_h)} = 0, \notag\\
    \max_{a\in\{0,1\}, h\in[H]} \frac{\min\big\{d_h^{\star}(1, a), \frac{1}{S}\big\}}{\myrho_h(1, a)} &\overset{\mathrm{(i)}}{\leq} \frac{1/S}{\mu(1)/2} 
	\overset{\mathrm{(ii)}}{=} \frac{2}{S\left(1-\frac{1}{SC}\right)}\leq \frac{4}{S} \leq 2C, \notag
\end{align}
where (i) arises from \eqref{eq:finite-lb-behavior-distribution}, (ii) relies on the definition in \eqref{finite-mu-assumption}, and the final two inequalities come from the assumption in \eqref{finite-C-assumption}. 
Taking this together with the straightforward condition $d_h^{\star}(s)=0$ ($s>1$) yields
\begin{align}
    \Cstar  = \max_{h\in[H]} \frac{\min\big\{d_h^{\star}(0, \theta_h), \frac{1}{S}\big\}}{\myrho_h(0, \theta_h)} = 2C.
\end{align}

\paragraph{Proof of inequality \eqref{eq:finite-Value-0-recursive}.}
%
By virtue of \eqref{eq:finite-x-h} and \eqref{eq:finite-lb-value-lemma}, we see that $x_h^{\pi^{\star,\theta},\theta} = p$ for all $h\in[H]$, which combined with \eqref{eq:finite-lemma-value-0-pi} gives
\begin{align}
\big\langle \rho, V_h^{\star, \theta} - V_h^{\pi, \theta} \big\rangle &= V_h^{\star, \theta}(0) - V_h^{\pi, \theta}(0) \notag\\
&= \big(\mu(1)p+\mu(0)\big) V_{h+1}^{\star,\theta}(0) + (1-p)\mu(1)V_{h+1}^{\star,\theta}(1) \notag \\
&\qquad  - \big(\mu(1)x_h^{\pi,\theta}+\mu(0)\big) V_{h+1}^{\pi,\theta}(0) - (1-x_h^{\pi,\theta})\mu(1)V_{h+1}^{\pi,\theta}(1) \notag\\
&   \overset{\mathrm{(i)}}{\geq} \big(\mu(1)x_h^{\pi,\theta}+\mu(0)\big) \left(V_{h+1}^{\star, \theta}(0) - V_{h+1}^{\pi,\theta}(0)\right) + \mu(1) \big(p-x_h^{\pi,\theta}\big)V_{h+1}^{\star,\theta}(0) \notag \\
&\qquad + (1-p)\mu(1)V_{h+1}^{\star,\theta}(1)  - \big(1-x_h^{\pi,\theta}\big)\mu(1)V_{h+1}^{\star,\theta}(1) \notag \\
&= \big(\mu(1)x_h^{\pi,\theta}+\mu(0)\big) \left(V_{h+1}^{\star, \theta}(0) - V_{h+1}^{\pi,\theta}(0)\right)  + \big(p-x_h^{\pi,\theta}\big) \mu(1)\left(V_{h+1}^{\star,\theta}(0) - V_{h+1}^{\star,\theta}(1)\right)  \notag \\
& \overset{\mathrm{(ii)}}{\geq} q\left(V_{h+1}^{\star, \theta}(0) - V_{h+1}^{\pi,\theta}(0)\right) +  \big(p-x_h^{\pi,\theta}\big) \mu(1)\left(V_{h+1}^{\star,\theta}(0) - V_{h+1}^{\star,\theta}(1)\right)  \notag \\
& \overset{\mathrm{(iii)}}{\geq} q\left(V_{h+1}^{\star, \theta}(0) - V_{h+1}^{\pi,\theta}(0)\right)  +  \frac{3}{8}(p-q) \big\|\pi_h^{\star, \theta}(0) - \pi_h(0)\big\|_1 \left(V_{h+1}^{\star,\theta}(0) - V_{h+1}^{\star,\theta}(1)\right)  \notag \\
&\overset{\mathrm{(iv)}}{\geq} q\left(V_{h+1}^{\star, \theta}(0) - V_{h+1}^{\pi,\theta}(0)\right)  + \frac{c_2\varepsilon}{8H^2}(H+1-h) \big\|\pi_h^{\star, \theta}(\cdot\mymid 0) - \pi_h(\cdot\mymid 0)\big\|_1 , \label{eq:finite-lb-recursion}
\end{align}
where (i) holds since $ V_{h+1}^{\pi,\theta}(1) \leq V_{h+1}^{\star,\theta}(1)$, (ii) follows from the fact that $x_h^\pi \geq q$ for any $\pi$ and $h\in[H]$, 
and (iv) arises from the facts \eqref{eq:finite-lb-value-lemma} and the choice \eqref{eq:finite-p-q-def} of $(p,q)$. 
To see why (iii) is valid, 
it suffices to note that  $\mu(1) = 1-\frac{1}{CS} \geq \frac{3}{4}$ (as a consequence of \eqref{finite-mu-assumption} and \eqref{finite-C-assumption}) and
\begin{equation*}
p-x_{h}^{\pi,\theta}=(p-q)\big(1-\pi_{h}(\theta_{h}\mymid0)\big)=\frac{1}{2}(p-q)\big(1-\pi_{h}(\theta_{h}\mymid0)+\pi_{h}(1-\theta_{h}\mymid0)\big)=\frac{1}{2}(p-q)\big\|\pi_{h}^{\star,\theta}(\cdot\mymid 0)-\pi_{h}(\cdot\mymid 0)\big\|_{1}. 	
\end{equation*}

To continue, under the condition 
\begin{align}
	\sum_{h=1}^H \big\|\pi_h(\cdot\mymid 0) - \pi_h^{\star, \theta}(\cdot\mymid 0)\big\|_1 \ge \frac{H}{8}, \label{eq:finite-lower-delta}
\end{align}
applying the relation in \eqref{eq:finite-lb-recursion} recursively yields
\begin{align}
    V_1^{\star, \theta}(0) - V_1^{\pi, \theta}(0) &\ge \sum_{h = 1}^H q^{h-1} \frac{c_2\varepsilon}{8H^2}(H+1-h)\big\|\pi_h^{\star, \theta}(\cdot\mymid 0) - \pi_h(\cdot\mymid 0)\big\|_1  \nonumber\\
    & = \sum_{h = 1}^H \bigg(1 - \frac{c_1}{H} - \frac{c_2\varepsilon}{H^2}\bigg)^{h-1} \frac{c_2\varepsilon}{8H^2}(H+1-h) \big\|\pi_h^{\star, \theta}(\cdot\mymid 0) - \pi_h(\cdot\mymid 0)\big\|_1  \notag \\
    & \overset{\mathrm{(i)}}{>} \frac{c_2\varepsilon}{16H^2} \sum_{h = 1}^H (H+1-h) \big\|\pi_h^{\star, \theta}(\cdot\mymid 0) - \pi_h(\cdot\mymid 0)\big\|_1 \notag \\
	& = \frac{c_2\varepsilon}{16H^2} \sum_{h = 1}^H h \big\|\pi_{H+1-h}^{\star, \theta}(\cdot\mymid 0) - \pi_{H+1-h}(\cdot\mymid 0)\big\|_1 \notag \\ 
    & \overset{\mathrm{(ii)}}{\geq} \frac{c_2\varepsilon}{16 H^2} \sum_{h=1}^{\lfloor H/16\rfloor}2h
	= \frac{c_{2}\varepsilon}{8H^{2}}\Big\lfloor\frac{H}{16}\Big\rfloor\Big(\Big\lfloor\frac{H}{16}\Big\rfloor+1\Big). 
	\label{eq:V-diff-intermediate-13579-inf-LB}
\end{align}
Here, (i) follows since 
\[
	\Big(1 - \frac{c_1}{H} - \frac{c_2\varepsilon}{H^2}\Big)^{h-1} \geq \Big(1 - \frac{2c_1}{H}\Big)^{H} > \frac{1}{2}, 
	\qquad \text{for all }h\in [H]
\]
holds as long as $0<c_1\leq 1/4$ and $c_2\varepsilon/H \leq c_1$.  
To see why (ii) is valid, we note that  
for 
any $0\leq x_{1},\cdots, x_{H}\leq x_{\max}$
obeying $\sum_{i=1}^{H}x_{i}\geq x_{\text{sum}}$,
the following elementary inequality holds: 
\[
	\sum_{i=1}^{H}x_{i}a_{i}\geq\sum_{i=1}^{\lfloor x_{\text{sum}}/ x_{\max} \rfloor}x_{\max}a_{i};
\]
this together with $\big\|\pi_h^{\star, \theta}(\cdot\mymid 0) - \pi_h(\cdot\mymid 0)\big\|_1\leq 2$ and \eqref{eq:finite-lower-delta} 
reveals that (by taking $a_h = h$ and $x_h = \big\|\pi_{H+1-h}^{\star, \theta}(\cdot\mymid 0) - \pi_{H+1-h}(\cdot\mymid 0)\big\|_1$) 
\[
\sum_{h=1}^{H}h\big\|\pi_{H+1-h}^{\star,\theta}(\cdot\mymid0)-\pi_{H+1-h}(\cdot\mymid0)\big\|_{1}\geq\sum_{h=1}^{\lfloor H/16\rfloor}2h, 
\]
thus validating  inequality (ii). 
As a result, we can continue the derivation to obtain
\begin{equation}
	\eqref{eq:V-diff-intermediate-13579-inf-LB} 
\geq\frac{c_{2}\varepsilon}{8H^{2}}\frac{\frac{H}{16}\big(\frac{H}{16}+1\big)}{2}>\varepsilon,
\end{equation}
provided that $c_2\geq 4096$.

\bibliography{bibfileRL}
\bibliographystyle{apalike} 
\end{document}